\pdfoutput=1
\RequirePackage{silence} % :-\
\documentclass[oneside,openright,titlepage,numbers=noenddot,%1headlines,
                headinclude,footinclude,cleardoublepage=empty,abstract=on,
                BCOR=5mm,paper=a4,fontsize=12pt
                ]{scrreprt}
\PassOptionsToPackage{utf8}{inputenc}
  \usepackage{inputenc}

\PassOptionsToPackage{T1}{fontenc} % T2A for cyrillics
  \usepackage{fontenc}

\PassOptionsToPackage{
  drafting=false,    % print version information on the bottom of the pages
  tocaligned=false, % the left column of the toc will be aligned (no indentation)
  dottedtoc=false,  % page numbers in ToC flushed right
  eulerchapternumbers=true, % use AMS Euler for chapter font (otherwise Palatino)
  linedheaders=false,       % chaper headers will have line above and beneath
  floatperchapter=true,     % numbering per chapter for all floats (i.e., Figure 1.1)
  eulermath=false,  % use awesome Euler fonts for mathematical formulae (only with pdfLaTeX)
  beramono=true,    % toggle a nice monospaced font (w/ bold)
  palatino=true,    % deactivate standard font for loading another one, see the last section at the end of this file for suggestions
  style=classicthesis % classicthesis, arsclassica
}{classicthesis}

\newcommand{\myTitle}{Towards Realistic Models of Change}

\newcommand{\myFirstName}{Isacco}
\newcommand{\myLastName}{Beretta}
\newcommand{\myWebsite}{https://scholar.google.com/citations?user=3YUmyg4AAAAJ&hl=en}

\newcommand{\myProfTitle}{Prof.}
\newcommand{\myProfFirstName}{Salvatore}
\newcommand{\myProfLastName}{Ruggieri}
\newcommand{\myProfWebsite}{https://pages.di.unipi.it/ruggieri/}

\newcommand{\myOtherProfTitle}{Prof.}
\newcommand{\myOtherProfFirstName}{Dino}
\newcommand{\myOtherProfLastName}{Pedreschi}
\newcommand{\myOtherProfWebsite}{https://scholar.google.com/citations?user=5efz6osAAAAJ&hl=en}

\newcommand{\myFaculty}{}

\newcommand{\myUni}{University of Pisa}

\newcommand{\mySubtitle}{A Causal Perspective on Decision-Making and Dynamical Systems}
\newcommand{\myName}{Isacco Beretta\xspace}

\newcommand{\myVersion}{\classicthesis}

\providecommand{\mLyX}{L\kern-.1667em\lower.25em\hbox{Y}\kern-.125emX\@}
\newcommand{\ie}{i.\,e.}

\newcommand{\eg}{e.\,g.}

\PassOptionsToPackage{ngerman,american}{babel} % change this to your language(s), main language last
    \usepackage{babel}

\usepackage{csquotes}
\PassOptionsToPackage{%
  backend=bibtex8,bibencoding=ascii,%
  language=auto,%
  style=authoryear-comp, % Author 1999, 2010
  sorting=nyt, % name, year, title
  maxbibnames=10, % default: 3, et al.
  natbib=true % natbib compatibility mode (\citep and \citet still work)
}{biblatex}
    \usepackage{biblatex}

\PassOptionsToPackage{fleqn}{amsmath}       % math environments and more by the AMS
  \usepackage{amsmath}

\usepackage{graphicx} %
\usepackage{scrhack} % fix warnings when using KOMA with listings package
\usepackage{xspace} % to get the spacing after macros right
\PassOptionsToPackage{printonlyused,smaller}{acronym}
  \usepackage{acronym} % nice macros for handling all acronyms in the thesis
\usepackage{tabularx} % better tables
\usepackage{listings}
\usepackage{classicthesis}
\usepackage{tikz}
\usetikzlibrary{decorations.pathmorphing,arrows.meta,positioning, fit, shapes.geometric,cd, backgrounds, patterns, calc,shapes}
\tikzstyle{fixed vertex} = [circle, draw=red,fill=red!10]
\tikzstyle{edge} = [-latex,thick]
\tikzstyle{o->} = [{Circle[open]}->, thick]
\tikzstyle{o-o} = [{Circle[open]}-{Circle[open]}, thick]
\tikzstyle{<-->} = [<->, thick]
\tikzstyle{redge} = [-latex,thick, draw=red]
\tikzstyle{uedge} = [thick, draw=black!20]
\tikzstyle{bluedge} = [-latex,thick, draw=blue]
\tikzstyle{snakedge} = [-stealth,thick, shorten <=1pt,decorate, decoration={snake,amplitude=.3mm}]
\tikzstyle{array}=[draw, fill=green!40, minimum width = 7mm, minimum height = 7mm]
\tikzstyle{hidden vertex} = [circle, draw=black,style=dashed]
\tikzstyle{vertex} = [circle, fill=black!10]
\tikzstyle{box} = [fill=black!10,draw = black,minimum width=0.7cm, minimum height=0.7cm]
\tikzstyle{array}=[draw, fill=green!40, minimum width = 7mm, minimum height = 7mm]

\hypersetup{%
  colorlinks=true, linktocpage=true, pdfstartpage=3, pdfstartview=FitV,%
  breaklinks=true, pageanchor=true,%
  pdfpagemode=UseNone, %
  plainpages=false, bookmarksnumbered, bookmarksopen=true, bookmarksopenlevel=1,%
  hypertexnames=true, pdfhighlight=/O,%nesting=true,%frenchlinks,%
  urlcolor=CTurl, linkcolor=CTlink, citecolor=CTcitation, %pagecolor=RoyalBlue,%
  pdftitle={\myTitle},%
  pdfauthor={\textcopyright\ \myName, \myUni, \myFaculty},%
  pdfsubject={},%
  pdfkeywords={},%
  pdfcreator={pdfLaTeX},%
  pdfproducer={LaTeX with hyperref and classicthesis}%
}

\makeatletter
\@ifpackageloaded{babel}%
  {%
    \addto\extrasamerican{%
    }%
    \addto\extrasngerman{%
    }%
    }{\relax}
\makeatother

\begin{document}
\frenchspacing
\raggedbottom
\selectlanguage{american} % american ngerman
%\renewcommand*{\bibname}{new name}
%\setbibpreamble{}
\pagenumbering{roman}
\pagestyle{plain}
%********************************************************************
% Frontmatter
%*******************************************************
%\include{FrontBackmatter/DirtyTitlepage}
%*******************************************************
% Titlepage
%*******************************************************
\begin{titlepage}
  \doublespacing
  \large
  \hfill
  \begin{center}

    \includegraphics[width=3cm]{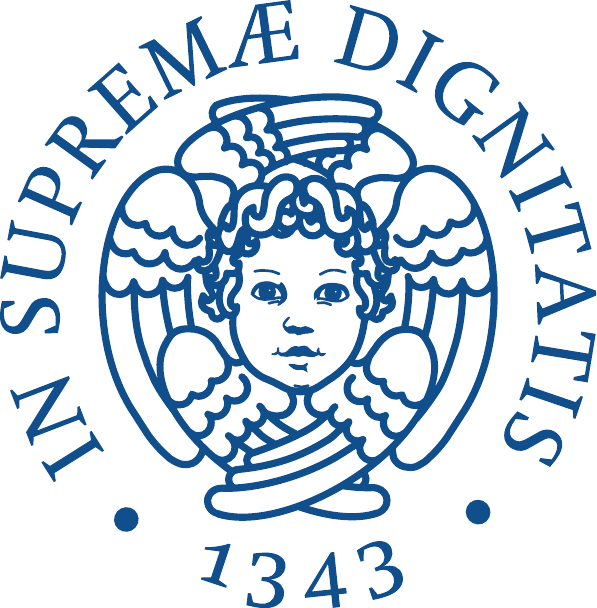}\\
    \vspace{10pt}
    \vspace{20pt}
    %\doublespacing
    %\textcolor{CTtitle}{\huge\textbf{\myTitle}}\\
    
    % Unipi is fixed
    {\Large \textsc{Università di Pisa}\par}
    {\Large \textbf{National Ph.D. in Artificial Intelligence}\par}
    {\normalfont \textit{XXXVII cycle} \par}
    \vspace{20pt}

    {\normalfont\bfseries\boldmath\Large
    \linespread{1}\selectfont % 1.0 is default; < 1.0 reduces space
    \begin{center}
        \myTitle:
        \\
        \mySubtitle
    \end{center}
    }	
    %\begingroup
    %{\setstretch{1.2}\color{black}\huge\spacedallcaps{\myTitle} \\ \medskip}
    %\endgroup
    %\emph{\mySubtitle}
  \end{center}
  \vspace{0.8cm}
  %\hrule
  \vspace{1.2cm}
  \onehalfspacing
  \begin{center}

    \begin{minipage}[t]{0.5\textwidth}
      \begin{flushleft}
        \textbf{Candidate:}\\
        \href{\myWebsite}{\color{black}{{\myFirstName} {\myLastName}}}
      \end{flushleft}
    \end{minipage}
    \begin{minipage}[t]{0.4\textwidth}
      \begin{flushright}
        \textbf{Supervisors:} \\
        \href{\myProfWebsite}{\color{black}{{\myProfTitle} {\myProfFirstName} {\myProfLastName}}}\\
        \href{\myOtherProfWebsite}{\color{black}{{\myOtherProfTitle} {\myOtherProfFirstName} {\myOtherProfLastName}}}\\
        %\href{\myOtherProfWebsite}{\color{black}{{\myOOtherProfTitle} {\myOOtherProfFirstName} {\myOOtherProfLastName}}}\\
      \end{flushright}
    \end{minipage}\\[1.5cm]
    \vspace{10pt}
    \small{Thesis submitted in partial fulfillment of the requirements for the degree of \par
    Doctor of Philosophy in Artificial Intelligence}
    
    %\myDepartment\\
    %\emph{\myProgram}
    %\href{https://www.unipi.it/}{\myUni}\\
    %\myLocation\\[1em]
%    {\mySubmissionMonth} {\mySubmissionYear}
  \end{center}
\end{titlepage}

\cleardoublepage%*******************************************************
% Dedication
%*******************************************************
\thispagestyle{empty}
\phantomsection
\pdfbookmark[1]{Dedication}{Dedication}

\vspace*{3cm}

\begin{comment}    
\begin{center}
    \emph{Give a neural network a thousand photos of cats and dogs, and it will learn to classify them with decent accuracy. Give it a billion, and it will achieve remarkable performance. But after all that, show it a horse, and it won’t have the slightest idea what it’s looking at.}
\end{center}
\end{comment}

%\begin{quote}
%\begin{minipage}{0.8\linewidth}
%\emph{Give a neural network a thousand photos of cats and dogs, and it will learn to classify them with decent accuracy. Give it a billion, and it will achieve remarkable performance. But after all that, show it a horse, and it won’t have the slightest idea what it’s looking at.} \\
%\noindent\rule{\linewidth}{0.4pt}
%\hrule
%\medskip
%\end{minipage}
%\end{quote}

%\medskip

%\begin{center}
%    Dedicated to the loving memory of Rudolf Miede. \\ \smallskip
%    1939\,--\,2005
%\end{center}

%\cleardoublepage\include{FrontBackmatter/Foreword}
\cleardoublepage%*******************************************************
% Abstract
%*******************************************************
%\renewcommand{\abstractname}{Abstract}
\pdfbookmark[1]{Abstract}{Abstract}
% \addcontentsline{toc}{chapter}{\tocEntry{Abstract}}
\begingroup
\let\clearpage\relax
\let\cleardoublepage\relax
\let\cleardoublepage\relax

\chapter*{Abstract}
%In recent decades, the \textit{statistical learning} research has made remarkable progress. The emergence of large datasets has enabled data-driven protocols that have transformed traditional scientific practice, with a huge impact on fields once dominated by mathematical models. However, these protocols share a key limitation: they struggle to generalize beyond the data they have seen, failing under \textit{distribution shift}. \textit{Causality offers a promising solution to this challenge}: causal models focus on the underlying mechanisms of change rather than surface-level associations, and enable robust reasoning under external interventions. This paradigm allows systems not only to learn from data, but also to reason about hypothetical scenarios, adapt to new environments, and inform better decision-making. This thesis explores how these two methodological approaches can be integrated to inform and enhance each other. The work is structured in two parts, each focusing on a different aspect of this integration.

While \textit{statistical learning} has made remarkable progress in recent decades, driven by the availability of large datasets and the rise of data-driven protocols that have transformed traditional scientific practice, it still faces a key limitation: it struggles to generalize beyond the data it has seen, e.g. often failing under distribution shift. \textit{Causality offers a promising solution to this challenge}: causal models focus on the underlying mechanisms of change rather than surface-level associations, \textit{enabling robust reasoning under external interventions}. This paradigm allows systems not only to learn from data, but also to reason about hypothetical scenarios, adapt to new environments, and inform better decision-making. This thesis explores how these methodological approaches can be integrated to inform  each other. The work is structured in two parts. %, each focusing on a different aspect of this integration.

The first part examines the problem of \textit{algorithmic recourse}—that is, identifying which changes an individual can make to overturn an unfavorable decision. We propose an alternative perspective to the established framework, arguing for a more realistic characterization of the cost function being minimized, a reformulation of the problem using causal terms, and the inclusion of a temporal dimension in policy evaluation. 

The second part focuses on dynamical systems, extending the notions of \textit{causal interventions and effects} to both discrete and continuous stochastic processes. It is shown that the equilibria of these processes—and how they change following an intervention—can be linked to Structural Causal Models (SCMs), that can be viewed as compressed representations of the underlying temporal dynamics. We argue for the relevance of such temporal causal models, emphasizing that in real-world scenarios, actions take time to produce their effects. 

The common goal of both parts is practical: the utility, robustness, and trustworthiness of these models depend critically on their \textit{ability to realistically represent the underlying phenomena}. A deeper integration of causality with data-driven methods is essential for building models that can adapt to real-world complexities and provide more reliable recommendations.

\vfill

\endgroup

\vfill

%\cleardoublepage\include{FrontBackmatter/Publications}
\cleardoublepage%*******************************************************
% Acknowledgments
%*******************************************************
\pdfbookmark[1]{Acknowledgments}{acknowledgments}

\bigskip

\begingroup
\let\clearpage\relax
\let\cleardoublepage\relax
\let\cleardoublepage\relax
\chapter*{Acknowledgments}

\bigskip

My thanks to Francesco, for convincing me that pursuing a PhD was the right choice.
My thanks to Dino and to Salvatore, for their guidance as supervisors.
My thanks to Isabel, for her guidance and for helping me shape the work I am most proud of.
My thanks to Marta, to Eleonora, and to Virginia, for being close to me through different moments of this journey.
And my thanks to Martina, for doing all of these things at once.

\endgroup

\cleardoublepage%*******************************************************
% Table of Contents
%*******************************************************
\pagestyle{scrheadings}
%\phantomsection
\pdfbookmark[1]{\contentsname}{tableofcontents}
\setcounter{tocdepth}{1} % <-- 2 includes up to subsections in the ToC
\setcounter{secnumdepth}{3} % <-- 3 numbers up to subsubsections
\manualmark
\markboth{\spacedlowsmallcaps{\contentsname}}{\spacedlowsmallcaps{\contentsname}}
\tableofcontents
\automark[section]{chapter}
\renewcommand{\chaptermark}[1]{\markboth{\spacedlowsmallcaps{#1}}{\spacedlowsmallcaps{#1}}}
\renewcommand{\sectionmark}[1]{\markright{\textsc{\thesection}\enspace\spacedlowsmallcaps{#1}}}
%*******************************************************
% List of Figures and of the Tables
%*******************************************************
\clearpage
\begingroup
    \let\clearpage\relax
    \let\cleardoublepage\relax
    %*******************************************************
    % List of Figures
    %*******************************************************
    %\phantomsection
    %\addcontentsline{toc}{chapter}{\listfigurename}
    %\pdfbookmark[1]{\listfigurename}{lof}
    %\listoffigures

    %\vspace{8ex}

    %*******************************************************
    % List of Tables
    %*******************************************************
    %\phantomsection
    %\addcontentsline{toc}{chapter}{\listtablename}
    %\pdfbookmark[1]{\listtablename}{lot}
    %\listoftables

    %\vspace{8ex}
    % \newpage

    %*******************************************************
    % List of Listings
    %*******************************************************
    %\phantomsection
    %\addcontentsline{toc}{chapter}{\lstlistlistingname}
    %\pdfbookmark[1]{\lstlistlistingname}{lol}
    %\lstlistoflistings

    %\vspace{8ex}

    %*******************************************************
    % Acronyms
    %*******************************************************
    %\phantomsection
    \pdfbookmark[1]{Acronyms}{acronyms}
    \markboth{\spacedlowsmallcaps{Acronyms}}{\spacedlowsmallcaps{Acronyms}}
    \chapter*{Acronyms}
    \begin{acronym}[CPDAGG]
    \setlength{\itemsep}{0.17cm}
    \acro{AI}{Artificial Intelligence}
    \acro{AR}{Algorithmic Recourse}
    \acro{BM}{Brownian Motion}
    \acro{BN}{Bayesian Network}
    \acro{CD}{Causal Discovery}
    \acro{CEx}{Contrastive Explanation}
    \acro{CI}{Causal Inference}
    \acro{CPDAG}{Completed Partially Directed Acyclic Graph}
    \acro{CRe}{Consequential Recommendations}
    \acro{DAG}{Directed Acyclic Graph}
    \acro{DE}{Differential Equation}
    \acro{dE}{Difference Equation}
    \acro{DSP}{Discrete-time Stochastic Process}
    \acro{GD}{Gradient Descent}
    \acro{GF}{Gradient Flow}
    \acro{MEC}{Markov Equivalence Class}
    \acro{ML}{Machine Learning}
    \acro{OdE}{Ordinary Difference Equation}
    \acro{ODE}{Ordinary Differential Equation}
    \acro{PDAG}{Partially Directed Acyclic Graph}
    \acro{RDE}{Random Differential Equation}
    \acro{RdE}{Random Difference Equation}
    \acro{SCM}{Structural Causal Model}
    \acro{SDE}{Stochastic Differential Equation}
    \acro{SdE}{Stochastic Difference Equation}
    \acro{SP}{Stochastic Process}
    \acrodefplural{SP}[SPs]{Stochastic Processes}
    \acro{VAR}{Vector AutoRegressive Model}
    \acro{XAI}{eXplainable Artificial Intelligence}
    \end{acronym}

\endgroup

\clearpage\renewcommand\nomgroup[1]{%
\vspace{10pt}
  \item[\bfseries
  \ifstrequal{#1}{A}{Variables}{%
  \ifstrequal{#1}{B}{Probability}{%
  \ifstrequal{#1}{C}{Statistics}{%
  \ifstrequal{#1}{D}{Causality}{}}}}%
]}
\makenomenclature
\renewcommand{\nomname}{Notation and Terminology} % Customize the title

% RANDOM VARIABLES
\nomenclature[A]{$X,Y$}{Real-valued Random Variables}
\nomenclature[A]{x}{particular value or realization of $X$}
\nomenclature[A]{$\X=(X_1,\dots,X_n)$}{$n$-dimensional Random Vector}
\nomenclature[A]{$\Noise$}{Real-valued Noise Variable}
\nomenclature[A]{$\noise$}{realization of the noise $\Noise$}

% PROBABILITIES
\nomenclature[B]{$\Prob$}{Probability Measure, also denoting the Probability Distribution of $X$, $\Prob_X$, when the context is clear}
\nomenclature[B]{$\prob$}{Probability Density Function}
\nomenclature[B]{$\prob(x)$}{density of $\Prob_X$ evaluated at the point $x$}
\nomenclature[B]{$\prob(y \mid x)$}{Conditional density given $X=x$ evaluated at $y$}

% STATISTICS
\nomenclature[C]{$\expected{X}$}{Expected Value of $X$}
\nomenclature[C]{$\Var[X]$}{Variance of $X$}
\nomenclature[C]{$\Cov[X,Y]$}{Covariance between $X$ and $Y$}
\nomenclature[C]{$\normal(\mu,\sigma^2)$}{Normal distribution with mean $\mu$ and variance $\sigma^2$}
\nomenclature[C]{$\uniform(a,b)$}{Uniform distribution between values $a$ and $b$}

% GRAPHS
\nomenclature[D,1]{$\graph$}{A graph. The symbols $X, Y, \dots$ are used to denote its nodes, overloading the notation for their corresponding random variables}
\nomenclature[D,2]{$\parents[i]_{\graph}$}{Parents of node $\Xcomp[i]$ in graph $\graph$. Subscript omitted when clear from the context}
\nomenclature[D,3]{$\children[i]_{\graph}$}{Children of node $\Xcomp[i]$ in graph $\graph$}
\nomenclature[D,3]{$\descendants[i]_{\graph}$}{Descendants of node $\Xcomp[i]$ in graph $\graph$}
\nomenclature[D,3]{$\ancestors[i]_{\graph}$}{Ancestors of node $\Xcomp[i]$ in graph $\graph$}
\nomenclature[D]{$\ce{}$}{Causal Effect}
\nomenclature[D]{$\ce{}_t$}{Causal Effect over Time}

\printnomenclature[1.4in]  % Print the notation list
%********************************************************************
% Mainmatter
%*******************************************************
\cleardoublepage
\pagestyle{scrheadings}
\pagenumbering{arabic}
%\setcounter{page}{90}
% use \cleardoublepage here to avoid problems with pdfbookmark
%
%
%
\chapter{Introduction}\label{ch:introduction}

\section{Motivation}\label{ch:causality_motivation}
Two main strategies are adopted in investigating natural and social systems, as well as designing procedures to perform practical tasks.

The first is a \emph{model-driven} approach, grounded in general principles and an established theoretical framework, often emphasizing \emph{causal relationships} to explain and predict the system's behavior, understanding the mechanisms by which elements influence one another. Model-based methods offer the advantages of interpretability and alignment with scientific theory, making them well-suited for generalization across domains. %, even under conditions not initially considered. 
However, they can struggle when applied to highly complex systems, where the underlying laws are too difficult to model from a human perspective or are not fully understood.

The second is a \emph{data-driven} approach, and it is particularly well-suited to systems governed by unknown or poorly understood laws. Instead of relying on pre-existing theoretical frameworks~\citep{anderson}, it draws insights from observational data to uncover \emph{patterns and statistical dependencies}. While this approach offers notable strengths, such as its flexibility in high-dimensional contexts, it also has significant limitations. Data-driven methods often lack transparency and interpretability, and without rigorous validation, they risk producing spurious or non-generalizable results, especially when applied to previously unseen scenarios or systems that have changed since the data were collected.

By comparing these two paradigms, we can appreciate the complementary nature of theory-driven and data-driven inquiry. A balanced integration of both may lead to more robust, insightful models that leverage each other's strengths while mitigating their weaknesses. These preliminary remarks introduce a general discussion: \textit{under what conditions and within what limits is it possible to integrate these two methodologies while retaining their strengths}? 

\subsection{The Role of Causality}

In what follows, a specific perspective will be adopted, limiting the analysis to the role that the concept of causality could play in addressing the question just posed. More concretely, on the one hand, we are interested in exploring how \ac{ML} can support the discovery of causal structures and enable reliable causal inference. On the other hand, principles of causal reasoning can also inform the design and development of \ac{ML} systems.
After all, causal reasoning is a core component of human intelligence, and causality plays a fundamental role in how we structure and understand knowledge~\citep{pearl2009causality}. Therefore, adopting a cause-effect framework could help us build more transparent, interpretable, and robust learning protocols.

\textbf{Beyond the i.i.d. assumption.} 
Most techniques used to draw inferences from data are inherently associative. These methods focus on identifying statistical dependencies among variables based solely on observational data, typically operating under the assumption that data points are \emph{independently and identically distributed (i.i.d.)} and generated from a fixed probability distribution. This foundational assumption enables the development of powerful and efficient algorithms, particularly for tasks such as regression, classification, and pattern recognition. However, while this associative framework has demonstrated remarkable effectiveness in many practical applications, it proves insufficient when epistemic goals extend beyond prediction to encompass explanation, understanding, or policy formulation. Such objectives are common in disciplines like the social sciences, epidemiology, medicine, and engineering, where addressing real-world problems often requires identifying causal-effect relationships. In these contexts, it becomes essential to answer questions such as: \emph{What would happen if a specific action were taken?} or \emph{How would the outcome change under a particular intervention?}

Accordingly, our focus shifts toward characterizing how a system’s behavior might change in response to external interventions, and evaluating counterfactual scenarios i.e., hypothetical alternatives to observed events. Gaining this level of insight necessitates moving beyond purely associative models to embrace causal frameworks. Such models offer a deeper understanding of the underlying generative mechanisms and help guard against common inferential pitfalls in observational analysis, such as spurious correlations and counterintuitive phenomena like Simpson’s Paradox~\citep{wagner1982simpson}. Reflecting this shift, the machine learning community has grown interested in integrating causal reasoning into inductive learning paradigms. This increasing attention stems from recognizing that causal structures can significantly enhance predictive models, rendering them more robust, interpretable, and better suited to inform real-world decision-making.

\section{Research Questions}\label{sec:research_questions}

In the first stage, this thesis explores how causal reasoning can be integrated into inductive procedures traditionally grounded in associative methods. The goal is to understand whether adopting a causal perspective leads to more accurate and realistic models of the systems under investigation. 
In the second stage, the focus shifts toward an additional layer of refinement i.e., integrating the temporal dimension into the causal framework.

These two research directions represent distinct and progressive steps within a broader conceptual pipeline, aiming to improve the ability of predictive and decision-making models to reflect the true complexity of the phenomena they are meant to describe in an increasingly realistic way. %The first step introduces structure through causality; the second adds dynamics through time.

This thesis is structured around the following Research Questions (RQs):

\begin{enumerate}[label=\textbf{RQ\arabic*}]
    \item \textit{What is the role of causality in building more realistic and interpretable decision-making systems?} \\
    This question, addressed primarily in Part~\ref{pt:algorithmic_recourse}, investigates whether integrating causal reasoning is a necessary and/or sufficient condition for developing decision-making protocols that more accurately reflect the real-world systems they are intended to model.

    \item \textit{How can classical causal frameworks be extended to incorporate the temporal dimension of real-world systems?} \\
    This question, explored in Part~\ref{pt:time_aware_causality}, focuses on the limitations of time-unaware models and examines the benefits of embedding temporal structure into causal reasoning, particularly when modeling interventions over time.
\end{enumerate}

\section{Thesis Contributions and Outline}\label{sec:contributions_outline}

The thesis is organized into three parts. 

\begin{itemize}
    \item \textbf{Part I – Causality.}  
    This first part introduces the core mathematical tools of causal inference and establishes the common language used throughout the rest of the thesis. It serves both as an accessible introduction to the formal notation and as a structured overview of the main levels of causal analysis. In doing so, it aims to give the reader a clear understanding of what causal models promise to deliver and the assumptions they require—and the limitations that emerge as one moves to more complex settings.
    \item \textbf{Part II – Towards Realistic Algorithmic Recourse. (\textbf{RQ1})}  
    This part focuses on decision-making systems, with particular attention to the problem of Algorithmic Recourse—i.e., identifying what changes an individual should make to overturn an unfavorable outcome. 
    
    \textbf{Chapter~\ref{ch:algo_recourse} - Algorithmic Recourse} introduces the classical formulations of the problem (Contrastive Explanations and Consequential Recommendations), discusses their conceptual differences, and connects them to the causal framework. It also provides a formal mathematical foundation and identifies key limitations in existing approaches. 
    
    \textbf{Chapter~\ref{ch:cost-subjective} - Cost is Subjective} builds on this analysis by shifting the focus to a user-centered perspective. It questions the assumption of a fixed, universal cost function and argues instead for a subjective, experience-driven notion of effort. The chapter introduces \EMA, an interactive tool that allows users to explore the decision space by directly modifying features. Unlike traditional methods, this approach relies on the user's knowledge and judgment to define plausible and actionable changes, without embedding predefined constraints. The tool supports experts and non-experts and is evaluated through examples on standard datasets. As the most human-centric contribution of the thesis, this chapter addresses the problem of realism without resorting to formal causal reasoning, highlighting limitations that remain even when such reasoning is later introduced.
    
    \textbf{Chapter~\ref{ch:timecar} - Causal Algorithmic Recourse and Time} initiates a reflection on the role of time in causal models and introduces one of the earliest contributions of the thesis to integrate temporal reasoning into Algorithmic Recourse (\textbf{RQ2}). Building on recent causal formulations, this work highlights a critical limitation of existing approaches—their inability to account for temporal dependencies between features and interventions. The chapter proposes a novel method that incorporates time into the cost evaluation of recourse by leveraging the topology of the causal graph. It also revisits the concept of \textit{actionability}, showing how the feasibility of interventions can vary significantly when time is taken into account. Finally, the work explores the fairness implications of time-unaware models, particularly their risk of exacerbating inequalities for disadvantaged groups, and supports its findings with semi-synthetic datasets.

    \item \textbf{Part III – Time-Aware Causality. (\textbf{RQ2})}  This part is dedicated to extending the classical causal framework, represented in its most complete form by Structural Causal Models, to dynamical systems, specifically stochastic processes.  
    
    \textbf{Chapter~\ref{ch:dynamical_sistems_intro} - Autonomous Dynamical Systems }introduces the mathematical objects related to dynamical systems and examines how the typical assumptions of Structural Causal Models transform in temporal models. In particular, it explores the connection between the notion of \textit{autonomy} in dynamical systems and the concept of \textit{modularity} in Structural Causal Models. The chapter also provides a first contextualization of \textit{causal inference over time}, proposing an initial link between estimating causal effects in the original static setting and its extended, time-aware formulation. 
    
    \textbf{Chapter~\ref{ch:Ci_over_time} - Causal Inference over Time} presents a unified framework for performing causal inference over time for discrete-time stochastic processes. It introduces two classes of time-aware interventions and provides the first general method for mapping Vector Autoregressive models to Structural Causal Models at equilibrium, preserving causal semantics under stability assumptions. Through theoretical results and empirical evaluations on both synthetic and real-world datasets, this chapter demonstrates the practicality of causal reasoning over time for interventional forecasting and retrospective counterfactuals. 
    
    \textbf{Chapter~\ref{ch:CIcontinuous_time} - Continuous Time and Non-linear Systems} expands the formal framework introduced in the previous chapter by moving \textit{from discrete-time to continuous-time representations}. It introduces stochastic differential equations as the mathematical foundation for modeling systems affected by ongoing random fluctuations. It presents the necessary stochastic calculus tools, such as Brownian motion, Itô processes, and the Fokker–Planck equation. It uses them to define and estimate the causal effect of interventions over continuous time. It also introduces the concept of stochastic gradient flows, allowing for closed-form equilibrium analysis and providing a natural mechanism for studying long-term causal effects. It presents a recursive procedure for computing the stationary distribution under intervention for acyclic systems. Finally, it bridges the gap between discrete-time and continuous-time models, outlining conditions under which results can be transferred across frameworks.
\end{itemize}

A more detailed characterization of the contributions related to the works published during the PhD is provided at the beginning of each chapter in the dedicated \textit{Contributions} section.

\section{Publications}\label{sec:PhDPubs}
We list below the papers that form the basis of this thesis, published at the time of writing. We will highlight when a chapter is based on a published work.

\begin{refsection}[ownpubs]
    \small
    \nocite{*} % is local to to the enclosing refsection
    \printbibliography[heading=none]
\end{refsection}

The following publication during the PhD studies does not contribute to the contents of this thesis:

\fullcite{nannini2024mapping}
%\begin{itemize}
%\nobibliography* \item \bibentry{nannini2024mapping}
%\end{itemize}

%\emph{Attention}: This requires a separate run of \texttt{bibtex} for your \texttt{refsection}, \eg, \texttt{ClassicThesis1-blx} for this file. You might also use \texttt{biber} as the backend for \texttt{biblatex}. See also \url{http://tex.stackexchange.com/questions/128196/problem-with-refsection}.

%
% PART 1
\cleardoublepage
\begin{comment}\ctparttext{You can put some informational part preamble text here.
Illo principalmente su nos. Non message \emph{occidental} angloromanic
da. Debitas effortio simplificate sia se, auxiliar summarios da que,
se avantiate publicationes via. Pan in terra summarios, capital
interlingua se que. Al via multo esser specimen, campo responder que
da. Le usate medical addresses pro, europa origine sanctificate nos se.}
\end{comment}
\part{Causality}\label{pt:causality}
\chapter{Mathematical Framework}\label{ch:causality_math_framework}

In this chapter, we introduce the mathematical objects underlying the causal framework, following the structure of Pearl’s Ladder of Causation~\citep{pearl2009causality}. As we move upward along this ladder, new mathematical tools are introduced and additional assumptions are layered on top of the previous ones. A summary of this pipeline is illustrated in Figure~\ref{fig:causal_pipeline_graph}.

\begin{figure}[h]
\centering
\begin{tikzpicture}[node distance=2.4cm and 1.2cm, every node/.style={align=center}]
    % Nodes (boxes)
    \node (obs) [draw, rectangle, rounded corners, minimum width=3.7cm, minimum height=2.1cm, fill=blue!10] 
        {\textbf{Observational} \\[1pt] \footnotesize Probability $\Prob(\X)$};
        
    \node (intv) [right=of obs, draw, rectangle, rounded corners, minimum width=3.7cm, minimum height=2.1cm, fill=green!10] 
        {\textbf{Interventional} \\[1pt] \footnotesize Causal graph \( \mathcal{G} \)};
        
    \node (cf) [right=of intv, draw, rectangle, rounded corners, minimum width=3.7cm, minimum height=2.1cm, fill=red!10] 
        {\textbf{Counterfactual} \\[1pt] \footnotesize Structural functions \( f_i \)};
        
    % Arrows
    \draw[->, thick] (obs) -- (intv);
    \draw[->, thick] (intv) -- (cf);

    % Assumptions node above arrows
    \node (assumptions1) at ($(obs)!0.5!(intv) + (0,2)$) 
        [font=\footnotesize, text width=4.4cm] 
        {\textbf{Causal assumptions:\\} Markov condition, Causal Sufficiency, Faithfulness};

    \node (assumptions2) at ($(intv)!0.5!(cf) + (0,2)$) 
        [font=\footnotesize, text width=5.0cm] 
        {\textbf{Structural assumptions:\\} Functional form, Modularity, Invariance};

    % Vertical lines from assumptions to arrows
    \draw[thick] (assumptions1.south) -- ($(obs)!0.5!(intv)$);
    \draw[thick] (assumptions2.south) -- ($(intv)!0.5!(cf)$);

    % Labels below
    \node[below=0.3cm of obs] {\footnotesize Statistical inference};
    \node[below=0.3cm of intv] {\footnotesize Interventions};
    \node[below=0.3cm of cf] {\footnotesize Counterfactual reasoning};
\end{tikzpicture}
\caption{\textbf{The Causal Inference Pipeline}: From Observational data to Interventions and Counterfactual reasoning}
\label{fig:causal_pipeline_graph}
\end{figure}
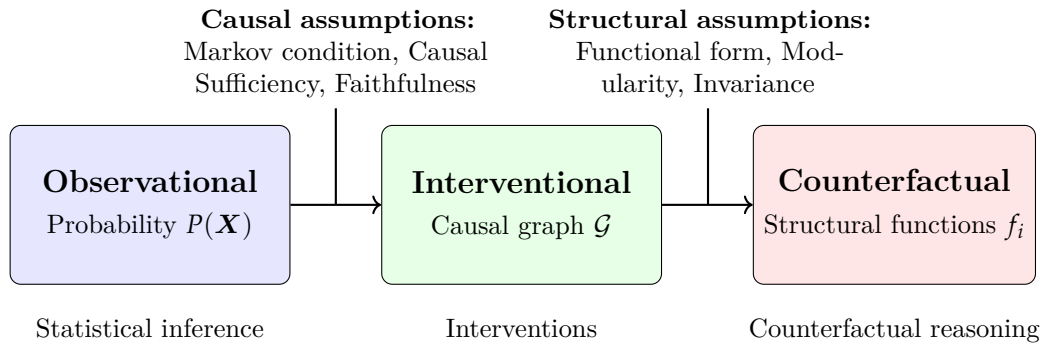

Understanding a system via the lens of causality involves progressing through a series of conceptual and methodological stages, each carrying its own theoretical and practical implications. Climbing the ladder of causality requires increasingly strong assumptions, many of which are not directly testable. This progression highlights a fundamental trade-off: as causal tools become more powerful, their domain of applicability narrows~\citep{barocas2020assumptions}. 

%At each level, new mathematical objects are introduced, gradually increasing the complexity of the overall framework. 
At the observational level, the central object of analysis is the probability distribution. Understanding the system involves identifying statistical patterns and examining relationships between variables through conditional probability queries. The interventional level introduces a causal graph, which encodes the directional dependencies among variables. This structure allows the model to move beyond correlations, enabling experimental-style queries that reflect the effects of active interventions. The counterfactual level adds another layer of complexity i.e., the functional identification of the mechanisms driving cause-effect relationships. While previous levels focused on whether and in what direction a causal influence exists, counterfactual reasoning aims to uncover exactly how those effects are produced. 

\section{The Ladder of Causation}
When the i.i.d. assumption breaks down, interventions and counterfactuals typically induce \emph{distribution shifts}, altering the data-generating process and invalidating standard learning protocols~\citep{holland1986statistics}. A model trained on one context may perform poorly when conditions change due to policy shifts, environmental variability, or human decisions. Addressing such scenarios requires models that engage with the causal structure underlying data distributions, rather than relying on surface-level associations.

From a practical perspective, the challenge is compounded by the fact that causal information is rarely available in a clean or controlled form. In most real-world settings, data are \emph{observational} rather than \emph{experimental}, meaning that we observe outcomes as they naturally occur, without randomized interventions. This limitation significantly constrains our ability to draw reliable causal conclusions. To navigate this complexity, there is a need for an epistemological framework that distinguishes between different levels of causal reasoning. 
As illustrated in \cref{tab_causality}, the \emph{Ladder of Causation}~\citep{pearl2009causality, bareinboim2022pearl} provides a conceptual hierarchy for structuring this progression. %Climbing this ladder requires increasingly sophisticated tools and assumptions, but it also offers deeper insights into the structure of reality and the consequences of action. 
In the following discussion, we will climb this ladder step by step, using the bivariate case as a guiding example.

\begin{table}[ht!]	
	\centering
	\begin{adjustbox}{max width=\textwidth}
		\begin{tabular}{@{}lllll@{}}
			\toprule
			Level &
			Typical Question &
			Model &
			Data &
			ML Methods \\ \midrule \\
			\textbf{3. Counterfactual} &
			\begin{tabular}[c]{@{}l@{}}What would Y have been \\ if I had acted on X?\end{tabular} &
			\begin{tabular}[c]{@{}l@{}}$\hat{\noise}_j := y^j-f(x^j)$\\
				$x^j:=\hat{x}$\\	
				$(y^j)^{\hat{x}}:=f(\hat{x})+\hat{\noise}^j$\end{tabular} &
			No Data &
			\begin{tabular}[c]{@{}l@{}}No ML \\ methods\end{tabular} \\ \\
			\textbf{2. Interventional} &
			\begin{tabular}[c]{@{}l@{}}How would acting on X \\ change Y?\end{tabular} &
			\begin{tabular}[c]{@{}l@{}}
				$x^j:=\hat{x}$\\$y^j:=f(\hat{x})+\noise^j$ \end{tabular}&
			\begin{tabular}[c]{@{}l@{}}$(\hat{x},y^i) \sim  \prob^{\hat{x}}$\\ $i.i.d$\end{tabular} &
			\begin{tabular}[c]{@{}l@{}}Reinforcement \\ Learning\end{tabular} \\ \\
			\textbf{1. Observational} &
			\begin{tabular}[c]{@{}l@{}}How would seeing X\\ change my belief on Y?\end{tabular} &
			$y^j=f(x^j)+\noise^j$ &
			\begin{tabular}[c]{@{}l@{}}$(x^i,y^i) \sim \prob$\\$i.i.d$\end{tabular} &
			\begin{tabular}[c]{@{}l@{}}Supervised/\\ Unsupervised\\ Learning\end{tabular} \\ \\\bottomrule
		\end{tabular}
	\end{adjustbox}
	\caption{\textbf{The Ladder of Causality.} The Ladder of Causality for the static-bivariate case.}
	\label{tab_causality}
\end{table}

\begin{itemize}
	\item[1.] \textbf{How would seeing X change my belief on Y?} On the first step, observational data $(x^i,y^i)$ are i.i.d. according to the a priori probability $\Prob$, and questions admit answers as conditional probabilities, say $\prob(y|x)$. We assume that a functional mechanism $f$ governs the relationship between $X$ and $Y$, excluding measurement errors, unmodelled external factors, and additional stochasticity captured by the noise $\noise$. For simplicity, we can think of the \ac{ML} protocol as aiming to recognize $f$, for example, by adopting a regression procedure.   
	\item[2.] \textbf{How would acting on X change Y?} Moving up to the next step, if we intervene by manually fixing the value of $X$, data are generally no longer distributed as $\Prob$. For each possible external intervention, e.g., $X:=\hat{x}$, data are distributed as the corresponding new distribution $\prob^{\hat{x}}$, to be determined. Generally, observational data are insufficient to extrapolate the system's cause-effect relations, and only a partial causal structure can be uniquely determined. No matter how elaborate \ac{ML} protocols may be, the observational distribution and the set of independence relations inferred from data are compatible with a non-trivial class of alternative causal structures. On the other side, causal inferences require causal assumptions: some conditions under which it is possible to answer interventional questions by examining observational data are available, but assuming a causal graph can correctly capture the cause-effect relationships between variables involved. 
	\item [3.] \textbf{How would Y have been if I had acted on X?} In this step, data are unavailable by definition at the counterfactual level, and no experiment in the world will be sufficient to answer counterfactual questions. More precisely, counterfactual inferences require assumptions about the mechanism underlying the observed phenomenon: the model $Y:=f(X)+\noise$ is necessary to make counterfactual conclusions. In a nutshell, the noise actual value $\hat{\noise}^j$ is determined via the observational couple $(x^j,y^j)$, providing $\hat{\noise}_j := y^j-f(x^j)$. Consequently, the counterfactual quantity $(y^{\hat{x}})^j$ is obtained by intervening with $X:=\hat{x}$ on the modified equation $y=f(\hat{x})+\hat{\noise}$.  
\end{itemize} 
The points above show how causality is essential for accessing information that is not necessarily obtainable on the observational level, regardless of the size of the system under consideration. Accordingly, the integration of causal structures within inductive protocols bears the potential to enhance their \emph{expressive capacity}. In this regard, the cause-effect problem in the \emph{bivariate} context may be considered both the most fundamental from a conceptual point of view and the least realistic in modeling concrete situations.
\\

\textbf{The Bivariate Case.} Given only two observable and \emph{correlated} variables, denoted here as $X$ and $Y$, a fundamental question arises: in which direction, if any, does a potential causal influence flow between them? That is, does $X$ cause changes in $Y$, does $Y$ influence $X$, or is the observed correlation merely the result of a hidden confounder, i.e., a third not-observed variable which causes both $X$ and $Y$, providing their non-zero correlation, or coincidence? This basic \emph{causal identification} problem has been the subject of significant theoretical interest, particularly within causal inference and algorithmic information theory. While it is generally acknowledged that real-world systems cannot be accurately modeled as mere collections of bivariate relationships, the analysis of such simplified pairs remains highly relevant. Much of the scientific literature 
%--- especially in the social sciences ---
is rich in discussions of variable pairs commonly assumed or debated to share a causal connection. Examples include: taxation and government expenditure, educational attainment and income level, military spending and aggressive behavior, inflation and unemployment, or social status and political affiliation. From a cognitive viewpoint, it seems clear that humans tend to act with this simplified causal relationship in mind, whereas more complex causal structures are employed in less common contexts. Despite their simplicity, these binary associations often serve as entry points for formulating and testing hypotheses about underlying mechanisms. Furthermore, from a methodological standpoint, the bivariate case provides a valuable setting to illustrate foundational concepts of causal inference. It also helps to clarify how such concepts may intersect with, and potentially inform, modern \ac{ML} approaches, particularly in contexts where large-scale observational data are abundant, but causal structure is unclear. 
\\
\\
The analysis for the bivariate static case builds upon the preliminary ideas summarized in \cref{tab_biv}, and aims to expose general properties of causality that are essential for both theoretical development and application to real-world inference problems. At this point, considering only the interventional level of the ladder, \cref{tab_biv} collects the two causal alternatives available when the \emph{confounder} is excluded a priori. In a nutshell, the scenario can be outlined as follows:    

\begin{table}[t]
	\centering
	\resizebox{\textwidth}{!}{%
		\begin{tabular}{@{}cccc@{}}
			\toprule
			\begin{tabular}[c]{@{}c@{}} Observational \\ Level\end{tabular} &
			\multicolumn{1}{c}{\begin{tabular}[c]{@{}c@{}}Causal \\ Structure\end{tabular}} &
			\multicolumn{1}{c}{\begin{tabular}[c]{@{}c@{}}Post-Interventional\\ Level\end{tabular}} &
			\begin{tabular}[c]{@{}c@{}}Interventional \\ Distribution\end{tabular} \\ \midrule
			\multirow{2}{*}{
				\begin{tikzpicture}
				\node[vertex] (x) at (0,-0.5) {X};
				\node[vertex] (y) at (1.5,-0.5) {Y};
				\draw[uedge] (x) -- (y);  
				\end{tikzpicture}} &
			\begin{tikzpicture}
			\node[vertex] (x) at (0,-0.5) {X};
			\node[vertex] (y) at (1.5,-0.5) {Y};
			\draw[edge] (x) -- (y);  
			\end{tikzpicture} &
			\begin{tikzpicture}
			\node[fixed vertex] (x) at (0,0) {X};
			\node[vertex] (y) at (1.5,0) {Y};
			\draw[edge] (x) -- (y);  
			\end{tikzpicture} &
			\raisebox{0.5\totalheight}{$\prob^{\hat{x}}(y) = \prob(y|\hat{x})$} \\
			&
			\begin{tikzpicture}
			\node[vertex] (x) at (0,0) {X};
			\node[vertex] (y) at (1.5,0) {Y};
			\draw[edge] (y) -- (x);  
			\end{tikzpicture} &
			\begin{tikzpicture}
			\node[fixed vertex] (x) at (0,0) {X};
			\node[vertex] (y) at (1.5,0) {Y};
			\end{tikzpicture} &
			\raisebox{0.5\totalheight}{$\prob^{\hat{x}}(y) = \prob(y)$} \\ 
			\bottomrule
		\end{tabular}%
	}
	\caption{\textbf{Cause-Effect Pairs}. Whenever the presence of a confounder is ruled out, the intervention $X:={\hat{x}}$ makes it possible to distinguish which case the system belongs to by comparing the interventional distribution $\prob^{\hat{x}}$ with the observational one.}
	\label{tab_biv}
\end{table}

\begin{itemize}
	\item Unlike statistical correlation, causation is inherently \emph{directional}. Therefore, a satisfactory theory of causality must account for this asymmetry — an aspect we must be able to detect, at least to some extent, when attempting to infer causal information from statistical dependencies. However, given only the observational joint distribution $\prob(X,Y)$ and no further assumptions, the direction of causality between $X$ and $Y$ cannot be determined. Indeed, for any such joint distribution, it is always possible to construct both a functional model $\Phi_{X \rightarrow Y}$, i.e., $Y := f(X, \noise)$ where $f$ is deterministic and $\noise$ is a noise term independent of $X$, and a model $\Phi_{Y \rightarrow X}$ of the opposite direction.
	
	\item This inherent \emph{symmetry} in the joint distribution is theoretical and remains unaffected by the amount of observational data available. This highlights a key distinction from standard predictive modeling, where increasing the volume of data typically improves performance. In contrast, more data does not resolve the fundamental ambiguity of causal direction.
	
	\item Importantly, while both functional models $\Phi_{X \rightarrow Y}$ and $\Phi_{Y \rightarrow X}$ can reproduce the same joint distribution, they generally differ in structure. This observation opens the door to resolving causal asymmetry by constraining the class $\Phi$ of admissible models or introducing a complexity criterion to distinguish between them. In the latter case, an \emph{Occam’s Razor} principle might be invoked: the direction yielding a simpler or more plausible model is preferred. Either way, the symmetry can only be broken by incorporating prior knowledge or additional assumptions, highlighting the necessity of theoretical input when working with purely observational data.
\end{itemize}

It is important to emphasize that further complications arise upon closer examination. In real-world scenarios, dependencies between random variables $X$ and $Y$ may not stem from a direct causal relationship, but rather from alternative underlying mechanisms. For instance, the observed association might result from selection bias, where both variables are conditionally dependent on unobserved \emph{confounding factors.} In such cases, the joint distribution $\prob(X, Y)$ reflects a distorted view of the true data-generating process, leading to potentially spurious inferences. Moreover, dependencies may emerge from \emph{shared temporal structures} rather than causality per se. A classic example involves systems governed by deterministic physical laws, where variables co-evolve over time due to common dynamics rather than a causal influence of one variable on the other. In these settings, the temporal correlation between $X$ and $Y$ may reflect their simultaneous response to an underlying process, rather than indicating a directional causal link. These considerations underscore the inherent challenges in distinguishing causation from mere statistical association. They also reinforce the necessity of incorporating domain knowledge, careful model design, or interventional data to correctly identify causal relationships,  especially in observational studies where such confounding factors and structural dependencies are likely present.

\textbf{Towards a general discussion.} The challenges and limitations encountered in the bivariate case become significantly more intricate in the multivariate setting, where the underlying causal structures are typically richer, more entangled, and susceptible to a wider range of confounding influences. In multivariate systems, not only does the number of possible causal configurations grow combinatorially with the number of variables, but the potential for complex dependencies, feedback loops, and latent confounders increases correspondingly. Moreover, the task of causal inference in high-dimensional settings is often compounded by computational constraints, such as the scalability of inference algorithms and the tractability of evaluating competing models. In \cref{pt:causality}, we shall outline the general mathematical framework, with particular attention to the difficulties mentioned above.

%In light of these challenges, our approach is to delineate a structured and methodical \emph{causal pipeline} — one that is as linear as possible and organized into clearly demarcated modules. This modular decomposition serves two complementary purposes. On the one hand, it provides a broad, high-level framework for understanding the typical stages of scientific inquiry involved in causal analysis, from data acquisition and modeling to inference and validation. On the other hand, it enables a precise characterization of the specific assumptions, methodological choices, and critical bottlenecks that emerge at each step of the process. By isolating the various stages of causal inference, we aim to both clarify the roles played by different components of the pipeline and make explicit the points at which domain knowledge, computational resources, or statistical assumptions become essential. This systematic approach not only enhances interpretability but also facilitates the identification of failure modes and the development of more resilient causal methodologies. Finally, this examination will contribute to clarifying where our work has offered original insights, even while operating within a broader landscape of themes and challenges already explored in the existing literature.
Given these challenges, our approach delineates a structured and methodical \emph{causal pipeline}, designed to be as linear as possible and organized into clearly demarcated modules. This decomposition serves two complementary purposes. On the one hand, it offers a broad, high-level framework that captures the typical stages of causal analysis, from data acquisition and modeling to inference and validation. On the other hand, it enables a precise characterization of the assumptions, methodological choices, and critical bottlenecks encountered throughout the process. By isolating each stage, we aim to clarify the specific role played by every component and to highlight where domain knowledge, computational resources, or statistical assumptions become essential. This systematic organization not only enhances interpretability but also facilitates the identification of potential failure modes and the development of more resilient methodologies. Finally, it helps situate our contributions within the broader landscape of existing research, clarifying their scope and limitations. In the following three sections, we will explore this progression in detail, introducing each causal level along with its associated formalisms and assumptions.

\section{Observational Level}\label{sec:observational_level}

In statistical inference, the study of a system typically begins with the collection of a finite number of observations. The fundamental goal is to infer the properties of the underlying data-generating process by detecting patterns, regularities, and systematic deviations across the observed instances. In formal terms, the objective is to recover, or approximate, the probability distribution $\Prob(\X)$ from which each observation is assumed to be a sample. This constitutes the first step of any statistical modeling effort and already requires the adoption of preliminary assumptions. The most basic of these is the i.i.d. assumption.

\begin{assumption}[i.i.d.\ Sampling]\label{ass:iid-formal}
Let $\X_1, \dots, \X_n$ be random (possibly vector valued) variables defined over a common sample space $\mathcal{X}$. The \emph{i.i.d.\ assumption} requires that:
\begin{enumerate}
    \item Each variable $\X_k$ has the same marginal distribution:
    $$
        \Prob_{\X_k}(\x) = \Prob_{\X_1}(\x) =: \Prob_{\X}(\x) \quad \text{for all } 1 \leq k \leq n;
    $$
    \item The joint distribution factorizes as the product of its marginals:
    $$
        \Prob_{\X_1,\dots,\X_n}(\x_1,\dots,\x_n) = \prod_{k=1}^n \Prob_{\X_k}(\x_k).
    $$
\end{enumerate}
That is, the variables are identically distributed and mutually independent. A dataset $\mathcal{D} = \{\x_1, \dots, \x_n\}$ is said to satisfy the i.i.d.\ assumption if it represents a joint realization of $\X_1, \dots, \X_n$. Consequently, the probability of the entire dataset factorizes as
$$
    \Prob(\mathcal{D}) = \prod_{i=1}^n \Prob_{\X}(\x_i).
$$
This assumption guarantees two essential properties. First, it ensures that the target of inference—the distribution $\Prob(\X)$—is well-defined. Second, it guarantees that the complexity of this target is fixed a priori and does not grow with the sample size. Each observation contributes equally and cumulatively to its identification.
\end{assumption}

\paragraph{Probability Factorization.}\label{sec:probability_factorization_math_framework}

Once the inferential target is identified, a natural question arises: how complex is it to determine $\Prob$? How many parameters are required to describe it accurately? How many samples are needed to estimate these parameters reliably? And how does this complexity scale with the dimension of $\X$? More importantly, how can we ensure this complexity remains manageable as the dimensionality grows?

A simple example may help clarify these concerns. Consider an $n$-dimensional binary random variable $\X$, with joint distribution $\Prob(\X)$. A naive representation of $\Prob(\X)$ consists of a table assigning a probability to each possible realization $\x$ of $\X$. Since there are $2^n$ such realizations, this representation requires $2^n - 1$ parameters, which quickly becomes computationally and cognitively intractable.

On the other hand, intelligent agents typically build internal models of the world through compact representations of observed data. A key step in this process is decomposing the system into simpler, independent components. In particular, recognizing \textit{statistical independence} among variables enables significant simplification. The formalism of \ac{BN}~\citep{koller2009probabilistic} provides the tools to express and exploit such independence relations.

\begin{remark}
The chain rule always allows us to factorize a joint distribution regardless of the dependencies among variables:
$$
\Prob(\X) = \Prob(\Xcomp[1]) \Prob(\Xcomp[2] \mid \Xcomp[1]) \cdots \Prob(\Xcomp[n] \mid \Xcomp[1], \dots, \Xcomp[n-1]) = \prod_{i=1}^n \Prob(\Xcomp[i] \mid X^{j<i}).
$$

Now, assume the variables are indexed such that each $X^i$ depends only on its immediate predecessor: the joint distribution satisfies the \emph{Markov condition}
$$
\Prob(\Xcomp[i] \mid X^{j<i}) = \Prob(\Xcomp[i] \mid X^{i-1}),
$$
and the factorization simplifies to
$$
\Prob(\X) = \Prob(\Xcomp[1]) \prod_{i=1}^{n-1} \Prob(X^{i+1} \mid X^i).
$$

This reduction significantly decreases the number of parameters that need to be estimated, and only values $2n 1$ are required. More generally, the independent relations among the variables can be structured according to a partial order, which, even in such cases, entails reducing the number of parameters. The existence and identification of such independence structures act as a mechanism of \emph{compression in representation}, achieved here through the right choice of factorization for $\Prob$.
\end{remark}

\subsection{Finite-data problem}
One important problem concerns the limited data availability. Although many inductive methods are \emph{asymptotically consistent}, i.e., %better and better 
approximating an ideal %operating 
condition as the number of available data increases, consistency does not always induce the expected practical advantage, and some problems are often %simply 
unavoidable due to the finite data regime.
	
Consider the problem of testing for conditional independence. Suppose we are studying a system involving many variables, and we try to determine whether $X$ and $Y$ are conditionally independent given a set of variables $\bm{S}$. Formally, this means verifying whether
$$
\Prob(X, Y \mid \bm{S}) = \Prob(X \mid \bm{S}) \Prob(Y \mid \bm{S}).
$$
We denote this condition by $X \ind_{\Prob} Y \mid \bm{S}$. The statistical procedure adopted to perform the test will typically rely on a subset of the available data. Specifically, it must verify that the independence condition holds \textit{for each possible realization} $\bm{S} = \bm{s}$.
Reasoning along these lines, the possible realizations will be, considering the simplified case of binary variables, $2^{|\mathbf{S}|}$. Therefore, each test will only have the $2^{|\mathbf{S}|}$-th part of the entire dataset. Consequently, the efficiency of statistical independence tests can only be preserved by accessing an exponentially increasing amount of data on the system, %quickly 
causing the procedure to be practically unfeasible. The reasoning is directly extended to non-binary discrete variables. In general, the relation between the problem's dimensionality and the amount of data required represents a discipline in itself, occupying a relevant role also in \ac{AI} and \ac{ML}~\citep{mohri2012foundations}.
	
When one considers continuous data, the problem often gets worse. First of all, many of the most popular independence tests, such as \textit{G-test, Fisher-test, and Entropy-based tests}, involve a binning procedure, i.e., a block discretization of the continuous distribution, thus ending up with the problems analyzed before. As an additional problem, the choice of bin size can significantly affect the results, introducing an extra layer of uncertainty and arbitrariness. Finally, working in the finite data regime introduces algorithm-specific limitations — some of which are theoretical in nature. For example, under these conditions, key causal learning algorithms suffer from known weaknesses: \pc{} becomes order-dependent~\citep{colombo2014order}, while \ges{} no longer guarantees finding the global optimum concerning its score \citep{chickering2002optimal}.
	
\section{Interventional Level}\label{sec:interventional_level}

The interventional level necessitates the introduction of a \textit{causal graph}, which generalizes the factorization example discussed above by providing a principled framework for encoding dependencies among variables in a compact and interpretable manner. The canonical factorization associated with the graph accommodates the modeling of external interventions, allowing them to be represented as modifications to one or more underlying factors.

\begin{definition}[Graph]
A \emph{graph} is a pair $\graph = (\GraphVertices, \GraphEdges)$, where:
\begin{itemize}
    \item \( \GraphVertices \) is a finite set of \emph{nodes} (or \emph{vertices}) $\GraphVertices=\{\GraphVertices^i\}_{i=1}^n$.
    \item \( E \subseteq V \times V \) is a set of \emph{directed edges}, where each edge $\edge{i,j} \in \GraphEdges$ represents a directed link from node $\GraphVertices^i$ to node $\GraphVertices^j$.
\end{itemize}
\end{definition}

\begin{definition}[Path]
Given a graph $\graph = (\GraphVertices, \GraphEdges)$, a \emph{path} is a finite sequence of nodes
$$
\GraphVertices^{i_1},\ \GraphVertices^{i_2},\ \dots,\ \GraphVertices^{i_k}
$$
such that for every consecutive pair $(\GraphVertices^{i_j}, \GraphVertices^{i_{j+1}})$ with $1 \leq j < k$, either 
$$
\edge{i_j,i_{j+1}} \in \GraphEdges \quad \text{or} \quad \edge{i_{j+1},i_j} \in \GraphEdges.
$$

A path is said to be \emph{directed} if all edges follow the same direction:
$$
\edge{i_j,i_{j+1}} \in \GraphEdges \quad \text{for all } 1 \leq j < k.
$$

A \emph{cycle} is a path where $i_1 = i_k$ and all intermediate nodes $\GraphVertices^{i_2}, \dots, \GraphVertices^{i_{k-1}}$ are distinct.
\end{definition}

\begin{definition}[\ac{DAG}]
A graph is a \ac{DAG} if it contains no directed cycles.
\end{definition}

In the context of probabilistic and causal modeling, each node corresponds to a random variable in the system, and each directed edge represents a direct causal influence between two variables. In this regard, some basic terminology can be helpful. Given a node \( \Xcomp[k] \in \graph \)\footnote{From this point onward, with a slight abuse of notation, we will use the same symbols to denote both the nodes in the graph and their corresponding random variables. This choice is intended to make the one-to-one correspondence between them more direct and easier to follow.}, we define:

\begin{itemize}
    \item The set of nodes with directed edges pointing into $ \Xcomp[i] $ is called its \textbf{parents}, denoted by $ \parents[i]_\graph $.
    
    \item The set of nodes that receive directed edges from $ \Xcomp[i] $ is called its \textbf{children}, denoted by $ \children[i]_\graph $.
    
    \item The set of nodes from which there exists a directed path into $ \Xcomp[i] $ is called its \textbf{ancestors}, denoted by $ \ancestors[i]_\graph $.
    
    \item The set of nodes that are reachable by a directed path starting from $ \Xcomp[i] $ is called its \textbf{descendants}, denoted by $ \descendants[i]_\graph $.
\end{itemize}

In what follows, we may omit the superscript \( \graph \) when the underlying \ac{DAG} is clear from context. 

At this point, we can characterize how much a graph \( \graph \) conforms to the independence properties inherent in the observational distribution. In this regard, the \emph{Markov property} is a foundational assumption in graphical modeling: a distribution $ \Prob $ is said to be \emph{Markov for} a \ac{DAG} $ \graph $ if the conditional independencies in $ \Prob $ are reflected in the structure of $ \graph $ and can be exploited for efficient computation and representation.

\begin{definition}
    The \ac{DAG} $ \graph $ is \textbf{$ \Prob $-compatible} if the following \textit{Local Markov Condition} holds for all nodes $ \Xcomp[i] \in \graph $:
\begin{equation}\label{eq:local_markov}
    \prob(\Xcomp[i] \mid \neg \descendants[i]_\graph, \parents[i]_\graph) 
    = \prob(\Xcomp[i] \mid \parents[i]_\graph),
\end{equation}
that is, each variable is independent of its non-descendants given its parents.
\end{definition} 
Under this condition, the joint distribution $ \prob(\mathbf{X}) $ admits the following \textit{canonical factorization}:
\begin{equation}\label{eq:canonical_factorization}
    \prob(\X) = \prod_{i=1}^{n} \prob(\Xcomp[i] \mid \parents[i]_\graph).
\end{equation}

This decomposition highlights the efficiency gained by modeling dependencies via a \ac{DAG}: rather than specifying the full joint distribution directly, one only needs to describe local conditional distributions over each variable given its parents.

\begin{remark}
Beyond its computational efficiency, the canonical factorization in~\cref{eq:canonical_factorization} carries significant cognitive and epistemological implications. Graphical models offer a way of representing dependencies that mirrors how humans organize knowledge: through modular, local interactions among components. Such representations align with cognitive constraints on memory and reasoning, allowing agents to store, interpret, and manipulate information efficiently \citep{pearl2011bayesian}.
\end{remark} 

\begin{remark}
   From a statistical and computational perspective, this modular structure enables scalable inference and learning by exploiting independence relations~\citep{koller2009probabilistic}. The factorized form reduces the memory requirements for storing a distribution and the computational complexity of querying it. 
\end{remark}

\begin{remark}
Although multiple \acp{DAG} can encode the same set of conditional independencies, one may prefer the graph that provides the most compact and structured account of the observed relations. In this spirit, the causal graph $ \graph $ may be interpreted, following an \textit{Occam's Razor-type principle}~\citep{walsh1979occam}, as the minimal description consistent with the observed data. This viewpoint naturally connects with algorithmic information theory, where simplicity and explanatory power are tied to measures such as Kolmogorov complexity \citep{Janzing2008CausalIU, PetersJanzingSchoelkopf17}.
\end{remark}

The $\Prob$-compatibility condition, as captured by \cref{eq:local_markov} and \cref{eq:canonical_factorization}, can also be characterized in terms of \emph{d-separation} relations. 

\begin{definition}[\textit{d}-separation]
    Given three sets of nodes $\mathbf{A}$, $\mathbf{B}$, and $\mathbf{C}$ in $\GraphVertices$, we say that $\mathbf{A}$ and $\mathbf{B}$ are \emph{d-separated by} $\mathbf{C}$ — denoted $ \mathbf{A} \ind_{\graph} \mathbf{B} \mid \mathbf{C} $ — if, for every $ A \in \mathbf{A} $, every $ B \in \mathbf{B} $, and every path $\gamma$ in $\graph$ connecting $A$ to $B$, the following conditions hold: at least one of (1) or (2) is satisfied, and (3) is satisfied\footnote{Throughout the following, we use the arrow shorthand $ \rightarrow C \rightarrow $, $ \leftarrow C \rightarrow $, and $ \rightarrow W \leftarrow $ to refer to local configurations along a path $\gamma$, without requiring that the variables involved be adjacent. These motifs describe the directions of edges at a node along $\gamma$, not the global structure of the path.}:

\begin{enumerate}
    \item There is a \emph{chain} that is blocked by $\mathbf{C}$; that is, a subpath of the form $ \rightarrow C \rightarrow $ or $ \leftarrow C \leftarrow $ appears in $\gamma$, with $C \in \mathbf{C}$;
    
    \item There is a \emph{fork} that is blocked by $\mathbf{C}$; that is, a subpath of the form $ \leftarrow C \rightarrow $ appears in $\gamma$, with $C \in \mathbf{C}$;
    
    \item No \emph{collider} along $\gamma$ is activated by $\mathbf{C}$; that is, whenever a subpath of the form $ \rightarrow W \leftarrow $ appears in $\gamma$, it holds that $(W \cup \descendants[W]_\graph) \cap \mathbf{C} = \emptyset$.
\end{enumerate}
\end{definition}

Using this notion, it can be shown that $ \graph $ is $ \Prob $-compatible if and only if the following \textbf{Global Markov Condition} holds:
\begin{equation}
    \mathbf{A} \ind_{\graph} \mathbf{B} \mid \mathbf{C} \quad \Rightarrow \quad \mathbf{A} \ind_{\Prob} \mathbf{B} \mid \mathbf{C}.
\end{equation}

In practice, we are often interested in designing an algorithm that, given $\Prob(\X)$ or a dataset sampled from $\Prob$, returns a graph $\graph$ that encodes \emph{exactly} the set of independence relations implied by $\Prob$. This motivates the introduction of the Faithfulness Assumption.

\begin{assumption}[Faithfulness]
    Given a \ac{DAG} $\graph$, a probability distribution $\Prob$ is said to be \graph-faithful if
    \begin{equation}
        \mathbf{A} \ind_{\graph} \mathbf{B} \mid \mathbf{C} \quad \Leftrightarrow \quad \mathbf{A} \ind_{\Prob} \mathbf{B} \mid \mathbf{C}.
    \end{equation}
\end{assumption}

\begin{remark}
    This assumption has a precise meaning not only in the relationship between $\graph$ and $\Prob$, but also in the internal structure of the causal system. To illustrate this point, we introduce a classical counterexample. Consider the \ac{DAG} $\graph$ depicted in Figure~\ref{fig:faithfulbroken}, defined by the following linear equations:
    \begin{equation}
        \begin{aligned}
            \Xcomp[1] &:= \varepsilon_1, \\
            \Xcomp[2] &:= \alpha \Xcomp[1] + \varepsilon_2, \\
            \Xcomp[3] &:= \Xcomp[1] + \beta \Xcomp[2] + \varepsilon_3,
        \end{aligned}
    \end{equation}
    where $\varepsilon_1, \varepsilon_2, \varepsilon_3$ are independent noise terms.

\begin{figure}[h]
    \centering
    \begin{tikzpicture}[
        scale=0.9,
        transform shape,
        vertex/.style={circle, draw, minimum size=0.9cm, inner sep=0pt}]
        
        % Nodes
        \node[vertex] (x1) at (0,3) {$\Xcomp[1]$};
        \node[vertex] (x2) at (-2,0) {$\Xcomp[2]$};
        \node[vertex] (x3) at (2,0) {$\Xcomp[3]$};
        
        % Edges with labels
        \draw[edge] (x1) -- (x2) node[scale=1.2,midway,sloped,above] {$\alpha$};
        \draw[edge] (x1) -- (x3) node[scale=1.2,midway,sloped,above] {$1$};
        \draw[edge] (x2) -- (x3) node[scale=1.2,midway,sloped,above] {$\beta$};
        
    \end{tikzpicture}
    \caption{\textbf{Violation of Faithfulness}. A simple \ac{DAG}. When $\alpha \beta = -1$, the causal effects from $\Xcomp[1]$ to $\Xcomp[3]$ cancel out, leading to a violation of Faithfulness.}
    \label{fig:faithfulbroken}
\end{figure}
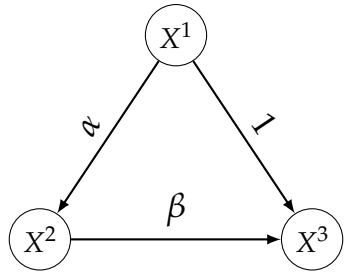

    If we set $\alpha \beta = -1$, the two causal paths from $\Xcomp[1]$ to $\Xcomp[3]$ cancel each other out, resulting in $\Xcomp[1] \ind_{\Prob} \Xcomp[3]$ for the distribution $\Prob$ generated by $\graph$, even though $\Xcomp[1] \not\ind_{\graph} \Xcomp[3]$. In fact, a simpler \ac{DAG} structure, namely $\Xcomp[1] \rightarrow \Xcomp[2] \leftarrow \Xcomp[3]$, would be compatible with $\Prob$.

    The Causal Faithfulness assumption rules out such cancellations of causal effects due to unlikely parameter coincidences, which would otherwise hide causal relationships from independence tests. In this sense, it can be seen as analogous to Occam’s razor, favoring simpler explanations. Although it is often argued that non-faithful configurations form a set of measure zero in the parameter space and are thus negligible in a probabilistic sense, recent results indicate that, in finite-sample settings and when allowing for a margin of error in statistical tests, violations of Faithfulness can be much more problematic than traditionally assumed~\citep{uhler2013geometry}.
\end{remark}

It is worth stressing that a $ \Prob $-compatible \ac{DAG} $ \graph $ is simply a compact representation of the set of conditional independence relations in $ \Prob $, and does \emph{not necessarily} reflect the true causal structure of the system. This becomes apparent when recognizing that the set of $ \Prob $-compatible graphs is not in general a singleton, as will be discussed below.

\subsection{Markov Equivalence Class}\label{se:MEC}

Given $n$ random variables $\Xcomp[1], \dots, \Xcomp[n]$, the number of possible \acp{DAG} increases super-exponentially with $n$. However, the conditional independence relations encoded by a distribution $\Prob$ impose constraints on the space of admissible graphs. 

\begin{definition}
    The subset of \acp{DAG} of $n$ nodes that are $\Prob$-compatible is called the \ac{MEC} of $\Prob(X^1,...,X^n)$, and is denoted by $\text{\ac{MEC}}_\Prob$.
\end{definition}

The class \ac{MEC} can be fully characterized by two structural features of a \ac{DAG}: its \emph{skeleton} and its set of \emph{v-structures}.

\begin{definition}[Skeleton]
Given a \ac{DAG} $\graph = (\GraphVertices, \GraphEdges)$, its \emph{skeleton} is the undirected graph $\graph'$ obtained by replacing all directed edges $\edge{i,j} \in \GraphEdges$ with undirected ones. Formally,
$$
\graph' = (\GraphVertices, \{ \{\GraphVertices^i,\GraphVertices^j\} \mid \edge{i,j} \in \GraphEdges \ \text{or} \ \edge{j,i} \in \GraphEdges \}).
$$
\end{definition}

\begin{definition}[V-structure]
Let $\graph = (\GraphVertices, \GraphEdges)$ be a DAG. A triple of distinct nodes $(\GraphVertices^i, \GraphVertices^j, \GraphVertices^k)$ forms a \emph{v-structure} if the following two conditions hold:
\begin{itemize}
    \item Both edges $(\GraphVertices^i, \GraphVertices^j) \in \GraphEdges$ and $(\GraphVertices^k, \GraphVertices^j) \in \GraphEdges$;
    \item There is no edge between $\GraphVertices^i$ and $\GraphVertices^k$ in the skeleton $\graph'$, i.e., $\{ \GraphVertices^i, \GraphVertices^k \} \notin \graph'$.
\end{itemize}

In shorthand, we write this configuration as:
$$
\GraphVertices^i \rightarrow \GraphVertices^j \leftarrow \GraphVertices^k,
$$
where $\GraphVertices^j$ is called a \emph{collider}. The arrow notation is used informally to indicate the presence of the corresponding ordered pairs in $\GraphEdges$.
\end{definition}

With these definitions in place, the following result characterizes when two \acp{DAG} belong to the same \ac{MEC}.

\begin{theorem}[~\citet{v-structuretheorem}]
Two \acp{DAG} are in the same \ac{MEC} if and only if they share the same skeleton and set of v-structures.
\end{theorem}

The class $\text{\ac{MEC}}_\Prob$ is graphically represented through a \ac{PDAG} $\graph_\Prob$, obtainable via three simple steps: 
\begin{itemize}
    \item We start with only nodes without any edges.
    \item The skeleton of $\graph_P$ is the same of any \ac{DAG} in the \ac{MEC}. We add all the required undirected edges to the graph.
    \item If an edge of $\graph_P$ has the same direction for all the \ac{DAG}s in the $\text{\ac{MEC}}_\Prob$, we direct the edge in that direction. 
\end{itemize}
The new graph $\graph_P$ is referred to as the \ac{CPDAG} of $\Prob$. The Markov equivalence class, $\text{\ac{MEC}}_\Prob$, can be understood as the set of all alternative causal structures capable of generating the probability distribution $\Prob$. These structures are collectively represented by $\graph_P$, where the directed edges correspond to the shared cause-effect relations. 

\subsection{Causal Sufficiency}

The Markov Condition ensures that the independence relations represented in the graph must also be reflected in the associated probability distribution. In particular, for any \ac{DAG}, the root nodes must be mutually independent. Indeed, it is easy to verify that if $X$ and $Y$ are both roots, applying \cref{eq:local_markov} yields
\[
\prob(X,Y) = \prob(X,Y \mid \emptyset) = \prob(X)\prob(Y).
\]
This property can be verified across variables in the graph by recursively applying \cref{eq:local_markov}, considering the condition on the variables' parents. From a modeling perspective, this reflects an important assumption:

\begin{assumption}[Causal Sufficiency]\label{ass:causal_sufficiency}
    A set of variables $\X$ is said to be causally sufficient if, for every variable $C$ that acts as a direct common cause of two distinct variables $X_i, X_j \in \X$, it holds that $C \in \X$. Formally, for all $X_i \neq X_j \in \X$ and any variable $C$,  
    \[
    (C \to X_i) \wedge (C \to X_j) \quad \Rightarrow \quad C \in \X.
    \]
\end{assumption}

Causal Sufficiency requires that the available variables capture all the relevant factors needed to model the system, and that no latent (i.e., unmeasured or unmodeled) variables influence the observed behavior. If this assumption does not hold, we may only observe a subset $\mathbf{O}$ of a larger set of variables $\mathbf{V}$ governing the system, where $\mathbf{V}$ also contains latent variables $\bm{L}$.

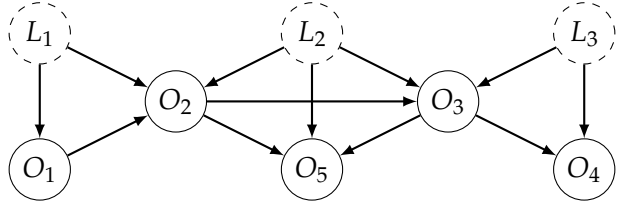
\begin{figure}[t]
    \centering
    \begin{tikzpicture}[
        scale=0.9,
        transform shape,
        vertex/.style={circle, draw, minimum size=0.9cm, inner sep=0pt},
        latent/.style={circle, draw, dashed, minimum size=0.9cm, inner sep=0pt}]
        
        % Observed variables
        \node[vertex] (o1) at (-4,-1) {$O_1$};
        \node[vertex] (o2) at (-2,0) {$O_2$};
        \node[vertex] (o3) at (2,0) {$O_3$};
        \node[vertex] (o4) at (4,-1) {$O_4$};
        \node[vertex] (o5) at (0,-1) {$O_5$};
        
        % Latent variables
        \node[latent] (l1) at (-4,1) {$L_1$};
        \node[latent] (l2) at (0,1) {$L_2$};
        \node[latent] (l3) at (4,1) {$L_3$};
        
        % Edges from latent to observed
        \draw[edge] (l1) -- (o1);
        \draw[edge] (l1) -- (o2);
        \draw[edge] (l2) -- (o2);
        \draw[edge] (l2) -- (o3);
        \draw[edge] (l2) -- (o5);
        \draw[edge] (l3) -- (o3);
        \draw[edge] (l3) -- (o4);
        
        % Edges between observed variables
        \draw[edge] (o1) -- (o2);
        \draw[edge] (o2) -- (o3);
        \draw[edge] (o3) -- (o4);
        \draw[edge] (o2) -- (o5);
        \draw[edge] (o3) -- (o5);
        
    \end{tikzpicture}
    \caption{\textbf{Partially Observed Causal Graph}. Observed variables $O_i$ are influenced by latent variables $L_j$. Causal Sufficiency does not hold.}
    \label{fig:latent_observed_graph}
\end{figure}

In principle, since $\bm{L}$ are unknown, there exist infinitely many causal models that can explain a given distribution, each introducing different hidden variables and connecting the observed variables through distinct causal pathways. Without any restriction on the class of models considered, making meaningful assertions about the underlying causal structure is unfeasible. 

Following the principle of Occam's razor~\citep{walsh1979occam}, it is reasonable to prefer the simplest models consistent with the observed data. Theories that survive this form of selection are called \emph{minimal}~\citep{pearl2009causality}.

Nonetheless, Occam’s razor remains a heuristic guideline rather than guaranteeing that the selected model reflects the underlying mechanism. Its application, while reasonable, represents an assumption that is arbitrary in principle. Overreliance on simplicity can, in some cases, lead not just to simpler models but to \emph{oversimplified} ones that fail to capture essential features of the system.

Assumption~\ref{ass:causal_sufficiency} effectively enforces $\mathbf{V} = \mathbf{O}$, that is, it posits the ``simplest imaginable'' situation by excluding latent variables altogether. However, this assumption often comes at a high cost, as it can produce unrealistic models, undermining the rationale for applying Occam’s razor in the first place.

\subsection{Interventions}

In the following, we define the notion of a Causal Intervention.

\begin{definition}[Causal Intervention]\label{def:interventional_intervention_math_framework}
    Consider a system described by a causal graph $\graph$ and a corresponding probability distribution $\prob$, represented through the canonical factorization~\cref{eq:canonical_factorization}. 
    An intervention $\intervention$ on the variable $\Xcomp[i]$ is defined as the replacement of the conditional probability term $\prob(\Xcomp[i] \mid \parents[i])$ with a modified term $\hat{\prob}(\Xcomp[i] \mid \hat{\parents}^i)$, where $\hat{\parents}^i \subseteq \parents[i]$. 
    In other words, a controlled shift in the conditional distribution of $\Xcomp[i]$ is performed, altering its dependency structure.

    As a result, a new canonical factorization is obtained by replacing the corresponding factor in the original one:
    \begin{equation}\label{eq:interventional_canonical_factorization}
        \prob(\X) = \prod_{j=1}^{n} \prob(\Xcomp[j] \mid \parents[j]) 
        \overset{\intervention}{\longmapsto} 
        \prob^\intervention(\X) := \hat{\prob}(\Xcomp[i] \mid \hat{\parents}^i) \cdot \prod_{j \neq i} \prob(\Xcomp[j] \mid \parents[j]).
    \end{equation}
\end{definition}

\begin{definition}[Do-Intervention]
An intervention is called \emph{hard} if the value of $\Xcomp[i]$ is set to a specific constant $\hat{x}^i$, that is,
\[
\hat{\prob}(\Xcomp[i] \mid \hat{\parents}^i) = \delta_{\hat{X}^i}(\Xcomp[i]),
\]
where $\delta_{\hat{x}^i}$ denotes the Dirac delta distribution concentrated at $\hat{x}^i$.

Compared to the general definition of intervention, a hard intervention sets $\hat{\parents}^i = \emptyset$, eliminating any dependence of $\Xcomp[i]$ on its original parents in the causal model. As a consequence, the causal graph is modified: $\graph$ is transformed into a new graph $\graph^\intervention$ by removing all arrows pointing into $\Xcomp[i]$, as depicted in \cref{fig:do_intervention_graph}.
This type of intervention is commonly denoted by $do(\Xcomp[i] = \hat{x}^i)$ in the literature and is often referred to as a \emph{do-intervention}.
\begin{figure}[t]
    \centering
    \begin{tikzpicture}[
        scale=0.9,
        transform shape,
        vertex/.style={circle, draw, minimum size=0.9cm, inner sep=0pt},
        thickvertex/.style={circle, draw, very thick, minimum size=0.9cm, inner sep=0pt}]
        
        % Original graph
        \node[vertex] (x1) at (0,2) {$X_1$};
        \node[vertex] (x2) at (2,2) {$X_2$};
        \node[vertex] (x3) at (1,0) {$X_3$};
        
        \draw[edge] (x1) -- (x2);
        \draw[edge] (x1) -- (x3);
        \draw[edge] (x2) -- (x3);
        
        \node at (1,-1.2) {Original graph $\graph$};

        % Intervened graph
        \node[vertex] (x1i) at (5,2) {$X_1$};
        \node[thickvertex] (x2i) at (7,2) {$X_2$};
        \node[vertex] (x3i) at (6,0) {$X_3$};

        % Only some arrows now
        \draw[edge] (x1i) -- (x3i);
        \draw[edge] (x2i) -- (x3i);
        
        % Label intervention over X2i
        \node at (6.5,2.5) {$\intervention$};
        
        \node at (6,-1.2) {Intervened graph $\graph^\intervention$};
        
        % Indicate the do operation
        \node at (3.5,1) {$\xrightarrow{do(X_2 = x_2)}$};
        
    \end{tikzpicture}
    \caption{\textbf{Effect of a Do-Intervention on the Causal Graph}. The intervention $do(X_2 = x_2)$ removes all incoming arrows to $X_2$.}
    \label{fig:do_intervention_graph}
\end{figure}
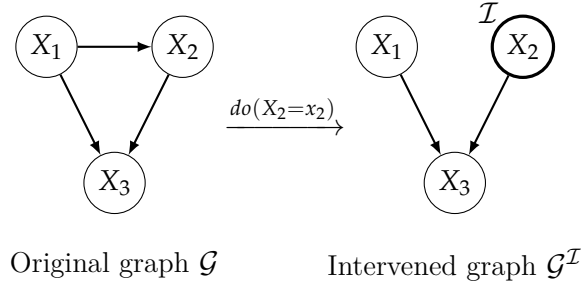

\end{definition}

\begin{remark}[Modularity for Generalization]
When environmental conditions change, a machine trained in one environment cannot be expected to perform well unless the localized changes can be identified. For instance, in computer vision, distribution shifts may arise due to aberrations, geometric transformations, or compression artifacts. 

Problems of this nature are typically addressed through data augmentation, pre-training, or other strategies that introduce data biases. However, these solutions may not be sufficient, and some structural knowledge of the system may be required. Researchers in \ac{AI} have recognized this challenge, motivating the development of subtasks such as \emph{domain adaptation}, \emph{transfer learning}, and \emph{life-long learning} to improve robustness.

In this context, \emph{modularity} plays a central role by enabling certain forms of generalization~\citep{pearl2009causality,schol2}. Specifically, it is assumed that, starting from the canonical factorization~\cref{eq:canonical_factorization}, a single factor $\prob(\Xcomp[i] \mid \parents[i])$ can be replaced by a modified factor $\hat{\prob}(\Xcomp[i] \mid \hat{\parents}^i)$ without altering the other terms, yielding~\cref{eq:interventional_canonical_factorization}.

This process formalizes the functional and semantic independence of the mechanisms characterizing the system, allowing the identification of the specific mechanism that has changed, whether due to external perturbations or the application of the model to a different, though related, problem.

Consequently, the generalization achieved is \emph{not from one dataset to another sampled from the same distribution}, but rather \emph{from one problem to another}. This form of \emph{out-of-distribution generalization} is a key capability of \emph{intelligence}, which relies on interventions, domain shifts, and temporal structure to build modular representations. In this regard, causal models may represent an important step toward achieving more transparent and robust \ac{AI} systems~\citep{pearl_seven_tools,pearlMackenzie18}.
\end{remark}

\begin{definition}[Causal Effect]
    The \emph{causal effect} $\ce$ of an intervention $\intervention$ is defined as the difference in the expectation of the variables $\X$ under the probability distribution before and after the intervention. Formally, %we write
    \begin{equation}\label{eq:causal_effect_in_framework}
        \ce^\intervention := \expected{\X}_{\prob^\intervention} - \expected{\X}_{\prob} 
        = \int_{\Xspace} \X \prob^\intervention(\X) \, \mathrm{d}\X - \int_{\Xspace} \X \prob(\X) \, \mathrm{d}\X.
    \end{equation}
    The causal effect $\ce$ can also be characterized conditionally on some system variables. Conditioning on the event $\bm{C} = \bm{c}$, the object of interest becomes $\expected{\X \mid \bm{C} = \bm{c}}$, while the structure of the definition remains unchanged.
\end{definition}

\section{Counterfactual Level}\label{sec:counterfactual_level}

So far, the following perspective has been adopted: a causal model consists of a graph $\graph$ that specifies which variables $\parents[i]$ directly determine the value of each variable $\Xcomp[i]$, inducing a partial order on $\Xcomp[1], \Xcomp[2],\dots, \Xcomp[m]$. Moreover,~\cref{eq:interventional_canonical_factorization} allows the computation of \emph{interventional} distributions.

However, the underlying mechanisms of the system have remained implicit, represented through the conditional probability distributions $\prob(\Xcomp[i] \mid \parents[i])$. This approach suffices when we are interested in answering causal queries that can be expressed in probabilistic terms, placing us at the second level of Table~\ref{tab_causality}. 

In this sense, a causal graph can be viewed as an intermediate description between fully mechanistic models, such as differential equations in physics, and purely statistical models based on the i.i.d. assumption. 

A further step toward a more explicit modeling of causal mechanisms is provided by the language of \acp{SCM}~\citep{pearl2009causality, scholk, schol2, schol3}. In place of the probabilistic factor $\prob(\Xcomp[i] \mid \parents[i])$, we consider a relationships as
\begin{equation}\label{eq:scm_equations}
\Xcomp[i] := f_i(\parents[i], \Noise^i),
\end{equation}
which is interpreted as an assignment from right to left, rather than a symmetric equation. 
$\Noise^i$ denotes the exogenous factors influencing $\Xcomp[i]$, modeled as random noise due to our lack of control over them.

This finer-grained representation enables us to formalize interventional queries more richly, including cases that cannot be expressed purely in probabilistic terms~\citep{Halpern16}, as well as to address counterfactual queries~\citep{pearl2009causality,Halpern16}.

First, we distinguish between the \emph{endogenous} variables $\X$ of the system and the \emph{exogenous} variables $\Noise$, which represent the noise terms. If Assumption~\ref{ass:causal_sufficiency} holds, each exogenous variable $\Noise^k$ influences one and only one observable variable $\Xcomp[k]$, ensuring the following independence properties: $\parents(\Xcomp[k]) \ind \Noise_k$ and $\Noise_i \ind \Noise_j$ for all $i \neq j$, with $i,j,k \in \{1, \dots, n\}$. 

The graph $\graph$ can then be extended by including the exogenous nodes $\Noise^k$, with directed edges $\Noise^k \rightarrow \Xcomp[k]$ for each $k$. This extended graph explicitly encodes that each variable $\Xcomp[k]$ is generated by its parents $\parents[k]$ together with its exogenous noise $\Noise^k$, through the underlying mechanism given by~\cref{eq:scm_equations}.

Given the distribution 
$$
\prob(\vectorNoise) = \prod_{i=1}^n \prob(\Noise^i),
$$
the functional assignments $f_k$ induce a joint distribution $\prob(\X)$, which inherits structural properties from the topology of $\graph$. Intuitively, independent noise terms can be viewed as sources of randomness that propagate information through the graph $\graph$. Then, given~\cref{eq:canonical_factorization}, each factor $\prob(\Xcomp[k] \mid \parents[k])$ is computed from the mechanism $f_k$ combined with the marginal distribution $\prob(\Noise^k)$.

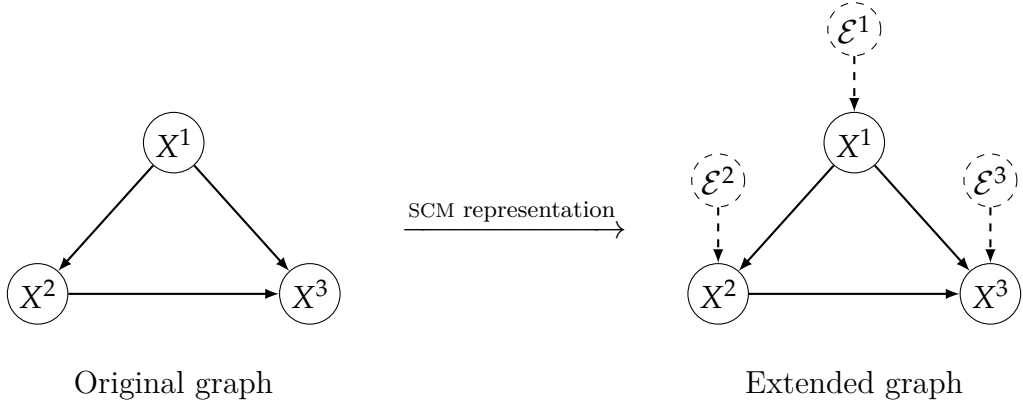
\begin{figure}[t]
    \centering
    \begin{tikzpicture}[
        scale=1,
        transform shape,
        vertex/.style={circle, draw, minimum size=0.8cm, inner sep=0pt},
        noise/.style={circle, draw, dashed, minimum size=0.7cm, inner sep=0pt}]
        
        % Original Graph
        \node[vertex] (x1) at (-2,2) {$\Xcomp[1]$};
        \node[vertex] (x2) at (-3.8,0) {$\Xcomp[2]$};
        \node[vertex] (x3) at (-0.2,0) {$\Xcomp[3]$};
        
        \draw[edge] (x1) -- (x2);
        \draw[edge] (x1) -- (x3);
        \draw[edge] (x2) -- (x3);
        
        \node at (-2,-1.2) {Original graph};
        
        % Extended Graph
        \node[vertex] (x1i) at (7,2) {$\Xcomp[1]$};
        \node[vertex] (x2i) at (5.2,0) {$\Xcomp[2]$};
        \node[vertex] (x3i) at (8.8,0) {$\Xcomp[3]$};
        
        % Noise nodes
        \node[noise] (n1) at (7,3.5) {$\Noise^1$};
        \node[noise] (n2) at (5.2,1.5) {$\Noise^2$};
        \node[noise] (n3) at (8.8,1.5) {$\Noise^3$};
        
        % Original edges
        \draw[edge] (x1i) -- (x2i);
        \draw[edge] (x1i) -- (x3i);
        \draw[edge] (x2i) -- (x3i);
        
        % Noise edges
        \draw[edge, dashed] (n1) -- (x1i);
        \draw[edge, dashed] (n2) -- (x2i);
        \draw[edge, dashed] (n3) -- (x3i);
        
        \node at (7,-1.2) {Extended graph};
        
        % Arrow between graphs
        \node at (2.5,1) {$\xrightarrow{\text{\ac{SCM} representation}}$};
        
    \end{tikzpicture}
    \caption{\textbf{Extension of a Causal Graph}. The original graph $\graph$ is extended by explicitly including exogenous noise variables $\Noise^i$, each influencing its corresponding endogenous variable $\Xcomp[i]$.}
    \label{fig:extended_graph}
\end{figure}

The advantage of this formulation lies in the fact that the object of interest is no longer the probability distribution $\prob(\X)$, but rather the random variable $\X$ itself. Specifically, interpreting the noises $\vectorNoise$ as random variables, we say that $\X$ is a solution of the \ac{SCM} $\scm[{\rmF, \vectorNoise}]$, where $\rmF$ is the set of structural equations $\rmF := \{\Xcompalt[i] := f_i(\parents[i], \vectorNoise^i)\}_{i=1}^{\sizex}$, if $\X$ satisfies all equations in $\rmF$. 

This perspective motivates a new level of causal query: assuming that a realization $\x$ of $\X$ has already occurred through a specific realization of the noises $\vectornoise$ and the application of the assignments in $\rmF$, in principle we can infer the values of $\vectornoise$ and answer questions of the form: \quotes{What would have happened if, prior to generating this realization, we had applied a certain intervention?}. 
In other words, we can formulate counterfactual queries. 
It is important to note a subtle point: the uncertainty of the model is initially captured by the probability distribution over $\vectorNoise$, but once a realization has been produced, the entire system of mechanisms becomes deterministic. \textit{There is no longer any uncertainty about the outcome of the process}. 

\subsection{Interventions and Counterfactuals}

As the final and most expressive level, \acp{SCM} are capable of answering all queries addressed at the previous levels. They can describe the observational distribution $\prob(\X)$, answer \emph{interventional queries} about the effects of external manipulations, and go further to enable \emph{counterfactual queries} to assess what would have happened to a particular observation if one observed variable $\Xcompalt[i]$ had taken a different value.

To this end, we reformulate the notion of intervention introduced in \cref{def:interventional_intervention_math_framework} within the language of \acp{SCM}.

\begin{definition}[Intervention - \acp{SCM}]\label{def:intervention_scm_math_framework}
An intervention $\intervention$ on an \ac{SCM} $\scm$ yields a new \ac{SCM} $\scm^\intervention$ in which one or more mechanisms $f_i(\parents[i], \vectorNoise^i)$ are replaced by modified mechanisms
\[
\hat{f}_i(\hat{\parents}^{i}, \hat{\vectorNoise}^{i}),
\]
where $\hat{\parents}^{i} \subseteq \parents[i]$ and $\hat{\vectorNoise}^{i} \subseteq \vectorNoise^i$.

We refer to a \textit{hard intervention} when $f_i$ is replaced by a constant value $\alpha^{i}$, and $\hat{\parents}^{i} = \hat{\vectorNoise}^{i} = \varnothing$. This type of intervention is denoted by the do-operator $do(\Xcompalt[i] = \alpha^{i})$. Conversely, we refer to a \textit{soft intervention} when at least one argument of $f_i$ is retained.
\end{definition}

\begin{definition}[Causal Effect - \acp{SCM}]\label{def:causal_effect_scm_math_framework}
The causal effect $\ce$ of an intervention $\intervention$ is evaluated as the difference between the values of the observable variables before and after the intervention, that is,
\begin{equation}\label{eq:causal_effect_scm_math_framework}
\ce^\intervention := \expected{\X^\intervention - \X},
\end{equation}
where $\X$ and $\X^\intervention$ are the solutions of $\scm$ and $\scm^\intervention$, respectively.
\end{definition}

\begin{definition}[Counterfactual Realization and Causal Effect]
    Given a realization $(\vectornoise, \x)$ produced by an \ac{SCM} $\scm$ through its structural equations $\rmF$, the \emph{counterfactual realization} $\x^\intervention$ is defined as the outcome of the intervened \ac{SCM} $\scm^\intervention$, evaluated under the same noise realization $\vectornoise$. 

    Moreover, the \emph{counterfactual causal effect} of $\intervention$ is the deterministic quantity
    $$
    \x^\intervention - \x.
    $$
\end{definition}

\begin{remark}
The average causal effect $\ce$ and the counterfactual causal effect are two fundamentally distinct quantities, and they can yield very different values. Consider, for example, the following toy model that describes the effect of a drug on a patient. The system consists of only two variables. The first is a binary variable that determines whether the treatment was administered. The second represents the patient's health status, which is also binary (good or poor health). The model can be described by the following structural equations:
\begin{equation}
    \begin{split}
        T &:= \Noise^T, \quad &\Noise^T \sim \text{Bernoulli}(0.5),\\
        S &:= \Noise^S T + (1-\Noise^S)(1-T), \quad &\Noise^S \sim \text{Bernoulli}(0.9).
    \end{split}
\end{equation}

The final health status of the patient depends on the treatment, and it is highly probable (90\%) that the treatment has a positive effect. However, there is a small probability (10\%) that the treatment has a negative effect, e.g., due to fragile or hidden health conditions. We can compute the expected value of $S$ as a function of $T$, obtaining
$$
\expected{S} = T - 0.1.
$$
Without intervention, we have $\expected{S} = 0.5 - 0.1 = 0.4$. If we intervene to ensure that the drug is always administered, that is, if we apply $do(T=1)$, we obtain $\expected{S}^{do(T=1)} = 1 - 0.1 = 0.9$. Consequently, the average causal effect is $\ce = 0.5$. On average, the drug gives an effective outcome in improving the patient's health. However, the situation changes for a specific patient if it turns out that $\Noise^S = 0$. In this case, evaluating the counterfactual realizations by fixing $\Noise^S=0$, we obtain $\x^{do(T=0)} = 1$ and $\x^{do(T=1)} = 0$. In other words, for this particular patient, the treatment is harmful.
\end{remark}

\begin{remark}
The previous example further clarifies an important point. Counterfactual queries refer to an alternative scenario to the one observed, and can only be evaluated \emph{after} a specific realization has occurred. In this sense, their utility is limited to very specific contexts. Although counterfactuals represent a powerful tool, this condition imposes a limitation that must be considered. If we were to ask whether the doctor made the right decision in administering the drug, the answer would be positive: \textit{based on the information available at the time of the decision}, administering the drug was the best alternative, as the expected causal effect $\ce$ was positive. Moreover, the counterfactual nature of the problem makes evaluating the quality of the model highly problematic. By definition, any estimate of a counterfactual outcome can not be directly verified or falsified, as we have no access to a ground-truth realization of the alternative scenario.
\end{remark}

\subsection{Modularity and Invariance}

The \acp{SCM} formulation relies on \textbf{invariance assumptions}, implicitly encoded in their definition; specifically, through the modeling of the system using a set of separate structural equations collected in $\rmF$. 
These assumptions do not concern the mathematical objects within the \ac{SCM} per se, but rather the expressiveness of the modeling framework itself, i.e., it assumes that the \ac{SCM} is sufficiently expressive to capture the underlying phenomena faithfully.

\begin{assumption}[Strict Functional Mechanisms]
Each mechanism $f_k(\parents[k], \Noise^k)$ depends only on the values of its arguments, not on any higher-order properties such as their mean or variance. In other words, $f_k$ is a function of the specific realizations of $\parents[k]$ and $\Noise^k$.

This assumption does not universally apply and may be invalid in a range of realistic settings. Consider, for example, a public competition in which the top 10 most competent participants are selected. In such a setting, the outcome for each participant depends not only on their attributes but also on the attributes of others. Consequently, this scenario cannot be adequately modeled within the \ac{SCM} framework, and it would lead to paradoxical results: for instance, if every participant's score were increased by 5 points, the relative ranking and thus the final selection would remain unchanged. 
\end{assumption}

\begin{assumption}[Modularity]\label{ass:modularity}
The overall behavior of the system depends only on the aggregate behavior of its individual mechanisms, which are assumed not to interact with each other. In other words, each mechanism $f_k$ represents an independent module that describes a subsystem isolated from the rest.

Modularity often provides a way to represent a system in a compressed form, excluding higher-order interactions that are either unknown or uncontrollable, and should therefore be interpreted in an approximate sense. In this view, the \ac{SCM} should be understood as an \textit{effective model}: sufficiently accurate for the intended purpose, but necessarily simplified relative to the true complexity of the underlying phenomenon.
\end{assumption}

Nonetheless, it is important to note that Assumption~\ref{ass:modularity} may not always hold in practice. For instance, in a biological system in which $\Xcomp[1]$ represents the presence of a virus and $\Xcomp[2]$ denotes the immune response of a patient, one might initially model their relationship as
\[
\Xcomp[2] := f_2(\Xcomp[1], \Noise^2),
\]
where $\Noise^2$ captures individual variation. 
However, the immune system may autonomously adapt over time, modifying its internal response mechanism to the virus without external intervention. This change implies that $f_2$ evolves as a function of past exposures or internal regulatory feedback, violating the assumption that each mechanism operates in isolation. Such endogenous adaptation challenges the modular view and reflects the limitations of treating mechanisms as entirely independent.

\begin{assumption}[Interventional Invariance]
After an intervention, all mechanisms not directly affected by it are assumed to remain unchanged, i.e., \textit{the principle of modularity extends to the interventional level}. This assumption applies to the system itself, not to the \ac{SCM}, where it is just a rule for how interventions are translated into model modifications.

Violations of this assumption can occur in systems that adapt to external manipulations. For instance, consider the biological scenario discussed above and suppose we intervene to eliminate the virus, setting $\Xcomp[1] := 0$ via vaccination. The virus may adapt by developing long-term immunity to the vaccine. This constitutes a change in a mechanism that was not directly targeted by the intervention, violating the principle of interventional invariance.
\end{assumption}

\cleardoublepage
\begin{comment}\ctparttext{You can put some informational part preamble text here.
Illo principalmente su nos. Non message \emph{occidental} angloromanic
da. Debitas effortio simplificate sia se, auxiliar summarios da que,
se avantiate publicationes via. Pan in terra summarios, capital
interlingua se que. Al via multo esser specimen, campo responder que
da. Le usate medical addresses pro, europa origine sanctificate nos se.}
\end{comment}
\part{Towards Realistic Algorithmic Recourse}\label{pt:algorithmic_recourse}
\chapter{Algorithmic Recourse}\label{ch:algo_recourse}
\ac{AR} refers to the class of problems concerned with determining what changes an individual should make to improve their outcome. Depending on the specific context and application domain, this problem may be known by different names. In this chapter, we present two of the most well-known formulations: \textit{Contrastive Explanations} and \textit{Consequential Recommendations}\citep{karimi2020survey}.

Imagine the following situation: Mario, a 38-year-old homeowner, applies for a personal loan to fund an extensive home renovation project. The financial institution, using a binary classification system, evaluates the application and ultimately decides to reject it based on the applicant's details, such as income or credit history. Faced with this outcome, Mario might naturally ask the following questions:  

\begin{enumerate}
    \item[\textbf{Q1}] What factors led to my loan application being rejected?  
    \item[\textbf{Q2}] How can I improve my chances of getting approved for a loan in the future?  
\end{enumerate}  

\paragraph{Contrastive Explanations}Q1 aims at receiving an explanation for the outcome, \ie, an understanding of why the application was declined. One particularly intuitive approach involves the use of \ac{CEx}.  

These explanations provide clarity by offering \emph{counterfactual scenarios} that demonstrate how the outcome could have been different if certain attributes or circumstances were modified. Such examples might include either real profiles of successful applicants or synthetic, hypothetical scenarios where the decision would have resulted in approval.
For instance, Mario might receive feedback such as:  
\begin{itemize}
    \item If your annual income were \$10,000 higher, your loan would have been approved.  
    \item If your credit score were 50 points higher, you would have met the approval threshold.  
\end{itemize}  

\paragraph{Consequential Recommendations}  
Q2 asks for a set of actions to be performed to change the person's status, to achieve a favourable outcome.
Typically, \ac{CRe} are tailored to the specific subject and presented in the form of actionable advice. Examples of \ac{CRe} include:

\begin{itemize}
    \item If you were willing to extend the mortgage term from 10 to 20 years, we could approve your loan.
    \item If you managed to secure a 20\% salary increase, you would become eligible.
\end{itemize}

\paragraph{Similarities and Differences}
It is important to specify that \ac{CEx} and \ac{CRe} are conceptually distinct and can yield significantly different responses under certain conditions. 
\begin{example}A student takes the Calculus 101 exam, which consists of 10 multiple-choice questions. The student answers only 5 questions correctly and thus fails the exam. A \ac{CEx} in this context might be something like: \quotes{if you had selected answer 'A' for question 3 and 'B' for question 6, you would have achieved a score of 7, passing the exam.} However, if the student were to retake the exam the following week, this type of information would hold little relevance. In contrast, a \ac{CRe} might highlight the necessity of improving understanding of the sections on derivatives and integrals covered in the course.
\end{example}

From a mathematical point of view, these concepts can be understood based on two main factors that give them clear meaning and highlight their differences. The first factor is \textit{the presence of random elements that can affect the values of the variables involved} and, as a result, the final outcome (for example, a student's mental state on the day of the test). The second factor is \textit{the way the variables in the system are connected and influence each other}. This aspect will be addressed in its more advanced form in \cref{ch:timecar}.

In the next section, we will introduce the mathematical formalism required to analyze \ac{AR}. We will start with its simplest definition and gradually refine it as the chapters progress.

\section{Preliminaries}\label{sec:algo_recourse_preliminaries}

\begin{definition}
Given a binary predictive model \( h: \Xspace \to \{0,1\} \), where \( \Xspace = \Xspace^1 \times \dots \times \Xspace^\sizex \) represents the feature space of \( \X \), consider the original (factual) instance of interest \( \xf \). We define \( \xcf \) as a counterfactual instance of \( \xf \) if \( h(\xf) \neq h(\xcf) \). 
\end{definition}

In most cases, we are interested in scenarios where \( h(\xf) = 0 \). Here, the model outcome assigning \( 0 \) is interpreted as "unfavorable," while an outcome of \( 1 \) is interpreted as "favorable". 

\paragraph{Minimal Cost}
While there may exist infinitely many alternative counterfactual instances, some are preferable to others. To formalize this preference, we introduce a notion of cost, defined as a function 
\begin{equation}\label{eq:cfcost}
\begin{split}
    \costfunction : \Xspace \times \Xspace &\to \R^+ \\
    \xf,\xcf &\mapsto \costfunction(\xf,\xcf) 
\end{split}
\end{equation}

\begin{definition}[\acf{CEx}]\label{def:contrastive_explanation}
\ac{CEx} consists in producing the \textit{optimal counterfactual instance} by optimizing cost with respect to the space of instances, that is  
\begin{equation}\label{eq:optimal_cf}
    \xcf^* := \argmin_{\xcf}\left[ \costfunction(\xf,\xcf)\right] \quad s.t. \quad h(\xcf)\neq h(\xf).   
\end{equation}  
\end{definition}

%The problem can also be formulated by considering an action applied to $\xf$ with the objective of reaching $\xcf^*$. \textit{An initial observation is that multiple actions may result in the same overall change}. Moreover, the specific effects of each action depend on the underlying structure of the system, particularly \textit{the causal relationships between the variables involved}.
The simplest case to consider is one in which the features are independent of each other, as illustrated in \cref{fig:independent-features}. Under this assumption, there exists only a single possible action to transition from $\xf$ to $\xcf$, namely, $\delta = \xcf - \xf$. In other words, the required action involves adding to each feature of $\xf$ \textit{the exact amount needed to match the corresponding value in $\xcf$}. \Cref{eq:optimal_cf} then becomes

\begin{equation}\label{eq:optimal_delta}
    \delta^* := \argmin_{\delta}\left[ \costfunction(\xf,\xf+\delta)\right] \quad s.t. \quad h(\xf + \delta)\neq h(\xf). 
\end{equation}

\begin{figure}
    \centering
    \begin{tikzpicture}[scale=1,
            transform shape,
            vertex/.style={circle, draw, minimum size=0.8cm, inner sep=0pt, font=\small}]
    \node[vertex] (x1) at (-2, 0) {$\Xcomp[1]$};
    \node[vertex] (x2) at (-1, 0) {$\Xcomp[2]$};
    \node[vertex] (x3) at (0, 0) {$\Xcomp[3]$};
    \node[vertex] (x4) at (1, 0) {$\Xcomp[4]$};
    \node[vertex] (x5) at (2, 0) {$\Xcomp[5]$};
    \node[vertex] (y) at (0, -2) {$Y$};
    
    \draw[edge] (x1) to (y);
    \draw[edge] (x2) to (y);
    \draw[edge] (x3) to (y);
    \draw[edge] (x4) to (y);
    \draw[edge] (x5) to (y);
    \end{tikzpicture}
    \caption{\textbf{Independent Features.} The simplest non-trivial causal model relevant to the problem of Algorithmic Recourse: the features are independent of one another, and each contributes individually to the final outcome.}
    \label{fig:independent-features}
\end{figure}
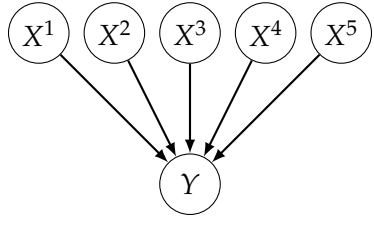

\section{First Challenge: Choosing Cost}\label{sec:algo_recourse_cost_challenges_1}

So far, we have defined \ac{AR} as a constrained optimization problem over a given $\costfunction$ function. However, this $\costfunction$ has been treated as an abstract quantity, leaving its specific form open to choice. In this section, we will address this point by presenting some of the most common choices in the literature and analyzing their limitations. \Cref{ch:cost-subjective} will further explore this topic from the user's perspective, arguing for the necessity of moving beyond a purely algorithmic approach, which may be insufficient to capture the full complexity of the problem at hand.

\subsection{Common choices}\label{sec:cost_choices}

The most direct, simple, and widely used approach is to select a cost function based on some distance in the feature space, typically some $\normlp$ norm, defined by
$$\norm{\x}_p:=((x^1)^p+\dots+(x^d)^p)^{1/p}.$$
Common examples include the Euclidean distance, characterized by the $\normltwo$ norm and defined as
$$\distance(\x,\y)=\norm{\x-\y}_2,$$
the Manhattan distance, induced by the $\normlone$ norm, the one induced by the $\normlzero$ norm, and the $\normlinfinity$ norm. \citet{KarimiBBV20}, for example, propose a linear combination of these. Each of these norms reflects a specific preference regarding the type of counterfactuals to be generated and emphasizes different optimality criteria. For instance, $\normlzero$ acts as a feature change counter. Consequently, its minimization promotes \textit{sparsity} in the modified features. 

Weights are commonly applied to the distance metric to address the varying scales of features, especially when the dataset is not pre-normalized \citep{Wachter2017}. Examining the broader field of counterfactual generation beyond the specific recourse problem, the selection of distance measures remains largely consistent and predominantly relies on $\normlp$ norms \citep{guidotti2022counterfactual}.  
\subsection{Constraints and Limitations}
To make $\costfunction$ functions as realistic and useful as possible, recent work in the field has focused on incorporating additional constraints (for a comprehensive review see  \citet{verma2024counterfactual,guidotti2022counterfactual}). Key examples include:
\begin{itemize}
\item \textit{Actionability}: Only allow changes to features that can be realistically modified (e.g., job skills, not date of birth or race).
\item \textit{Plausibility}: Ensure counterfactuals follow realistic data patterns (e.g., avoid suggesting $1M$ salary for entry-level positions).
\item \textit{Sparsity}: Prefer solutions with fewer feature changes (e.g., modifying 1-2 attributes rather than many).
\item \textit{Diversity}: When generating multiple counterfactuals, ensure they represent meaningfully different alternatives.
\item \textit{Robustness}: Counterfactuals should remain valid under model uncertainty or small input perturbations.
\end{itemize}

Intuitively, we can view these constraints as additional penalty terms added to the original cost function. Formally, given a set of constraint costs $\{c_i\}_{i=1}^n$ (where each $c_i$ represents the penalty for violating a constraint), the total cost becomes:
$$
\costfunction = \costfunction_{\text{original}} + \sum_{i=1}^n c_i.
$$
For \textit{hard constraints} (those that must be absolutely satisfied), the corresponding $c_i$ is set to $\infty$, effectively prohibiting any solution that violates them.

Despite this refinement process, certain limitations inherent in using distance metrics as a foundation persist and cannot be easily resolved through additional penalty terms. In this section, we begin by examining the simplest issues: \textit{translational invariance} and \textit{symmetry}. Later in our work, we will address more complex challenges including \textit{causality}  (\cref{sec:dependent_variables}), \textit{user dependence} (\cref{ch:cost-subjective}), and \textit{time unawareness} (\cref{pt:time_aware_causality}).

\subsection{Limitations}
\paragraph{Translational Invariance} A key limitation of distance metrics induced by $\normlp$ norms is their \textit{translational invariance} property, expressed as:
$$\distance(\x+\delta,\y+\delta) = \distance(\x,\y) \quad \forall \x,\y,\delta.$$

\begin{example}
Consider two individuals, $\x$ and $\y$, who receive identical recommendations to qualify for a bank loan: \quotes{increase your monthly salary by \$700}. However, $\x$ currently earns \$8,000, while $\y$ earns \$1,200. Intuitively, fulfilling this recommendation requires substantially less relative effort for $\x$ than for $\y$, as the difficulty depends critically on the initial salary - a factor not reflected in translationally invariant distance metrics.
\end{example}

To address this limitation, \citet{ustun2019} proposes a percentile-based measure defined as:
\begin{equation*}
    \costfunction(\x,\y) = |\Prob(\x)-\Prob(\y)|,
\end{equation*}
where $\Prob$ represents the cumulative probability function of the target population.  

\paragraph{Symmetry}

Nevertheless, this approach still presents significant limitations, particularly regarding the \textit{symmetry} of $\costfunction$:
\begin{equation}\label{eq:cost_simmetry}
    \costfunction(\x,\y) = \costfunction(\y,\x).
\end{equation}
This property proves undesirable in nearly all realistic scenarios. Typically, at least at a basic analytical level, we assign each variable a directionality: movement in one direction improves the variable's state, while movement in the opposite direction worsens it. Intuitively, we expect the effort required for improvement to always exceed that needed for deterioration. Note that $\normlp$ and percentile shifts are both symmetric.

\section{Second Challenge: Dependent Variables}\label{sec:dependent_variables}
Variables in a system are not always independent. In most real-world cases, they interact with each other, meaning that a single change can be beneficial for one variable while simultaneously harmful for another. In this section, we formalize this problem by introducing the \emph{causal formulation of \ac{AR}}.  

\begin{example}[The Winemaker's Dilemma]
\begin{figure}[h!]
\centering
\begin{minipage}{0.5\textwidth}
    \centering
    \begin{tikzpicture}[scale=1.1, transform shape, vertex/.style={rectangle, draw, minimum size=0.9cm, inner sep=4pt, font=\small}]
        \node[vertex] (quality) at (0, 1) {Quality};
        \node[vertex] (quantity) at (0, -1) {Quantity};
        \node[vertex] (earnings) at (3,0) {Earnings};
        
        \draw[edge, latex-latex] (quality) -- (quantity);  % Bidirectional edge
        \draw[edge] (quality) -- (earnings);
        \draw[edge] (quantity) -- (earnings);
    \end{tikzpicture}
\end{minipage}%
\begin{minipage}{0.5\textwidth}
    \centering
    \includegraphics[width=0.7\textwidth]{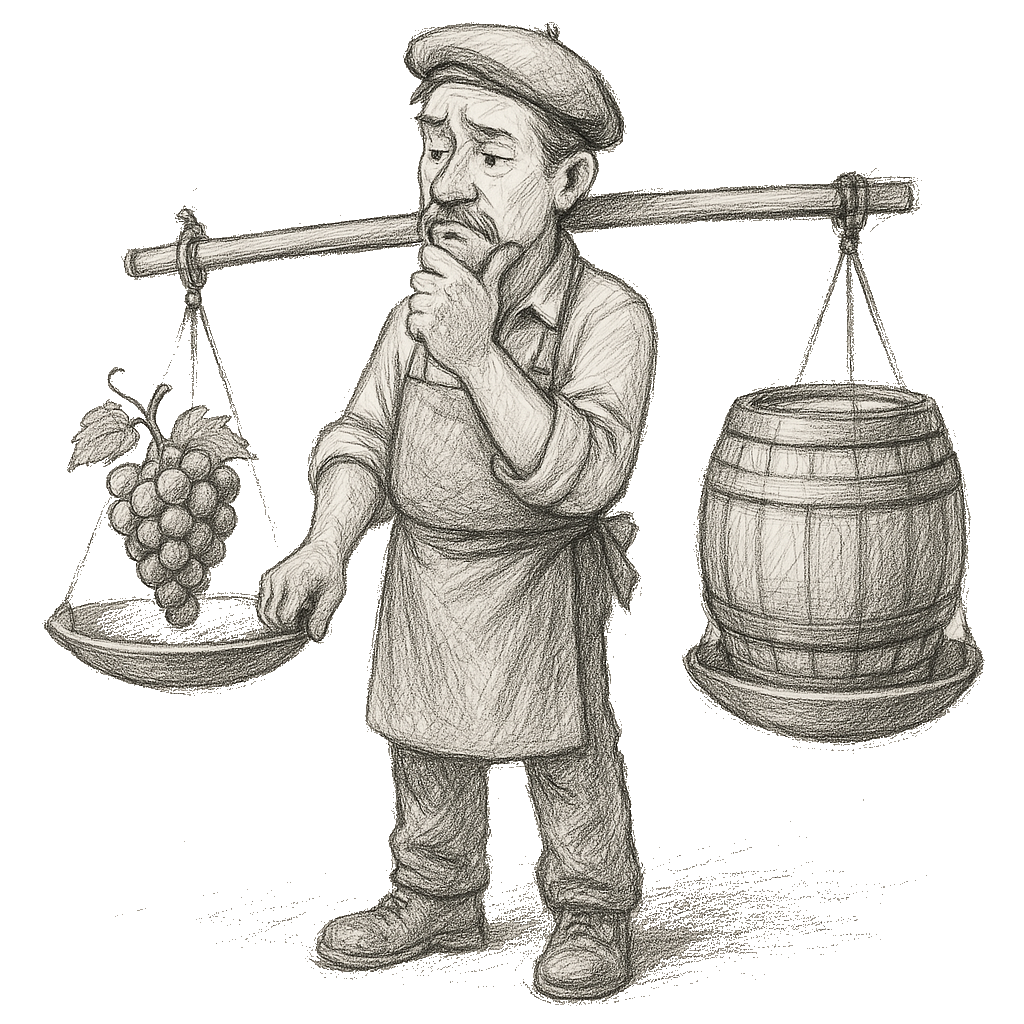}
\end{minipage}
\label{fig:winemaker_graph}
\end{figure}
In winemaking, there is a trade-off between grape quality and yield. Focusing on high-quality grapes often means reducing yield, as vines are pruned to enhance the flavor of fewer fruits. This improves the wine's quality but lowers production. Conversely, maximizing yield (for instance, through the use of fertilizers) can result in less flavorful grapes, leading to lower-quality wine. The outcome of this trade-off is total earnings, which depend on both the price of high-quality wine and the volume produced, balancing the benefits of quality with the costs of reduced yield. \emph{What actions should the winemaker take to maximize their profit?}
\end{example}

\subsection{Causality in Algorithmic Recourse}\label{sec:Causal_AlgorithmicRecourse}

Assuming independence between variables, the formulation of Contrastive Explanation through \cref{eq:optimal_cf} and \cref{eq:optimal_delta} provided a sufficiently clear protocol for modifying the factual instance $\xf$ into the optimal counterfactual $\xcf^*$. Since altering a single variable had no effect on the others, the transformation from one state to the other was straightforward: one simply needed to examine the component-wise differences and sequentially add the missing quantity, one variable at a time.  

Without this assumption, the process becomes more complex: \textit{as soon as the first modification is made, other variables may change as well}. Even if one compensates for these unintended drifts through further adjustments, new drifts may arise, leading to a cascading effect. In other words, without understanding the mechanisms governing the dependencies within the system, reaching $\xcf^*$ may no longer be straightforward.  

Conversely, knowledge of these mechanisms could not only ensure that $\xcf^*$ is reached but might even simplify the process, potentially requiring fewer modifications. In other words, one could leverage this knowledge to their advantage. Consider \cref{fig:cex_vs_causal_plot}. 
\begin{figure}[h!]
    \centering
    \includegraphics[width=1\linewidth]{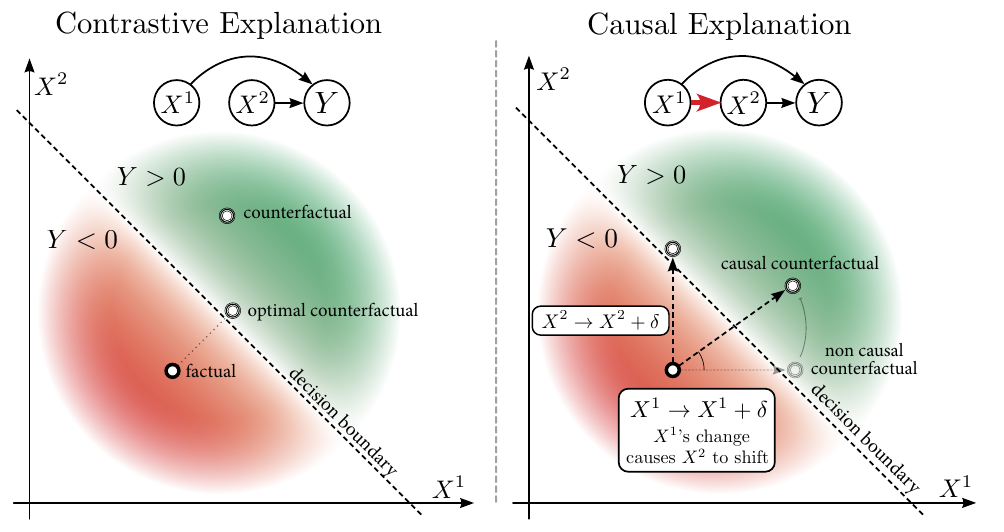}
    \caption{\textbf{Causal vs non-Causal \ac{AR}.} (Left) \textit{Classic Algorithmic Recourse} setting. (Right) \textit{Causal Algorithmic Recourse setting}, where variables can have dependencies.}
    \label{fig:cex_vs_causal_plot}
\end{figure}

The \textit{Contrastive Explanations} setting (Left) relies on the following assumptions:  
\begin{enumerate}
    \item The variables $\Xcomp[i]$ are mutually independent.  
    \item The final decision is determined only by the sign of $Y$. 
    \item Every point can be reached, meaning all features are actionable.  
    \item The optimal $\xcf^*$ minimizes a given distance metric (e.g., Euclidean norm).  
\end{enumerate}  

Introducing dependencies between variables leads to what we call the \textit{causal setting}, as these dependencies can be modeled by a \ac{SCM}. The key differences are:  
\begin{enumerate}
    \item The variables are no longer independent, as indicated by the red arrow between $\Xcomp[1]$ and $\Xcomp[2]$.  
    \item The cost of an action depends only on changes to the causes, not their effects. The $\costfunction$ function inherits the structure of the causal model. For example, if someone changes their job, their salary is likely to change as a consequence, without incurring additional cost.  
\end{enumerate}  

Let us analyze the two actions shown in the figure on the right. 
\begin{itemize}
    \item The first action consists of increasing $\Xcomp[1]$ by a quantity $\delta$. Since $\Xcomp[1]$ is a cause of $\Xcomp[2]$ and has a positive effect on it, increasing $\Xcomp[1]$ also increases $\Xcomp[2]$. In a non-causal setting, this intervention would result in a horizontal shift, producing the \textit{non-causal counterfactual}. In the causal setting, however, the shift in $\Xcomp[2]$ modifies the final state, leading to a point much farther from the decision boundary (a more positive outcome). Similarly, a weaker intervention in the same direction, that may have led to a negative outcome in the non-causal model, could result in a positive outcome in the causal model, despite the estimated cost being the same.
    \item The second action, instead, increases the value of $\Xcomp[2]$, which is not a cause of $\Xcomp[1]$. Consequently, the counterfactual obtained in both settings remains exactly the same. The knowledge of the causal structure informs us that, in this example, acting on $\Xcomp[1]$ is preferable. 
\end{itemize}  

In the causal setting, knowing the dependencies between features makes it possible to evaluate the effect of a given intervention on $\xf$. However, the inverse problem \ie, determining the necessary actions to reach a desired target point $\xcf$, is not straightforward. Additionally, the original metric used to determine the optimality of counterfactuals is no longer valid, as the new cost function is now adjusted to account for cause-effect relationships. In other words, formulating the optimization problem as in \cref{eq:optimal_cf} is no longer meaningful. Instead, it seems more useful to shift the optimization problem to the space of actions. \textit{Rather than asking where one wants to arrive, the central question becomes how to get there}. 

\begin{figure}
    \centering
    \includegraphics[width=0.8\linewidth]{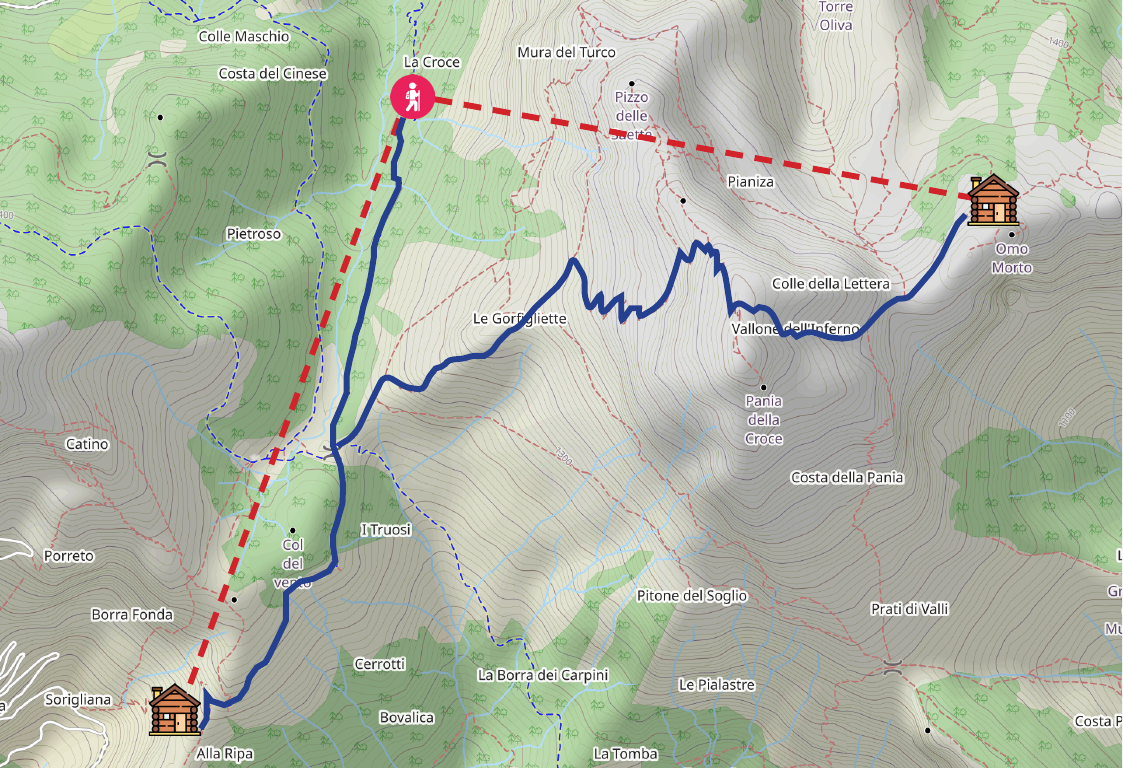}
    \caption{\textbf{\ac{CEx} vs \ac{CRe}.} Illustration of the difference between \ac{CEx} and \ac{CRe}. While \ac{CEx} suggests reaching the nearest refuge in a straight line, ignoring obstacles or terrain, \ac{CRe} leverages additional structural information (the topographic map) to recommend the fastest feasible route, even if it leads to a different refuge. This example highlights how incorporating knowledge of system structure enables more effective and realistic recommendations.}
    \label{fig:hiking_problem}
\end{figure}
\begin{example}
    Consider the following analogy, depicted in \cref{fig:hiking_problem}. A mountaineer wants to reach the nearest refuge to spend the night. A \emph{distance-based recommendation} would suggest the closest refuge in a straight line, without providing specific guidance on how to get there.  In contrast, an \emph{action-based recommendation}, relying on a topographic map of the area, would evaluate alternative routes and ultimately suggest the one with the shortest duration. Note that the selected refuge may differ depending on the recommendation approach used. The key distinction between these two approaches is that the second \textit{incorporates additional knowledge of the system’s structure} (\ie, the topographic map), enabling the optimization of a more meaningful metric. Moreover, while the first approach \textit{only identifies the destination}, the second \textit{explicitly provides the route to follow}.  
 \end{example}

 \begin{definition}[\acf{CRe}]\label{def:Consequential_recommendation}
\ac{CRe} consists in producing the \emph{optimal action} to implement in order to change the evaluation of the function $h$, by optimizing the cost function $\costfunction$ over the space of actions $\actionspace$. Formally, we define:
\begin{equation}\label{eq:optimal_action}
\bm{a}^* := \argmin_{\bm{a}\in \actionspace} \left[ \costfunction(\xf,\bm{a}) \right] \quad \text{subject to} \quad h(\bm{a}(\xf)) \neq h(\xf).
\end{equation}

Let us briefly comment on the formulation. First, an action is a mapping $\bm{a} : \Xspace \to \Xspace$ that modifies the state of the system. Consequently, the cost function $\costfunction$ is a function
$$
\costfunction : \Xspace \times \actionspace \to \R^+,
$$
which evaluates the cost of implementing a given action starting from a specific initial state.

The objective remains the same as in \cref{def:contrastive_explanation}, namely to change the output of the evaluation function $h$. However, in general, we do not expect the resulting counterfactual instance to coincide with the one found in contrastive explanations. The action space $\actionspace$ is currently an abstract object that must be contextualized based on the specific structure assigned to the system. In our case, adopting a causal perspective, $\actionspace$ corresponds to the space of possible causal interventions, and the structural model considered will be that of an \ac{SCM} (See \cref{sec:counterfactual_level}).
\end{definition}

 \paragraph{Complexity of \ac{CEx} and \ac{CRe}}
Enriching the problem structure does not only bring benefits. On one hand, \ac{CRe} allows for a more meaningful and realistic study of \ac{AR}, but on the other, it requires a more complex formulation, not only conceptually but also computationally.

Comparing \cref{eq:optimal_delta} and \cref{eq:optimal_action}, the former optimizes in $\Xspace$, which has dimension $n$, corresponding to the number of features in the system. Using a brute-force Grid Search optimization algorithm, \ac{CEx} would suffer from the curse of dimensionality. For simplicity, let us assume that $\Xspace$ is the hypercube $[0,1]^n$. If we fix a resolution of $1/k$, meaning that each feature takes $k$ different values, a total of $k^n$ instances would need to be evaluated, leading to exponential complexity.  

In practice, the regularity of the functions $\costfunction$ and $h$ can be exploited to significantly simplify the problem. This can be achieved using gradient-based methods (similarly to approaches employed in neural network models) or, when gradients are unavailable, through approximate algorithms. As a result, \ac{CEx} is almost always computationally manageable. In contrast, \ac{CRe} introduces an additional layer of complexity.

 \begin{remark}
     \ac{CRe} is at least as computationally complex as \ac{CEx}.  

     Assume that all system features are actionable. Then, we can apply an intervention to each component $\Xcomp[i]$ that modifies its value, denoted as $\intervention^i$. For example, considering additive interventions (adding or subtracting a given quantity $\delta$ from the initial value), $\intervention^i$ defines a one-dimensional space.  

     In principle, the space of all possible interventions (i.e., the action space $\actionspace$) consists of all combinations of individual interventions $\intervention^i$. Consequently, in the worst case, its dimensionality is $2^n \cdot n$.  

     At the beginning of \cref{ch:causality_math_framework}, we saw that exploiting the causal structure of the model can significantly reduce its complexity. By leveraging the independence and modularity of causal mechanisms, it is often possible to achieve a more compact representation, sometimes reducing the complexity from exponential to polynomial, or even linear (see, for example, \cref{sec:probability_factorization_math_framework}).  

     In the best-case scenario, where variables are completely independent (see \cref{fig:independent-features}), the action space simplifies to the Cartesian product  
     $$\actionspace = \actionspace^1 \times \dots \times \actionspace^n.$$  
     Consequently, under these conditions, \ac{CEx} and \ac{CRe} become entirely equivalent.    
\end{remark}
\paragraph{Linearity} 
An important case where \ac{CRe} keeps the same computational complexity as \ac{CEx} is when the \ac{SCM} is linear and the interventions are additive. In this case, the structural equations can be written in compact form as
\begin{equation}
    \X = \A \X + \vectornoise,
\end{equation}
where $\A$ is a matrix that represents the strengths of the causal relationships between the variables. 

An additive intervention on component $\Xcomp[i]$ means adding the term $\delta_i \hat{\bm{e}}_i$ to the right-hand side of the equation, where $\hat{\bm{e}}_i$ is the $i$-th unit vector of the standard basis in $\R^n$. Explicitly writing the equation in vector form gives
\begin{equation}
    \begin{bmatrix}
        \Xcomp[1] \\
        \vdots \\
        \Xcomp[i] \\
        \vdots \\
        \Xcomp[n]
    \end{bmatrix}
    =
    \A
    \begin{bmatrix}
        \Xcomp[1] \\
        \vdots \\
        \Xcomp[i] \\
        \vdots \\
        \Xcomp[n]
    \end{bmatrix}
    +
    \begin{bmatrix}
        \noisecomp[1] \\
        \vdots \\
        \noisecomp[i] \\
        \vdots \\
        \noisecomp[n]
    \end{bmatrix}
    +
    \begin{bmatrix}
        0 \\
        \vdots \\
        \delta_i \\
        \vdots \\
        0
    \end{bmatrix}.
\end{equation}

Solving for $\X$, we obtain
\[
\X = [\I - \A]^{-1} (\vectornoise + \delta_i \hat{\bm{e}}_i).
\]
We can easily identify the effect of the intervention in the term $[\I - \A]^{-1} \delta_i \hat{\bm{e}}_i$.

Suppose we apply another intervention on top of the first one, but on a different variable $\Xcomp[j]$. Thanks to the linearity of $[\I - \A]^{-1}$, by the superposition principle, the total effect is simply the sum of the individual effects:
\[
[\I - \A]^{-1}\delta_i \hat{\bm{e}}_i + [\I - \A]^{-1}\delta_j \hat{\bm{e}}_j.
\]
Therefore, the search for $\bm{a}^*$ can be performed one variable at a time, and the size of the search space reduces to the dimension of the system.

\chapter{Cost is Subjective}\label{ch:cost-subjective} 

\begin{tcolorbox}[colback=CTsemi!5!white, colframe=CTsemi, boxrule=0.5mm]
\textbf{Main Sources:} \fullcite{DBLP:conf/xai/BerettaCM23}, \fullcite{cappuccio2024interactive}.
\end{tcolorbox}

\paragraph{Author's Note}
This chapter presents one of the first attempts to explore counterfactual explanations from a more user-centered point of view. The ideas discussed were developed during the early stages of the PhD and are mostly meant as a starting point for reasoning about how people interact with algorithmic decisions.
The objective is to propose a preliminary solution to the problem of rendering this framework more compatible with human cognitive processes, providing an in-depth discussion of the methodological aspects and limitations of classical algorithmic recourse from the perspective of the user.
\vspace{30pt}

Progress in the performance and efficiency of automatic decision-making systems has incentivized AI-based solutions in pervasive and impactful contexts of daily life, such as finance, healthcare, and transportation.
A problematic aspect of these models is the lack of explainability: it is difficult to provide the reasons behind the automated decisions due to the complexity of the algorithms and the large amounts of data processed. This lack of transparency and accountability can lead to mistrust and concerns about the ethical implications of using AI in decision-making~\citep{guidotti2018survey}.
In sensitive real-world contexts, accounting for the algorithmic decision is necessary for the user to understand and contest the motivations behind the system's output. This is especially important when the outcome strongly impacts human life or has harmful consequences.
Moreover, relevant policies such as the GDPR (General Data Protection Regulations)~\citep{GDPR,Seizov_AItransparencyGDPR} claim the relevance of appealing in case of rejection, i.e., to request what changes users should make to be accepted in receiving a positive response, for example, to a loan application. In this way, users can be empowered to understand how the algorithmic decision affects them.

\begin{figure}[t] 
    \centering
    \includegraphics[width=1\linewidth]{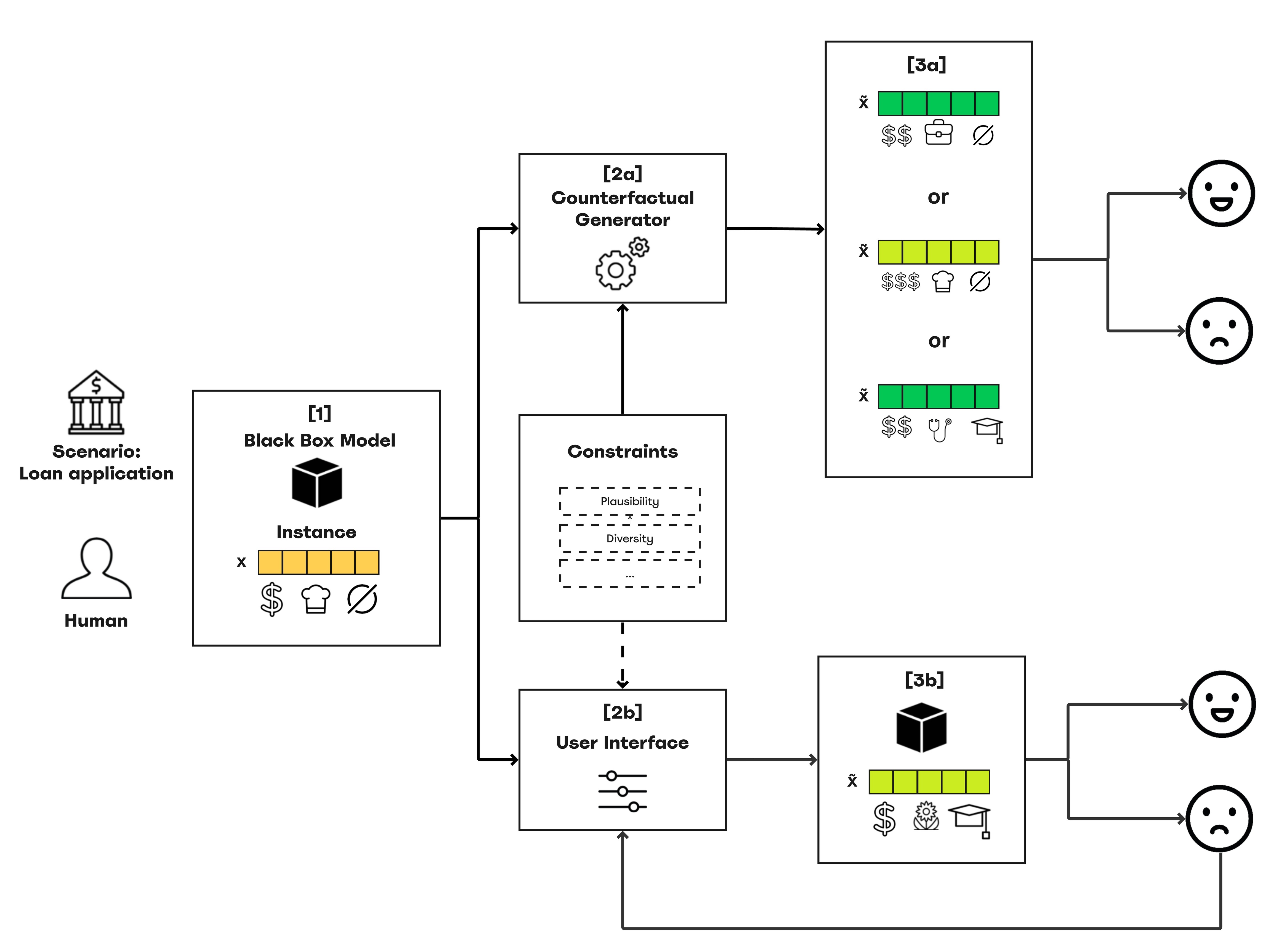}
    \caption{\textbf{Workflow of \EMA.}}
    \label{fig:workflow}
\end{figure}

In this work, we present \EMA, an interactive interface that allows an exploration of the decision boundary of any black box model.
We report in Fig.~\ref{fig:workflow} a visual workflow of our approach. Imagine a person applying for a loan to a bank that leverages an automatic decision-making system. The instance $x$ represents her current situation (stage [1] of the workflow): she has a certain amount of capital, her current working position is being a chef in a restaurant, and she has no degree. The decision emitted by the black box model is negative, hence her request is denied. 

Following the diagram, in stage [2a] an automatic counterfactual technique generates instances trying to optimize a predefined cost function that embeds a set of constraints, e.g., \textit{plausibility} and \textit{diversity}.  
Three counterfactual instances receiving a positive outcome are proposed to the user (stage [3a]). The scenarios suggested are: (i) the person has a medium amount of capital and occupies a business role, (ii) the person remains a chef, but she has a high amount of capital, and (iii) the person has a medium amount of capital, she is a doctor and she has a degree in medicine. 
From the counterfactuals returned, the user satisfaction is symbolized by the smiley and gloomy faces. Following this pipeline, the subjective human satisfaction is impossible to capture and it remains unknown or an irrelevant feedback.

Our proposal, exemplified in stage [2b], consists of an interface that lets the users freely explore their options, changing the features at will. The set of constrains is not directly embedded nor expressed through the interface, but it is implicitly delegated to the common sense of the users. The resulting counterfactual user-crafted in stage [3b] has the same amount of money, by profession she is a florist and has a degree. The difference between our approach w.r.t. the traditional one consists in the possibility for the user to intervene again on the interface and generate a new, more effective counterfactual if the one created is not useful or applicable, creating a continuous loop through the interaction.

The rest of the chapter is organized as follows. In Section~\ref{related}, we report a critical literature discussion and a review of the state-of-the-art solutions that closely align with our contribution. Section~\ref{method} presents the proposed method. In Section~\ref{samples}, we report several examples testing different datasets to demonstrate our proposal’s novelty and effectiveness for counterfactual generation. Finally, Section~\ref{concl} indicates future research directions and discusses the limitations of our approach.

\section{Contributions}\label{sec:contributions_costsubj}

This chapter offers the following contributions:

\begin{itemize}
    \item \textbf{Human-Centered Perspective:} We highlight several limitations of existing counterfactual generation techniques, particularly the assumption of a universal, fixed cost function, and argue for a more subjective, user-driven notion of cost.

    \item \textbf{Interactive Exploration:} We propose an alternative approach to counterfactual generation that gives users direct control over feature modification, allowing them to explore the decision space through a simple and intuitive interface.

    \item \textbf{Prototype Implementation:} We present \EMA, a lightweight tool for local, model-agnostic interaction with decision systems, designed to support both experts and lay users.

    \item \textbf{Preliminary Evaluation:} We demonstrate the use of \EMA through illustrative examples on standard benchmark datasets, showing how user interaction can lead to flexible and realistic counterfactuals.

    \item \textbf{Design Reflection:} We discuss how this type of tool could open the way for new evaluation methods and user studies that shift the focus from algorithmic optimization to human experience and control.
\end{itemize}

\section{Related Work}\label{related}

\subsection{Counterfactual generation techniques}
Counterfactual approaches explain individual predictions describing what-if contrastive scenarios~\citep{karimi2021,DBLP:conf/pkdd/MolnarCB20,DBLP:conf/xai/BerettaC23}. Specifically, the explanations indicate to the user the minimum change necessary in the feature space representing them so that the output of the automatic system changes toward the desired outcome. 

Various approaches have been devised for generating counterfactuals, including optimization, heuristic search, instance-based, and decision trees. We refer to~\citet{guidotti2022counterfactual} for a more comprehensive overview. 
Evaluating the quality of counterfactual explanations requires careful consideration of multiple desired qualities, which have been defined in the literature as \textit{validity, sparsity, similarity, plausibility, discriminative power, actionability, causality}, and \textit{diversity}~\citep{DBLP:conf/kdd/Jia0RLH19,zhang2019should,DBLP:journals/ai/Guidotti21}. The complexity of this range of evaluation metrics presents a critical challenge in comparing the currently available algorithms and establishing benchmark procedures.
Moreover, most of the current counterfactual generation techniques do not allow for user interaction, thus limiting the practical applicability of these methods in real-world contexts~\citep{DBLP:journals/csur/KarimiBSV23}.

\subsection{Human-Centered Approach}
It is crucial to rethink new ways of interaction between humans and the \ac{XAI} algorithm, focusing more on the user's point of view and needs~\citep{miller2019explanation,liao2021human,abdul2018trends,amershi2014power}.
The effectiveness of an explanation is based on the extent to which the recipient regards it as comprehensible and useful. A human-centered approach is needed to fill the gap between the algorithmic perspective and an explanation that is truly intelligible to humans. In fact, as stated in ~\citet{liao2021human}, \ac{XAI} presents more of a design challenge than an algorithmic challenge.

Previous social science research has shown that human explanations are often contrastive. Therefore, counterfactuals are generally considered less cognitively demanding and more intuitive ~\citep{riveiroT21}.
Given these assumptions, our proposed interface and algorithm are outlined as starting points for several case studies, see Section~\ref{samples}, and are therefore designed for domain experts and lay users ~\citep{RiberaL19}.

In 2018 the DARPA explainable AI initiative framed the explainable AI process as a three-stage approach, distinguishing between the explainable model, the explanation user interface.~\citep{gunning2019darpa}.
This definition draws a difference between the model through which an explanation of a machine learning algorithm is generated and the means used to communicate it to the user~\citep{chromik2020taxonomy}.
From this perspective, the explanation process has to be outlined as a continuous dialogue between a sender and a receiver~\citep{abdul2018trends}. Therefore \ac{XAI} has to consider interactivity as a fundamental part of the process~\citep{madumal2019grounded,miller2019explanation, chromik2020taxonomy}, for example, through the design of novel interfaces that allow model inspection at will~\citep{Shneiderman20bridging}, and promote an understanding of the computational processes giving users control of the algorithmic actions~\citep{DBLP:journals/interactions/ShneidermanPCJE16}. 
%This paradigm shift also entails changing the desiderata, taxonomy, and evaluation criteria of Explainable AI.

\subsection{Interactive explanations and visualizations}
In the following, we frame \EMA in relation to existing literature solutions, and critically discuss the novelty and positioning of our contribution.

\texttt{DECE} is an advanced counterfactual generation method tailored for both global and local interpretability tasks~\citep{DBLP:journals/tvcg/ChengMQ21}. With a strong focus on visualization and user interface design, \texttt{DECE} automates counterfactual generation using established metrics such as diversity and validity. Notably, it facilitates user-driven constraint imposition through a decision tree framework, offering fine-grained control over the generation process. Moreover, \texttt{DECE} extends beyond dataset limitations by exploring counterfactuals independent of available data, while leveraging subset counterfactuals to dissect machine learning model decision boundaries.

\texttt{Prospector} is a method designed to generate instance counterfactuals via a visual interface~\citep{DBLP:conf/chi/KrausePN16}. It utilizes a semi-global score, specifically a partial dependence plot, to estimate the probable prediction when specific features of the original instance are modified. This process involves three distinct phases: patient selection, patient inspection, and the generation of partial dependence plots. Access to both the model and the dataset used for training is required. %\footnote{Unfortunately, the code for Prospector is no longer accessible or available online}
The tool is specifically designed for data scientists. It allows them to understand how features affect the prediction, thus improving a predictive model.

\texttt{DiCE} framework generates and evaluates counterfactuals according to two properties, namely \textit{plausibility} and \textit{diversity}~\citep{mothilal2020explaining}. The authors test \texttt{DiCE} on four real-world datasets and compare its counterfactuals with popular local explanation methods and other prior approaches.  

\texttt{T-LACE} generates reliable counterfactual explanations leveraging on two properties of a tailored transparent latent space, namely \textit{similarity} and \textit{linearity}~\citep{bodria2022transparent}. An interactive framework for auto-encoders allows a visual exploration of the latent space, shedding insights on the links between input features and model prediction~\citep{bodria2022explaining}. 

\texttt{What-If Tool}\footnote{\url{https://pair-code.github.io/what-if-tool/}} consists of a model-agnostic interactive interface designed to visually explore ML systems on a variety of inputs and analyze features importance both at local and global level~\citep{wexler2019if}. It is able to automatically generate counterfactuals by optimizing predefined metrics, while it does not allow direct user intervention. It allows practitioners to interactively adjust models rather than individual instances, making it ideal for refining model performance. Its user-friendly design requires minimal coding skills, catering to a wide range of practitioners seeking insights through counterfactual reasoning.

For a thorough overview, we refer to~\citep{ExplanatoryModelAnalysis2021Biecek}, where authors present techniques and strategy designed to explore and explain model's predictions at \textit{instance-level}, i.e., for a single record, and at \textit{dataset-level}, i.e., for an entire collection.  

\paragraph{Our contribution}
We outline some of the distinguishing properties of \EMA, our method, setting it apart from existing literature. Firstly, it prioritizes the user, who not only operates the tool but also indirectly defines the optimization metric, ensuring subjective considerations are captured. Moreover, it requires minimal technical skills, as no coding knowledge is needed, and users interact through intuitive sliders. It works with any model and does not require access to the dataset, making it easy to explore different models without constraints. Furthermore, it is purely local, focusing solely on the specific instance at hand, ensuring relevance to the user's context. Users can adjust instances feature by feature, enabling thorough exploration of the feature space. Overall, our method is user-friendly, adaptable, and effective for exploring machine learning systems.

\section{User-Centered Limitations}\label{sec:user_cost_limitations}

Building on the discussion in \cref{sec:algo_recourse_cost_challenges_1}, we now focus on limitations that emerge when algorithmic recommendations interact with the social and personal context of the user. While previous efforts have tried to improve $\costfunction$ through better distance metrics or added constraints, many still operate under a fixed and context-agnostic definition of cost, limiting their usefulness in real-world settings. In particular, we highlight two critical issues that are especially relevant when considering user interaction:

\textbf{Limitation 1.} \textit{Stationarity}: The cost function $\costfunction(x, x')$ is predetermined and remains fixed after deployment. It does not update based on user feedback or preferences, preventing any adaptation that could improve alignment between the model’s suggestions and the user's actual needs (\textit{see Research Challenge 9 in~\citep{counterfactual_survey_2020}}). Consider, for example, a credit scoring system that assumes obtaining a formal degree (e.g., a university diploma) is easier than switching to a different job. Based on this assumption, the model prioritizes education-related changes in its recourse suggestions. However, after deployment, it may observe that users consistently prefer to modify their employment status rather than pursue further education. This could be due to the long time horizon, financial cost, or perceived inaccessibility of academic pathways. In contrast, changing jobs may feel more attainable to users in practice. Such behavior reveals that the cost function has underestimated the difficulty of education-related changes and overestimated that of employment shifts. A non-stationary system could learn from this feedback and revise the relative weights in $\costfunction$, ultimately generating more practical and user-aligned recommendations. In other words, incorporating user feedback may offer a powerful means of calibrating the cost function to better reflect real-world preferences and constraints.

\textbf{Limitation 2.} \textit{Universality}: $\costfunction$ cannot depend on exogenous factors of the system. In most existing approaches, it is implicitly assumed that the same cost function is suitable for all users. However, in practice, the feasibility and difficulty of taking action can vary widely based on personal circumstances that are not captured in the dataset. For example, the ease of changing jobs may depend on the types of employment available in the local area, access to transportation, or caregiving responsibilities at home. Requesting a salary increase might be easier for individuals with strong workplace relationships, and accumulating wealth may depend on whether someone owns additional assets (such as family heirlooms or property). These factors, though highly relevant, are typically unobserved by the system. The cost of an action is not an objective or universal quantity—it is shaped by each individual’s social, economic, and psychological context. Ignoring this subjectivity can lead to overly rigid or unrealistic recommendations (\textit{see Research Challenge 10 in~\citep{counterfactual_survey_2020}}). One preliminary attempt to address this issue is proposed by~\citep{DBLP:journals/corr/abs-1912-03277}, which introduces user-specific cost modeling, although much remains to be explored in this direction.

These challenges motivate a shift from purely algorithmic optimization toward user-centered methods that treat individuals not just as passive recipients of recourse, but as active agents capable of judging what is realistic, desirable, or worthwhile. We explore one such approach, introducing a prototype system that allows users to directly explore the decision space and define their own counterfactuals.

Each of these properties imposes specific limitations on the ability to provide effective recommendations to each individual. 
Our contribution aims to address these limitations, especially from a user-focused viewpoint.

\section{Feature Space Navigator}\label{method} 
%Building upon~\citep{DBLP:conf/xai/BerettaCM23}, 
Our method proposes an alternative formulation of the counterfactual generation problem, departing from the algorithmic perspective and instead shifting the focus to user interaction. 
Our proposal leverages the autonomy and control of the users that can explore the behavior of the decision model accordingly with their own knowledge base, reducing the need for a dedicated explanation algorithm. From this perspective, users are designated as the ``domain experts'' of their own real-word scenario. 

\EMA is local, post-hoc and agnostic, i.e., designed to account for any black box\footnote{\EMA is available at \url{https://github.com/Elecapp/Feature_space_navigator}.}. 
It offers an interactive interface of the areas where the model returns the desired outcome by providing a graphical visualization. 
The entire process is delegated to the users that are free to explore the feature space of the system and generate autonomously, ``by hand'', effective counterfactual and prototype instances that they deem feasible, desirable, and suitable to their specific evaluations and needs\footnote{User interactions through the interface are not stored; hence, no concern occurs regarding the safety of private information that users might input.}. 
Users are therefore empowered with complete autonomy and control over the process without the mediation of an explanation algorithm between the human and the decision model.
In the following, we outline the interface, the underlying algorithm, and discuss evaluation aspects of the outputs returned by \EMA.

\paragraph{Interface}
The interface, depicted in Section~\ref{samples}, visually displays how model predictions react in response to the adjustments introduced by the users, letting them to identify relevant prototypes and counterfactuals.  
It allows users to independently explore the feature space by selectively modifying the original instance on the basis of her own knowledge and experience, i.e., changing the value of one of the features that characterize the instance $\x$. This can be done introducing minimal changes through the adjustment of the slider \textit{feature by feature} within its designated range.
Specifically, each slider is characterized by a colored gradient representing the model's scores for the different values of the feature. Users can choose which feature to modify by moving the cursor to a value that improves their score, selecting the feature they find most comfortable to change.
After each adjustment, the instance $\x$ is promptly updated with new values, generating new gradients and enabling further modifications until the user is content with the set of changes and insights.
Through a number of attempts, generating both prototypes and counterfactuals, the user intuitively guesses the feature importance of the classifier decision boundary through gradients change. 

\begin{figure}
  \begin{center}
    \includegraphics[width=0.80\textwidth]{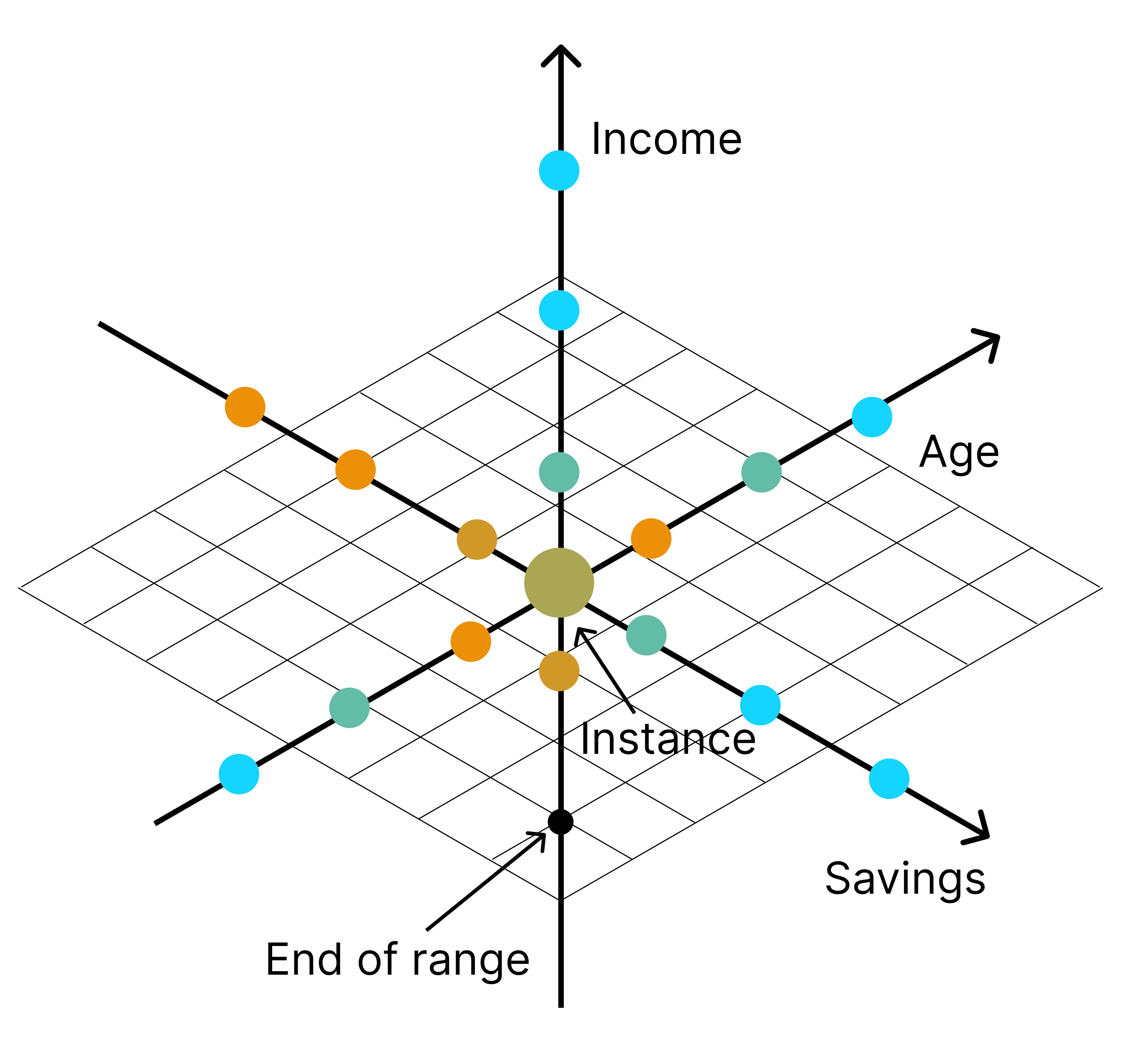}
  \end{center}
  \caption{\textbf{The Space of Changes.} Example of allowed movements through the sliders.}
  \label{fig:piano}
\end{figure}

\begin{algorithm2e}[t]
\caption{\textsc{GridEvaluation}{}($ranges$, $bins$, $X$, $x$, $model$)}\label{alg}
\SetKwInOut{Input}{Input}
\SetKwInOut{Output}{Output}
\Input{$ranges$ - features ranges, $bins$ - number of bins, $X$ - dataset, $x$ - original instance, $model$ - black box}
\Output{$\hat{Y}$ - $model$'s predictions}
$d \gets \text{len}(x)$\tcp*{\small storing number of features}
$\hat{Y} \gets  [0]_{d \times \text{bins}}$ \tcp*{\small creating a matrix of zeros}
\For{$i=0$ \KwTo $d$ }{
    \If{$categorical(X_i)$ %\tcp*{\small True if the value is categorical}
    } 
    {
    \For{$v \in unique(X_{i})$ %\tcp*{\small for each feature's categorical value}
    }
    { 
        $\hat{x} \gets  \text{copy}(x)$ \tcp*{\small creating a copy of $x$}
        $\hat{x}_i \gets  v $ \tcp*{\small changing $\hat{x}$'s $i$-th feature with the new value}
        $\hat{Y}_{i,b} \gets \text{model}(\hat{x})$ \tcp*{\small predicting score}
    }
    }
    \Else
    {
    $min, max \gets ranges_i$ \tcp*{\small saving min and max of current feature} 
    \For{$b=0$ \KwTo $\text{bins}$ }{
        $v \gets min + \frac {b}{\text{bins}} \cdot (max - min) $ \tcp*{\small computing the new value} 
        $\hat{x} \gets  \text{copy}(x)$ \tcp*{\small creating a copy of $x$}
        $\hat{x}_i \gets  v $ \tcp*{\small changing $\hat{x}$'s $i$-th feature with the new value}
        $\hat{Y}_{i,b} \gets \text{model}(\hat{x})$ \tcp*{\small predicting score}
    }
    }
} 
\Return $\hat{Y}$
\end{algorithm2e}

\paragraph{Algorithm}
The user can move along each axis in the feature space but can not make diagonal movements.
We present a simplified representation of the process in Figure~\ref{fig:piano}, i.e., the allowed movements using sliders. 
The cumulative nature of the modifications ensures the ability to reach any point in the space. 
In Algorithm~\ref{alg}, we report the pseudo-code for a single step of the process. It requires access to the model's probability outputs and is specifically tailored for tabular data with a binary target variable.
The number of predictions required scales linearly with the number of features, ensuring the interface update remains essentially real-time.

\section{Scenarios}\label{samples}

\paragraph{Datasets}
In this work, we implemented \EMA on two datasets widely used in academic research and freely accessed through the UCI Repository~\citep{UCIRepository}. 
\begin{itemize}
    \item The \textsc{Statlog (German Credit Data)} dataset~\citep{misc_statlog_(german_credit_data)_144} categorizes individuals as having good or bad credit risk through a binary classification task. It comprises $1000$ instances and $20$ features, all of which are categorical and integer types. The categorical features are handled in the interface by preserving the original categorical feature's record alongside the value generated by the LabelEncoder.
    \item The \Wine collection~\citep{misc_wine_109} comprises data resulting from a chemical analysis conducted on wines originating from the same Italian region but derived from three distinct grape varieties. This examination evaluated the levels of 13 different components present in each wine type. With 178 instances and 13 features, encompassing both integer and real numbers, the primary task associated with this dataset is multi-class classification. The target variable has three possible values denoting distinct cultivars: each wine is categorized into one of three classes. From an interface point of view, we handled the multiclassification by creating a 3-step gradient.
\end{itemize}

\paragraph{Classifiers}
We performed an $85$-$15$ split in train and test for \German and a $70$-$30$ split for \Wine.
We ran our evaluation on Random Forest implemented by the \texttt{scikit-learn} library\footnote{\url{https://scikit-learn.org/stable/index.html}}. The best parameters are identified through a Randomized Search over a dictionary of possibilities.
Both the classifiers were performed using a fixed random state.

\subsection{Understanding the Effects of Interventions}
In the following, we provide examples of potential usage of the interface on the adopted datasets, aiming to illustrate the functionality of \EMA.

\paragraph{\German Scenario}
Let us examine a case from the \German  dataset, where a user applies for a loan. The initial configuration, depicted in Figure~\ref{fig:sub1_g}, is associated with a negative outcome—the model classifies the user as a bad credit risk, and the loan is therefore denied.

Using the interface, the user begins exploring alternative options. In a first attempt, shown in Figure~\ref{fig:sub2_g}, they reduce the \textit{credit amount} requested, assuming that a lower loan might be perceived as less risky. This modification slightly improves the score, but not enough to change the model's decision—the predicted credit risk remains negative.

Realizing that credit amount alone may not be sufficient to alter the outcome, the user continues their exploration. In the third step, shown in Figure~\ref{fig:sub3_g}, they adjust the \textit{duration in month}—a variable that may influence the perceived repayment capacity. This time, the result changes: the model now classifies the instance as a good credit risk, and the loan would be granted.

This simple scenario illustrates the core idea behind our approach: by enabling users to interactively explore the feature space, they can experiment with combinations of changes that feel plausible or acceptable to them. Instead of relying on a fixed optimization routine, the process becomes a negotiation between what is desirable, feasible, and effective—driven by the user’s own judgment and understanding of their situation.

\paragraph{\Wine Scenario}
In this example, a hypothetical chemist is using the interface to determine the likely origin of a wine sample, based on its chemical composition. Initially, the sample is classified as class $2$, as shown in Figure~\ref{fig:sub1_w}.

Looking to explore how the prediction might change, the user adjusts the \textit{flavanoids} level—one of the more influential chemical features related to classification. As depicted in Figure~\ref{fig:sub2_w}, this modification leads to a shift in the predicted class: the sample is now associated with cultivar class $1$.

Curious about the full decision boundary, the user continues to explore this feature further. In a final step, shown in Figure~\ref{fig:sub3_w}, they increase the \textit{flavanoids} value once more. This additional change pushes the prediction into a different region of the feature space, resulting in a new classification as class $0$.

This scenario demonstrates how users can explore the sensitivity of the model to single features and observe non-linear shifts in its decision-making. In this case, modifying just one feature—flavanoids—across a narrow range was enough to traverse the boundaries between all three wine classes, giving the user a good sense of the model’s internal behavior and thresholds.

\begin{remark*} While in this work we instantiated \EMA on the two datasets provided as examples, it is important to recognize that the scenarios discussed serve merely as illustrative cases. The interface developed through this approach is generalizable and applicable to various real-world contexts. Therefore, while the specific instances showcased here offer valuable insights, the broader implications extend beyond these samples, demonstrating the versatility and utility of \EMA in addressing varied challenges across different domains.
\end{remark*}

\begin{figure}[ht!]
\centering
\subfloat[Starting situation.]{\includegraphics[width=1\textwidth]{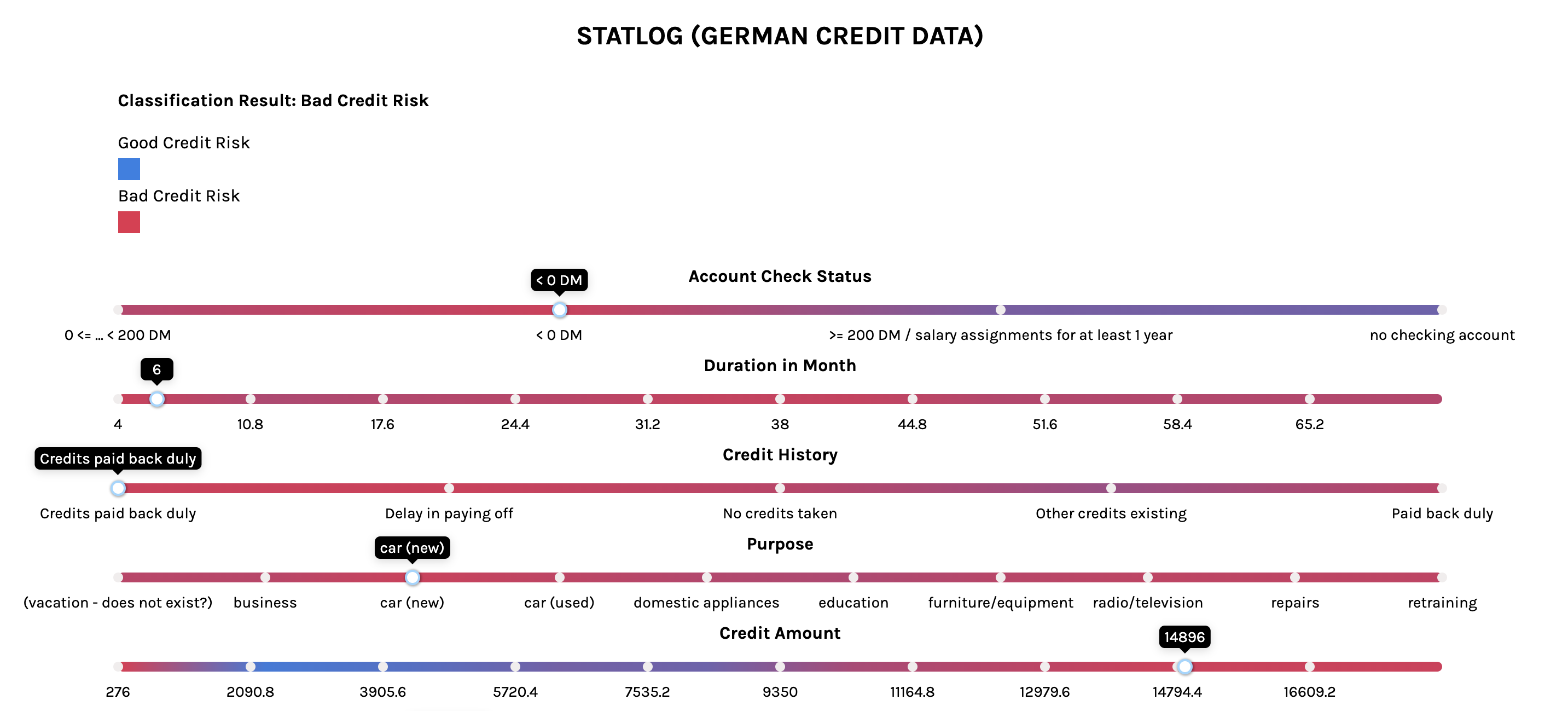}\label{fig:sub1_g}}\hskip1ex
\subfloat[User changes \textit{Credit Amount}; credit risk is still not good enough.]{\includegraphics[width=1\textwidth]{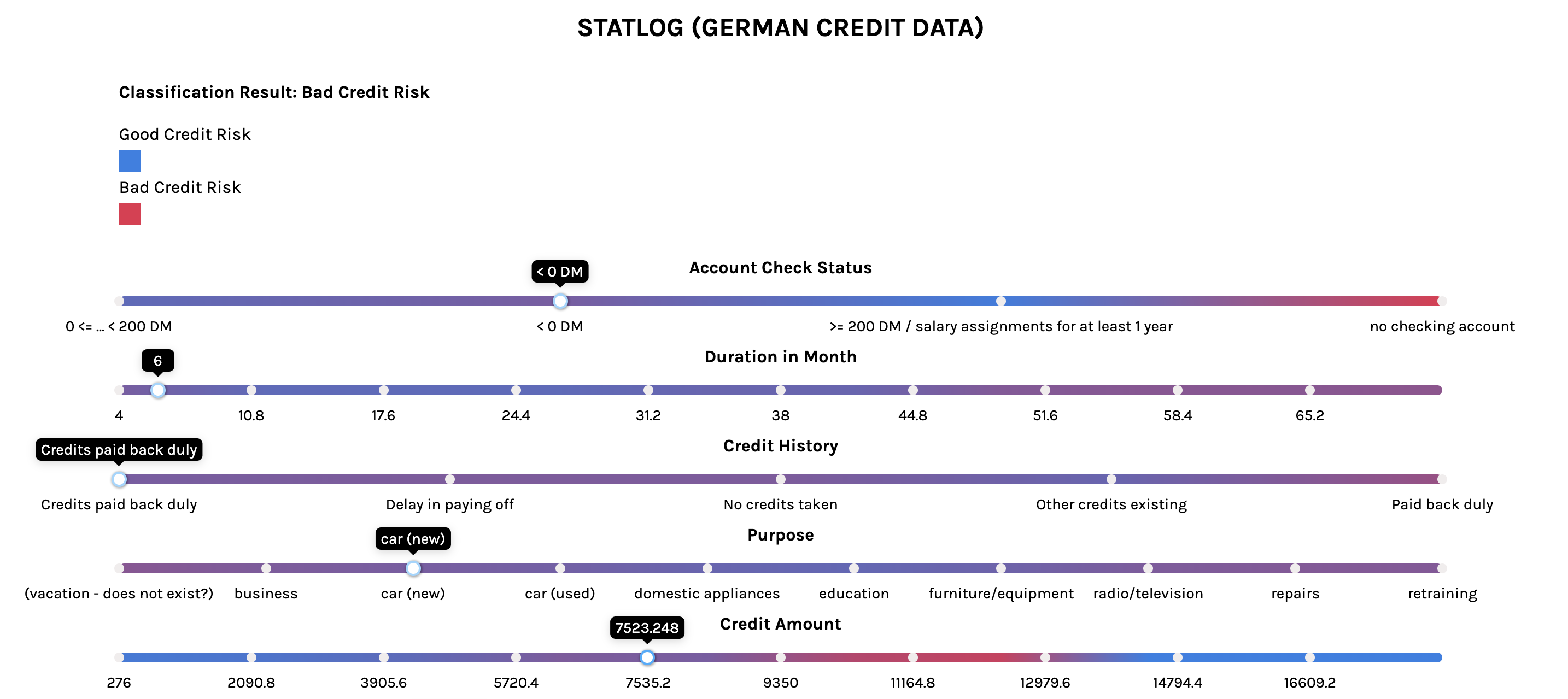}\label{fig:sub2_g}}\hskip1ex
\subfloat[User changes \textit{Duration in Month}; credit risk becomes good.]{\includegraphics[width=1\textwidth]{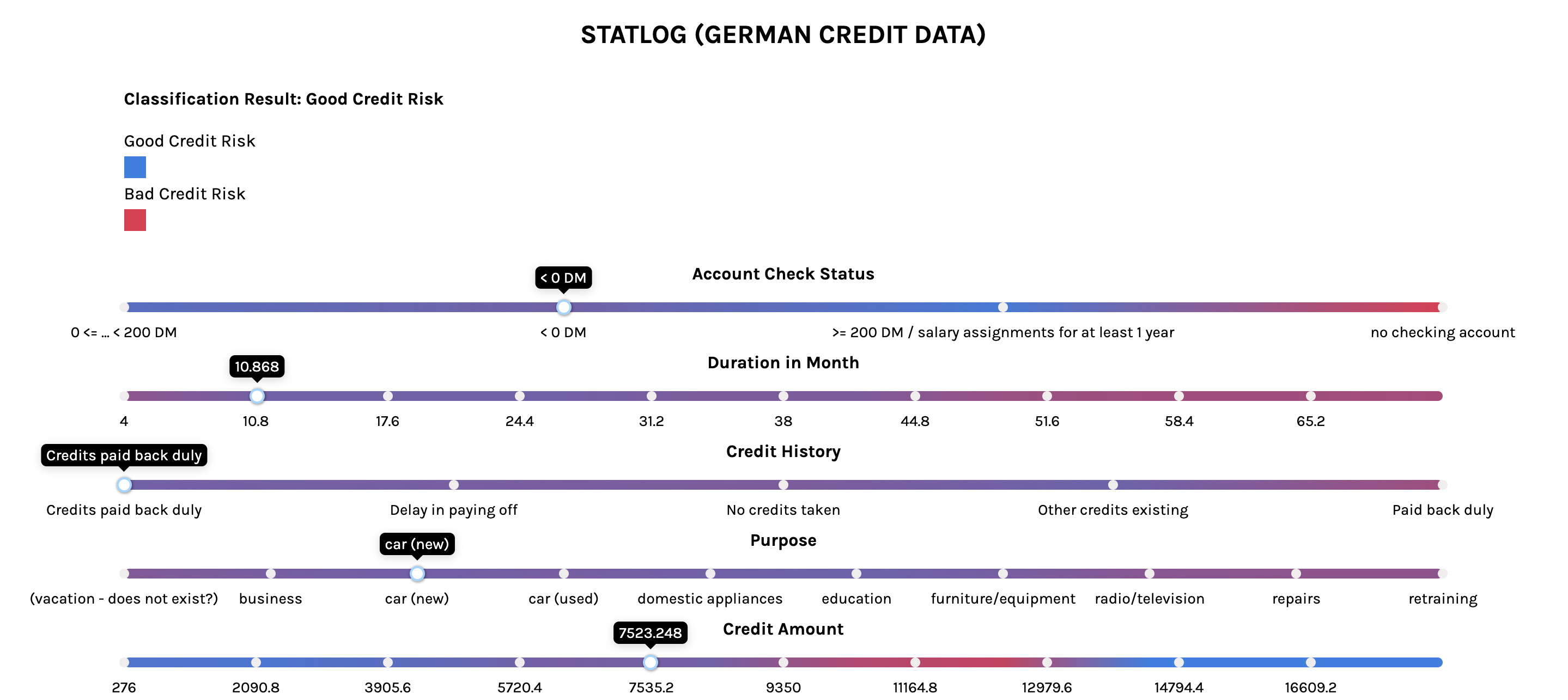}\label{fig:sub3_g}}
\caption{\textbf{\German Scenario.}}
\end{figure}

\begin{figure}[ht!]
\centering
\subfloat[Starting situation.]{\includegraphics[width=1\textwidth]{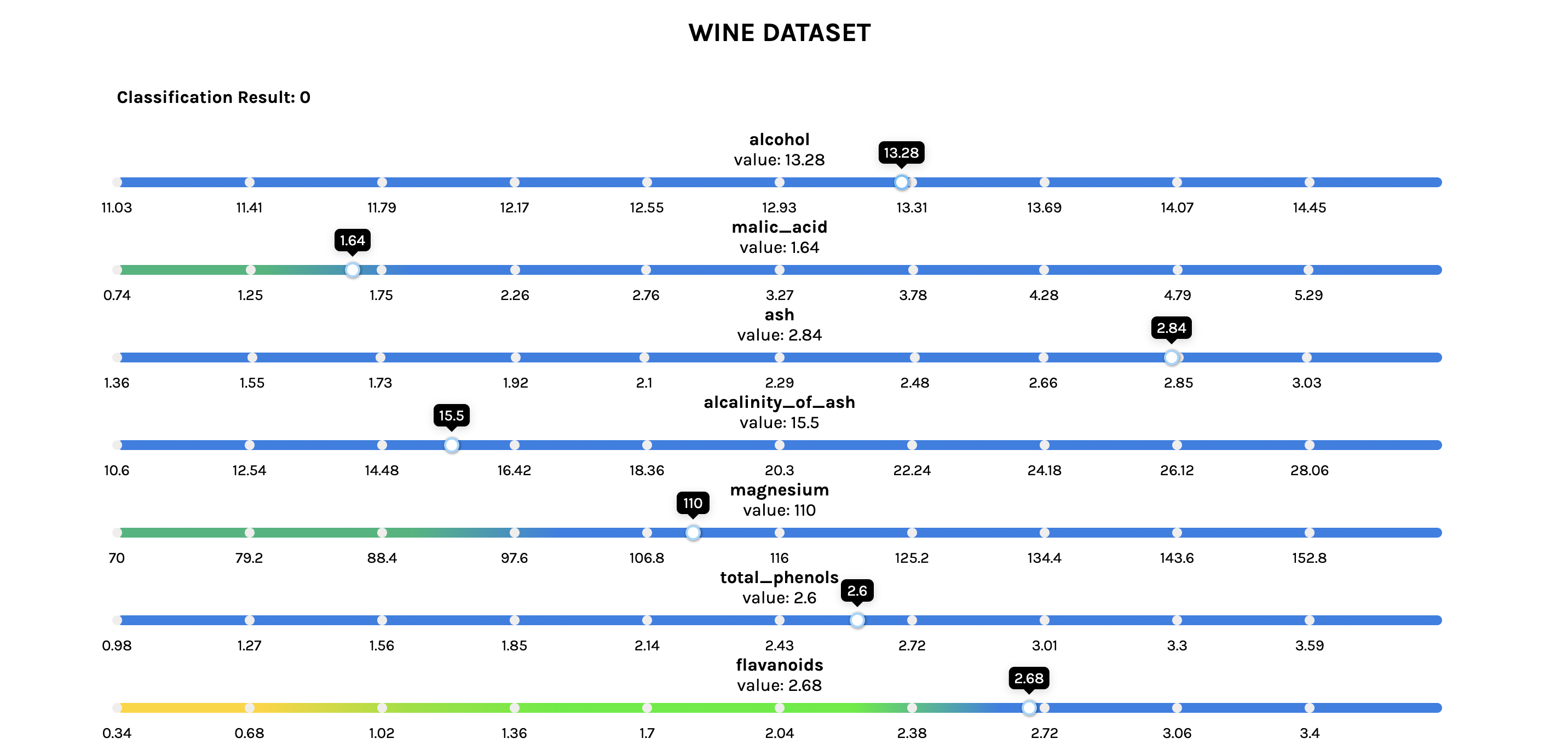}\label{fig:sub1_w}}\hskip1ex
\subfloat[User changes \textit{flavanoids}; wine class becomes $1$.]{\includegraphics[width=1\textwidth]{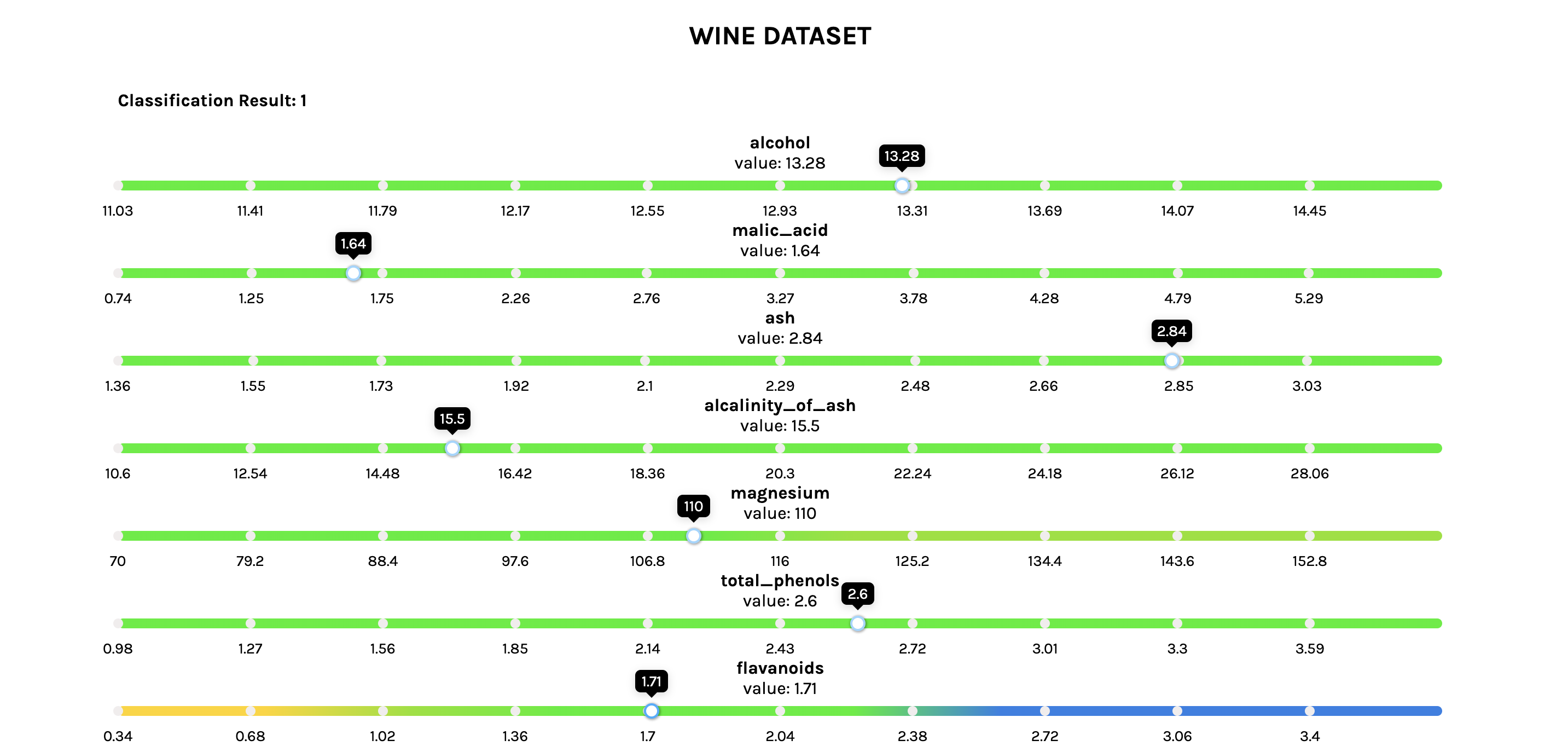}\label{fig:sub2_w}}\hskip1ex
\subfloat[User changes \textit{flavanoids}; wine class becomes $2$.]{\includegraphics[width=1\textwidth]{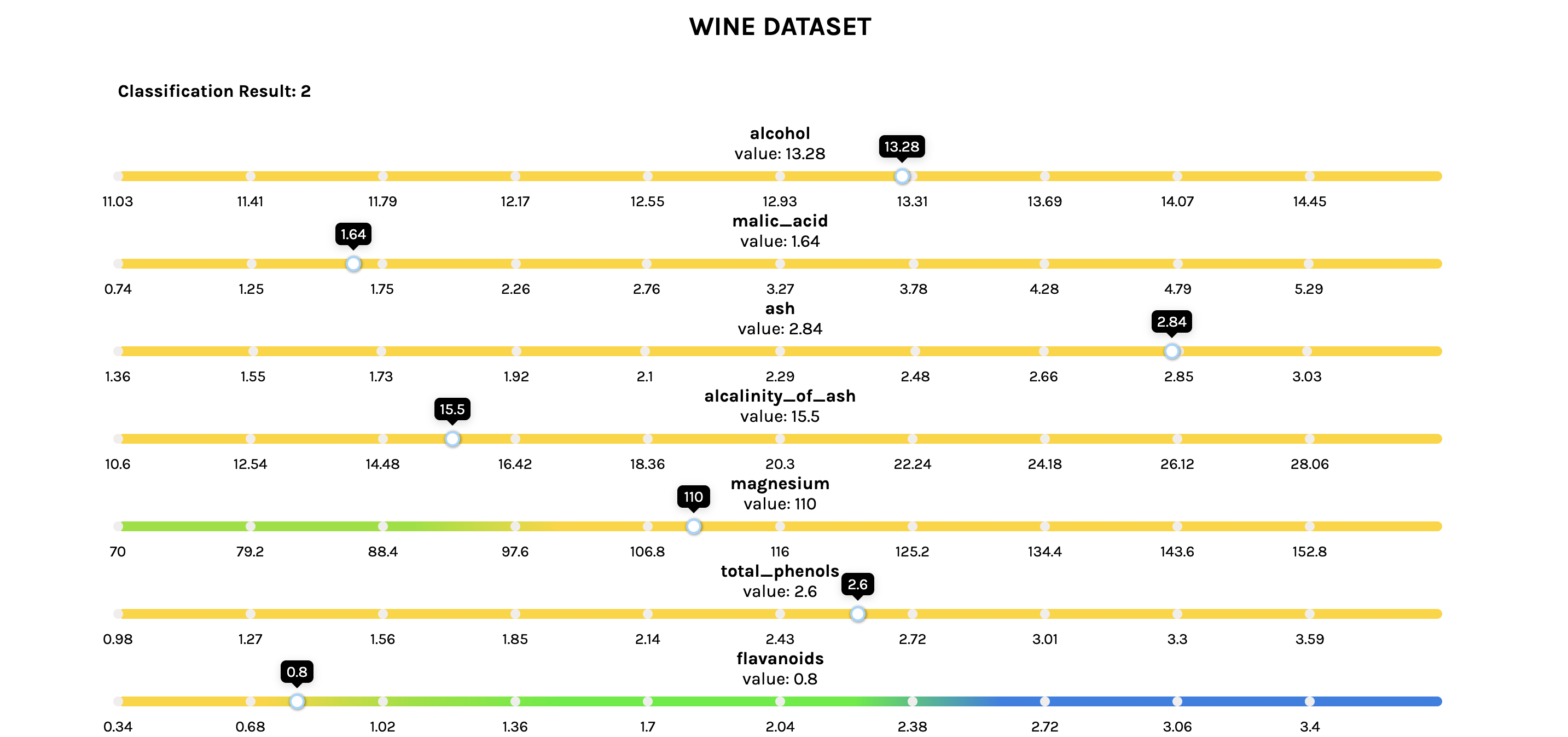}\label{fig:sub3_w}}
\caption{\textbf{\Wine Scenario.}}
\end{figure}

\section{Conclusions}\label{concl}
In this work, we have critically examined the limitations of the techno-solutionist approach to explanations in the context of counterfactual generation, reaffirming interactivity as a core value in the explanation interface between the model and the user.
We have developed \EMA, an interactive graphical interface that facilitates exploration of the feature space of the decision boundary of a black box, providing an intuitive visualization of the areas where the model produces the desired outcomes. 
This user-centric approach enhances transparency and control in the decision-making process, contributing to the advancement of \ac{XAI} methodologies.

\paragraph{Future Work}
Embracing a human-centered perspective, the evaluation requires a conceptual shift by involving users within it in order to understand whether the explanation is convincing for the intended user, the ultimate stakeholder in the process~\citep{DBLP:conf/hcomp/LiaoZLDD22,DBLP:conf/chi/WangVCC23}.
Concretely, a qualitative evaluation in the form of a user study is necessary. The study will provide a measure of human validation regarding the comprehensibility and usability of the interface, as well as the usefulness of exploration as a means of generating explanations for the underlying decision model. 
Indeed, since our approach does not assume the existence of a pre-established cost function, making a quantitative assessment of its optimality is not useful. Therefore, a direct comparison with algorithmic methods should be dismissed, in fact the evaluation of the method's effectiveness can not rely solely on data-driven metrics mentioned in Section~\ref{related}, as generating a counterfactual depends entirely on user interaction.

\paragraph{Limitations}
Although our approach delegates the responsibility of evaluating counterfactuals to the user, the issue of \textit{plausibility} still applies, at least in part. While users are generally equipped with common sense and contextual knowledge, they may still be confused when the system labels as "valid" or "positive" a counterfactual instance that appears implausible or even absurd from their perspective. For example, an individual might wonder why a suggestion that feels unachievable is considered a valid recourse.

This highlights a broader challenge: how to incorporate soft constraints or plausibility filters into our interactive system without compromising user freedom. While we have not fully addressed this aspect in the current version, we believe it is a promising direction for future work. Ideally, users would be able to explore freely while receiving subtle guidance from the system—perhaps through visual cues or lightweight warnings—when proposed changes deviate too far from realistic or consistent patterns.

Another important aspect is the internal consistency between feature values. Some feature combinations, though technically valid, may be incompatible in practice. Many algorithmic methods address this through statistical constraints or generative models that ensure counterfactuals lie within the data distribution. In the absence of such mechanisms, our system risks presenting users with combinations that, while permitted by the interface, may not make sense in real-world scenarios. Incorporating light-touch plausibility checks, while preserving the simplicity and transparency of the tool, is an open challenge.

Our method also assumes that users will tend to make fewer, more meaningful changes compared to automatic algorithms. This encourages \textit{sparsity}, which is generally desirable. However, as the number of editable features grows, the cognitive burden on users increases. Exploring a complex decision boundary manually can become tiring or confusing, especially for users without technical or domain-specific expertise. Providing mechanisms to highlight the most influential features or prioritize promising directions could help mitigate this challenge.

Finally, we acknowledge that counterfactual explanations carry broader limitations. They can be susceptible to manipulation or gaming, and may recommend actions that are formally valid but ethically questionable or socially undesirable. These concerns, raised in recent work~\citep{slack2021counterfactual,barocas2020assumptions,kasirzadeh2021use}, must be taken seriously in the design of responsible and trustworthy recourse systems.

%By embracing a user-centric perspective, the field of \ac{XAI} can overcome the drawbacks of a purely technologically driven perspective. Our approach holds the potential to not only enhance the interpretability and transparency of AI models~\citep{transparency_LarssonH20} but also foster trust and effective decision-making in human-AI interactions~\citep{trust_Nicodeme20}.
\chapter{Causal Algorithmic Recourse and Time}\label{ch:timecar}

\begin{tcolorbox}[colback=CTsemi!5!white, colframe=CTsemi, boxrule=0.5mm]
\textbf{Main Sources:} \fullcite{DBLP:conf/xai/BerettaC23}, \fullcite{beretta2024enhancing}.
\end{tcolorbox}

\paragraph{Author's Note} 
This chapter marks the beginning of a reflection on the role of time in causal models, initiated during the second year of the PhD. Although the theoretical framework presented in the included works represents an early version of what is further developed in \cref{pt:time_aware_causality}, and although the incorporation of the temporal dimension remains shallow, these contributions analyze key methodological limitations of existing \ac{AR} approaches and lay the foundation for a richer notion of actionability and fairness, rooted in time-aware recommendations.
\vspace{30pt}

\ac{AR} aims to provide actionable recommendations to reverse negative outcomes from automated decision-making systems~\citep{karimi2020survey}. 
A possible circumstance in which such recommendations are required is when a credit scoring model denies a loan to an applicant, and the individual desires to understand what should be different to change the outcome of the \ac{AI} system (i.e., to have the loan accepted). Since the individual’s attributes change may have downstream effects on other features, observing and identifying causal mechanisms is crucial in analyzing real-world scenarios to avoid infeasible actions. From this perspective, ~\citet{karimi2021, karimi2020algorithmic} propose a fundamental reformulation of the recourse problem, incorporating knowledge of causal relationships into the process of recommending recourse actions. The ability to assess the causal relationships explicitly guarantees plausible counterfactuals~\citep{byrne2019} and improves the user's perception of a decision's quality since it reflects the tendency of human beings to think in terms of cause-effect~\citep{pearl2018book}.

A significant limitation of state-of-the-art methods is their \textit{inability to incorporate the temporal dimension}. Neglecting the temporal interdependencies between features and actions can result in erroneous identification of the feature that requires modification at a particular moment, leading to ineffective or sub-optimal recourse actions. As a result, there is a need to devise causal \ac{AR} techniques that can incorporate temporal information to provide more reliable and plausible recommendations.
Hence, we investigate the usefulness of integrating the temporal dimension into a causal \ac{AR} problem by incorporating the topological information of the causal graph in the cost function evaluation. 
Besides, we discuss the necessity of interpreting the causal model as a representation of a dynamical process i.e., one that involves the evolution of its instances over time.
Lastly, we highlight the critical role of time in influencing the fairness of \ac{AR} and explore how temporal factors contribute to systemic inequities.

The rest of the chapter is organized as follows. 
Section~\ref{sec:contributions_timecar} outlines the main contributions. 
Section~\ref{sec:why_timecar} presents the essential background and elaborates on the importance of considering the time dimension along with a practical example. 
Section~\ref{sec:timecar} details the proposed approach. 
Section~\ref{sec:fairtime} discusses time unawareness in the domain of fairness, examining the implications of recommendations for minority populations.
Section~\ref{sec:exp_timecar} reports the experimental evaluation.
Finally, Section~\ref{sec:conclusions_timecar} summarizes the findings and highlights open research directions.

\section{Contributions}\label{sec:contributions_timecar}

This chapter makes the following contributions:

\begin{itemize}
    \item \textbf{Problem Identification:} We identify an important limitation in existing causal \ac{AR} methods - their failure to account for temporal dependencies - and argue that this omission undermines the realism and feasibility of recourse recommendations.
    
    \item \textbf{Temporal Cost Modeling:} We introduce the first method in the literature that explicitly incorporates the temporal dimension into the cost evaluation of recourse actions. Our approach leverages the topology of the causal graph to model time-aware action effects.
    
    \item \textbf{Actionability Revisited:} We propose a new interpretation of actionability that integrates time constraints, highlighting how similar interventions may carry vastly different temporal costs and how this affects the practicality of recourse.
    
    \item \textbf{Fairness Implications:} We analyze how neglecting time in \ac{AR} can lead to unfair outcomes, particularly for marginalized groups, and demonstrate that temporal considerations are critical for equitable algorithmic decision-making.
    
    \item \textbf{Empirical Insight:} Through a controlled semi-synthetic setup based on the German Credit dataset, we illustrate how existing time-unaware causal AR methods can produce suboptimal or impractical recommendations.
\end{itemize}

\section{Why Time Matters?}\label{sec:why_timecar}
\subsection{Preliminaries}\label{sec:timecar_preliminaries}
\paragraph{Action Cost}
Identifying the optimal action in the causal \ac{AR} problem necessarily requires defining a notion of \quotes{intervention cost}, typically using a function 
$ \costfunction: \Xspace \times \actionspace \to \R^+,$ where $\X \in \Xspace$ is the individual. 
We use the notation $\bm{a}_{\bm{\delta}} \in \actionspace$ to denote an action that changes $\X$ by an amount $\bm{\delta}$. The cost associated with actions that result in greater changes to $\X$ is intuitively expected to be higher. In other words, we expect the function $\costfunction$ to increase in $|\bm{\delta}|$ monotonically. The choice of $\costfunction$ determines the optimization outcome, regardless of whether the problem is formulated at the observational, interventional, or counterfactual level.
The most commonly used cost function in the literature is the $\ell_p$ norm (see \cref{sec:cost_choices}), which we denote by $\costfunction_{\ell_p}$. Specifically, in this case we have 
\[
\costfunction_{\ell_p}(\X,\bm{a}_{\bm{\delta}}) = \norm{\bm{\delta}}_p.
\]
To ensure scale invariance across the features of $\X$, $\costfunction_{\ell_p}$ is often replaced by a normalized variant:
\[
\costfunction_{\ell_pn}(\X,\bm{a}_{\bm{\delta}}) = \norm*{\frac{\bm{\delta}}{\bm{\sigma}^2}}_p, \quad \text{where} \quad \sigma^2_i = \Var[X_i]
\]
and the division between $\bm{\delta}$ and $\bm{\sigma}^2$ is understood component-wise.

Despite the constraints inherent in this formulation, it is widely considered a rational and viable choice, mainly due to the inherent challenges of formulating an effective cost function without access to supplementary information.

%\paragraph{Actionable Recourse}
%The distinction between \ac{AR} and the \textit{causality-aware} variant is defined by how the action $\mathbb{A}_\delta$ operates on a particular instance $X$. In the former $\mathbb{A}_\delta(X):=X+\delta$, whereas in the latter, the action is considered as a causal intervention
%$$\mathbb{A}_\delta(X):=\mathbf{F}_{\mathbb{A}_\delta}(X), \quad \text{where} \quad  \mathbf{F}_{\mathbb{A}_\delta}=\{X_i := f_i(\parents[i], E^{(i)}+\delta_i\}_{i=1}^{n}.$$

\subsection{The Role of Time}
In everyday experiences, we typically observe a temporal ordering between the cause and the effect, where the former precedes the latter. This relation could be exemplified by turning on a light switch in a room, where the action of flipping the switch is the cause of the light turning on. 
In the context of causal graphs applied to cross-sectional data, time is often ignored, 
leaving room for other notions of dependence between variables.
However, in the framework of \ac{AR}, it seems natural to include time as a relevant parameter in defining the cost of a specific action. 
We typically assume that \textit{a change in the value of one variable in the causal graph instantaneously affects the descendant variables}. 
In short, probability distributions, including interventional ones, represent a static and unchanging phenomenon of a fundamentally descriptive type. From another perspective, when considering %thinking of 
a physical system, its structural equations describe the system's behavior
in response to specific physical interventions, ultimately leading to a new and distinct equilibrium state.
However, the propagation of the effects of these interventions to the downstream variables may not occur immediately.
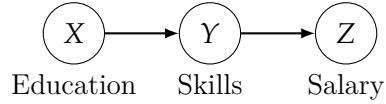
\begin{figure}[t]
    \centering
        \begin{tikzpicture}[
            scale=0.9,
            transform shape,
            vertex/.style={circle, draw, minimum size=0.9cm, inner sep=0pt}]
        	\node[vertex,label=below:Education] (x) at (-2,0) {$X$};
        	\node[vertex,label=below:Skills] (z) at (0,0) {$Y$};
        	\node[vertex,label=below:Salary] (y) at (2,0) {$Z$};
        	
        	\draw[edge] (x) -- (z);
        	\draw[edge] (z) -- (y);
        \end{tikzpicture} 
    \caption{\textbf{Simple Causal Graph}. A causal graph representing the relationship between college education, individual skills, and job salary.}
    \label{fig:causal_dag}
\end{figure}
For example, Figure~\ref{fig:causal_dag} reports a causal graph consisting of three variables, $X$, $Y$, and $Z$, representing a person's college education, skills, and job salary, respectively. 
%where $X$ stands for having a college education or not, $Y$ represents the individual's skill, $Z$ their job's salary. 
We can assume that the system is described by a linear model with additive noise, %by the structural equations
which can be expressed by the following structural equations:
$$ X:=\Noise_{X}, \quad Y:=a X+\Noise_{Y},\quad Z:=b Y+\Noise_{Z}.$$ 
where, $\Noise_{X}$, $\Noise_{Y}$, and $\Noise_{Z}$ represent noise terms, and $a$ and $b$ are constants. 
%The model states that college generally provides better preparation, often leading to a better salary. 
Such a model implies that a college education generally leads to better preparation and a higher salary. 
However, the process leading from $X$ to $Y$ can take time (in this case, years), but the model lacks this kind of temporal information 
and thereby is incapable of considering it. 
Suppose a person %that a particular individual 
wants to increase his salary $Z$ and queries the model for advice on achieving this goal. 
The alternatives are only two: the person can improve their skills $Y$ by taking a training course, or they can attend college $X$ to obtain skills as a result of the action. 
%The optimal action will likely turn out to be the former. 
The optimal action would likely be to take a training course. 
However, the model may not be able to detect this fact. 
In particular, if the coefficient $a$ that links $X$ to $Y$ is sufficiently large, according to~\citep{karimi2020algorithmic}, the optimal action would be to intervene on $X$, rather than on $Y$. 
%In general, we could expect that every time a node has many descendants and its effects are relevant, the intervention set, i.e. the variables to intervene on, will probably include it. In other words, \citep{karimi2020algorithmic} could have a bias towards \emph{root nodes}. Our goal is to identify this bias if present and be able to control it. 
Generally, whenever a node has many children and/or is the beginning of a long chain, it is likely to be part of the intervention set, \ie, the set of variables on which to intervene. 
%Generally, whenever a node has many descendants and  effects, the intervention set, which is the set of variables to intervene on, would likely include the node. 
This suggests that the actual formulation of the causal algorithmic recourse problem could be biased towards \emph{root nodes}.

\section{Time Weighted Causal Graphs}\label{sec:timecar}
A possible approach to consider the temporal aspect as a relevant parameter of the process is to use the topological information of the graph as an indicator of the time interval between the action and its consequences.
One way to integrate the model with the missing information is to manually assign additional weights to each edge of the graph, incorporating information about the \textit{characteristic response time} of the child variable w.r.t. the change of the parent variable.
This allows for analyzing the dynamic flow of the causal effect over time.

Given the \ac{DAG} $\graph$, we reformulate the cost function of an action $\bm{a}_{\bm{\delta}}$ over $\X$ \textit{to obtain an effect over $Y$}, as 
\begin{equation}\label{eq:timecar}
    \costfunction(\X,\bm{a}_{\bm{\delta}},Y) = c_{s}\left(\X,\bm{a}_{\bm{\delta}}\right) + \lambda c_{t(\graph)}\left(S(\bm{a}_{\bm{\delta}}),Y\right),
\end{equation}
where $S(\bm{a}_{\bm{\delta}})$ is the support of $\bm{a}_{\bm{\delta}}$, i.e. the set of variables \textit{directly} modified by $\bm{a}_{\bm{\delta}}$. 
%We added the $\lambda$ as a free parameter to tune explicitly the balance between the two parts of the cost function, $C_s$ that represents the cost function in the features' space and $C_t$ that indicates the \textit{time} part. 
We add the parameter $\lambda$ as a free variable to explicitly adjust the balance between the two components of the cost function, namely $c_s$ that denotes the cost function in the feature space and $c_{t(\graph)}$ that reflects the time part.
%Note that, with some abuse of notation, 
It should be noted that the parameters of $c_s$ represent %are intended as 
features and their values, while $c_{t(\graph)}$ involves the topological properties of the graph and thus depends on nodes and their graph relationships. In the following, we provide an example of a cost function considering the temporal dimension:
$$ \costfunction(X,\bm{a}_{\bm{\delta}},Y) = \norm{\bm{\delta}}_2 + \lambda \sup_{V\in \mathcal{S}(\bm{a}_{\bm{\delta}})} d_{lp}^\graph(V,Y)$$
where $d_{lp}^\graph(V,Y)$ %being 
denotes the longest path distance between node $V$ and $Y$ over $\graph$.

~\cref{fig:graphexample} illustrates the graphical example of the above formula. 
Assuming for simplicity that each edge has %the same weight ($1$)
a weight of $1$, we consider an action over the set $\{W,X\}$, observing different directed paths towards $Y$. 
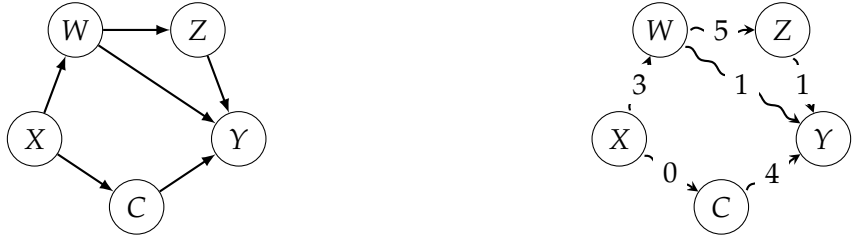
\begin{figure}[!t]
\centering
\begin{subfigure}{0.45\textwidth}
    \centering
    \begin{tikzpicture}[
        scale=0.9,
        transform shape,
        vertex/.style={circle, draw, minimum size=0.8cm, inner sep=0pt}]
        \node[vertex] (x) at (0,-0.3) {$X$};
        \node[vertex] (w) at (0.6,1.3) {$W$};
        \node[vertex] (z) at (2.4,1.3) {$Z$};
        \node[vertex] (c) at (1.5,-1.3) {$C$};
        \node[vertex] (y) at (3,-0.3) {$Y$};
        
        \draw[edge] (x) -- (w);
        \draw[edge] (w) -- (z);
        \draw[edge] (w) -- (y);
        \draw[edge] (z) -- (y);
        \draw[edge] (x) -- (c);
        \draw[edge] (c) -- (y);
    \end{tikzpicture}
    %\caption{Causal graph}
\end{subfigure}
\hfill
\begin{subfigure}{0.45\textwidth}
    \centering
    \begin{tikzpicture}[
        scale=0.9,
        transform shape,
        vertex/.style={circle, draw, minimum size=0.8cm, inner sep=0pt}]
        \node[vertex] (x) at (0,-0.3) {$X$};
        \node[vertex] (w) at (0.6,1.3) {$W$};
        \node[vertex] (z) at (2.4,1.3) {$Z$};
        \node[vertex] (c) at (1.5,-1.3) {$C$};
        \node[vertex] (y) at (3,-0.3) {$Y$};
        
        \draw[snakedge] (x) -- (w)
        node[scale=1,midway,fill=white] {$3$};
        \draw[snakedge] (w) -- (z)
        node[scale=1,midway,fill=white] {$5$};
        \draw[snakedge] (w) -- (y)
        node[scale=1,midway,fill=white] {$1$};
        \draw[snakedge] (z) -- (y)
        node[scale=1,midway,fill=white] {$1$};
        \draw[snakedge] (x) -- (c)
        node[scale=1,midway,fill=white] {$0$};
        \draw[snakedge] (c) -- (y)
        node[scale=1,midway,fill=white] {$4$};
    \end{tikzpicture}
    %\caption{Time-weighted causal graph}
\end{subfigure}
\caption{\textbf{Causal Graph vs Time Weighted Causal Graph}. (Left) An example of a causal graph. (Right) A possible way to incorporate time information over the causal graph. A wavy edge weight $\tau_{ij}$ is meant to represent time between intervention over node $i$ and the observed effect on node $j$.}
\label{fig:graphexample}
\end{figure}

By applying $\costfunction$ and focusing on the \emph{time} part of the formula, we obtain $$c_t = \sup_{V\in \{X,W\}}d_{lp}(V,Y) = d_{lp}(X,Y) = 3.$$ 
The $\sup$ operation is motivated by %this 
the following observation: 
%if we act over some variables and we interpret 
when we act on some variables and interpret the process as a dynamic information flux from them to the target, the total causal effect will be observed after every single process finishes, equivalently, after the last one does.
This proposal for $c_t$ may be overly simplistic and lacking flexibility, as adding a single edge with a high weight can radically change the cost function regardless of its causal effect. 
A less rigid formulation of $c_t$ could consider the average response time across the multitude of causal processes involved, each of which can be represented by a causal path between the intervention variable and the target. 
The average could be weighted by the causal impact of each process, ensuring that minor processes do not excessively skew the evaluation.

Consider the case of linear models with additive noise characterized by \acp{SCM} of the following form:
$$ \Xcomp[j] := \sum_{X_i \in \parents[j]} \beta_{ij} X_i + \Noise_j.$$
In this setting, the path impact can be assessed by calculating the product of the coefficients $\beta_{ij}$ associated with their corresponding edges $(i,j)$. 
Furthermore, the total causal effect of a cause $X$ over $Y$ is the sum of the effects attributed to the paths between $X$ and $Y$.
Suppose we also have information regarding the response times associated with edges in $\graph$, as depicted in Figure \ref{fig:timegraph}. We denote these quantities as $\tau_{ij}$.
Considering the set of all these paths and defining the weight of each path $\pi$ as $w_\pi := \prod_{(i,j) \in \pi}|\beta_{ij}|$, we define the following:
$$ c_t(X,Y) = \frac{1}{Z_{XY}}\sum_{\pi | X\overset{\pi}{\to }Y}w_\pi t_\pi, \quad Z_{XY} := \sum_{\pi | X\overset{\pi}{\to }Y}w_\pi, \quad t_\pi := \sum_{(i,j) \in \pi}\tau_{ij}.$$
In brief, $t_\pi$ represents the propagation time of the causal effect \textit{through} $\pi$, and $c_t$ is the weighted average of $t_\pi$, weighted by the relative importance of the causal effect of each path on the overall process. 
The value of $c_t$ in this revised formulation intuitively represents the time at which a significant portion of the causal effect becomes observable, rather than necessarily capturing its entirety. 
In this sense, it offers greater robustness and flexibility compared to a cost that relies solely on calculating the longest path.

\paragraph{General Remarks} 
We have presented some formulations incorporating temporal dimension within the causal \ac{AR} framework. 
However, it is essential to recognize that a general solution is insufficient in resolving the issue across diverse contexts. Instead, deeper considerations must be given to the unique demands of the user and the specific properties inherent in the problem being addressed. 

If the user needs to complete the action within a time constraint, \eg{}, purchase a house and move in within two months of applying for a loan, instead of directly adding the term $c_t$ to the total cost, an alternative approach could be to use it as a constraint in the optimization problem. This approach would redefine the problem as finding the most cost-effective action, where the effects are achieved within a predetermined maximum time, leading to a favorable decision.

Moreover, the problem could be such that $\tau_{ij}$ is not a predetermined fixed value, but rather it may depend on the specific instance under consideration. For instance, in the context of achieving fitness through a diet and exercise program, the required time would vary based on factors such as the user's age, gender, and weight.
In such cases, it would be necessary to enrich the data structure and fit $\tau$ based on available data. 

Lastly, the nature of the data can have a significant impact on the formulation of the problem. As an example, where data provides relevant information regarding the temporal dimension, the choice of $c_t$ may inherently depend on it. In particular, when working with a time series dataset, it may be possible to construct a more nuanced $c_t$ that aligns with specific practical contexts.

Considering these aspects, it follows that the issue is both complex and multifaceted. Although fully addressing it is beyond the scope of this work, developing effective methods will require applying our proposal to real-world scenarios.

\section{Time Awareness for Fairer Recommendations}\label{sec:fairtime}
Building on the formulation proposed in~\cref{sec:timecar}, this section extends the discussion to the domain of fairness, specifically examining the implications of recommendations for marginalized demographic groups.
Indeed, long wait times for implementing recourse actions can profoundly impact these communities' financial stability, educational opportunities, and overall well-being. Such delays exacerbate pre-existing inequalities, perpetuating cycles of disadvantage that are challenging to overcome. 
In the following, we highlight the critical role of time in influencing the fairness of \ac{AR}, particularly for minority populations, and explore how temporal factors contribute to systemic inequities. Finally, we propose a novel interpretation of actionability that incorporates time constraints, ensuring that recommendations remain feasible within predefined temporal limits. 

\begin{figure}[t!]
  \centering
  \begin{subfigure}[t]{0.5\textwidth}
    \includegraphics[width=\textwidth]{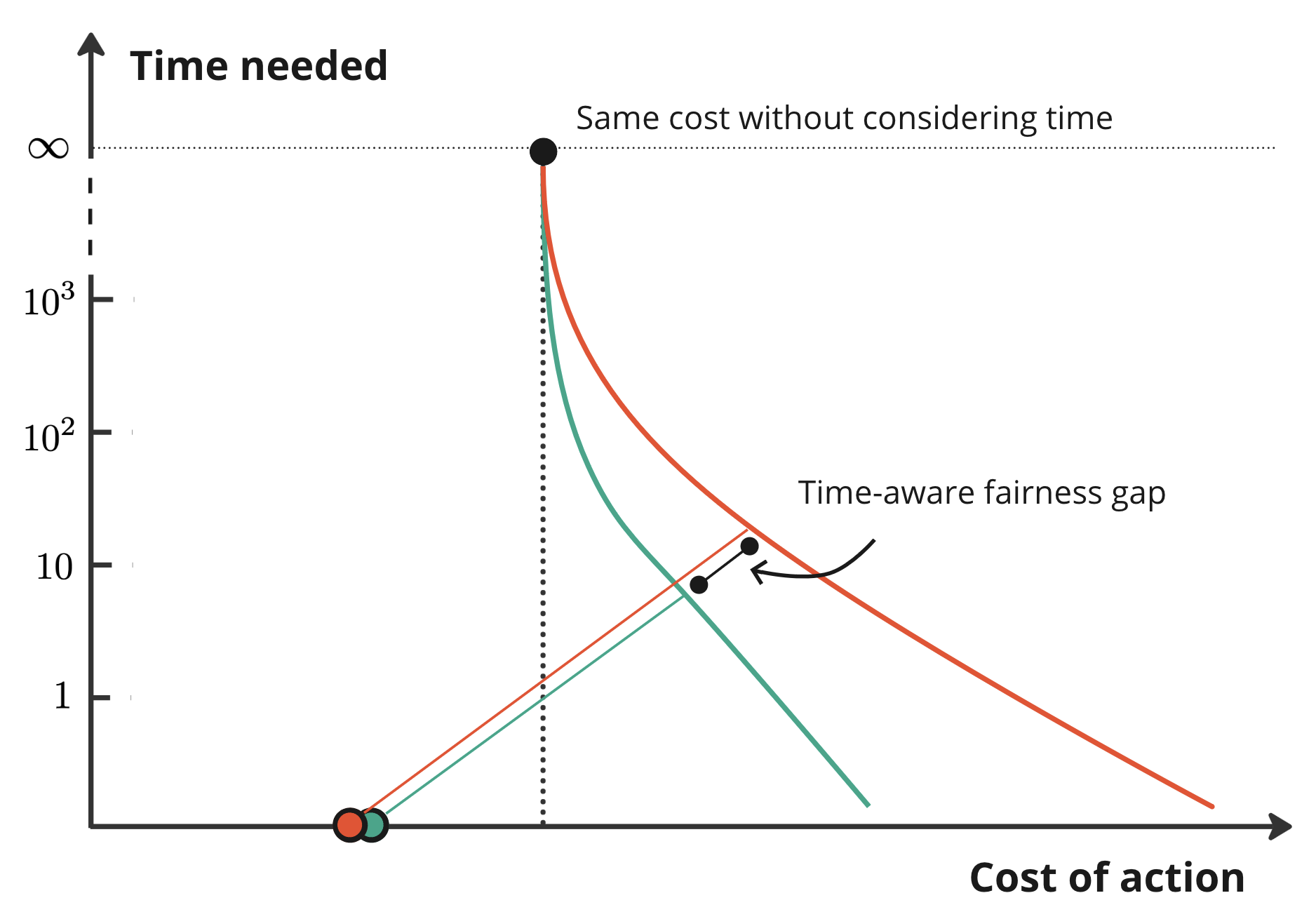}
    \caption{Not accounting for time could introduce hidden biases in recommendation algorithms.
    }\label{fig:same_cost}
  \end{subfigure}%
  \hfill
  \begin{subfigure}[t]{0.42\textwidth}
     \raisebox{2.8mm}
    {\includegraphics[width=\textwidth]{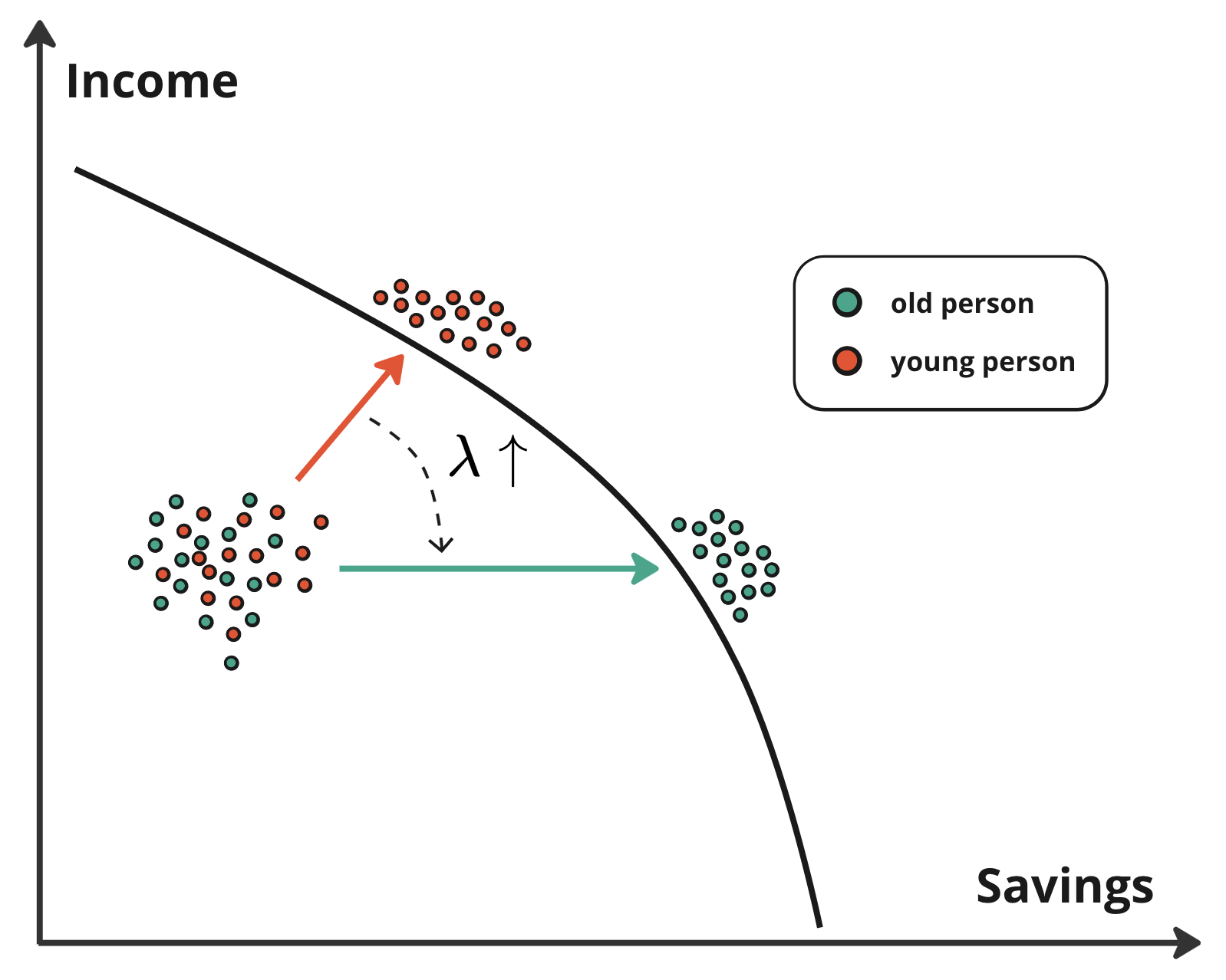}}
    \caption{Let $cost(\delta) = c_s(\delta) + \lambda \bm{1}(\delta_{income} \neq 0)$, it could be that $\lambda$ is depending on \textit{Age}.}\label{fig:old_young}
  \end{subfigure}
  \caption{\textbf{Time Awareness Scenarios for Fair Recommendations.} Possible scenarios where the time cost matters from a fairness perspective. 
  (a) An example of how similar feature costs can result in significantly different temporal costs. (b) An example of time being considered as a resource whose value varies based on individual characteristics.}
\end{figure}

\paragraph{Same Cost, Different Times}

Given the formulation in~\cref{eq:timecar}, we consider a scenario in which \textit{sensitive attributes} are included among the features, denoted by $\mathcal{S} \subset \Xspace$ (with a slight abuse of notation, we use the same symbol for both the feature space and the set of features). We examine the case of two individuals, $\x_1$ and $\x_2$, who differ in their sensitive attributes, such that $\bm{s}_1 \neq \bm{s}_2$.
We hypothesize that the cost of recommendations produced by the time-unaware automatic decision system for these individuals, $c_s(\x_1, \bm{a}_{\bm{\delta_1}})$ and $c_s(\x_2, \bm{a}_{\bm{\delta_2}})$, are approximately equal. This suggests that, despite differences in sensitive attributes, the system proposes similar intervention costs for both individuals. However, the temporal cost $c_t$ may differ significantly between them—meaning one individual might require more time than the other to reach the desired outcome. This scenario is illustrated in~\cref{fig:same_cost}.

\paragraph{Not Everyone Values Time Equally}
In another scenario, time may be regarded as a resource whose value varies based on individual characteristics. ~\cref{fig:old_young} demonstrates this idea through a specific case related to applying for a loan. The value of $\lambda$ in~\cref{eq:timecar} might be higher for the older population as they are likely closer to retirement and have a limited window to recuperate from financial setbacks. Conversely, younger individuals might have a lower $\lambda$ value given their longer time horizon to adjust their savings behavior.
Hence, financial models must be calibrated to accommodate varying $\lambda$ values across different demographic segments. This understanding enables the creation of customized recommendations sensitive to each individual's dynamics and the time-related evaluation of changes within their societal and economic contexts.

\begin{figure}[t!]
    \centering
\includegraphics[width=1\textwidth]{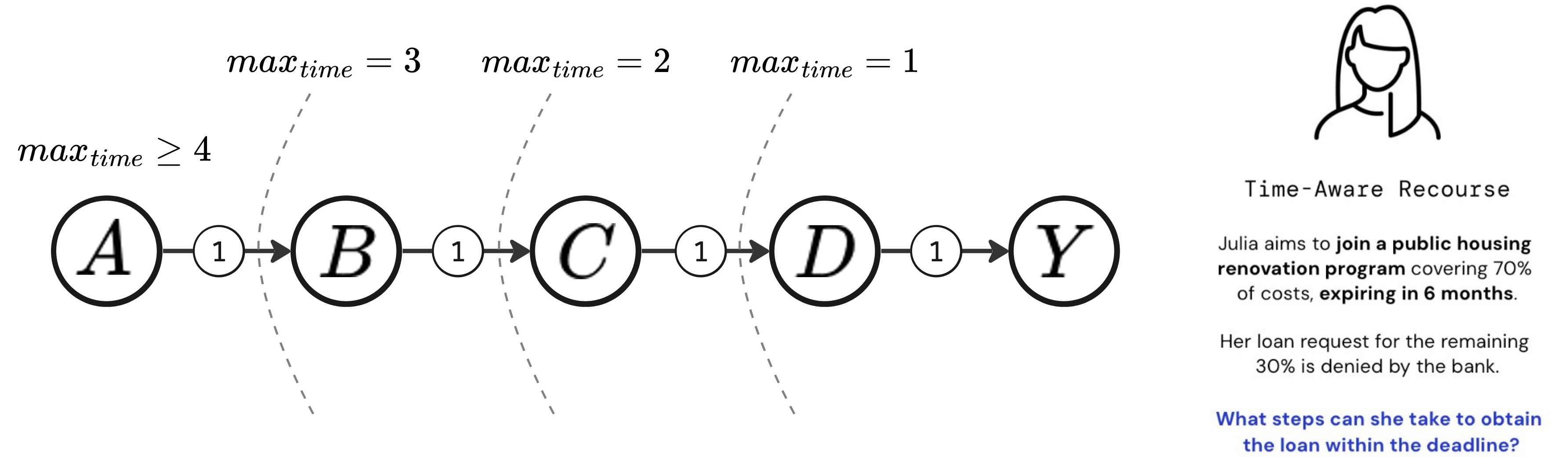}
    \caption{\textbf{Actionability as Time Constraint.} (Left) As $max_{time}$ decreases, the graph undergoes a pruning process that reduces the number of actionable variables. For example, only C and D will be actionable if the user specifies a maximum time of $2$ years for the request. (Right) A real-life application is one where an individual has finite time available.}
    \label{fig:actionable_dag}
\end{figure}

\paragraph{Actionability as a Time Constraint} \label{sec:broader}
%In the context of \ac{AR}, \textit{plausibility} refers to the perceived consistency and reasonableness of the recommendations provided by recourse approaches. 
%
%From a psychological perspective, providing plausible explanations enables users to form mental models that align with their prior knowledge~\citep{beretta2022}. 
%
%When the temporal dimension is incorporated into causal reasoning, an \ac{AR} approach could ensure that the actions suggested are psychologically congruent with human intuitions and mental frameworks. 
%
%This compatibility fosters a sense of trust and confidence in the algorithmic system, facilitating user acceptance and engagement. 
\textit{Actionability} is considered one of the crucial aspects in a counterfactual generation process, as highlighted in~\citet{guidotti2022counterfactual, , DBLP:journals/csur/KarimiBSV23}. 
We propose expanding the concept beyond the notion of \textit{being able to act upon} to include \textit{the ability to do so within a reasonable timeframe}   (~\cref{fig:actionable_dag}). 
In fact, if the action required to implement a recommendation is excessively time-consuming or impractical, the recommendation becomes unhelpful for the user. 
In the constrained optimization framework of \ac{AR}, the actionability threshold is directly controlled by the maximum time constraint, denoted as $max_{time}$. 
This parameter can be determined a priori or \textit{adapted based on the user's specific requirements each time a request is made.} 
In the latter case, the parameter enables personalized control over actionability for the applicant. 
From this perspective, we propose a new interpretation of \textit{fair} recommendations expressed as follows: 
\textit{a time-aware algorithmic recourse model is fair if its recommendations remain fair under any fixed time constraint.}

\section{Empirical Evaluation}\label{sec:exp_timecar}
\subsection{Research Questions}
To evaluate our proposed approach, we aim to answer the following questions:

\begin{itemize} 
    \item \textit{Q1}: How does the action cost influence the selection of actionable variables? 
    \item \textit{Q2}: Does the current formulation of causal \ac{AR} exhibit a bias favoring interventions on root nodes in the causal graph?
\end{itemize}

The code and the dataset are available at \url{https://github.com/marti5ini/time-car}. Experiments were run on Ubuntu 20.04LTS, 252GB RAM, 3.30GHz x 36 IntelCore i9.

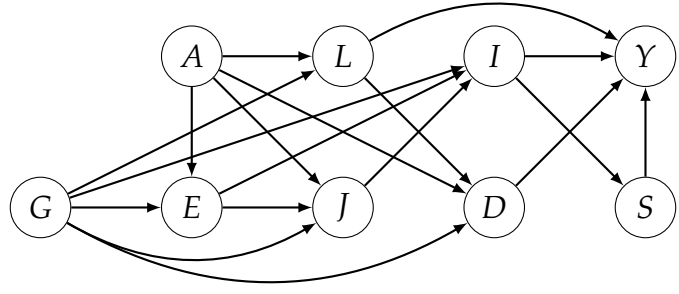
\begin{figure}[!t]
\centering
\begin{tikzpicture}[scale=1, transform shape,
        vertex/.style={circle, draw, minimum size=0.8cm, inner sep=0pt}]
    \node[vertex] (g) at (-2,0) {$G$};
    \node[vertex] (a) at (0,2) {$A$};
    \node[vertex] (e) at (0,0) {$E$};
    \node[vertex] (j) at (2,0) {$J$};
    \node[vertex] (l) at (2,2) {$L$};
    \node[vertex] (d) at (4,0) {$D$};
    \node[vertex] (i) at (4,2) {$I$};
    \node[vertex] (s) at (6,0) {$S$};
    \node[vertex] (y) at (6,2) {$Y$};
    
    \draw[edge] (g) -- (e);
    \draw[edge] (a) -- (e);
    \draw[edge] (g) -- (l);
    \draw[edge] (a) -- (l);
    \draw[edge] (e) -- (j);
    \draw[edge] (a) -- (j);
    \draw[edge] (g) to [out=-30,in=210] (j);
    \draw[edge] (g) to [out=-30,in=210] (d);
    \draw[edge] (a) -- (d);
    \draw[edge] (l) -- (d);
    \draw[edge] (j) -- (i);
    \draw[edge] (g) -- (i);
    \draw[edge] (e) -- (i);
    \draw[edge] (i) -- (s);
    \draw[edge] (l) to [out=30,in=150] (y);
    \draw[edge] (d) -- (y);
    \draw[edge] (i) -- (y);
    \draw[edge] (s) -- (y);
\end{tikzpicture}
\caption{\textbf{German Credit Causal Graph}. The assumed causal graph for the German Credit dataset. Nodes are labeled with the initials of their corresponding variables, namely \textit{Gender, Education, Age, Job, Loan Amount, Loan Duration, Income} and \textit{Saving}. The node Y denotes the \textit{Credit Score}.}
\label{fig:german}
\end{figure}

\subsection{Experimental Settings}\label{sec:exp_setup}
We consider a semi-synthetic \ac{SCM} based on the German Credit dataset\footnote{\url{https://www.kaggle.com/datasets/uciml/german-credit}}. While the corresponding causal graph is shown in~\cref{fig:german}, the loan approval \ac{SCM} consists of the following structural equations and noise distributions: 
\begin{center}
    \resizebox{0.95\textwidth}{!}{%
        \begin{minipage}{\textwidth}
            \begin{align*}
            (Gender) \qquad G &:=\Noise_G, \qquad &\Noise_G \sim \operatorname{Bernoulli}(0.5) \\[0.3em]
            (Age) \qquad A &:=-35+\Noise_A, \qquad &\Noise_A \sim \operatorname{Gamma}(10,3.5)\\[0.3em]
            (Education) \qquad E &:= G+ A +\Noise_E, \quad &\Noise_E \sim \mathcal{N}(0,1)\\[0.3em]
            (Job) \qquad J &:= G + 2 A + 4 E + \Noise_J, \quad &\Noise_J \sim \mathcal{N}(0,2)\\[0.3em]
            (Loan\; Amount) \qquad L&:= A + 0.5 G+ \Noise_L, \quad &\Noise_L \sim \mathcal{N}(0,3)\\[0.3em]
            (Loan\; Duration) \qquad D&:=G - 0.5 A + 2L+\Noise_D, \quad &\Noise_D \sim \mathcal{N}(0,2)\\[0.3em]
            (Income) \qquad I&:=0.5G + A + 4 E + 5 J + \Noise_I, \quad &\Noise_I \sim \mathcal{N}(0,4)\\[0.3em]
            (Savings) \qquad S&:= 5 I +\Noise_S \quad &\Noise_S \sim \mathcal{N}(0,2)
            \end{align*}
         \end{minipage}
    }
\end{center}
The \textit{Credit Score} $Y$ is obtained according to:
$$
Y \sim \sigma (2I+3S-L-D).
$$

We model \textit{Age}, \textit{Gender} and \textit{Loan Duration} as non-actionable variables but consider the latter to be mutable, i.e., it cannot be manipulated directly but is allowed to change (e.g., as a consequence of an intervention on \textit{Loan Amount}).
Our analysis specifically focuses on \textit{Education} and \textit{Income}, as we suppose that these two variables may display the most significant bias. The reason behind our belief is that both variables lie on the same causal pathway. Moreover, from a semantic perspective, one would anticipate that education is a way slower process than an income increase. A potential illustration of this idea is depicted in~\cref{fig:timegraph}, in which edges starting from \textit{Education} have quite significant weights compared to the others.
Notably, the structural equations used in \citep{karimi2020algorithmic} differ from those defined in our work. 
The rationale for the variation is rooted in the observation that the \ac{SCM} used had excessive noise to each variable, thus rendering the causal effect of any reasonable intervention practically irrelevant and precluding the possibility of testing alternative scenarios. 
Moreover, while the original structural equations accounted for both linear and nonlinear relationships, it was deemed sufficient for us only to consider the linear ones. 
Indeed, for the scope of our work, studying systems other than linear was unnecessary because the time dimension is independent of the \ac{SCM} form.
Using the loan approval \ac{SCM} described above, we have generated a synthetic observational dataset $\mathcal{D}$ with $10000$ samples. For each feature $\Xcomp[i]$, we applied an intervention described by: 
$$\Xcomp[i] \rightarrow \Xcomp[i] + \alpha \hat{\sigma}_{i}, \quad \text{where} \quad \alpha \in \mathbb{R}^+, $$ 
generating for each intervention an interventional dataset $\mathcal{D}_{\Xcomp[i]}$\footnote{Note that the cost of the action is equal to $\frac{\alpha\hat{\sigma}_i}{\hat{\sigma}_i}=\alpha$, regardless of the variable on which the action is executed.} Given a fixed intervention cost, our goal is to determine which action yields the greatest impact on the target variable $Y$.
To this end, we estimate the Causal Effect Derivative (\texttt{CED}), which, in the linear case, can be interpreted as the causal effect produced by a unit-cost intervention on a specific variable.
The \texttt{CED} is computed using the following finite-difference expression, which captures the effect-to-cost ratio of the intervention:

\begin{equation}\label{eq:estimate} \frac{\overline{Y}_{\mathcal{D}_{\Xcomp[i]}} - \overline{Y}_{\mathcal{D}}}{\alpha}. \end{equation}

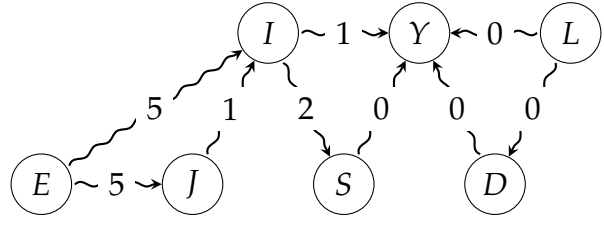
\begin{figure}[!t]
\centering
\begin{tikzpicture}[scale=1, transform shape,
        vertex/.style={circle, draw, minimum size=0.8cm, inner sep=0pt}]
    \node[vertex] (e) at (-2,0) {$E$};
    \node[vertex] (l) at (5,2) {$L$};
    \node[vertex] (j) at (0,0) {$J$};
    \node[vertex] (d) at (4,0) {$D$};
    \node[vertex] (i) at (1,2) {$I$};
    \node[vertex] (s) at (2,0) {$S$};
    \node[vertex] (y) at (3,2) {$Y$};
    \draw[snakedge] (e) -- (i)
    node[scale=1,midway,fill=white] {$5$};
    \draw[snakedge] (e) -- (j)
    node[scale=1,midway,fill=white] {$5$};
    \draw[snakedge] (j) -- (i)
    node[scale=1,midway,fill=white] {$1$};
    \draw[snakedge] (i) -- (s)
    node[scale=1,midway,fill=white] {$2$};
    \draw[snakedge] (d) -- (y)
    node[scale=1,midway,fill=white] {$0$};
    \draw[snakedge] (i) -- (y)
    node[scale=1,midway,fill=white] {$1$};
    \draw[snakedge] (s) -- (y)
    node[scale=1,midway,fill=white] {$0$};
    \draw[snakedge] (l) -- (y)
    node[scale=1,midway,fill=white] {$0$};
    \draw[snakedge] (l) -- (d)
    node[scale=1,midway,fill=white] {$0$};
\end{tikzpicture}
\caption{\textbf{Actionable Weighted Causal Graph}. German Credit actionable causal graph where the edge labels represent the response times for parent-child relationships.}
\label{fig:timegraph}
\end{figure}

\subsection{Results}

\paragraph{Q1: Determining the Optimal Cost}
As stated in~\cref{sec:why_timecar}, the cost function used in~\citet{karimi2021} is the normalized $\ell_1$ norm. 
Such normalization is required because if the distributions of the variables being manipulated have different scales, this could result in different costs. ~\cref{fig:costo} depicts a possible scenario that describes the phenomenon mentioned above. The blue-colored distribution has a smaller variance, while the red has a larger one. 
As a result, the intervention cost in the case of the first distribution is higher than that of the second, given the same amount of $\delta$.
Although this formulation is commonly used, we believe it may not be sufficient, given that it entails critical aspects that we endeavor to explain below.

For instance, we consider a simple \ac{SCM} system, such as the linear Gaussian model. In this setting, each variable $\Xcomp[i]$ is defined by the structural equation
\[
\Xcomp[i] := \sum_{\Xcomp[j] \in \parents[i]} \beta_{ji} \Xcomp[j] + \Noise_i, \quad \Noise_i \sim \mathcal{N}(0,\hat{\sigma}^2_i), \quad \beta_{ji} \in \mathbb{R}.
\]
Here, $\hat{\sigma}^2_i$ represents the variance of the exogenous noise term $\Noise_i$, which we refer to as the \textit{proper variance} of $\Xcomp[i]$. 

The total variance of $\Xcomp[i]$, denoted $\sigma^2_i = \Var[\Xcomp[i]]$, includes both the contribution from the noise and the influence of its parents. Letting $\bm{\beta}_i$ be the vector of structural coefficients $\beta_{ji}$ \textit{restricted to the set of parents of} $\Xcomp[i]$, and $\bm{\Sigma}_{\parents[i]}$ the covariance matrix of the parent variables, the variance of $\Xcomp[i]$ can be compactly expressed as
\[
\sigma^2_i = \bm{\beta}_i^\top \bm{\Sigma}_{\parents[i]} \bm{\beta}_i + \hat{\sigma}^2_i.
\]
Since $\bm{\Sigma}_{\parents[i]}$ is a covariance matrix, it is positive semi-definite, implying that $\bm{\beta}_i^\top \bm{\Sigma}_{\parents[i]} \bm{\beta}_i \geq 0$. Therefore, we obtain the inequality
\[
\sigma^2_i \geq \hat{\sigma}^2_i,
\]
meaning that the total variance of $\Xcomp[i]$ is always at least as large as its proper variance.

Furthermore, $\Xcomp[j]$ can be recursively expanded as a linear function of its ancestors, with a coefficient that reflects the cumulative effect along the causal paths.
%In most cases, $a_{ij}\geq 1$ results in an avalanche effect of the variances. 
If $\beta_{ij}\geq 1$, we have a \emph{avalanche effect} of the variances. This means that the variables become increasingly spread along the causal order. A comprehensive discussion on this topic can be found in~\citep{reisach2021beware,reisach2023}, where this property is exploited to infer the causal structure of a directed acyclic graph.
When considering \ac{AR}, we argue that using the normalized $\ell_1$ norm with ${\sigma}^2_i$ as a cost function can lead to significant distortions if applied to the system being studied. This is because interventions become less expensive as one moves down the graph's topological order, regardless of each variable's internal properties. To address this issue, \textit{we have decided to normalize the cost function based on the proper variances} $\hat{\sigma}^2_i$.
\begin{figure}[!t]
\centering
\includegraphics[width=6cm]{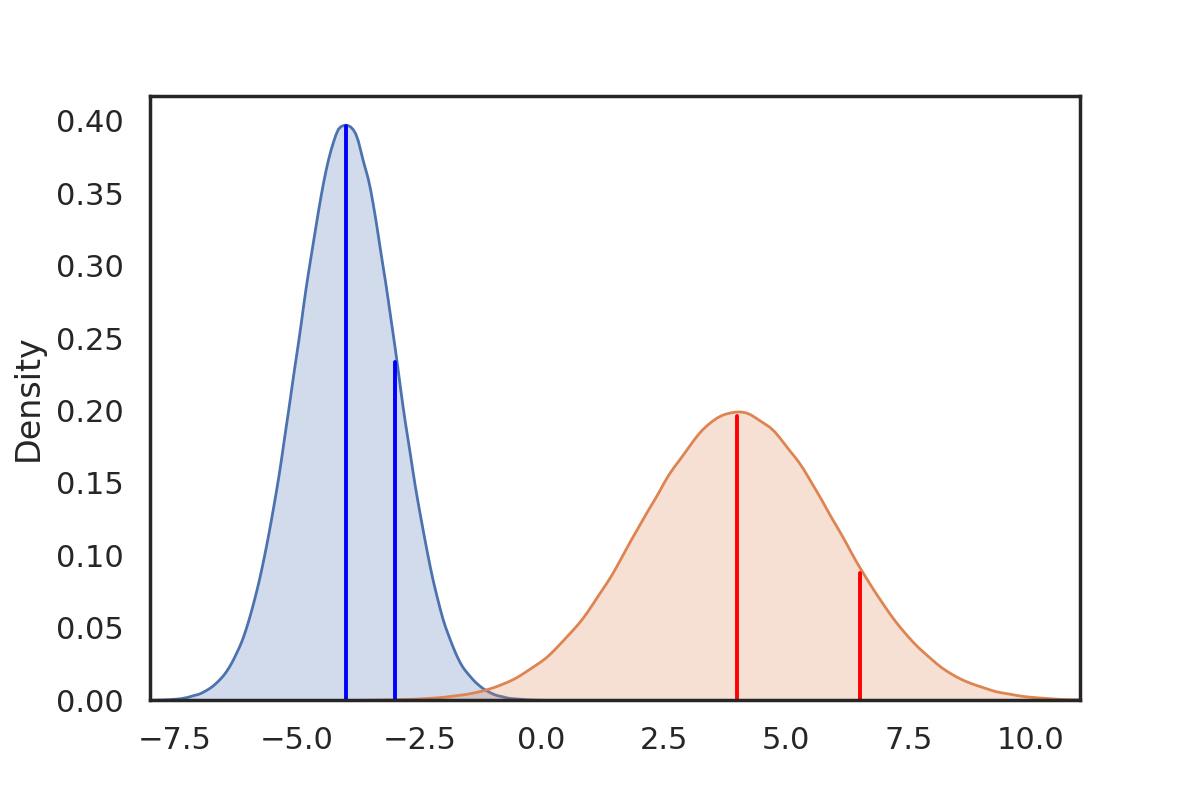}
\caption{\textbf{Impact of Distribution Scale on Costs.} An example demonstrating how the scale of distributions can lead to variations in intervention costs.}
\label{fig:costo}
\end{figure}

\paragraph{Q2: Identifying a potential bias in time-unaware Causal AR}
In~\cref{tab:causal-effect-derivative-scm}, we report each feature's estimated Causal Effects Derivative (\texttt{CED}) obtained by applying~\cref{eq:estimate} respectively to the \ac{SCM} reported in~\citep{karimi2020algorithmic} and the one proposed in this work. Concerning the former, it can be observed that nearly every \texttt{CED} is close to zero, giving rise to an unrealistic problem. 
\begin{table}[t!]
    \centering
    \caption{\textbf{Causal Effect Derivative.} Comparison between the estimated \texttt{CED} of the \ac{SCM} used in~\citep{karimi2020algorithmic} and the ones we proposed. Non-actionable variables are in red. Best-estimated actionable values are in bold.}
    \label{tab:causal-effect-derivative-scm}
    \setlength{\tabcolsep}{5pt} % Adjust column spacing
    \footnotesize
    \renewcommand{\arraystretch}{1.2} % Adjust row spacing
    \begin{tabular}{lcc}
    \toprule
    Feature & \texttt{CED} in \citep{karimi2020algorithmic} & Our \texttt{CED} \\
    \midrule
    \textcolor{red}{Gender}    & \textcolor{red}{0.004}  & \textcolor{red}{0.273} \\
    \textcolor{red}{Age}       & \textcolor{red}{-0.016} & \textcolor{red}{0.329} \\
    Education                  & 0.000                  & \textbf{0.181} \\
    Loan Amount                & -0.018                 & -0.099 \\
    Job                        & 0.015                  & 0.087 \\
    \textcolor{red}{Duration}  & \textcolor{red}{-0.021} & \textcolor{red}{-0.037} \\
    Income                     & \textbf{0.058}         & 0.137 \\
    Savings                    & 0.038                  & 0.066 \\
    \bottomrule
    \end{tabular}
\end{table}
For instance, the absence of any importance of variables such as \textit{Gender}, \textit{Age}, and \textit{Education} in the context of credit score prediction appears highly unusual. One possible explanation could be related to the previous discussion on using variance as a normalization coefficient. If the variance increases along the causal direction (multiplying at each edge), variables that are further away from the target will have a smaller variance and, therefore, a higher intervention cost. Conversely, those closer to the target will have a larger variance and a lower cost. In fact, the four variables with the highest \texttt{CED} are at a distance of $1$ from Y, while all those at a distance of at least $2$ have lower \texttt{CED} values.

\begin{figure}[h!]
\centering
\includegraphics[width=0.8\textwidth]{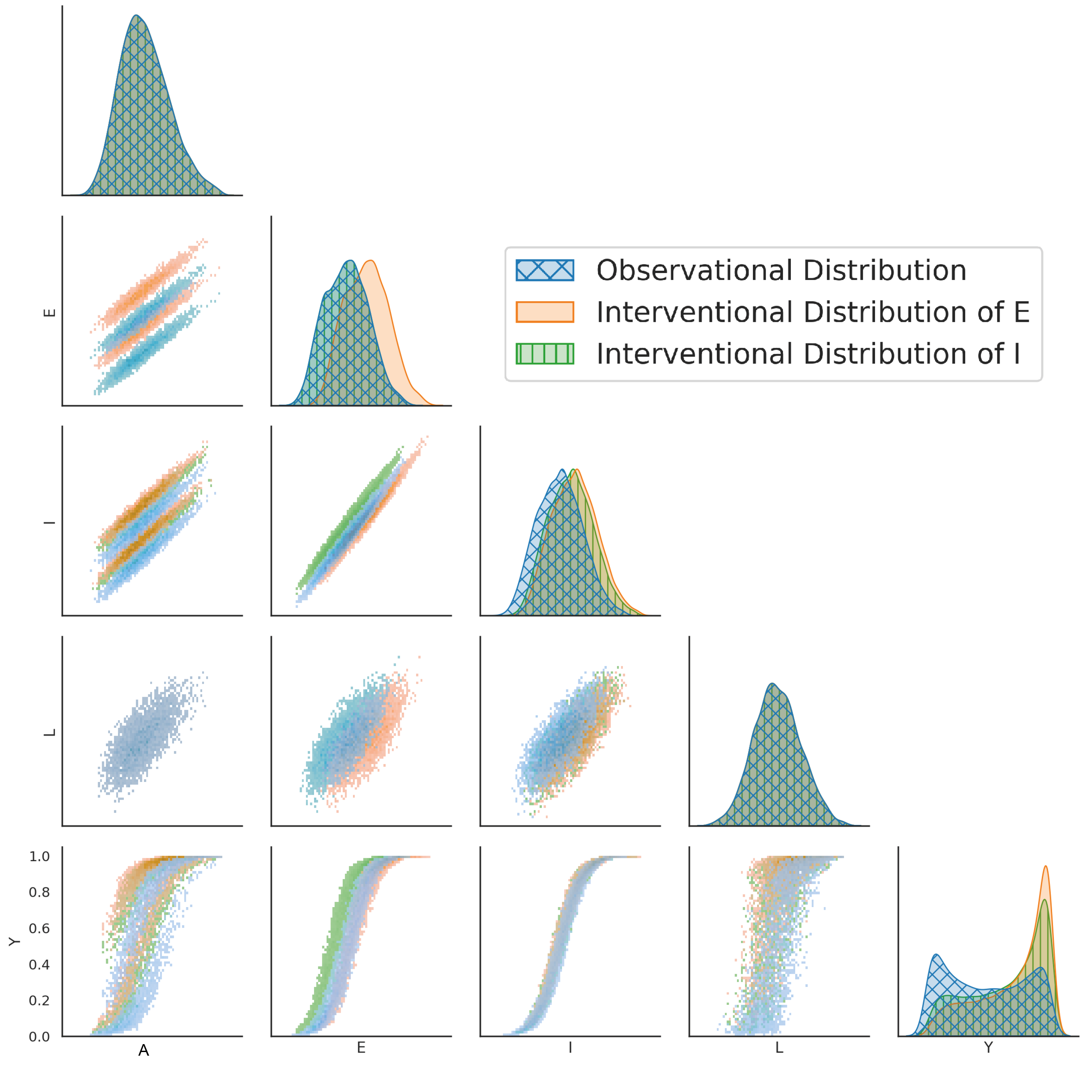}
\caption{\textbf{Observational vs. Interventional Distributions.} Pair plot illustrating observational and interventional distributions for the German Credit dataset. The results show that intervening on \textit{Education} yields a stronger impact on improving the outcome of $Y$ compared to \textit{Income}, as highlighted in the bottom-right subplot.}
\label{fig:plot}
\end{figure}

~\cref{fig:plot} presents the pair plot of the \ac{SCM} system described in~\cref{sec:exp_setup}, for three different distributions: Observational, Interventional on \textit{Education}, and Interventional on \textit{Income}. 
Only a subset of variables was selected, focusing on the two treatment variables and on $Y$. We added \textit{Age} and \textit{Loan Amount} to provide a clearer view of the system's structure and how interventions modify instances. Regarding \textit{Loan Amount}, it is observed that since it does not depend on either \textit{Education} or\textit{ Income}, its distribution remains the same in all cases. The same applies to \textit{Age}, which provides additional insight: the pair plots (\textit{Age, Education}) and (\textit{Age, Income}) show a bimodal distribution, explainable through the presence of fork paths $A\rightarrow E \leftarrow G$ and $ A\rightarrow E\rightarrow I\leftarrow G$, revealing the presence of a gender gap in \textit{Education} and \textit{Income} of individuals described by the \ac{SCM}. 
Finally, the most significant observation pertains to the \textit{Target} variable $Y$. Notably, it is evident that for the same cost, \textit{Education} has a stronger causal effect than \textit{Income}. As a result, the recommended action for addressing the unfavorable outcome through recourse would involve prioritizing the improvement of \textit{Education}, confirming the bias towards the root nodes we hypothesized in~\cref{sec:why_timecar}.

\section{Broader Impact}\label{sec:timecar_impact}
A crucial aspect emphasized by the \ac{XAI} community is the user's demand for plausible explanations~\citep{guidotti2018survey,ustun2019}. 
In the context of \ac{AR}, \textit{plausibility} refers to the perceived consistency and reasonableness of the recommendations provided by recourse approaches. 
From a psychological perspective, providing plausible explanations enables users to form mental models that align with their prior knowledge and reasoning abilities. 
When the temporal dimension is incorporated into causal reasoning, an \ac{AR} approach could ensure that the actions suggested are psychologically congruent with human intuitions and mental frameworks. 
This compatibility fosters a sense of trust and confidence in the algorithmic system, thereby facilitating user acceptance and engagement. 

In conclusion, by incorporating the temporal dimension into  \ac{AR} reasoning, the plausibility of explanations is enhanced, aligning them more closely with user expectations and cognitive processes. 
Moreover, considering actionability addresses the user's need for practical and timely actions. 
These insights highlight that time-aware \ac{AR} approaches are valuable, as they bridge the gap between human intuition, psychological congruence, and efficient decision-making processes.

\paragraph{Prediction vs. Improvement}

An important point to clarify is the distinction between the predicted outcome ($\hat{Y}$) and the actual outcome ($Y$).
Machine learning (ML) models are typically statistical in nature and do not inherently capture causal relationships. As a result, $\hat{Y}$ may behave very differently from $Y$. Consider, for example, disease diagnosis in the medical domain. Suppose we have several variables indicating the presence or absence of specific symptoms in a patient. A predictive model $\hat{Y}$ uses statistical correlations between these symptoms and the disease to enhance accuracy.
However, if the patient takes a medication that suppresses some symptoms, $\hat{Y}$ may change significantly—despite the fact that the actual disease status $Y$ remains unchanged.
From a causal perspective, intervening on effects (e.g., symptoms) does not influence their causes (e.g., the disease itself). For a more thorough discussion of this issue, see~\citep{konig2022}. In our work, we define a predictive model that is aware of the underlying causal relationships, ensuring that $\hat{Y}$ aligns more closely with $Y$, allowing us to treat them as a unified entity.
Importantly, not all variables in the system causally influence the outcome—only a relevant subset does.

\section{Conclusions}\label{sec:conclusions_timecar}
In this work, we discussed the problem of \ac{AR} from a causal perspective, focusing on incorporating the temporal dimension into evaluating the cost of an action. Firstly, we justified this integration by discussing its connection with the causal graph's topological structure and proposing a new \textit{time-aware} causal \ac{AR} formulation. Then, we tested our theoretical intuition through an experiment inspired by the credit score model on the German Credit Dataset. Our findings confirm our expectations: if the causal model is unaware of the response times between variables, it could recommend actions that, although optimal considering only the \acp{SCM}, would require too much time to be practically viable in most applications. %These results motivate future work to develop and evaluate causal algorithms that effectively incorporate temporal information to enhance the quality of recommendations.
Moreover, by showing scenarios where identical cost actions lead to disparate time requirements for different individuals and by revising actionability to include time constraints,
we identify the need for time-aware models that ensure fairness and align with human psychological expectations, encouraging trust in automated decision-making and promoting fairer outcomes.
As a final remark, we would like to highlight how incorporating the temporal dimension into \ac{AR} is a conceptually distinct problem from its causal formulation. A similar discussion to the one presented in this work could be made under different causal knowledge conditions up to its absence. In our opinion, the advantage of the causal framework stems from using the graph structure as a support for the finer estimation of the temporal relationships between the system's variables.
% 
%
%
%
%
%
%
%
%
% PART 3
\part{Time Aware Causality}\label{pt:time_aware_causality}
\chapter{Autonomous Dynamical Systems}\label{ch:dynamical_sistems_intro}

Adopting the framework introduced by~\citet{pearl2009causality}, a causal model can be understood as a set \( \X \) of observable random variables whose dependency relations are encoded through either a \ac{DAG} or a \ac{SCM} (\cref{ch:causality_math_framework}). Both representations allow for the formalization of causal relationships and support causal inference, which in turn enables the prediction of changes in the probability distribution $\Prob(\X)$ under the assumption that an external \emph{intervention} has been performed on the system.

We aim to investigate how this framework can be extended to accommodate systems that evolve over time, addressing two central objectives:

\begin{itemize}
    \item \textbf{Introducing interventions into causal time series analysis.} Traditional techniques in the literature~\citep{hamilton1994series} typically focus on identifying statistical dependencies from observational data, without explicitly modeling the effect of interventions. A paradigmatic example is \emph{Granger causality}, which estimates the causal influence of a time series \( \X_t \) on another series \( \Y_t \) by comparing the prediction errors of two linear regression models, one that includes the past values \( \X_{<t} \), and one that does not. While effective within its scope, such methods lack the formal tools to reason about interventional changes.

    \item \textbf{Enhancing the expressiveness of causal models to capture temporally distributed effects.} Interventions in dynamic systems often have effects that unfold over different temporal scales. Capturing this behavior is crucial for applications in fields such as economics, epidemiology, climate science and medicine\citep{DBLP:journals/corr/abs-2410-08816} where delayed or temporally diffuse causal effects are common. A more expressive causal modeling framework would allow for a principled treatment of such temporal complexities, improving both the interpretability and applicability of causal inference in time-dependent domains.
\end{itemize}

The most elementary approach consists of indexing the variables in $\X$ with a temporal parameter, by incorporating the time dependence directly into the \ac{SCM} equations. We consider a sequence of variables $\X_t$ indexed by time $t \in \mathbb{Z}$. A causal model with temporal structure can be described as a system of equations:
\begin{equation}\label{time_scm}
\left\{\Xcomp[i] _t = f_i(\parents[i]_t, \noise^i_t)\right\}_{i=1}^n,
\end{equation}
where $\parents[i]_t$ denotes the set of parent nodes of $X^i_t$ in the causal graph, and $\noise^i_t$ represents an exogenous noise term. We assume that the causal structure must be consistent with the temporal order, ensuring that causes necessarily precede their effects:
\begin{equation}
\parents[i]_t \subseteq \{\X_s : s < t \}.
\end{equation}
Remarkably, the strict inequality $s<t$ reflects the assumption that there are no instantaneous influences between different variables. 

\section{A Dynamical System Perspective}\label{sec:dynamical_system_perspective}
Causal analysis in this temporal formulation inherits several challenges from the broader causal modeling literature. For instance, the observed variables in $\Xspace$ may not be causally sufficient, as the underlying causal structure could involve latent confounders. Furthermore, identifying the correct causal direction among observed variables is often non-trivial. In addition to these familiar challenges, the temporal setting introduces new complexities: most notably, time series data typically consist of a single observed trajectory, which limits our ability to infer structure or validate assumptions from repeated trials. Given these limitations, the core idea of the following chapters is to approach the problem from a different perspective: instead of integrating time into causal models, we aim to \textit{embed causality within time-based models}. We argue that the greater flexibility of these models may provide a more effective solution. Moreover, this direction is grounded in a fundamental principle already mentioned above: the phenomena we study, and reality itself, unfold over time, and causal relationships should reflect this intrinsic characteristic; every action requires time to materialize and to produce an effect in the world.

At the same time, static \ac{SCM}s have proven to be highly effective in representing real-world systems and remain one of the main tools in causal analysis. This raises a crucial question: if the phenomena we study naturally evolve over time, how can a static theory successfully model them? And how can we reconcile these two perspectives? In this work, we put forward the following hypothesis: \textit{\acp{SCM} can be understood as a compression protocol for representing a dynamical model}, one that discards all information related to the temporal dimension while preserving only the structural features observable over long timescales. In other words, \acp{SCM} capture a dynamical system's equilibrium states while admitting a causal interpretation.

\textbf{Compressing Temporal Information.} To discuss the role of causality in dynamical systems, we begin by considering the framework of Statistical Mechanics. In this context, a macroscopic system $x_t$ is described through time averages of functions defined over the microscopic state at time $t$ of its constituent particles. More formally, given a microscopic observable $F$, its time average over a sufficiently long trajectory of the system is given by
\begin{equation}
\bar{F} = \lim_{T \to \infty} \frac{1}{T} \int_0^T F(x_t) dt.
\end{equation}
To make sense of the temporal average considered, it is essential to take into account the existence of two different \emph{temporal scales}: the macroscopic time scale, which corresponds to the observation times of the system, is much larger than the microscopic time scale, which governs the speed of changes at the molecular level. Yet, calculating the average in principle requires knowledge of the system's microscopic state $x_t$ and its temporal evolution.

For an \textit{ergodic system}, time averages can be replaced by phase-space averages computed through an appropriate \emph{invariant measure} $\prob(x)$:
\begin{equation}
\langle F \rangle = \int F(x) \prob(x)\, dx.
\end{equation}
This substitution allows us to describe the long-term behavior of a system using equilibrium equations that summarize its macroscopic properties. A classical example is the ideal gas law:
\begin{equation}
pV = nRT,
\end{equation}
which links pressure, volume, and temperature without explicitly modeling the movement of individual particles. This law is derived by averaging over time or microstates and thus compresses detailed dynamics into a compact summary. However, it does not directly encode causal relations. It represents a constraint on the observational distribution $\Prob(p,V,T)$, but says nothing about how interventions affect the system.

Still, under certain conditions, the equilibrium equation can acquire a causal interpretation. For instance, if we heat a gas in a container with fixed volume, the equation tells us how the pressure will change at equilibrium. In such cases, the equation becomes predictive of intervention outcomes. But to use it causally, we must carefully specify the nature of the intervention and verify that the assumptions of the model hold under manipulation.

\paragraph{SCM as a model of dynamical systems' equilibrium}
Let us consider a dynamical system, potentially involving stochastic components. We can assign an initial observational distribution to represent the state of the system, along with its associated uncertainty. This distribution evolves over time according to the inherent dynamics of the system. Under certain conditions, the system may converge to an invariant (or stationary) distribution. This limiting distribution can be interpreted as representing an equilibrium state of the system.

Now, imagine intervening on the system at a given point in time, altering its dynamics. As a result, the system may evolve toward a different invariant distribution. This new equilibrium reflects the long-term causal effect of the intervention on the system’s behavior.

Our goal is to reinterpret this shift as the causal response of the system, not just in the dynamic sense, but in terms of a corresponding intervention on a static \ac{SCM}. In other words, we aim to characterize the long-run behavior of the dynamical system as if it were governed by a static \ac{SCM} whose equilibrium distribution directly encodes the causal consequences of interventions.

This perspective allows us to bridge the gap between dynamical causal processes and the static causal models typically used in causal inference by treating the invariant distributions of the dynamical system as the key objects of interest. In this framework, long-term causal behavior can be studied through the lens of shifts in equilibrium distributions, interpreted via interventions on the corresponding static \ac{SCM}. 

\section{Autonomy} The equation \ref{time_scm} relies on another important assumption: the function \( f_i \) is invariant in time. Although its inputs, namely the set of parent variables \( \parents[i]_t \) and the noise term \( \noise^i_t \), may vary over time steps, the functional form of \( f_i \) remains fixed. This means that the same causal law governs the evolution of the system at every point in time, capturing the notion of mechanisms \emph{autonomy}: the internal workings of a system remain unaffected by changes in the distribution of other variables or by interventions applied elsewhere in the system. This point is subtle, and at least three remarks are needed.
\begin{itemize}
    \item From a methodological perspective, the condition of autonomy corresponds to the ability to conceptualize a system as distinct and separable from its surrounding environment, such that it interacts with the environment only through specific channels. This perspective allows for a clear boundary between the internal dynamics of the system and the external factors influencing it, thereby enabling meaningful causal analysis. 

    \item More specifically, \textit{autonomy can be interpreted as a modularity assumption}. That is, each mechanism $f_i$ is considered a fundamental and self-contained component of the system, responsible for determining the behavior of a particular variable given its direct causes. These mechanisms are assumed to reflect stable aspects of the world's structure, remaining invariant over time and under a range of interventions. This invariance is what makes it possible to isolate and study individual causal relationships, as well as to predict the consequences of hypothetical manipulations.
    
    \item The autonomy condition enables the modeling of long-term causal effects by ensuring that the underlying causal mechanisms \textit{remain stable over time}. This stability, in turn, allows for the system to converge toward a well-defined, time-invariant distribution as $ t \to \infty$. In other words, autonomy supports the existence of asymptotic behavior that is not merely the result of statistical regularities but reflects persistent causal structure. This is particularly important when analyzing systems whose equilibrium or steady-state properties carry meaningful causal interpretations, where long-term outcomes are shaped by sustained interactions governed by invariant mechanisms.
\end{itemize}

\subsection{Is Climate Change Caused by Humans? Causal Ambiguity in Non-Stationary Processes}

Climate change is a powerful example of how causal reasoning becomes difficult when the assumption of autonomy does not hold. In a standard causal model, the mechanisms that describe how a system evolves are assumed to remain stable over time. However, the climate system is shaped by many complex and long-term processes, such as solar activity, volcanic eruptions, and ocean circulation. These processes may change over time in ways that we do not fully understand or cannot directly control.

Due to these changes, the causal relationships in the climate system might not be stable. For example, the relationship between greenhouse gas emissions and global temperature may look different in different periods, depending on the background influence of other natural factors. This makes it hard to say with certainty whether recent temperature increases are caused mainly by human activity or by natural changes.

This challenge is related to the idea of \textit{causal sufficiency}. A model is causally sufficient if it includes all the important variables that influence the system. If some relevant factors are missing, the model’s mechanisms may appear to change over time, when in fact they are reacting to unmodeled influences. In this case, the system may still be predictable, but we lose the ability to make strong causal claims.

One possible solution would be to expand the model to include more of these hidden influences. If we could measure and include them, we could recover a model where the mechanisms remain stable, and causal inferences become more reliable. But in practice, this is often hard to do, especially when the missing factors involve complex natural systems or long-term trends. Until then, our ability to estimate the effects of intervention remains limited by this uncertainty.

\section{Not Only Long-Term Effects}\label{sec:static_vs_dynamic}

The discussion in \cref{sec:dynamical_system_perspective} on the macroscopic analysis of an ideal gas offers two alternative perspectives on the role of temporal compression.

\textbf{The first perspective} concerns cases where the fast dynamics of the system are either too complex to model or irrelevant for the researcher’s goals. In the case of a gas, the motion of individual particles fits both categories: we are unable to predict each particle’s movement, and even if we could, this information would not be practically useful. Therefore, only the slow, macroscopic time scale matters, and studying long-term effects becomes the main objective. 

This principle, common to many modeling practices involving abstraction, applies in at least two scenarios: 
\begin{itemize}
    \item When macroscopic variables emerge from the behavior of a large number of microscopic components (e.g., gas particles, macroeconomic indicators from individual agents, or biological systems from cellular processes).
    \item When a small number of components evolve across multiple time scales, such as neuronal activity (fast scale) and the slow evolution of synaptic structures, or stock market price fluctuations (fast scale) and their monthly or yearly trends (slow scale).
\end{itemize}

In the first case, compression typically occurs through some form of phase-space averaging; in the second case, it occurs through temporal averaging. Sometimes, as in the ideal gas example, both mechanisms coexist. From this perspective, \ac{SCM}s provide an effective model to describe macroscopic dynamics and, under suitable assumptions, \textit{capture all the useful information} about the system.

\textbf{The second perspective} addresses situations where the researcher is also interested in the fast dynamics of the system and wants to understand the effects of actions during the evolution itself. Typically, the fast scale is not so fast as to be practically unobservable, and its effects offer valuable insights.

\begin{example}
Consider a government policy aimed at controlling the birth rate. Suppose %— hypothetically — 
that the government could set the exact number of births per year indefinitely. After a sufficiently long period (roughly the maximum human lifespan, say 110 years), and assuming all other mechanisms remain unchanged, the result would be a completely homogeneous population pyramid, meaning the same number of individuals at each age. 

However, this long-term description would not capture all relevant aspects. The evolution of the population pyramid during the transition period would offer critical information, for example, about the available workforce or the financial pressure on pension systems. In fact, policy evaluation typically focuses on short- to medium-term effects, and very often the system is not expected to reach equilibrium at all, due to future policies or long-term social, cultural, or political changes.
\end{example}

One particularly interesting scenario occurs when the information of interest lies entirely in the microscopic description of the system, and the equilibrium itself is of little or no practical use. 

\begin{comment}Consider, for instance, weather forecasting: even if we had a model that accurately describes the global average rainfall or the typical distribution of temperatures, such information would be of little relevance for a meteorologist, whose goal is to predict the \textit{localized state} of the atmosphere in both space and time. The same holds for the study of social or ecological dynamics.
\end{comment}
\begin{example}
    Consider the \textit{Lotka–Volterra} model, which describes the interaction between two animal species that share the same environment, a prey and a predator. When the prey population increases, predators have more food available, leading to a growth in their population as well. However, an excess of predators eventually reduces the prey population, initiating a new cycle. This dynamic repeats over time. Although the system may admit a theoretical equilibrium point, the true interest lies not in identifying this equilibrium, but in studying the cyclical dynamic behavior itself.
\end{example}

In such contexts, \textit{compressing the temporal dimension would mean precisely discarding the most valuable information}. In other words, analyzing the associated \ac{SCM} would be ineffective from all practical perspectives. Access to the full underlying temporal model is essential, and both inference and causal reasoning must be performed at the level of the dynamical system itself.

\section{Objective of the Next Chapters}

In the upcoming chapters, we aim to address both perspectives, exploring their connections in greater depth.

\begin{itemize}
    \item On the one hand, we are interested in understanding how to map between the two descriptions, that is, how to relate dynamical models to their equilibrium representations. This interest is motivated by several, partly complementary, reflections. We want to better understand in what sense a \ac{SCM} can be considered causal, and whether this causal meaning can be extended to the underlying dynamical system. We argue why invoking causality at the dynamical level is meaningful and how this perspective remains consistent with the classical causal framework. At the same time, this connection offers a new perspective for interpreting \ac{SCM}s more broadly, helping to clarify their underlying assumptions and limitations. For instance, stating that the causal graph is acyclic implicitly makes assumptions about the nature of the underlying dynamical system. What does this imply? Under what conditions is this assumption justified, and when might it fail to hold?

    \item On the other hand, we are interested in extending causal inference beyond the static formulation. In many cases, it is necessary to evaluate the effects of an intervention \textit{over time}, and \ac{SCM}s are no longer a sufficient description. Actions in the real world require time to produce effects, and understanding how these effects develop demands a temporal causal model. In \cref{ch:Ci_over_time} and \cref{ch:CIcontinuous_time}, we propose a new framework for this purpose, introducing the notion of intervention in stochastic processes, modeling how interventions alter dynamics, and developing tools to estimate causal effects over time. Our approach aims to maintain as much continuity as possible with the formalism of \ac{SCM}s, identifying modular and independent mechanisms, and interpreting them as functional relationships describing differential equations (or difference equations in the discrete-time case).
\end{itemize}

\chapter{Causal Inference over Time}\label{ch:Ci_over_time}
\begin{tcolorbox}[colback=CTsemi!5!white, colframe=CTsemi, boxrule=0.5mm]
\textbf{Main Sources:} \fullcite{cinquini2024practical}.
\end{tcolorbox}

\paragraph{Author's Note}
This chapter presents a mature stage of the PhD research and brings together several research threads developed throughout the project. Building on structural causal models and temporal systems, it introduces a clear and unified framework for performing causal inference over time. The chapter connects ideas from dynamical systems, vector autoregressive models, and causal reasoning, offering both theoretical foundations and practical tools to study the effects of interventions in systems that evolve over time. The results represent an important milestone in the overall research path.
\vspace{30pt}

Dynamical systems often exhibit complex behaviors that unfold over time, leading to delayed responses and feedback loops. 
Understanding the causal effect of interventions within such systems is crucial across disciplines such as climate~\citep{runge2019inferring} and social sciences~\citep{wunsch2022time}, where different time scales play a central role. 
For instance, monetary policy adjustments may have immediate effects on consumer spending, but their impact on inflation and economic growth only becomes evident in the medium/long-term. 
Similarly, the consequences of human actions on climate change may take decades to manifest, with the risk of endorsing public policies that underestimate their relevance. To address these issues, it is essential to perform \emph{causal inference over time}. 
From the perspective of causality, \acp{SCM} provide a formal framework to enable causal inference from cross-sectional data. However, adapting existing methods to capture temporal dynamics remains a challenge~\citep{bongers2021foundations}. Alternatively, autoregressive models offer practical methods for time-series analysis and forecasting~\citep{lutkepohl2005new}, but their formalization of causal effects is limited. 
First, they model interventions as \emph{shocks} in a specific point in time, with effects that fade away after a certain period~\citep{ moneta2011causal}. Second, a common methodological stance in these works is that of Granger causality~\citep{Granger_1969}. This approach treats relationships between variables as predictive dependencies that are only potentially causal; it assesses whether one variable can improve the forecasting of another, rather than explicitly identifying a causal mechanism between them. In other words, inferring a true cause-effect relationship requires additional assumptions about the system, such as causal sufficiency. Without these, no definitive conclusions can be drawn regarding the system's causal structure. This challenge is addressed specifically within the field of Causal Discovery, which is not the central focus of this thesis. For a deeper discussion of Granger's approach and its methodological implications, we refer the reader to his original works \citep{granger1980testing, granger1988some, granger2003some} and some other relevent works \citep{comte1996noncausality,eichler2012causal,eichler2012graphical,eichler2013causal}. The specific implications of assuming causal sufficiency a priori are discussed in more recent literature \citep{eichler2010granger, white2010granger}. Our work is conceptually closer to this latter scenario.\\
Our work combines the strengths of both frameworks, i.e., \acp{SCM} and  \acp{VAR}, to enable robust reasoning about the causal effect of interventions on dynamical systems over time.
Our experiments on synthetic and real-world datasets show that the proposed framework achieves strong performance in terms of observational forecasting while enabling accurate estimation of the causal effect of interventions on dynamical systems. 

The rest of the chapter is organized as follows.   
Section~\ref{sec:preliminaries_ciovertime} introduces the key background on structural causal models and stochastic processes. 
Section~\ref{sec:causal-perspective} presents the theoretical foundations for causal inference over time, focusing on stochastic difference equations. 
Section~\ref{sec:interventions} provides a practical implementation of causal interventions and details how \acp{VAR} can be mapped to \acp{SCM}. 
Section~\ref{sec:practical_implications} discusses real-world applications and implications of the proposed framework. 
Section~\ref{sec:experiments} reports the experimental evaluation. 
Finally, Section~\ref{sec:conclusions_ciovertime} concludes the chapter and outlines directions for future research.

\section{Contributions}\label{sec:contributions_ciovertime}

This chapter makes the following contributions:

\begin{itemize}
    \item \textbf{Formalization of  Causal Inference over Time:} We introduce a framework for performing causal inference over discrete-time stochastic processes, extending \acp{SCM} to account for temporal dynamics.

    \item \textbf{Intervention Modeling in Dynamical Systems:} We define two types of interventions, additive and forcing, that allow us to study the effects of sustained actions over time, and we show how to estimate their impact on system trajectories.

    \item \textbf{Mapping \ac{VAR} to \ac{SCM}:} We provide a method for mapping \ac{VAR} models to equivalent \acp{SCM} at equilibrium, showing that long-term causal effects can be recovered using \acp{SCM}.

    \item \textbf{Theoretical Guarantees:} We prove that under stability conditions, the mapping preserves causal semantics, \ie, observational and interventional distributions commute between the time-based and equilibrium representations.

    \item \textbf{Practical Applications and Case Studies:} We apply our framework to synthetic and real-world datasets, demonstrating its ability to support both interventional and counterfactual reasoning in time-evolving settings.
\end{itemize}

\section{Preliminaries}\label{sec:preliminaries_ciovertime}
\subsection{Structural Causal Models}
\paragraph{Equation} 
A \ac{SCM} $\scm[{\rmF,\vectorNoise}]$ determines how a set of \sizex endogenous (observed) random variables $\rmX := \{\Xcompalt[1], \dots, \Xcompalt[\sizex]\}$ are obtained from a set of %are obtained via $k$ 
exogenous variables $ \vectorNoise := \{\Noise_1,\ldots, \Noise_d\}$, with prior distribution $\distribution(\vectorNoise)$, via a set of structural equations $\rmF := \{\Xcompalt[i] := f_i(\parents[i], \vectorNoise^i)\}_{i=1}^{\sizex}$. 
Each $f_i$ computes $\Xcompalt[i]$ from its causal parents\footnote{Unlike in acyclic \acp{SCM}, $\parents[i]$ loses its hierarchical interpretation since two nodes can be mutually parents.}  $\parents[i]\subseteq\X$ and a set $\vectorNoise^{i}\subseteq \vectorNoise$. 
We refer to $\X$ as a solution of $\scm$.
We assume $\parents[i]$ to be minimal, i.e.,~it only contains variables $\Xcompalt[j]$ such that$ \quad \partial_{\Xcompalt[j]}{f_i}\neq~0$. 
This formulation extends the definition in~\citep{pearl2009causality} (also called \textit{Markovian} \ac{SCM} in~\citet{bongers2018causal}) 
to include cycles. 

\paragraph{Graph}
A \ac{SCM} $\scm$ induces a directed graph $ \graph_{\scm} = (\GraphVertices, \GraphEdges)$ that describes the functional dependencies in $\rmF$: $\GraphVertices$ is the set of nodes for which $V_i$ represents $\Xcompalt[i]$ and $\GraphEdges$ is the set of the edges $\edge{i,j} \in \GraphEdges \iff \Xcompalt[i] \in \parents[j]$. 

\paragraph{Intervention}
Besides describing the observational distribution $p(\X)$, \acp{SCM} allow answering \emph{interventional queries} about the effect of external manipulations, and enable \emph{counterfactual queries} assessing what would have happened to a particular observation if one observed variable $\Xcompalt[i]$ had taken a different value. 
%on a subset of the endogenous variables $I \subseteq \{1,\dots,\sizex\}$
%
An intervention $\intervention$ on a \ac{SCM} $\scm$ yields a new \ac{SCM} $\scm^\intervention$ for which one or more mechanisms $f_i(\parents[i], \vectorNoise^i)$ change to
$\tilde{f}_i(\tilde{\parentssym}^{i}, \tilde{\vectorNoise}^{i}),$
where $ \tilde{\parentssym}^{i}\subseteq \parents[i] \: \text{ and } \:\tilde{\vectorNoise}^{i} \subseteq \vectorNoise^i$.
We refer to a \textit{hard intervention} when $f_i$ is replaced by a constant value $\alpha^{i}$, and $\tilde{\parentssym}^{i}=\tilde{\vectorNoise}^{i}=\varnothing$. This type of intervention is denoted by the do-operator $do(\Xcompalt[i] = \alpha^{i})$. 
On the other hand, we refer to a \textit{soft intervention} when at least one argument of $f_i$ is retained.
The causal effect $\ce$ of an intervention is evaluated in terms of differences between the values of the observable variables before and  after the intervention $\intervention$, i.e.,  
% $\X^\intervention$, and
% %Specifically, the effect can be considered at the population level (Average $\ce$), for subgroups (Conditional $\ce$), or individuals (Individual $\ce$). 
% %
% %Each of these cases can be expressed through the formula 
% can be expressed through the formula
\begin{equation}\label{eq:causal_effect}
\ce^\intervention =\expected{\X^\intervention -\X}.
\end{equation}
%where $C$ represents the conditioning describing the subgroup of interest, ranging from the entire population to a single realization. 

These queries can be seen as operations on $\mathcal{M}$ that yield a new \ac{SCM} $\tilde{\mathcal{M}}$. A \ac{SCM} is (interventionally/counterfactually) \textit{solvable} if each operation has a solution $\tilde{\mathcal{M}}$, while it is \textit{uniquely solvable} if, additionally, the solution is unique. In \citep{bongers2018causal}, the latter is denoted as \textit{simple} \ac{SCM}. In general, simple \acp{SCM} allow for cycles and hidden confounders. However, not all \acp{SCM} with cycles and hidden confounders are simple. 

A \ac{SCM} is said to be \textit{Markovian} if it induces a \ac{DAG} and $\Noisecompalt[i]\cap\Noisecompalt[j]=\varnothing,\:  \forall i,j \in \{1,\dots,d\}$. Markovian \acp{SCM} are always uniquely solvable, hence a subset of simple \acp{SCM}.

\subsection{Discrete-time Stochastic Processes}
\begin{definition}\label{def:dsp}
    A \ac{DSP} is a function
$$ \X : \;T \times \Omega \to \sR^\sizex$$
    %&t, \omega \mapsto \X(t,\omega)
%
where $t\in{T}$ is a time index in $\sZ$, such that $\X_t$ (denoting $\X(t,\cdot)$) is a random variable on a probability space $(\Omega,\gF,\Prob)$\footnote{The probability space $(\Omega, \gF, \Prob)$ consists of a sample space $\Omega$, a $\sigma$-algebra $\gF$ of events, and a probability measure $\Prob$ that assigns probabilities to events in $\gF$.}. 
We refer to $\X(\omega)$ (denoting $\X(\cdot,\omega)$)  
as a \textit{realization} or \textit{trajectory} of $\X$ and denote the i-th component of $\X$ with $\Xcompalt[i]$.
\end{definition}

A \ac{DSP} can be described through a \ac{dE}, i.e.,~a recurrence relation that allows computing $\X_t$ based on its past values. 
\acp{dE} can serve as discrete approximations of differential equations, sharing some fundamental theoretical properties. In many applications, the former are approximation tools to simulate the latter numerically. 
As with differential equations, it can be beneficial to introduce stochastic components into the equation to model phenomena that are not entirely predictable. 
This unpredictability may stem from incomplete knowledge of the relevant variables of the process or its intrinsic randomness. 
We will denote with \ac{dE} the class of difference equations in general, regardless of whether they are deterministic or not. 

Including stochastic components in differential equations gives rise to various classes of equations.
This section introduces the notions of \ac{OdE}, describing deterministic processes; \ac{RdE}, which involve randomness in the initial state $\X_0$ and in the evolution parameters; and \ac{SdE}, which describe inherently stochastic trajectories. Understanding these concepts is helpful for contextualizing our work, and for clarifying how it extends the ideas in~\citep{Mooij2013deterministic_case}.
\paragraph{Ordinary Difference Equation}
An \ac{OdE} is an expression that describes the evolution of an indexed variable $x_t$ through a functional relationship
\begin{equation*}
\x_t = f(\x_{<t}),
\end{equation*}
where $\x_{<t}:=\{\x_{t-1},\x_{t-2},\dots\}$ represents the past of $\x_t$. Specifically, a $p^{th}$-order \ac{OdE} is an expression where the preceding variables of $\x_t$ include up to $\x_{t-p}$, i.e, $\x_{<t}=\{\x_{t-i}\}_{i=1}^p$. 

\paragraph{Random Difference Equation} 
A \ac{RdE} is a \ac{dE} of the form
\begin{equation}\label{eq:RDE}
    \X_{t} = f(\X_{<t},\vectornoise),
\end{equation}
where $\vectornoise:\Omega\to\sR^e$ is a random variable\footnote{In ~\citep{bongers2018causal}, \(\vectornoise\) is regarded as a stochastic process that converges eventually or asymptotically to a time-independent random variable. In this work, simplifying the concept, we consider it as a random variable throughout the whole process.}  and such that for every $\omega \in \Omega$, the equation 
\begin{equation}\label{eq:RDE_sample_path_solution}
    \X_{t}(\omega) = f(\X_{<t}(\omega),\vectornoise(\omega))
\end{equation}
is an \ac{OdE}.

Note that $\X_t$ in \Eqref{eq:RDE} is a random variable and characterizes a \ac{DSP} $\X:T\times\Omega\to\sR^n$ induced by $\vectornoise$.
In this sense, we say that $\X$ is a solution of \Eqref{eq:RDE}. Equivalently, we say that $\X$ is a \textit{solution} of \Eqref{eq:RDE} if \Eqref{eq:RDE_sample_path_solution} is satisfied for \textit{almost all} $\omega\in\Omega$.
\begin{example}[\ac{RdE}]\label{ex:rde}
The system described by 
\begin{equation}\label{eq:rde_example}
    \begin{cases}
        X_t=a(\Omega)X_{t-1} + b(\Omega)\\
        X_0=X_0(\Omega)
    \end{cases}
\end{equation} 
represents a first-order linear \ac{RdE} equipped with an initial condition for $X_0$, in which the stochasticity of $\omega$ generates a distribution over $X_0$ and over the parameters of the equation $a,b$. 
The general solution of \Eqref{eq:rde_example} can be found by recursive substitution, obtaining $X_t(w) = a(\omega)^t X_0(\omega) + b(\omega)\sum_{i=0}^{t-1}a(\omega)^i$.
\end{example}

\paragraph{Stochastic Difference Equation}
A \ac{SdE} describes a \ac{DSP} via a functional relationship of the form
\begin{equation}\label{eq:sde}
\X_t = \bm{f}(\X_{<t}) + \bm{g}(\X_{<t}) \odot \vectornoise_t, 
\end{equation}
where $\X_{<t}:=\{\X_{t-1},\X_{t-2},\dots\}$ represents the trajectory up to time $t$, % past of $\X_t$, 
$\bm{f}$ represents the system's deterministic mechanism, and the Hadamard product $\bm{g} \odot \vectornoise_t$ is the inhomogeneous stochastic part, where $\vectornoise_t$ denotes \textit{white noise}, i.e.,~$\forall t,t^\prime \in T \quad \expected{\vectornoise_t}=\bm{0}, \expected{\vectornoise_t\vectornoise_{t^\prime}}=\Sigma^2_{\vectornoise}\delta_{t, t^\prime}$.  
While for \acp{OdE} and \acp{RdE} trajectories $\X(\omega)$ may asymptotically converge to an equilibrium, \acp{SdE} cannot exhibit such convergence due to the ongoing influence of $\bm{g} \odot \vectornoise_t$. 

\begin{example}[\ac{SdE} - Random Walk]
    The equation $x_t = x_{t-1}+\noise_t, \quad \noise_t\sim\normal(0,\sigma^2_\noise)$, equipped with $x_0=0$, represents a \ac{SdE} with solution $x_t = \sum_{i=1}^t \noise_i \sim \normal(0,t\sigma^2_\noise)$. Random Walk is closely connected with its continuous counterpart, the Wiener Process, also called Brownian Motion.
\end{example}

\paragraph{Graph} 
Analogously to \acp{SCM}, we can associate a directed graph $\graph_{\dequation}$ to a \ac{dE} $\dequation$, consisting of nodes $V_i$ representing individual components $\Xcompalt[i]$, while an edge $\edge{i,j}$ is present if $\exists k>0$ such that $\partial_{x^i_t}
\Xcompalt[j]_{t+k}\neq 0$ in $\dequation$.

\subsection{Vector Autoregressive models}
In this chapter we focus on linear \acp{SdE} or in other words, the \ac{VAR} model~\citep{kilian2017structural}.

\paragraph{Equation}
Consider a \sizex-dimensional vector-valued stationary time series $\{\X_0, \dots, \X_T\}$ generated by a \ac{VAR} model with \emph{lag} $\order$, where a lag represents 
%the memory of the process, i.e.,
the number of previous time steps used to predict the current value of each variable.
Specifically, the \ac{VAR}($\order$) model is defined by 
\begin{equation}\label{eq:VAR}
    \X_t = \bm{\nu} + \A_1 \X_{t-1} + \dots + \A_p \X_{t-p} + \vectornoisevar_t, 
\end{equation}
where $\bm{\nu}$ is a \sizex-dimensional vector of intercept terms, 
$\{\A_i\}_{i=1}^p$ are $\sizex \times \sizex$ matrices 
%defined by$[\A_\tau]_{ij} = \alpha_{ij,\tau}$ with $\tau = 1, \dots, \order$ 
and 
$\vectornoisevar_t$ is a \sizex-dimensional white noise term.
%
% We assume the process $\X_t$ to be stable and stationary~\citep{hamilton1994time}.
% Under this assumption, it can be expressed using an infinitive-order Moving Average (MA) representation $\X_t = \bm{\mu} + \sum_{i=0}^\infty \Phi_i \vectornoisevar_{t-i}$ 
% (the in-depth description can be found in \ref{app:svarproperties}), where $\bm{\mu}$ is the mean term and $\Phi$ is the \textit{impulse response function} of the model.  
If the process $\X_t$ is  stationary~\citep{hamilton1994series}, \ie, the unconditional joint distribution of $\X_t$ over any finite set of time indices $t_1,\dots,t_k$ does not depend on time shifting, formally  $$\prob(x_{t_1},\dots, x_{t_k}) = \prob(x_{t_1+\tau},\dots ,x_{t_k+\tau}) \quad \forall \tau\in \R,$$
%then it can be expressed using an infinitive-order Moving Average (MA) representation $\X_t = \bm{\mu} + \sum_{i=0}^\infty \Phi_i \vectornoisevar_{t-i}$ (the in-depth description can be found in \ref{app:svarproperties}), where $\bm{\mu}$ is the mean term and $\Phi$ is the \textit{impulse response function} of the model.  
%
%
\Eqref{eq:VAR} can also be written as
    \begin{equation}\label{eq:var_lagoperator}
        \A(L)\X_t = \bm{\nu} + \vectornoisevar_t, \quad
        \text{with} \quad \:  \A(L) := \bm{I}_\sizex - \A_1 L - \dots - \A_\order L^\order, 
    \end{equation}
where $L$ is the \emph{lag operator} defined as $L \X_t \equiv \X_{t-1}$ and $\bm{I}_\sizex$ is the \sizex-dimensional identity matrix.  

\paragraph{Properties}\label{app:svarproperties}
A \ac{VAR} process $\X_t$ is stable if all roots of the determinantal polynomial of the VAR operator are outside the complex unit circle. This condition can be formally expressed as:

\begin{equation}\label{eq:determinantal}
    %det(I_{Kp} - \bm{A}_z) \neq 0 \quad \text{for} \quad |z| \leq 1.
    det(I_K - \A_1z - \ldots - \A_\order z^\order) \neq 0 \quad \forall z \in \mathbb{C}, |z| \leq 1,
\end{equation}
where $\mathbb{C}$ denotes the set of complex numbers. The stability condition ensures that the process is stationary~\citep{lutkepohl2005new}. Consequently, a stable \ac{VAR}($\order$) process is stationary.

In~\citep{wold1938study}, the author introduced the Wold decomposition Theorem, which states that every covariance-stationary time series $\X_t$ can be written as the sum of two uncorrelated processes, that is $\X_t = \bm{Z}_t + \bm{Y}_t$
where $\bm{Z}_t$ is a deterministic component that can be forecast from its past, and $\bm{Y}_t$ is a purely stochastic process with an infinite-order Moving Average ($MA(\infty)$) representation: 
\begin{equation}\label{eq:moving_average}
\bm{Y}_t = \sum_{i=0}^\infty \Phi_i \vectornoisevar_{t-i}.
\end{equation}
Suppose the $\Phi_i$ are absolutely summable and that there exists an operator $A(L)$ with absolutely summable coefficient matrices satisfying $A(L)\Phi(L) = I_K$.
Then $\Phi(L)$ is invertible with $\Phi(L)^{-1} = A(L)$ and $\bm{Y}_t$ can be approximated arbitrarily well by a finite-order \ac{VAR}($\order$) with $\order$ sufficiently large, and $A(L) = I - \A_1 L - \dots - \A_p L^\order$, where $I$ is the identity matrix. 

In particular, the matrices $\Phi_i$ of a \ac{VAR}$(p)$ model can be calculated by recursively applying the formula
\begin{equation}\label{eq:recursive_calcultation_phi}
    \Phi_0 = I, \quad \Phi_i = \sum_{j=1}^i \Phi_{i-j}\A_j, \quad \text{where} \; \A_j = 0 \; \text{for} \; j>p.
\end{equation}
\begin{remark*}
    Consider the model \( A(L)\X_t = \bm{\nu} + \vectornoisevar_t \), where \( \bm{\nu} \) is a (constant) intercept term. Using Wold's Theorem, we can express it by linearity of $\response$ as \( \X_t = \response(L)\bm{\nu} + \response(L)\vectornoisevar_t \). Noting that \( L\bm{\nu}=\bm{\nu} \), we obtain  
    \begin{equation*}
        \X_t=\bm{\mu} + \response(L)\vectornoisevar_t
    \end{equation*}
    where \( \bm{\mu}=\response(1)\bm{\nu}:=\sum_{l=0}^\infty\response_l\bm{\nu} \). Due to the invertibility of \( \response(L) \), we can express the same relation through \( \bm{\nu}=\response(1)^{-1}\bm{\mu}=A(1)\bm{\mu}=[1-\A_1-\dots-\A_p]\bm{\mu} \). From this observation, we automatically obtain the equality
    \begin{equation}\label{eq:phi1a1}
        \sum_{l=0}^\infty \response_l = \response(1) = A(1)^{-1} = [1-\A_1 - \dots - \A_p]^{-1}.
    \end{equation}
    
\end{remark*}

A limitation of \acp{VAR} is the inability to interpret the system in causal terms since the components of $\vectornoisevar_t$ are cross-correlated and act as hidden confounders.
A common approach to overcome this issue is to orthogonalize the noise terms. %To that end
In this context, the process of \ac{CD} is analogous to the one of \acp{SCM}~\citep{hyvarinen2010estimation, moneta2013causal, geiger15, malinskyS18}, and identifies a triangular matrix $\hat{\A}_0$ such that $\vectornoisesvar_t = \hat{\A}_0 \vectornoisevar_t$ consists of mutually uncorrelated elements. The transformed \ac{VAR}, known as the Structural VAR (SVAR) model in the literature~\citep{kilian2017structural}, is defined by
$\hat{\A}_0 \X_t = \hat{\A}_0 \bm{\nu} + \hat{\A}_1 \X_{t-1} + \dots + \hat{\A}_\order \X_{t-\order} + \vectornoisesvar_t,
$ where $\hat{\A}_i = \hat{\A}_0 \A_i$. %In this work, \textit{we assume that the causal order is known and causal discovery will be not considered in the pipeline}. \IV{be as explicit as possible. Applications of SVAR has been on causal discovery, i.e., estimating Azero. In contrast we provide a complemenbtary approach to perform causal inference over time, i.e.,e estimate the causal effects of interventions on a dynamical system. }
%
%\end{equation*}
%\sr{It seems here that a solution to the causal interpretation is already given by the SVAR. So, it should be explicitely stated why it is not enough ($\hat{\A}_0$ is difficult to compute?).}
%
%\begin{remark*}
From a modeling perspective, \ac{VAR} and SVAR are equivalent, %differing only in their representation. Specifically,  
as any SVAR can be expressed in its reduced-form \ac{VAR} by computing $\A_i = \hat{\A}_0^{-1}\hat{\A}_i$ for $i=0, \ldots, p$ in \ref{eq:VAR}. Notably, choosing one over the other does not affect its causal interpretation, provided that $\hat{\A}_0$ is known. For simplicity, in this work, we adopt the \ac{VAR} notation, %and  provide a step forward in the causal interpretation of \VAR{s}, by
to introduce a \emph{novel framework for causal inference over time, which complements the SVAR's causal discovery approach}. 
%Exploiting this property, all the results in this work are presented for the sake of clarity within the \ac{VAR} framework.
%\end{remark*}

\paragraph{Graph}
An edge $(i, V)$ is present iff $\exists k$ s.t. $\A_k[i, j] \neq 0$.

\section{Causal perspective on Discrete-time Stochastic Processes}\label{sec:causal-perspective}
In this section, we present the theoretical foundations for causal inference over time. Our main focus is on \acp{SdE}, which represent the most complex type among the equations considered. 

First, we briefly summarize the results from \citet{Mooij2013deterministic_case,bongers2018causal,DBLP:journals/corr/abs-2106-12430} regarding \acp{OdE} and \acp{RdE}(\cref{sec:equilibration}). Then, we extend these results to \acp{SdE}, providing a formal definition of causal interventions in this context (\cref{subsec:causal_int_sde}). Finally, we demonstrate how a \ac{SCM} can serve as a simplified representation of the asymptotic behavior of an underlying dynamical system (
\cref{subsec:mapping_sde_scm}).%extending from~\citep{bongers2018causal} from continuous-time to discrete-time systems and further relaxing its assumptions to tackle \acp{SdE}. 
%
% \paragraph{Motivating example} 
% Answering interventional queries formulated over time, such as ``If I receive a vaccination, what level of protection can I expect after one week and over six months?'' requires a higher expressiveness compared to \acp{SCM}, which can address most ``If I receive a vaccination, what level of protection can I expect?''. Causal inference over time aims to precisely cover this gap, i.e., being able to thoroughly reason on this type of question by adopting a perspective that embraces the temporal evolution. 
\subsection{Stochastic Processes Equilibration - \acp{dE}}\label{sec:equilibration}
Let $\X$ be a \ac{DSP}.  
We say that $\X$ \textit{(strongly) equilibrates} if $\X_t$ converges to a random variable $\Xinf$ for $t\to\infty$ almost surely, i.e., 
\begin{equation*}
    \Prob(\limsup_{t\to\infty} \{\omega\in\Omega:|\X_t(\omega)-\Xinf(\omega)|>\varepsilon\})=0\quad \forall \varepsilon>0, 
\end{equation*}
\text{denoted by $\X_t\overset{a.s}{\to}\Xinf$}. %
We call $\Xinf$ the \textit{equilibrium state} of $\X$.
Moreover, we say that $\X$ \textit{weakly equilibrates} if $\X_t$ converges in distribution to a random variable $\Xinf$, i.e., for all $x$ for which $F_\X$ (the CDF of $\X$) is continuous, 
\begin{equation*}
    \lim_{t\to\infty}F_{\X_t}(x)=F_{\Xinf}(x),\quad \text{denoted by $\X_t\indist\Xinf$}.
\end{equation*}
In general, $\X_t\overset{a.s.}{\to}\Xinf \Rightarrow \X_t\indist\Xinf$. 
\begin{remark}
 If $\X$ is a solution of an \ac{RdE}, almost sure convergence is achieved if and only if, for almost all $\omega\in\Omega$, the solution of the \ac{OdE} in \Eqref{eq:RDE_sample_path_solution} converges asymptotically to a fixed value $\Xinf(\omega)$. 
 On the other hand, a \ac{DSP} that is a solution of a \ac{SdE} (for convenience, we recall \Eqref{eq:sde}: $
\X_t = \bm{f}(\X_{<t}) + \bm{g}(\X_{<t}) \odot \vectornoise_t$) cannot equilibrate unless the term $g(\X_{<t})$ converges as $t\to\infty$ to $0$, thus transforming the equation into a \ac{RdE}. Apart from this scenario, no solutions of a \ac{SdE} can strongly equilibrate.
\end{remark}
\begin{example}[\ac{RdE} equilibration]
    Consider \cref{ex:rde}. Assuming  that the inequality $a(\omega)<1$ holds $\forall \omega\in\Omega$, we can write 
    \begin{equation*}
        X_t\overset{a.s}{\to}X_\infty=b(\Omega)\sum_{i=0}^{\infty}a(\Omega)^i = \frac{b(\Omega)}{1-a(\Omega)}.
    \end{equation*}
\end{example}

\subsection*{Equilibration as a map from \acp{RdE} to \acp{SCM}}\label{app:equilibration_RDE_SCM}
The process equilibration allows defining a map that associates each \ac{RdE} with an \ac{SCM}. 
Furthermore, this map preserves the semantic of intervention, in the same sense of Section~\ref{subsec:mapping_sde_scm}. Visually, the diagram in Figure~\ref{fig:simple_diagram_commutes} commutes, i.e., applying equilibration before or after intervention doesn't change the resulting $\X^\intervention_\infty$.

\begin{figure}[t!]
    \centering
    \begin{tikzpicture}
    [scale=0.8, every node/.style={rectangle split, rectangle split parts=2,rectangle split horizontal=true, rounded corners=0.1cm, minimum height=1cm, align=center,scale=0.8}]

    % Nodes
    \node[draw] (1) at (0,2) {\nodepart[text width=0.6cm]{one}$\dequation$\nodepart[text width=0.6cm]{two}$\X_t$};
    \node[draw] (2) at (4,2) {\nodepart[text width=0.6cm]{one}$\scm$\nodepart[text width=0.6cm]{two}$\X_\infty$};
    \node[draw] (3) at (0,0) {\nodepart[text width=0.6cm]{one}$\dequation^\intervention$\nodepart[text width=0.6cm]{two}$ \X_t^\intervention$};
    \node[draw] (4) at (4,0) {\nodepart[text width=0.6cm]{one}$\scm^\intervention$\nodepart[text width=0.6cm]{two}$ \X_\infty^\intervention$};

    \node [] at (2.14,1) {$=?$};

    % Arrows
    \draw[|->,shorten >=4pt,shorten <=4pt,color=MidnightBlue] (1) -- (2) node[midway, above] {$t \to \infty$};
    \draw[|->,shorten >=4pt,shorten <=4pt,color=BrickRed] (1) -- (3) node[midway, left] {$\intervention$};
    \draw[|->,shorten >=4pt,shorten <=4pt,color=MidnightBlue] (2) -- (4) node[midway, right,xshift=0.33cm] {$\intervention$};
    \draw[|->,shorten >=4pt,shorten <=4pt,color=BrickRed] (3) -- (4) node[midway, below] {$t \to \infty$};

    % Curved arrows
    \draw[thick,-stealth,in=0,out=90,shorten >=4pt,shorten <=4pt,color=MidnightBlue] (1.3,1.5) .. controls +(0:1.1cm) and +(90:1.1cm) .. (2.7,0.5);
    \draw[thick,-stealth,in=-90,out=180,shorten >=4pt,shorten <=4pt,color=BrickRed] (1.3,1.5) .. controls +(-90:1.1cm) and +(180:1.1cm) .. (2.7,0.5);
\end{tikzpicture}
\caption{\textbf{\ac{RdE} Equilibration.} Illustration of the relationship between \acp{RdE} \& \acp{SCM} and the effect of causal interventions on them. Applying equilibration (here denoted by $t \to \infty$) before or after intervention does not change the final resulting distribution $\X_\infty^\intervention$, i.e., the diagram commutes. For \acp{SdE}, when variables are localized in times,  commutation is not preserved, that is $(\X_\infty)^{\intervention_{\scm}} \neq (\X^{\intervention_{\dequation}})_\infty$.}\label{fig:simple_diagram_commutes}
\end{figure}
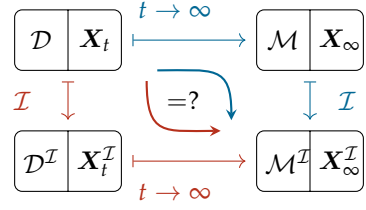

%We treat the cases of \acp{RdE} and \acp{SdE} separately to clarify their differences and avoid using overly general formalism.

Consider a Simple \ac{RdE} \(\dequation\) and its solution \(\X\). Let \(\Xinf\) be the random variable which \(\X\) equilibrates to. 
Then it is possible to associate an \ac{SCM} \(\scm_{\dequation}=(\rmF,\vectorNoise)\) where \(\vectorNoise=\vectornoise_{\dequation}\) and \(\rmF\) is defined through the equilibrium relation
\begin{equation}
    \Xinf = f_{\dequation}(\Xinf(<t,\cdot),\vectorNoise)=\tilde{f}(\Xinf,\vectorNoise),
\end{equation}
where \(\Xinf(t,\cdot)=\Xinf \: \forall t\) is the constant stochastic process associated with \(\Xinf\). In other words, \(\scm_\dequation\) fully inherits the functional relationships of \(\dequation\), substituting dynamic variables with their stationary counterparts. A proof for the general case in continuous-time can be found in ~\citep{bongers2018causal}. The discrete-time variant follows directly by applying methods of numerical approximation~\citep{hildebrand1987introduction,jornet2023theory}. For further discussion on the commutativity of the diagram in \ref{fig:simple_diagram_commutes}, see \cite{dash2003caveats, DBLP:conf/aistats/Dash05}.

\subsection{Causal Interventions on \acp{SdE}}\label{subsec:causal_int_sde}

We define an intervention $\intervention$ on a \ac{SdE} $\dequation$ as a modification of one or more component equations denoted by the mapping:
\begin{equation}\label{eq:sde_intervention}
\intervention:  \, f_i(\parents[i]_{<t}) + g_i(\parents[i]_{<t})\odot\vectornoise_t \quad \longmapsto \quad
\tilde{f_i}(\tilde{\parentssym}_{<t}^{i}) + \tilde{g_i}(\tilde{\parentssym}_{<t}^{i})\odot \vectornoise_t, \; \forall t\geq t_\intervention
\end{equation}
where \(\tilde{\parentssym}_{<t}^{i}\subseteq\parents[i]_{<t}\). 
Unlike \acp{SCM}, the intervention applies \textit{starting from a specific time} $t_\intervention$. 
In other words, the process follows the original equations for $t<t_\intervention$ and the modified ones for \(t\geq t_\intervention\). We denote the modified \ac{SdE} as $\dequation^\intervention$ to 
%We 
generalize \cref{eq:causal_effect} to account for time. To differentiate between interventions on a \ac{SCM} $\scm$ and on a \ac{SdE} $\dequation$, $\intervention_{\scm}$ and $\intervention_{\dequation}$ will be respectively adopted when necessary.

\begin{definition}[Causal Effect over time ($\ce{}_t$)]\label{def:ce_t_discrete}
    Let \(\X\) be a solution of a specific \ac{SdE} $\dequation$. We define the causal effect at time $t$ of an intervention $\intervention$ as
    \begin{equation}\label{eq:causal_effet_time}
        \ce{}^\intervention_t := \expected{\X^\intervention_t-\X_t|\X_{<t_\intervention}}, %\:|\: C},
    \end{equation}
where \(\X^\intervention\) is the solution of the modified \ac{SdE} $\dequation^\intervention$.
\end{definition}
The interpretation of  $\ce{}^\intervention_t$ is closely related to the causal effect of an intervention on a \ac{SCM} (\ref{eq:causal_effect}), $\ce{}^\intervention$: while the latter measures the causal effect of an exogenous intervention, \(\ce_t\) does so for any time step $t$ of the \ac{DSP}, i.e.,~\textit{it measures the causal effect of an intervention over time}. 
%, \textit{considering a certain time has passed since the intervention}.  
%
%on the causal decentants of on the observed variables regarding a specific exogenous intervention to the system, \(\ce_t\) does so on a finer scale, \textit{considering a certain time has passed since the action}. 
%
Importantly, as we will show in the next section, \(\ce{}^\intervention_t \to \ce^\intervention{}\) as \(t\to\infty\), i.e.,
there is an asymptotic correspondence between the two quantities.

\subsection*{Equilibration as a map from \acp{SdE} to \acp{SCM}}\label{app:equilibration_SDE_SCM}
Associating an \ac{SdE} with an \ac{SCM} involves two key steps. The first concerns the conditions under which the dynamical system converges to a stationary (observational) distribution. The second involves constructing a mapping between interventions in the two models and establishing their equivalence.   

In this chapter, we focus on the first aspect and discuss the second only in the linear case. While the existence of a suitable \ac{SCM} appears intuitively plausible in general, explicitly characterizing it requires mathematical tools beyond the scope of this chapter. For nonlinear systems, we will investigate this issue further in the next chapter (\ref{ch:CIcontinuous_time}).  

\paragraph{Existence of an observational limit distribution}
Several works discuss the existence and uniqueness of a limit distribution for stochastic systems ~\citep{arnold1995random, bhattacharya2003random, loskot1995limit}.
%and  positive results are obtained also where systems with much weaker assumptions than those within the scope of this work are considered, e.g., non i.i.d. $\vectornoise_t$. 
Here, we adapt the following theorem from \citep{lasota1989stochastic} to our formulation in \Eqref{eq:sde}:
%see also ~\citep{loskot1995limit}),
%
\begin{theorem}\label{thm:sdetodistribution}[Adapted \Eqref{eq:sde} from~\citep{lasota1989stochastic}]
Consider a \ac{SdE}\footnote{
Theorem~\ref{thm:sdetodistribution} considers Markovian systems, while our formulation allows for $f$ and $g$ to depend on multiple past states. As long as the states considered are finite in number, i.e., $s$ depends on $\X_{t-1},\X_{t-2},\dots,\X_{t-p},\: p<\infty$, it is always possible to represent the system as if it were Markovian, through the transformation $\bm{Z}_t := [\X_t,\X_{t-1},\dots,\X_{t-p+1}]$. Given this fact, for ease of exposition, we will consider strictly Markovian systems in the discussion.} 
described by 
\begin{equation*}
    \X_t = f(\X_{t-1}) + g(\X_{t-1})\odot\vectornoise_t.
\end{equation*}
If the following conditions hold: i) $f,g:\mathcal{X}\to\mathcal{X}$ are continuous in $\X_{t-1}$, and $\mathcal{X}\subset\R^{\sizex}$ is closed; ii) 
     $\vectornoise_t$ is i.i.d. white noise; and iii)
     for every $t$,
% \begin{enumerate}
%     \item $f,g:\mathcal{X}\to\mathcal{X}$ are continuous in $\X_{t-1}$, and $\mathcal{X}\subset\R^{\sizex}$ is closed.
%     \item $\vectornoise_t$ is i.i.d. white noise.
%     \item For every $t$,
    \begin{align*}
        &(\text{iii}.1) 
        \quad \forall x,y\in\mathcal{X},\: x\neq y\\\quad&\expected{\norm{f(x)+g(x)\odot\vectornoise_t-f(y)-g(y)\odot\vectornoise_t}}< \norm{x-y} \\\\
        &(\text{iii}.2) 
        \quad \forall x\in\mathcal{X}\\
        &\expected{\norm{f(x) + g(x)\odot\vectornoise_t}^2}\leq \alpha \norm{x}^2 + \beta ,\: \alpha<1,
    \end{align*}
%\end{enumerate}
where $\norm{\cdot}$ is the Euclidean norm in $\sR^\sizex$. 
Then, there exists a unique invariant distribution $p(\Xinf)$ to which $p(\X_t)$ converges as $t\to\infty$, regardless of the initial state $\X_0$.
\end{theorem}

%We identify the family of \acp{SCM} that share the same graph of the \SDE{} and are consistent with $\Xinf$. By considering the interventional distributions one by one, we progressively narrow down the permissible family of \acp{SCM} until obtaining a class of interventional equivalence. 
%
%To verify that interventions are consistent across the two frameworks, 
%\textit{it is necessary to specify the chosen intervention class}. In the next section, we will proceed in this direction, translating the theoretical formulation into a practical implementation.
%To asymptotically relate \ac{DSP}s and \acp{SCM}, the theorem's assumptions must be verified for both the observational and interventional \acp{DSP}. 
%Under this condition, we obtain a set of invariant observational and interventional distributions  $p(\Xinf)$ and $p(\Xinf^\intervention)$.  %\Xinf,\{\Xinf^\intervention\}_\intervention\}$. 
%
%In other words, it is in general \textit{possible to identify an \ac{SCM} consistent at the observational and interventional levels with the causal dynamics of $\X_t$ ad $\X^\intervention_t$ at equilibrium}. \textcolor{red}{Proving commutation in non linear setting is future work. for now we have just the first step (existence of a limit distribution and set of candidate SCMs. Think how to conclude this. probably make this point clear at start of section A.4.}

\paragraph{Rediscussing \cref{thm:sdetodistribution} in terms of Lipschitz continuity}

For a more intuitive interpretation of \cref{thm:sdetodistribution}, we show that condition (iii) holds under Lipschitz continuity\footnote{A function $f$ is Lipschitz with constant $L$ if for all $x,y$, it holds that $\|f(x)-f(y)\| \leq L \|x-y\|$.} of $f$ and $g$. %In particular, assuming $f$ and $g$ are Lipschitz with constants $L_f$ and $L_g$, 
% %
Given $f,g:\sR^\sizex\to\sR^\sizex$ Lipschitz continuous with Lipschitz constants $L_f,L_g$, we want to find sufficient conditions on $L_f$ and $L_g$ such that, for every fixed $x,y$,
\begin{align*}
         (1) \quad &\expected{\norm{f(x)+g(x)\odot\vectornoise-f(y)-g(y)\odot\vectornoise}}< \norm{x-y}. \\
         (2) \quad &\expected{\norm{f(x) + g(x)\odot\vectornoise}^2}\leq \alpha \norm{x}^2 + \beta ,\: \alpha<1.
    \end{align*}
    $\vectornoise$ being a random vector s.t. each of its component satisfy $\expected{\noise}=0,\expected{|\noise|^2}<\infty$. For $(1)$ we calculate
    \begin{align*}
        &\expected{\norm{f(x)+g(x)\odot\vectornoise-f(y)-g(y)\odot\vectornoise}}\leq \\
        &\norm{f(x)-f(y)}+\norm{g(x)-g(y)}\: \expected{\norm{\vectornoise}}\leq \\
        &L_f \norm{x-y} + L_g \norm{x-y} \: \expected{\norm{\vectornoise}}= \\
        &(L_f + L_g \: \expected{\norm{\vectornoise}})\norm{x-y}<\norm{x-y} \\
        &\iff \quad \boxed{L_f + L_g \: \expected{\norm{\vectornoise}}< 1}
    \end{align*}

    The same procedure for $(2)$ yields
    \begin{align*}
        &\expected{\norm{f(x) + g(x)\odot\vectornoise}^2} = \\
        &\expected{\norm{f(x)}^2} + \expected{\norm{g(x)\odot\vectornoise}^2}= \quad \{\hat{f}(x):=f(x)-f(0)\} \\
        &\expected{\norm{\hat{f}(x)+f(0)}^2}+\expected{\norm{\hat{g}(x)\odot \vectornoise+g(0)\odot \vectornoise}^2}\leq \\
        &L_f^2 \norm{x}^2 + 2 |\innerproduct{\hat{f}(x)}{f(0)}|  + \norm{f(0)}^2 + \\
        &L_g^2\norm{x}^2 \sigma^2 + 2|\innerproduct{\hat{g}(x)}{g(0)}|\sigma^2 + \norm{g(0)}^2 \sigma^2 \leq \quad \{\sigma^2:=\expected{\norm{\vectornoise}^2}\} \\
        &(L_f^2+L_g^2\sigma^2)\norm{x}^2+2(L_f\norm{f(0)} +L_g\norm{g(0)}\sigma^2)\norm{x} + \\
        &\norm{f(0)}^2 + \norm{g(0)}^2 \sigma^2 =: a\norm{x}^2+m\norm{x}+b.
    \end{align*}
    Observing that every polynomial $ax^2 + mx + b$ with $a > 0$ can be bounded above by $\alpha x^2 + \beta$ assuming $\alpha > a$ and provided that $\beta$ is sufficiently large, we finally obtain
    \begin{equation*}
    \begin{split}
        &a\norm{x}^2+m\norm{x}+b \leq \alpha x^2 + \beta, \: \alpha<1 \\
        &\iff \quad \boxed{L_f^2+L_g^2\expected{\norm{\vectornoise}^2} < 1}
    \end{split}
    \end{equation*}

In the simple case in which $g\equiv 1$, we obtain the inequality $L_f<1$, i.e.,~$f$ \textit{is a contraction}. 
Condition \((2)\) disappears as $L_f^2<L_f<1$.
In the general case, we interpret $f$ as a contraction map on $\Xinf$ and $g \odot \vectornoise_t$ as a diffusion process. The balance between these processes generates the invariant distribution.

\subsection{Mapping \acp{SdE} to \acp{SCM} - The linear case}\label{subsec:mapping_sde_scm}
Given an \ac{SdE} $\dequation$ and its solution $\X$, we study the conditions on $\dequation$ such that:
i) $\X_t$ converges in distribution to $\X_\infty$ as $t \to \infty$; and
ii) there exists an \ac{SCM} $\scm$ such that $\X_\infty$ is a solution of $\scm$ and, for every intervention $\intervention$, it holds that $(\X_\infty)^{\intervention_{\scm}} =(\X^{\intervention_{\dequation}})_\infty$. While i) is automatically satisfied by any finite memory stationary process, ii) requires more careful analysis, as discussed below. 

%its application on $\scm$ gives rise to the same intervened distribution that would be obtained by first intervening on $\dequation$ and letting the process converge as $t\to\infty$. In other words, the diagram shown in~\ref{fig:diagram_not_commute} must commute

%, it preserves the same convergence property after any intervention, i.e., $\X^\intervention_t$ converges in distribution to $\X_\infty^\intervention$. 
%
\paragraph{A negative result from \citep{janzing2018structural}} 
Consider the stable bivariate system defined by the equations $X_t = \varepsilon^{x}_t, Y_t= 0.5 \cdot X_{t-1} + \varepsilon^{y}_t$. For every \( t \), \( X_t \) and \( Y_t \) are independent of each other. Consequently, the joint distribution $p(X_t,Y_t)$ cannot capture the causal dependencies of the system (\textit{$X$ causes $Y$}). The lack of causal information in the cross-sectional dimension arises because the variables are \textit{localized in time}; their values change rapidly, leading to minimal or no correlation with their past values. On this specific point \citet{janzing2018structural} provide an explicit negative result: without first making the variables de-localized in time, there is no \ac{SCM} that can capture the \ac{SdE}. In fact, our definition of intervention (\ref{eq:sde_intervention}) \textit{acts on a variable of the system for a prolonged and indefinite period}.

\paragraph{$\gT$-transformation} To overcome this limitation  \citet{janzing2018structural} propose a transformation of $\X_t$ based on a frequency analysis of the time series.\footnote{Our $\Z_t$ (\ref{eq:long_run_normalized_mean}) can also be interpreted as a form of discrete Fourier transform of the time-series $\X_{1:t}$ with frequency zero.} Instead, our choice is inspired by the long-run normalized mean via the transformation $\gT: \text{\ac{DSP}} \mapsto \text{\ac{DSP}}$ defined by
\begin{equation}\label{eq:long_run_normalized_mean}
\gT(\X)_t := \Z_t = \bm{\mu} + \frac{1}{\sqrt{t}}\sum_{i=1}^t (\X_i-\bm{\mu}),
\end{equation}
%
%In contrast we need to treat variables that are \textit{de-localized in time},  and only under this condition we can hope to map the process into a causal informed \ac{SCM}
%
where $\bm{\mu} := \expected{\X}$.\footnote{The expectation here is taken over time as well. Nonetheless, for stationary processes, this simplifies to \(\expected{\X} = \expected{\X_t}\) for all \(t\).} 
Moreover, $\expected{\Z_\infty} = \expected{\X_\infty}= \bm{\mu}$ so that for every intervention $\intervention$, $\ce{}^\intervention_\infty$ (\ref{eq:causal_effet_time}) \textit{yields the same values}. However, unlike $\X$, $\Z$ can be mapped into an \ac{SCM} that precisely models \textit{its distribution shift over any intervention}, thereby satisfying property ii), represented as the commutativity of the diagram in \ref{fig:diagram_z_commutes}. \\

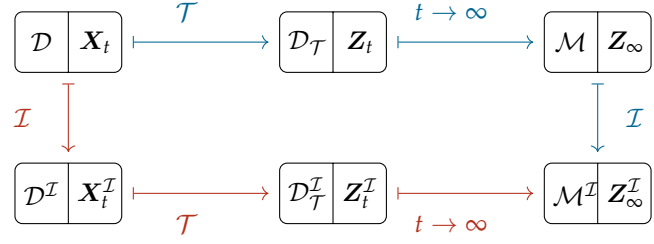
\begin{figure}[t!]
    \centering
\begin{tikzpicture}
    [scale=1.0, every node/.style={rectangle split, rectangle split parts=2,rectangle split horizontal=true, rounded corners=0.1cm, minimum height=1cm, align=center,scale=0.8}]

    % Nodes
    \node[draw] (1) at (0,2) {\nodepart[text width=0.6cm]{one}$\dequation$\nodepart[text width=0.6cm]{two}$\X_t$};
    \node[draw] (2) at (3.5,2) {\nodepart[text width=0.6cm]{one}$\dequation_\gT$\nodepart[text width=0.6cm]{two}$\Z_t$};
    \node[draw] (3) at (7,2) {\nodepart[text width=0.6cm]{one}$\scm$\nodepart[text width=0.6cm]{two}$ \Z_\infty$};
    \node[draw] (4) at (0,0) {\nodepart[text width=0.6cm]{one}$\dequation^\intervention$\nodepart[text width=0.6cm]{two}$\X_t^\intervention$};
    \node[draw] (5) at (3.5,0) {\nodepart[text width=0.6cm]{one}$\dequation^\intervention_\gT$\nodepart[text width=0.6cm]{two}$ \Z_t^\intervention$};
    \node[draw] (6) at (7,0) {\nodepart[text width=0.6cm]{one}$\scm^\intervention$\nodepart[text width=0.6cm]{two}$ \Z_\infty^\intervention$};

    % Arrows
    \draw[|->,shorten >=4pt,shorten <=4pt,color=MidnightBlue] (1) -- (2) node[midway, above] {$\gT$};
    \draw[|->,shorten >=4pt,shorten <=4pt,color=MidnightBlue] (2) -- (3) node[midway, above] {$t \to \infty$};
    \draw[|->,shorten >=4pt,shorten <=4pt,color=BrickRed] (4) -- (5) node[midway, below] {$\gT$};
    \draw[|->,shorten >=4pt,shorten <=4pt,color=BrickRed] (5) -- (6) node[midway, below] {$t \to \infty$};
    \draw[|->,shorten >=4pt,shorten <=4pt,color=BrickRed] (1) -- (4) node[midway, left] {$\intervention$};
    \draw[|->,shorten >=4pt,shorten <=4pt,color=MidnightBlue] (3) -- (6) node[midway, right, xshift=0.33cm] {$\intervention$};

    % Curved arrows
    %\draw[thick,-stealth,in=0,out=90,shorten >=4pt,shorten <=4pt] (1.3,1.5) .. controls +(0:1.1cm) and +(90:1.1cm) .. (2.7,0.5);
    %\draw[thick,-stealth,in=-90,out=180,shorten >=4pt,shorten <=4pt] (1.3,1.5) .. controls +(-90:1.1cm) and +(180:1.1cm) .. (2.7,0.5);
\end{tikzpicture}
\caption{\textbf{$\gT$-transformation.} $\gT$-transformation transfers causal information from the temporal to the cross-sectional dimension, and thus to the joint distribution $P(\Z_t)$. The diagram commutes, i.e., red and blue paths produce the same result.}
\label{fig:diagram_z_commutes}
\end{figure}

It is important to clarify that $\Z_t$ is not the process of interest, and the focus of the causal analysis remains on $\X_t$. However, due to the equivalence of long-run causal effects calculated in both processes and the ability to associate $\Z_t$ with the \ac{SCM} that models these effects, $\Z_t$ serves as a convenient intermediate mathematical tool. To demonstrate how this transformation ensures these desirable properties, we will focus on the subclass of linear systems, particularly on \ac{VAR} models. The reason for this choice is twofold. First, linear models, despite their simplicity, are still on par performance-wise with state-of-the-art \ac{ML} based forecast techniques~\citep{linear_analysis_toner_2024}, in particular when dealing with stochastic time series~\citep{parmezan2019evaluation}. Second, the mathematical treatment of interventions and the estimation of causal effects is particularly straightforward to implement and interpret, making this a useful first step for a possible extension to the nonlinear case.

%%Notably, it converges to a zero mean process with covariance matrix that is the discrete fourier transform of the Autocovariance Generating Function valutata in zero, ossia la varianza della componente costante del processo in analisi}

\section{From VARs to SCMs} \label{sec:interventions}
In this section, we provide implementations of two types of causal interventions, leveraging the strengths of the \ac{VAR} framework (\cref{subsec:causal_interventions}).
Then, we show that linear \acp{SCM} can model the long-term effects of stable \acp{VAR}, explaining the properties of its \ac{DSP} equilibrium (\cref{subsec:mapping_var_scm}). 

\subsection{Implementation of Causal Interventions}\label{subsec:causal_interventions}

Different application scenarios may need different types of interventions. Consider a government's fiscal policy. 
% Requiring that all citizens pay the same amount, say $10,000$ euros per year, is unrealistic. 
In such a setting, a feasible approach would be to implement an \emph{additive intervention} in the form of an annual tax increase of, e.g., $300$ euros per household on top of existing taxes. %This type of intervention can be modeled by $\intervention^{a}$. 
Alternatively, in other scenarios, e.g., when studying the effect of the \textit{key European Central Bank's interest rate}~\citep{belke2007ecb}, a more natural choice is to implement a \emph{forcing intervention} that enforces the convergence of an observed variable (e.g., interest rate) to a target value. 
%Conversely, when studying the effect of the \textit{key European Central Bank's interest rate}\footnote{\url{https://www.ecb.europa.eu/stats/policy_and_exchange_rates/key_ecb_interest_rates}} ~\citep{belke2007ecb}, fixed at a specific value for a certain period, $\intervention^{}$ would be the most natural choice.
In the following, we propose an implementation for \acp{VAR} of these two forms of interventions, % for linear autoregressive models in increasing order of complexity, 
showing their effects on the system and discussing their stability conditions.

\subsection*{Additive Interventions}\label{subsec:forcing}

Given a stable \ac{VAR}(\order) as in \cref{eq:var_lagoperator}, we define an additive intervention $\intervention_{a}$ at time $t_\intervention$ with force $\bm{F}$ as the mapping:
\begin{equation}\label{eq:forcing_intervention}
\begin{split}
   \intervention_{a}: \,  &\A(L)\X_t = \bm{\nu} + \vectornoisevar_t \quad \longmapsto \\
    &\A(L)\X_t = \bm{\nu} + \vectornoisevar_t + \mathds{I}(t\geq t_\intervention)\bm{F},
    \end{split}
\end{equation}
where $\mathds{I}(t\geq t_\intervention)$ is the indicator function, which equals $1$ if $t\geq t_\intervention$, otherwise $0$. 
In other words, we perform a translation while keeping the process dynamics unchanged.
%When necessary, we denote this type of intervention as $\intervention^F_{t_0}(\bm{F})$. 
%
In such case, the temporal causal effect $\ce{}_t$ is deterministic and takes values $\ce_{t}=0$  for $t <t_\intervention$ while, for %$t = t_\intervention +k$ with 
$k\geq 0$:
\begin{equation*}
    \ce_{t_\intervention+k} %= \expected{\X_{t_0+k}^\intervention-\X_{t_0+k}} 
    = \sum_{l=0}^k \response_l \bm{F}, 
\end{equation*}
where $\Phi$ is the \emph{impulse response function} of the \ac{VAR} model. Refer to \cref{app:svarproperties} for further details.
%
%$\intervention^F_{t_0}(\bm{F})$ is well-defined because it is %compositional and commutative. 
%
%Formally, the relation $\intervention^F_{t_0}(\bm{F}+\bm{F}') = \intervention^F_{t_0}(\bm{F})\odot \intervention^F_{t_0}(\bm{F}') = \intervention^F_{t_0}(\bm{F}')\odot \intervention^F_{t_0}(\bm{F})$ holds.
\begin{remark*}
    For this type of intervention, $\ce_t$ is deterministic and does not depend on the specific trajectory. The same property can be observed on the linear \ac{SCM} associated with the process, defining the intervention in a similar way: $\tilde{\X}=\tilde{\A} \tilde{\X}+\tilde{\vectornoisevar}$ changes into $\tilde{\X}=\tilde{\A} \tilde{\X}+\bm{F}+\tilde{\vectornoisevar}$.
\end{remark*}

\paragraph{Stability}\label{sec:stability_additive_discrete}
Additive intervention preserves the stability regardless of the value of $\bm{F}$, since $\A(L)$ does not change. See \cref{app:svarproperties} for stability conditions of \acp{VAR}.

\subsection*{Forcing Interventions}\label{subsec:do-like}

We define a forcing intervention $\intervention_f$ at time $t_\intervention$ with force $\bm{F}$ and target value $\hat{\X}$ as:
\begin{equation}\label{eq:do-like}
\begin{split}
    \intervention_{f}&: \, \A(L)\X_t = \bm{\nu} + \vectornoisevar_t \quad \longmapsto \\
    &\A(L)\X_t = \bm{\nu} + \vectornoisevar_t + \mathds{I}(t\geq t_\intervention)\bm{F}\odot(\hat{\X}-\X_t).
\end{split}
\end{equation}
We assume $\bm{F}$ to be positive in each component. 
This intervention acts as an attraction towards $\hat{\X}$, and $\bm{F}$ modulates the intensity of the attraction force. Applying an intervention on a single component $\Xcomp[i]$ toward the fixed value $\hat{X}$ and letting $F^{i}\to +\infty$ yields the do operator $do(\Xcomp[i] = \hat{X})$. We refer to~\citep{Mooij2013deterministic_case} for a detailed discussion of this point.
%$\intervention^{\tilde{do}}$ is both compositional and commutative in the same manner as $\intervention^F$.

\paragraph{Stability}
Forcing interventions $\intervention_{f}$ perturb the system dynamics by modifying the operator $\A(L)$. 
Specifically, by shifting the term $\bm{F} \odot \X_t$ to the left of the equation and rewriting it in matrix form as $\bm{F}_{diag}\X_t$, we obtain $\tilde{\A}(L):=\A(L)+\bm{F}_{diag}$. 
Hence, the stability of the intervened system is not guaranteed, and it is necessary to verify \textit{that all the eigenvalues of} $\tilde{\A}(L)$ \textit{are still inside the unit circle.} Intuitively, the stability of an observational system often relies on negative feedback loops. Fixing one variable can disrupt this balance, leading to runaway behavior. For example, turning off a pressure release valve in a pressurized tank can cause the pressure to build up uncontrollably, eventually leading to an explosion.

\begin{example}[Inverted Pendulum]\label{example:stability_not_granted}
When applying a forcing intervention, the intensity of $\bm{F}$ can play an important role. To show this, consider the toy model depicted in \cref{fig:inverted_pendulum} and described by the\ac{VAR}$(1)$ equation
\begin{equation}\label{eq:pendulum_matrix}
    \begin{bmatrix}
        \Xcomp[1]\\\Xcomp[2]
    \end{bmatrix}_t
    = 
    \sqrt{2}\begin{bmatrix}
        0 & -0.5 \\
        0.5 & 1
    \end{bmatrix} 
    \begin{bmatrix}
        \Xcomp[1]\\\Xcomp[2]
    \end{bmatrix}_{t-1}
    + 
    \begin{bmatrix}
        \vectornoise^{1}\\\vectornoise^{2}
    \end{bmatrix}_t.
\end{equation}
the roots of $det(I-\A_1z)$ are $z_{1,2}=\sqrt{2}>1$, indicating that the system is stable. Applying $\intervention_f$ to component $\Xcomp[1]$, we need to examine
\begin{equation*}
    det\left(
    \begin{bmatrix}
        1+F & 0\\
        0 & 1
    \end{bmatrix}-
    \sqrt{2}\begin{bmatrix}
        0 & -0.5\\
        0.5 & 1
    \end{bmatrix}z
    \right)
\end{equation*}
which have roots $z_{1,2}=\sqrt{2}\left(1+F\pm \sqrt{F^2+F}\right)$. It is easy to verify that for $F=1$, one root is less than $1$, and thus the model is no longer stable.
On the other hand, interveening on $\Xcomp[2]$ preserves stability regardless of $F$ intensity; the same procedure yields $z_{1,2}=\sqrt{2}\left(1\pm\sqrt{-F}\right)$, so that $|z_{1,2}|\geq \sqrt{2}>1$ $\forall F>0$. 
\end{example}

\paragraph{Stable \acp{VAR} with at most self-loops are interventionally stable}
It is possible to identify sufficient conditions for $\dequation$ to be always stable under intervention. In particular, it is easy to see that if all matrices $\A_i$ of the model are strictly lower triangular, and thus $\graph_\dequation$ is acyclic, $\X_t$ has a finite Moving Average representation. This property is preserved by any intervention $\intervention_f$.  

Acyclicity can be weakened to allow \textit{self-loops}. Given a \ac{VAR} model such that all its matrices $\A_i$ are (non-strictly) lower triangular, the stability of the intervened system directly follows from the stability of the observational one. To verify this, note that the determinantal polynomial of the model $$det(I-\A_1z-\dots-\A_pz^p) \quad (\text{as in} \: \text{\cref{eq:determinantal}})$$ reduces to the product of diagonal terms since all matrices are triangular, and each i-th root satisfies the relation $p_i(z)=1-\A_1[i,i]z-\dots-\A_p[i,i]z^p=0$. Observing that $p_i(0)=1>0$ and considering the root with the smallest norm $z_{min}$, using the continuity of $p_i$, we can deduce that $p_i(z)>0 \: \forall z$ such that $|z|<|z_{min}|$. Intervening on component $i$ yields $\tilde{p_i}(z)=p_i(z)+F$, from which we can conclude $\tilde{p_i}(z)>F>0$ for all $z$ with norm lower than $z_{min}$, and in particular for all those within the unit circle.

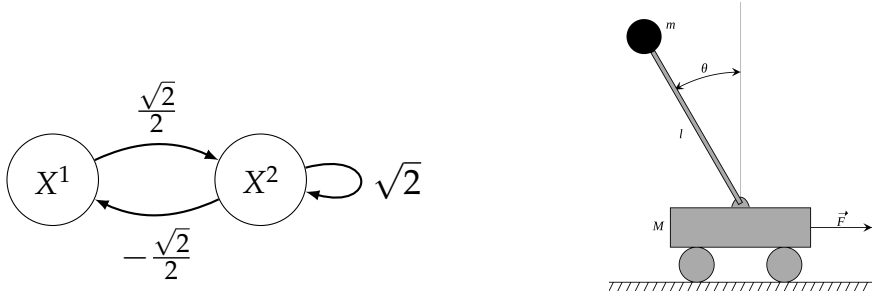
\begin{figure}
    \centering
    \begin{subfigure}[t]{0.49\textwidth}
        \centering
    \begin{tikzpicture}[scale=1.1,
            transform shape,
            vertex/.style={circle, draw, minimum size=1.1cm, inner sep=0pt, font=\small}]
    \node[vertex] (x1) at (0, 2.5) {$X^{1}$};
    \node[vertex] (x2) at (2.5, 2.5) {$X^{2}$};
    
    \draw[edge] (x1) to [bend left=25] node[midway, above] {$\frac{\sqrt{2}}{2}$} (x2);
    \draw[edge] (x2) to [bend left=25] 
    node[midway, below] {$-\frac{\sqrt{2}}{2}$}
    (x1);
    \draw[edge] (x2) to [loop right] 
    node[midway, right] {$\sqrt{2}$}
    (x2);
    \end{tikzpicture}
    \end{subfigure}%
    \begin{subfigure}[t]{0.49\textwidth}
        \centering
        \includegraphics[height=1.5in]{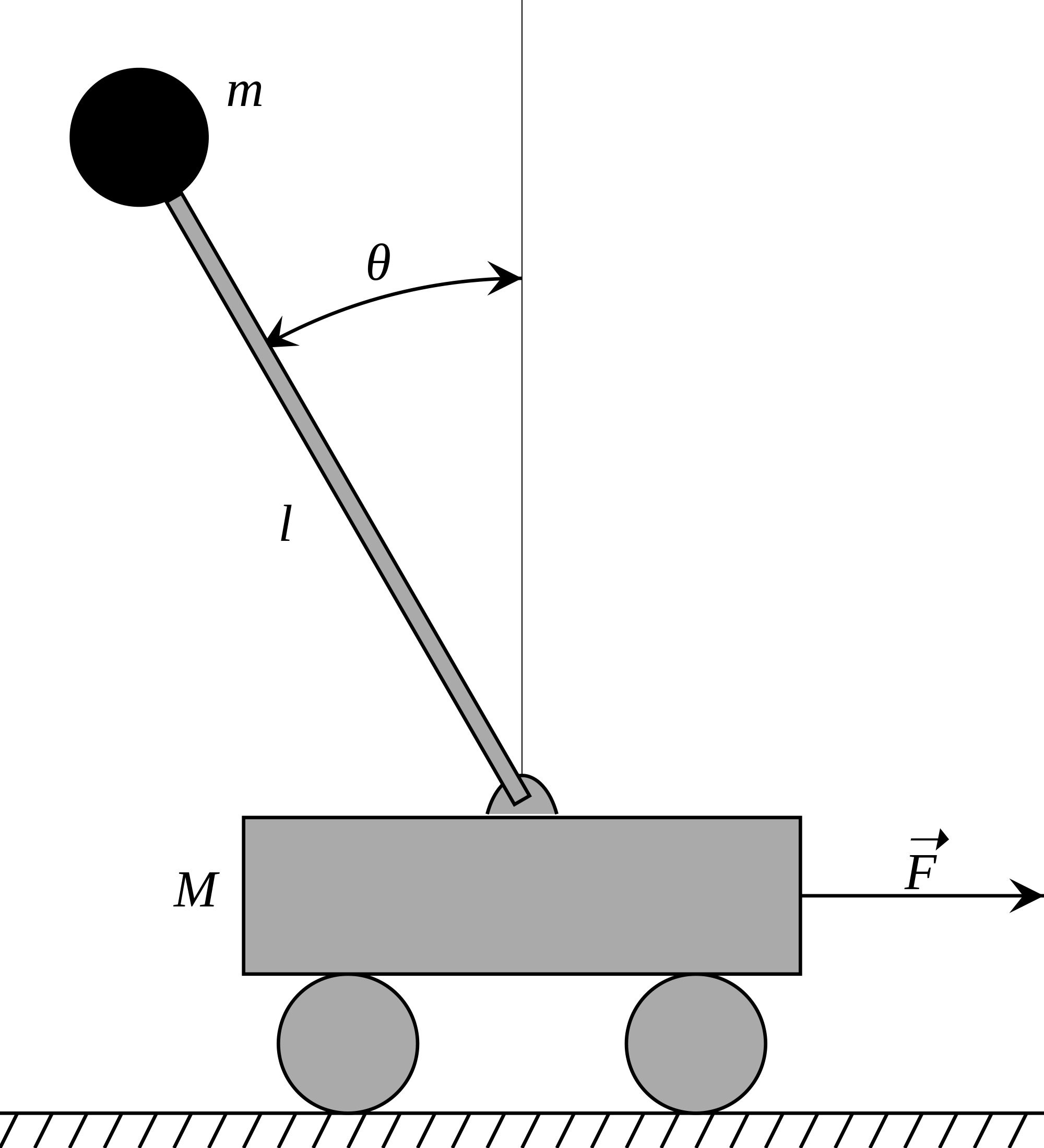}
    \end{subfigure}
    \caption{\textbf{Inverted Pendulum. } (Left) A simple model that lacks interventional stability: \( \Xcomp[1] \) functions as a stabilizer for \( \Xcomp[2] \), which exhibits a divergent dynamic. Intervening on \( \Xcomp[1] \) disrupts the negative feedback loop, making the system unstable. (Right) A concrete example of the model from control theory, the inverted pendulum: \( \Xcomp[1] \) is the position of the cart, while \( \Xcomp[2] \) is the angle \( \theta \); fixing the cart makes the system unstable, i.e., the pendulum falls. As a practical example, consider optimal control applications to nuclear reactors management~\citep{edwards1992robust}.}
    \label{fig:inverted_pendulum}
\end{figure}

%Generally, suppose that a sufficient condition to ensure intervention stability is, in addition to observational stability, the absence of negative feedback loops among the variables of the system.

% \paragraph{Why different interventions}
% Different scenarios may need different types of interventions. Consider, for example, a government's fiscal policy. Requiring that all citizens pay the same amount, say $10,000$ euros per year, is unrealistic. A more feasible approach would be to implement a tax increase, such as an additional $300$ euros annually per household on top of existing taxes. This type of intervention can be modeled by $\intervention^{a}$. Conversely, when studying the effect of the \textit{key European Central Bank's interest rate}\footnote{\url{https://www.ecb.europa.eu/stats/policy_and_exchange_rates/key_ecb_interest_rates}} ~\citep{belke2007ecb}, fixed at a specific value for a certain period, $\intervention^{do}$ would be the most natural choice.
\subsection{Mapping \acp{VAR} to \acp{SCM}}\label{subsec:mapping_var_scm}
%Given a stable \( \ac{VAR}$(p)$ \) model \(\dequation\) that describes the dynamics of a process \( \X_t \), we formulate the following theorem.
We provide the explicit mapping from \acp{VAR} to linear \acp{SCM} in the following theorem.
\begin{theorem}\label{thm:var_to_scm}
    Given a stable \ac{VAR}$(p)$ $\dequation$ defined by \cref{eq:VAR}, 
    there exists a linear \ac{SCM} $\scm$ with structural equations %given by
    \footnote{For simplicity we set $\bm{\nu}=0$, i.e.,~we assume $\expected{\X_t}=0$. The theorem applies in the general case up to a translation of both the \ac{VAR} and the associated \ac{SCM}.} 
    \begin{equation}\label{eq:var_to_scm}
    \begin{split}
        &\tilde{\bm{X}} = \tilde{\A} \tilde{\bm{X}} + \tilde{\vectornoisevar}, 
        \\\text{where }\quad &\tilde{\A} := \left[ \A_1 + \dots + \A_p \right] \text{ and } \tilde{\vectornoisevar} \sim \normal(\bm{0},\boldsymbol{\Sigma}_{\vectornoisevar}),  
    \end{split}
    \end{equation}
    such that, given the transformation $\bm{Z}_t = \frac{1}{\sqrt{t}}\sum_{i=1}^t\X_i$, the following properties hold: 
    \begin{enumerate}
    \item  $\dequation$ and  $\scm$ share the same \emph{causal graph}, i.e., $\graph_{\scm}=\graph_{\dequation}$;\label{samegraph}
    \item The \emph{observational distribution} induced by $\dequation$ at equilibrium $p (\bm{Z}_\infty)$ is equal to the one induced by $\scm$, $p( \tilde{\bm{X}})$;\label{scmissolution}
     \item The \emph{interventional distribution}
     $p (\bm{Z}_\infty^{\intervention_\dequation})$
     is equal to the one induced by the same intervention on $\scm$,
         $p( \tilde{\bm{X}}^{\intervention_{\scm}})$. \label{scminterventioncomutes}
    % \item The \emph{interventional distribution} induced by an intervention ${\intervention_\dequation}$ on the \ac{DSP} at equilibrium 
     % $p (\Xinf^{\intervention_\dequation})$, 
     % agrees with an interventional distribution of the \ac{SCM} 
     %     $p( \tilde{\bm{X}}^{\intervention_{\scm}})$, i.e., 
     % $p (\Xinf^{\intervention_\dequation}) = p( \tilde{\bm{X}}^{\intervention_{\scm}})$. \label{scminterventioncomutes}
    %\item $p (\Xinf) = p_{\dequation} (\Xin) $ is a solution of $\scm$, i.e., $\Xinf= \left[ \A_1 + \dots + \A_p \right]\Xinf+\vectornoise\ $;\label{scmissolution}
    %
    %\item the equilibrium distribution after an intervention over $\dequation$ and the one obtained by the same intervention over $\scm$ are equal, formally $\X^{\intervention_{\scm}} = \Xinf^{\intervention_\dequation}$.
    \end{enumerate}
\end{theorem}

\begin{proof}
Consider the $MA(\infty)$ representation of $\X_t$ (since we can always apply a translation to it without modifying the dynamic of the stochastic part, let us assume for simplicity that \( \expected{\X_t} = 0 \)):
\begin{equation*}
    \X_t = \sum_{l=0}^{\infty} \response_l \vectornoisevar_{t-l}.
\end{equation*}
To capture the long-run relationships of the process, we apply the transformation 
\begin{equation}\label{eq:t-trasformation} \quad 
    \gT: \X_t \mapsto \bm{Z}_t = \frac{1}{\sqrt{t}} \sum_{i=1}^t \X_i,
\end{equation}

By examining the equilibration $\bm{Z}_t \to \bm{Z}_\infty$, we can link a stable \ac{VAR}$(p)$ process to the following \ac{SCM}:
\begin{equation*}
    \tilde{\X} =\tilde{\A} \tilde{\X} + \tilde{\vectornoisevar}, \quad \text{where} \quad  \tilde{\A} := \left[ \sum_{i=1}^p \A_i \right]
\end{equation*} 

By computing the limit of $\Z_t$ we get:
\begin{align*}
    \lim_{t\to\infty}\bm{Z}_t &= \lim_{t\to\infty} \frac{1}{\sqrt{t}}\sum_{i=1}^t \sum_{l=0}^{\infty} \response_l \vectornoisevar_{i-l} = \\
    &\text{(by a.s. absolute convergence)}\\
    &= \lim_{t\to\infty} \frac{1}{\sqrt{t}} \sum_{l=0}^{\infty} \sum_{i=1}^t \response_l \vectornoisevar_{i-l}= \\
    &= \lim_{t\to\infty} \sum_{l=0}^{\infty} \response_l \frac{1}{\sqrt{t}} \sum_{i=1}^t  \vectornoisevar_{i-l}= \\
    &\text{(by dominated convergence)}\\
    &=  \sum_{l=0}^{\infty} \response_l \lim_{t\to\infty} \frac{1}{\sqrt{t}} \sum_{i=1}^t  \vectornoisevar_{i-l}= \\ 
    &\text{(by central limit theorem)}\\
    &\overset{d}{=}\left[ \sum_{l=0}^{\infty} \response_l \right] \tilde{\vectornoisevar}, \quad \tilde{\vectornoisevar} \sim \normal (0,\Sigma_{\vectornoisevar})= \\
    &= \response(1) \tilde{\vectornoisevar} = A(1)^{-1} \tilde{\vectornoisevar} = \left[1-\A_1-\dots-\A_p\right]^{-1} \tilde{\vectornoisevar},
\end{align*}
that is 
\begin{equation}\label{eq:explicit_scm}
\begin{split}
    \bm{Z}_\infty \overset{d}{=} A(1)^{-1} \tilde{\vectornoisevar}, \quad \text{or equivalently} \\
    \bm{Z}_\infty \overset{d}{=} [\A_1+\dots+\A_p] \bm{Z}_\infty + \tilde{\vectornoisevar}
\end{split}
\end{equation}
\end{proof}
For an alternative proof see~\citep[194]{hamilton1994series} and~\citep[429]{anderson1994statistical}, from which we report the main statement:
\begin{theorem}\label{thm:teorem_hamilton}~\citep[429]{anderson1994statistical}
Let $\X$ be a \ac{DSP} defined by
\begin{equation}
    \X_t = \bm{\mu} + \sum_{l=0}^\infty \Phi_i \vectornoisevar_{t-l}
\end{equation}
such that $\vectornoisevar_t$ is white noise and $\sum_{l=0}^\infty |\Phi_i|<\infty$. Then
\begin{equation*}
    \bm{Z}_t = \frac{1}{\sqrt{t}}\sum_{i=1}^{t} (\X_i-\bm{\mu}) \quad \indist \quad \normal(0,\sum_{i=-\infty}^{\infty}\gamma_i) \quad \text{as} \quad t\to\infty,
\end{equation*}

where $\gamma_i$ represents the autocovariance $\Var[\X_t,\X_{t-i}]$. 
\end{theorem}
Exploiting the i.i.d assumption for $\vectornoisevar_t$, $\gamma_i$ admits the representation
$$\gamma_i=\sum_{l=0}^{\infty}\response_{l+i} \Sigma_{\vectornoisevar} \response_l'= \sum_{k=0}^{\infty} \sum_{l=0}^{\infty} \delta_{k,l+i} \response_{k} \Sigma_{\vectornoisevar} \response_l'.$$
Using $\sum_{i=-\infty}^{\infty} \delta_{i,j} = 1 \: \forall j$, we rewrite
\begin{align}
\begin{split}\label{autocovariance_to_response}
    \sum_{i=-\infty}^{\infty}\gamma_i &= \sum_{i=-\infty}^{\infty} \sum_{k=0}^{\infty} \sum_{l=0}^{\infty} \delta_{k,l+i} \response_{k} \Sigma_{\vectornoisevar} \response_l'= \\
    &\text{(absolute convergence)}\\
    &= \sum_{k=0}^{\infty} \sum_{l=0}^{\infty} \sum_{i=-\infty}^{\infty}\delta_{k,l+i} \response_{k} \Sigma_{\vectornoisevar} \response_l'=\\
    &= \sum_{k=0}^{\infty} \sum_{l=0}^{\infty} \response_{k} \Sigma_{\vectornoisevar} \response_l'=\\
    &= \left[ \sum_{k=0}^{\infty} \response_{k} \right] \Sigma_{\vectornoisevar} \left[ \sum_{l=0}^{\infty}\response_l' \right] = \response(1) \Sigma_{\vectornoisevar} \response(1)',
\end{split}
\end{align}
that is compatible with \cref{eq:explicit_scm}. 
Additionally, it is worth noting that given a positive definite covariance matrix  $\Sigma_{\bm{Z}}$ and a fixed order of variables in $\bm{Z}$, Cholesky decomposition~\citep{lutkepohl2005new} guarantees the existence and uniqueness of a triangular matrix $L$ such that $\Sigma_{\bm{Z}} = LL'$. 
The same holds for LDL decomposition~\citep{lutkepohl2005new}, namely $\Sigma_{\bm{Z}} = LDL'$. 
Setting $L=\response(1)$ and $D=\Sigma_{\vectornoisevar}$, we can deduce \cref{eq:explicit_scm} from~\cref{thm:teorem_hamilton}.
\begin{remark*}
    Note that due to the influence of time in \acp{VAR}, the equivalent  \acp{SCM} at equilibrium, while linear, may lead to cycles in the causal graph (see, e.g., \cref{fig:census_graph}) and correlations between the exogenous variables, captured by the full covariance matrix $\boldsymbol{\Sigma}_{\vectornoisevar}$ in \cref{eq:var_to_scm}. Note also that the above Theorem implies that there is a direct relationship between interventions on \acp{DSP} $\intervention_\dequation$, and interventions on \acp{SCM} $\intervention_{\scm}$. 
\end{remark*}

% \begin{remark*}
%     It is possible to check condition \ref{scmissolution} through the $MA(\infty)$ representation of $\X_t$ (\cref{eq:moving_average}) and by replacing $\vectornoise_t$ with a fixed random variable $\vectornoise$. Indeed, we get $\X_t = \sum_{l=0}^\infty\response_l\vectornoise = \response(1)\vectornoise$.
%     Using the equality $\response(L) = A(L)^{-1}$ (see \cref{eq:phi1a1}), 
%     %(See \cref{app:svar_properties}) 
%     we rewrite it as $$A(1)\X_t = \left[I-\A_1 - \dots - \A_p \right]\X_t=\vectornoise \quad \Rightarrow \quad \X_t = \left[ \A_1 + \dots + \A_p \right]\X_t + \vectornoise,$$ that is  \cref{eq:var_to_scm}. We obtain the same result by removing time pedices from \cref{eq:VAR}.
% \end{remark*}

\paragraph{Equilibration commutes with additive intervention}

We show that applying the same additive intervention on the \ac{VAR} and its associated \ac{SCM} yields the same interventional distribution. We start by studying the long term effect of the additive term $\bm{F}$ over the \ac{VAR} model:

\begin{equation*}
\begin{split}
    &A(L)\X_t = \vectornoisevar_t + \bm{F} \quad \Rightarrow \quad \X_t = \Phi(L)\vectornoisevar_t + \Phi(1)\bm{F}, \\
    &\text{where} \quad \Phi(1)\bm{F} = \sum_{l=0}^\infty\Phi_l \bm{F}.
\end{split}
\end{equation*}
By applying $\gT$ (\cref{eq:t-trasformation}) and \cref{eq:explicit_scm}, we obtain
\begin{equation*}
    \bm{Z}^\intervention_\infty = \Phi(1) \tilde{\vectornoisevar} + \Phi(1)\bm{F}.
\end{equation*}

Using \cref{thm:var_to_scm}, we study the same intervention of the associated \ac{SCM}:

\begin{align*}
    &\tilde{\X}=\tilde{\A}\tilde{\X} + \noisescm + \bm{F} \Rightarrow \\
    & [I-\tilde{\A}]\X =\noisescm+ \bm{F} \Rightarrow \\
    & \tilde{\X} = [I-\tilde{\bm{A}}]^{-1}\tilde{\vectornoisevar} + [I-\tilde{\bm{A}}]^{-1}\bm{F}.
\end{align*}

The correspondence follows directly from $\Phi(1) = \A(1)^{-1} = [I- \tilde{\A}]^{-1}.$

\paragraph{Equilibration commutes with forcing intervention}

We report \cref{eq:do-like} for clarity of exposition:
\begin{equation*}
    \A(L)\X_t = \bm{\nu} + \vectornoisevar_t + \mathds{I}(t\geq t_0)\bm{F}\odot(\hat{\X}-\X_t)
\end{equation*}

We can rewrite the equation by shifting the term $\bm{F}\odot \X_t$ to the left, reformulating it in matrix form $\bm{F}_{diag}\X_t$, and obtaining $[\A(L)+\bm{F}_{diag}]\X_t = \bm{\nu} + \vectornoisevar_t + \mathds{I}(t\geq t_0)\bm{F}\odot\hat{\X}$. 
By defining $\A_\intervention(L):=\A(L)+\bm{F}_{diag}$ and $\tilde{\bm{F}}:= \bm{F}\odot\hat{\X}$, we obtain $\A_\intervention(L) \X_t = \bm{\nu} + \vectornoisevar_t + \mathds{I}(t\geq t_0)\tilde{\bm{F}}$, which is again an additive intervention, but on the intervened dynamic. For this reason, we can prove the case $\hat{\X}=\bm{0}$ and directly applying the result of additive intervention afterwards.
By applying $\gT$ (\cref{eq:t-trasformation}) and equilibrating, we obtain
\begin{equation*}
    \bm{Z}^\intervention_\infty = \A_\intervention(1)^{-1} \noisescm
\end{equation*}

Consider the analogue of forcing intervention over an \ac{SCM}:
\begin{equation*}
    \tilde{\X} = \tilde{\bm{A}} \tilde{\X} + \noisescm - \bm{F} \odot \tilde{\X}, \quad \text{where we fixed} \quad \hat{\X}=\bm{0}.
\end{equation*}

We can rewrite it as $[I-\tilde{\bm{A}}+\bm{F}_{diag}]\tilde{\X}=\noisescm$, that is $\A_\intervention(1)\tilde{\X} = \noisescm \Rightarrow \tilde{\X} = \A_\intervention(1)^{-1} \noisescm.\qed{}$

\section{Practical implications}\label{sec:practical_implications}

\begin{comment}- expressiveness - theorem of SVARs and linear but potentially cyclic and with conf. SCMs

- fitting : literature and libraries from time series data. cross'sectional data, future work. 

- causal inference over time: retrospective CF, what would have happen to an individual trajectory of an intervention would have been applied at a point in past
- observational counterfactual \& interventional counterfactual - remark the individual aspect as it is different as in SCMs
\end{comment}

\paragraph{Causal queries}
%In the time-series econometric literature, causal analysis typically relies on the concept of Granger causality. This approach has received various criticisms over the years, as it is highly sensitive to the presence of hidden confounders and can fail to capture ``structural causality''~\citep{demiralp2008bootstrap}, as intended in the context of \ac{SCM}. \SVAR{}s represent one attempt to incorporate this structure into autoregressive models and have many similarities with \ac{SCM}, including the causal discovery methods used ~\citep{hyvarinen2010estimation, moneta2011causal}. Causal inference in this framework is performed differently: while \ac{SCM} focus on estimating the \textit{permanent effects} of a particular action, \AR{} models estimate the effects of system shocks, i.e., interventions localized in time that temporarily perturb the process from its equilibrium state, causing \textit{transient effects that disappear after a certain period}~\citep{stock2001vector}. 
Our formulation of causal interventions on \acp{VAR} differs from the standard approach based on Granger causality by being closer %, we have proposed a different way of understanding interventions in \ac{DSP}, 
%much closer 
to that of \acp{SCM}. %, showing its utility even for time-series analysis.
%In particular, we show the possibility of formulating the following causal questions:
Consequently, it enables the generalization of interventional and counterfactual queries %, common in the \ac{SCM} framework, 
to account for time (see~\cref{app:queries_practice}).
That is, it allows for answering the following causal questions: 

\begin{itemize}
    \item \textbf{Forecasted Interventions} What are the expected effects on an individual trajectory (or a population) when intervening in the present, and how do they vary over time?
    
    \item \textbf{Retrospective Counterfactuals} What would have happened to an individual trajectory if an intervention had been applied at a specific point in the past? What state would it be in now?
\end{itemize}
Both causal queries acquire a meaning embedded in the temporal dimension in terms of \textit{forecasting for the future} (\cref{subsec:interventional_forecasting}) and \textit{retrospection for the past} 
 (\cref{app:counterfactuals}), respectively. %\sr{Something is missing here. The reader expects the 'how' the questions can be solved. It could simply be a forward reference to the empirical evaluation section.}

\paragraph{Expressiveness and Universality}
\acp{VAR}, although linear models, are remarkably expressive~\citep{kilian2017structural}. By the Wold Decomposition Theorem~\citep{wold1938study}, \textit{any purely nondeterministic, covariance-stationary process can be approximated arbitrarily well by an autoregressive model}. In this sense, VARs act as universal approximators for stationary dynamics.

However, as with many “universal consistency” results in machine learning, this theorem comes with important caveats. It does not quantify how quickly the approximation error decays with the sample size, nor how the complexity of the model (e.g., the autoregressive order) scales with the complexity of the underlying process. In practice, such missing rates make the theorem more of a conceptual reassurance than a constructive modeling principle.

To illustrate the gap between theory and practice, consider the following example:

\begin{example}{Discrete Gaussian Random Walk in a One-Dimensional Box}
The system is defined as:
\begin{equation}
\begin{split}
X_t' &= X_{t-1} + \varepsilon_t, \\
X_t &= \text{Reflect}(X_t'),
\end{split}
\end{equation}
where
\begin{equation}
\text{Reflect}(x) =
\begin{cases}
2 - x & \text{if } x > 1, \\
-2 - x & \text{if } x < -1, \\
x & \text{otherwise}.
\end{cases}
\end{equation}

This describes, for example, the motion of a particle confined to a one-dimensional box. For most of its trajectory, the process behaves like a simple AR(1) system; but whenever the boundary is reached, the reflection introduces a strongly nonlinear effect. Although the Wold theorem guarantees a valid representation, the corresponding MA($\infty$) expansion is essentially intractable, and it fails to capture the localized nonlinearity created by the reflections—thereby missing a defining structural feature of the process.

More generally, processes whose stationary distributions deviate sharply from Gaussianity are difficult to represent faithfully with the Wold decomposition. Unless the moving-average coefficients are chosen in a highly structured way, the infinite linear combination of Gaussian innovations tends to induce Gaussian marginals for $X_t$, by the Central Limit Theorem. This makes VAR-based representations ill-suited for processes with fundamentally non-Gaussian steady states.
\end{example}

we intend to explore non-linear \acp{DSP} in future work, as they may lead to better convergence rates and allow for %aiming to propose a 
causal interpretation of a broader family of dynamical models.

\paragraph{Feedback loops}
To properly understand complex systems, it is often useful to model feedback loops between their variables. 
Time-series models naturally capture this property, while \acp{SCM} require significant reformulation. 
The theory of cyclic \acp{SCM} has seen a significant advancement in recent years \citep{bongers2021foundations}, but
%The theory of cyclic \acp{SCM} has seen significant development in recent years \citep{bongers2021foundations}, however, 
practical approaches, both for causal discovery and causal inference, are still underdeveloped~\citep{bongers2016cyclic, Lorbeer2023cycliccausaldiscovery}. 
%
%For this reason, we believe that studying causal inference over time could be a promising direction.
%
Our formalization of causal inference on \acp{VAR} is a step forward in this direction.

\paragraph{Fitting}
\acp{VAR} estimation is typically performed using ordinary least squares. Various alternative methods are available, both in terms of constrained optimization (e.g., to use prior knowledge about some coefficients of the \ac{VAR} matrices ~\citep{sims1980macroeconomics}) and within a Bayesian framework ~\citep{koop2010bayesian}. Refer to ~\citet[chapters~3,4,5]{lutkepohl2005new} for a comprehensive discussion. 
%Several software libraries exist for studying autoregressive models: \texttt{statsmodels}~\citep{seabold2010statsmodels} and \texttt{PyFlux}~\citep{taylor2016pyflux} for \textit{Python}, \texttt{vars}~\citep{vars2008}, and \texttt{MTS}~\citep{tsay2022package} for \textit{R}. 
%
Importantly, although %the most common application is the study of 
\acp{VAR} are most commonly used on
time-series data (i.e., data from one single unit across a period of time), there are approaches tailored to the analysis of \emph{panel data} (i.e., the evolution of many units over time) ~\citep{sigmund2021panel}; and \emph{cross-sectional data} (i.e.,~many individuals at a single point of time), provided that they have at least some proxy variables of time ~\citep{deaton1985panel}. 
%The ability to study a system based on observing a single trajectory using temporal information distinguishes \AR{} models from \ac{SCM}, which require a population to study their distribution. In this regard, \AR{} models may have superior generalization capabilities. While this point is not explored in the current work, we believe it merits further investigation.
%
Such approaches open up a promising line of future work that can further generalize \acp{VAR} applicability for causal reasoning over time. 

\subsection{Answering causal queries in practice}\label{app:queries_practice}

\subsection*{Forecasting}\label{app:forecasting}
Forecasting is one of the main objectives of time series analysis in general~\citep{montgomery2015introduction}. We discuss predictors based on \ac{VAR} processes that minimize the Mean Square Error (MSE). 

Suppose we have a time series $\X_t$ up to time $t_0$ and we want to predict the state of the system after $h$ steps $\X_{t_0+h}$. We call a predictor for this quantity the $h$-steps predictor. In particular, the optimal predictor with respect to MSE is 
\begin{equation*}
    \expected{\X_{t_0+h}|\X_{\leq t_0}}.
\end{equation*}
The optimality of the expected value relies on the assumption that $\vectornoisevar_t$ in \cref{eq:VAR} is \textit{i.i.d. white noise}.
For a \ac{VAR}$(p)$ as in \cref{eq:VAR}, the optimal $h$-step predictor can be evaluated by recursively applying
\begin{equation}\label{eq:recursive_expected_var}
\begin{split}
    &\expected{\X_{t+k}|\X_{\leq t}} =\bm{\nu} +\A_1\expected{\X_{t+k-1}|\X_{\leq t}} + \\
    &\cdots + \A_p \expected{\X_{t+k-p}|\X_{\leq t}},
\end{split}
\end{equation}
where $\expected{\X_{t'}|\X_{\leq t}} = \X_{t'}$ if $t' \leq t$. In other words, it is possible to forecast all subsequent values of $\X_{t_0}$ by applying the original equation of the \ac{VAR} without considering the noise $\vectornoisevar_t$. For example, considering the $1$-step predictor, we get
\begin{equation*}
    \expected{\X_{t_0+1}|\X_{\leq t_0}} = \bm{\nu} + \A_1 \X_{t_0} + \cdots + \A_p \X_{t_0+1-p}.
\end{equation*}
Similarly, by exploiting the possibility of recursively applying \cref{eq:recursive_expected_var}, we obtain the estimate of the forecast error covariance (or MSE) matrix as
\begin{equation*}
\Sigma_{\X}(k) = \sum_{i=0}^{h-1} \Phi_i \Sigma_{\vectornoisevar} \Phi'_i.
\end{equation*}
Note that each term in the sum is a positive definite matrix, and consequently, the MSE matrix is monotonically non-decreasing. $\Sigma_{\X}(k)$ approaches the covariance matrix of $\X_t$, $\Sigma_{\X}$, for $k\to\infty$. In fact, the optimal long-range forecast $h \to \infty$ is just the process mean. In other words, the past of the process contains no information on the development of the process in the distant future.

\subsection*{Interventional Forecasting}\label{app:interventional_forecast}

Consider a \ac{VAR}$(p)$ model and suppose an intervention is applied to the system at time $t_0$. We report the procedure for performing interventional forecasts for additive and forcing interventions separately.

\paragraph{Additive intervention}
The model equations are modified only by adding the term $\bm{F}$, consequently, \cref{eq:recursive_expected_var} simply becomes
\begin{equation}\label{eq:forecast_additive_intervention}
\begin{split}
    &\expected{\X_{t+k}|\X_{\leq t}} = \hat{\bm{\nu}}+ \A_1\expected{\X_{t+k-1}|\X_{\leq t}} +\\
    &\cdots + \A_p \expected{\X_{t+k-p}|\X_{\leq t}}, \quad \text{where} \quad  \hat{\bm{\nu}} = \bm{\nu}+\bm{F}
    \end{split}
\end{equation}
The MSE matrix remains unchanged, as both $\Sigma_{\vectornoisevar}$ and all the matrices (and consequently the matrices $\Phi_i$, see \cref{eq:recursive_calcultation_phi}) are not modified.

Alternatively, exploiting the linearity of the predictor, it is possible to directly estimate the expected causal effect at time $h$ through
\begin{equation*}
    \ce_{t_0+h} %= \expected{\X_{t_0+k}^\intervention-\X_{t_0+k}} 
    = \sum_{l=0}^h \response_l \bm{F}, 
\end{equation*}
and add this quantity to the observational forecast. The resulting interventional forecast is equivalent.

\paragraph{Forcing intervention}
In this case, the intervention modifies the matrices $\A_i$ of the process, and the procedure is more complex. Let us proceed step by step. The intervened model is described by (see \cref{eq:do-like})

\begin{equation}\label{eq:forecast_forcing_intervention}
\begin{split}
    & \X_t = \bm{\nu} + \bm{F}\odot(\hat{\X}-\X_t) + \sum_{i=1}^p \A_i \X_{t-i} + \vectornoisevar_t \\
    \Rightarrow \quad & [I+\bm{F}_{diag}] \X_t = \bm{\nu} + \bm{F}\odot\hat{\X} + \sum_{i=1}^p \A_i \X_{t-i} + \vectornoisevar_t \\
    & \text{where} \quad \bm{F}_{diag} := I \cdot \bm{F} \\
    \Rightarrow \quad & [I+\bm{F}_{diag}] \X_t = \hat{\bm{\nu}} + \sum_{i=1}^p \A_i \X_{t-i} + \vectornoisevar_t \\
    & \text{where} \quad \hat{\bm{\nu}} := \bm{\nu} + \bm{F}\odot\hat{\X} \\
    \Rightarrow \quad & \X_t = [I+\bm{F}_{diag}]^{-1} \left(\hat{\bm{\nu}} + \sum_{i=1}^p \A_i \X_{t-i} + \vectornoisevar_t \right) \\
    \Rightarrow \quad & \X_t = \tilde{\bm{\nu}} + \sum_{i=1}^p \tilde{\A}_i \X_{t-i} + \tilde{\vectornoisevar}_t,
\end{split}
\end{equation}

where $\{\tilde{\bm{\nu}}, \tilde{\A}_i, \tilde{\vectornoisevar}_t\} := [I+\bm{F}_{diag}]^{-1}\{\hat{\bm{\nu}}, \A_i, \vectornoisevar_t\}$.

The resulting equation has the form of a standard \ac{VAR}$(p)$ model, so we can apply the forecasting exactly as in \cref{app:forecasting}. The same applies to the MSE matrix, calculated using the modified matrices $\Sigma_{\tilde{\vectornoisevar}}$ and $\tilde{\Phi}_i$.

Unlike the additive intervention, there is no way to directly calculate the causal effect. The estimate is therefore obtained through the difference.
\begin{equation*}
    \ce{}_{t_0+h} = \expected{\tilde{\X}_{t_0+h}|\X_{\leq t}} - \expected{\X_{t_0+h}|\X_{\leq t}},
\end{equation*}
that is \textit{the difference between the interventional forecast and the observational one}.

\subsection{Retrospective Counterfactuals}\label{app:counterfactuals_retro}
In the \textit{ladder of causality}~\citep{pearl2009causality,bareinboim2022pearl}, counterfactual statements belong to the last level, distinct from the interventional level, and a mature causal model should be able to address both. Embracing this perspective, we show how counterfactual queries can be formulated for causal inference over time. The procedure follows the same steps used in the \ac{SCM} framework.
\begin{itemize}
\item \textbf{Abduction} Based on the observed trajectory $\X_t$, we estimate the values of $\vectornoisevar_t$ by calculating the residuals $\hat{\vectornoisevar}_t = \X_t - \expected{\X_t|\X_{<t}}$. For a \ac{VAR}$(p)$ model, in particular, $\expected{\X_t|\X_{<t}} = \A_1 \X_{t-1} + \cdots + \A_p \X_{t-p}$.
\item \textbf{Action} We apply a specific intervention $\intervention$ by modifying the model's equations.
\item \textbf{Prediction} We simulate the evolution of the new process $\X^{\intervention}_t$ using $\hat{\vectornoisevar}_t$ in the equation.
\end{itemize}

Let $\X_t$ be a time series generated by a \ac{VAR}$(p)$ model. We want to evaluate what the trajectory of $\X_t$ would have been up to the present time $t_1$ assuming a specific intervention was applied at a particular past time $t_0$. We proceed as follows:
\begin{enumerate}
    \item Perform the abduction step and retrieve the residuals $\{\hat{\vectornoisevar}_t\}_{t=t_0}^{t_1}$.
    \item Apply the intervention to the model as in \cref{app:interventional_forecast}.
    \item Simulate the counterfactual trajectory from $\X_{t_0}$ to $\X_{t_1}$ through the recursive application of
    \begin{equation*}
        \X_t = \hat{\bm{\nu}} + \A_1 \X_{t-1} + \dots + \A_p \X_{t-p} + \hat{\vectornoisevar}_t.
    \end{equation*}
    where $\hat{\bm{\nu}}$ is defined as in \cref{eq:forecast_additive_intervention} or \cref{eq:forecast_forcing_intervention}, depending on the type of intervention. Note that \textit{if using forcing intervention}, the whole term on the right must be left-multiplied by $[I+\bm{F}_{diag}]^{-1}$ before performing the simulation.
\end{enumerate} 
\paragraph{Observation} Performing retrospective counterfactual without applying any intervention results in the perfect recovery of the factual trajectory, so that the difference with the interventional one can be interpreted in the same way as in \ref{app:interventional_forecast}, i.e. as the causal effect over time along the trajectory.

%\paragraph{Motivating Example.} 
%Counterfactual reasoning is often used to determine responsibility. Consider, for instance, the economic losses due to delays in the construction of a public work or the late diagnosis of a serious illness. In these cases, to accurately estimate the actual damages incurred, the temporal dimension is essential.

%Riportiamo qui alcuni risultati qualitativi ottenuti applicando counterfactuals nell'ambito dei loans, utilizzando nello specifico un dataset sintetico ispirato al German Dataset. 
\begin{comment}\ibnote{marti qui continua tu, rischio di ripetere o anticipare cose.

- l'importanza dei controfattuali nell'algorithmic recourse, diritto ad una spiegazione (cosa mi è mancato per, come avrei potuto ottenerlo)

- mostrare controfattuale su expertise, o comunque una qualche variabile difficilmente modificabile in futuro (avresti potuto, ormai è andata)

- un esempio su real indagando quanta sarebbe la forza lavoro oggi se l'italia avesse adottato una politica migratoria più aperta dagli anni 80

}
\end{comment}

%Molte discipline scientifiche, compresa la computer science, soffrono di uno sbilanciamento tra i generi in favore dei maschi in termini di presenza nei dipartimenti universitari. Negli ultimi anni sono state attuate diverse politiche per sopperire a questo gap. Tuttavia, data la lunghezza del percorso accademico, gli effetti di queste risultano ancora modesti. Sarebbe interessante investigare quale sarebbe oggi la situazione se queste politiche fossero state attuate a partire da 30 anni fa. 

\section{Empirical evaluation}\label{sec:experiments}
We evaluate \ac{VAR} models' accuracy and expressiveness in multivariate time series, focusing on two forecasting dimensions: observational (\ref{subsec:observational_forecasting}) and interventional (\ref{subsec:interventional_forecasting}). 

\subsection{Research Questions}
The experiments address the following three questions:
\begin{itemize}
     \item \textbf{RQ1.} How does the \ac{VAR} performance compare with SOTA models for forecasting multivariate time series?

     \item \textbf{RQ2.} How do additive and forcing interventions affect the system dynamics? 
     
     \item \textbf{RQ3.} How accurate is the causal \ac{VAR} framework in estimating the causal effect of interventions over time?  \end{itemize}

Code is publicly available on \url{https://github.com/marti5ini/ci-over-time}. 

%Specifically, we evaluate \acp{VAR} models' accuracy and expressiveness in multivariate time series, focusing on three  forecasting dimensions: (i) observational, (ii) additive intervention, and (iii) forcing intervention.
%
%\textbf{Q1} pertains to the efficacy of \ac{VAR} models given historical observations, we predict future values 
%
%In \textbf{Q2} we examine the \ac{VAR}'s capacity to forecast outcomes when introducing additive or forcing interventions into the system.

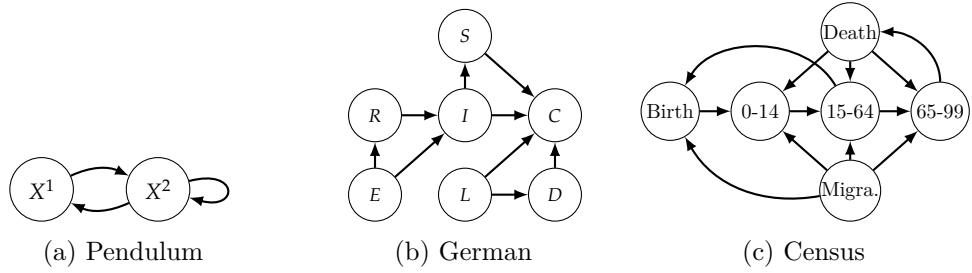
\begin{figure}[t]
    \centering
    \begin{subfigure}[b]{0.32\columnwidth}
        \centering
        \begin{tikzpicture}[
            scale=0.7,
            transform shape,
            vertex/.style={circle, draw, minimum size=1.2cm, inner sep=0pt, font=\normalsize}
        ]
        \node[vertex] (x1) at (0, 1.5) {$X^{1}$};
        \node[vertex] (x2) at (2.2, 1.5) {$X^{2}$};
        
        %\draw[edge] (x1) to [bend left=25] node[midway, above] {\scriptsize $\frac{\sqrt{2}}{2}$} (x2);
        %\draw[edge] (x2) to [bend left=25] node[midway, below] {\scriptsize $-\frac{\sqrt{2}}{2}$} (x1);
        %\draw[edge] (x2) to [loop right] node[midway, right] {\scriptsize $\sqrt{2}$} (x2);
        \draw[edge] (x1) to [bend left=25] (x2);
        \draw[edge] (x2) to [bend left=25] (x1);
        \draw[edge] (x2) to [loop right] (x2);
        \end{tikzpicture}
        \caption{\footnotesize Pendulum}
        \label{fig:inverted_pendulum_graph}
    \end{subfigure}%
    %\hfill
    \begin{subfigure}[b]{0.32\columnwidth}
        \centering
        \begin{tikzpicture}[
            scale=0.7,
            transform shape,
            vertex/.style={circle, draw, minimum size=1cm, inner sep=0pt, font=\footnotesize}
        ]
        \node[vertex] (r) at (0,1.5) {$R$};
        \node[vertex] (e) at (0,0) {$E$};
        \node[vertex] (l) at (1.7,0) {$L$};
        \node[vertex] (d) at (3.4,0) {$D$};
        \node[vertex] (i) at (1.7,1.5) {$I$};
        \node[vertex] (s) at (1.7,3) {$S$};
        \node[vertex] (y) at (3.4,1.5) {$C$};
        
        \draw[edge] (e) -- (r);
        \draw[edge] (l) -- (d);
        \draw[edge] (r) -- (i);
        \draw[edge] (e) -- (i);
        \draw[edge] (i) -- (s);
        \draw[edge] (l) -- (y);
        \draw[edge] (d) -- (y);
        \draw[edge] (i) -- (y);
        \draw[edge] (s) -- (y);
        \end{tikzpicture}
        \caption{\footnotesize German}
        \label{fig:german_graph}
    \end{subfigure}%
    %\hfill
    \begin{subfigure}[b]{0.32\columnwidth}
        \centering
        \begin{tikzpicture}[
            scale=0.7,
            transform shape,
            vertex/.style={circle, draw, minimum size=1.1cm, inner sep=0pt, font=\small}
        ]
        \node[vertex] (b) at (0, 1.5) {Birth};
        \node[vertex] (d) at (3.4, 3) {Death};
        \node[vertex] (m) at (3.4,0) {Migra.};
        \node[vertex] (a-0) at (1.7,1.5) {0-14};
        \node[vertex] (a-1) at (3.4,1.5) {15-64};
        \node[vertex] (a-2) at (5.1,1.5) {65-99};
        
        \draw[edge] (b) -- (a-0);
        \draw[edge] (a-1) to [out=120,in=60](b); 
        \draw[edge] (d) -- (a-2);
        \draw[edge] (d) -- (a-1);
        \draw[edge] (d) -- (a-0);
        \draw[edge] (m) -- (a-0);
        \draw[edge] (m) -- (a-1);
        \draw[edge] (m) -- (a-2);
        \draw[edge] (a-0) -- (a-1);
        \draw[edge] (a-1) -- (a-2);
        \draw[edge] (m) to [out=190,in=-60](b);
        \draw[edge] (a-2) to [out=90, in=0](d);
        \end{tikzpicture}
        \caption{\footnotesize Census}
        \label{fig:census_graph}
    \end{subfigure}
    \caption{\textbf{Causal graphs}. The causal graph for (a) and (b) is known, while for (c), it is assumed.
    In (b), nodes are labeled with the initials of each feature: Expertise, Responsibility, Loan Amount, Duration, Income, Savings, and Credit Score. In (c), $0-14$, $15-64$, and $65-99$ represent age groups.}
    \label{fig:datasets_graphs}
    \vspace{-0.2cm}
\end{figure}

\subsection*{Experimental Settings}
\paragraph{Datasets} We rely on two synthetic datasets, German\footnote{This dataset is inspired on \url{https://archive.ics.uci.edu/dataset/144/statlog+german+credit+data}} and Pendulum,
and the real-world Census dataset\footnote{\url{https://data.census.gov/}}.
German simulates a loan approval scenario with seven variables.
%: Expertise, Responsibility, Loan Amount, Duration, Income, Savings, and Credit Score.
%
Pendulum is a two-variable system where
$\Xcomp[1]$ operates as a stabilizer for $\Xcomp[2]$, which exhibits a divergent dynamic. 
%Intervening on $X^{(1)}$ disrupts the negative feedback loop, making the system unstable. 
%
Census includes demographic variables across three age groups, along with migration, birth, and death rates from $1992$ to $2023$ for $50$ countries. 
\ref{fig:datasets_graphs} illustrates the causal graphs for all datasets.
%
%\sr{Do we assume they are known or correctly discovered, right? We should mention this here} 
%we put this information in the caption of figure 2 
%

\paragraph{Metrics}
We measure the discrepancy between the \textit{h}-steps forecast $\hat{\X}_{t+h}|\X_{<t}$ and the true value $\X_{t+h}$ on the test set $\X_{test}$. %using as metrics $R^2$ score, Mean Absolute Error (MAE), and
%Symmetric Mean Absolute Percentage Error (SMAPE)~\citep{montgomery2015introduction}. 
%
We report Mean Absolute Error (MAE) focusing on the target variables (i.e., \textit{Credit Score} for German, $\Xcomp[1]$ for Pendulum, and age groups for Census). All results are averaged over ten runs.

\subsection{Observational Forecasting}\label{subsec:observational_forecasting}
%\textcolor{blue}{add definition of observational forecasting}
\paragraph{Baselines}
We compare the \ac{VAR} framework with three relevant works: i) DLinear~\citep{zeng2023dlinear}, a decomposition-based linear model that separates trend and seasonal components;
ii) TSMixer~\citep{chen2023tsmixer}, a Multi-layer Perceptron (MLP) based model that focuses on mixing time and feature dimensions; iii) TiDE~\citep{das2023tide}, a MLP based encoder-decoder model. 
To assess the effectiveness of the forecasting methods, we introduce an observational oracle forecaster that has full knowledge about the true data generating process and produces the optimal predictor, i.e., $\hat{\X}_{t+h}|\X_{<t} = \expected{\X_{t+h}|\X_{<t}}$. 

%For a fair comparison, all model hyperparameters have been tuned to optimize the performance on the smallest datasets (see~\ref{app:obs_forecasting}). 

\begin{table}[t]
    \centering
    \caption{\textbf{Observational Forecasting.}
    MAE scores (\emph{lower is better}) for 
    \ac{VAR}, DLinear~\citep{zeng2023dlinear},  TiDE~\citep{das2023tide} and TSMixer~\citep{chen2023tsmixer}.   
    Results averaged over $10$ runs. 
    Best model in bold, Oracle in $\mathtt{typewriter}$.
    For Census, size equals the number of countries times the number of years.}
    %\IV{Do you have here a scaling factor similar to Table 2?} \textcolor{blue}{No, we don't have it} \sr{what is the other number in subscript? Size is the size of the training set?}}
    \label{tab:comparative_forecasting}
    \setlength{\tabcolsep}{1.5pt}
    \footnotesize
    \renewcommand{\arraystretch}{1.18}
    \begin{tabular}{lcccccccc}
    \toprule
    & & & \multicolumn{5}{c}{Observational Forecasting} \\
    \cmidrule(lr){4-8}
    Dataset & Size & Horizon & Oracle & VAR & DLinear & TiDE & TSMixer \\
\midrule
    German & \multirow{4}{*}{\centering100} & \multirow{2}{*}{\centering1} & $\mathtt{.004}$ & $\mathbf{.008}$ & $.009$ & $.011$ & $.014$ \\
Pendulum & & & $\mathtt{.042}$ & $\mathbf{.043}$ & $\mathbf{.043}$ & $.218$ & $.217$ \\
\cmidrule(lr){3-8}
German & & \multirow{2}{*}{\centering10} & $\mathtt{.014}$ & $\mathbf{.055}$ & $\mathbf{.055}$ & $.094$ & $.139$ \\
Pendulum & & & $\mathtt{.399}$ & $\mathbf{.420}$ & $.440$ & $1.43$ & $1.43$ \\
\midrule
German & \multirow{4}{*}{\centering500} & \multirow{2}{*}{\centering1} & $\mathtt{.004}$ & $\mathbf{.004}$ & $\mathbf{.004}$ & $.011$ & $.014$ \\
Pendulum & & & $\mathtt{.042}$ & $\mathbf{.042}$ & $\mathbf{.042}$ & $.218$ & $.217$ \\
\cmidrule(lr){3-8}
German & & \multirow{2}{*}{\centering10} & $\mathtt{.014}$ & $\mathbf{.015}$ & $\mathbf{.015}$ & $.093$ & $.135$ \\
Pendulum & & & $\mathtt{.399}$ & $\mathbf{.401}$ & $.405$ & $1.43$ & $1.43$ \\
\hline \hline
\multirow{2}{*}{Census} &\multirow{2}{*} {\centering{$50\times32$}} & 1 & - & $\mathbf{.001}$ & $.006$ & $\mathbf{.001}$ & $.008$ \\
& & 5 & - & $.017$ & $.025$ & $\mathbf{.014}$ & $.024$\\
\bottomrule
    \end{tabular}
%\vspace{-0.19cm}
\end{table}

%\paragraph{Results} 
\paragraph{RQ1: \ac{VAR} vs SOTA models}
The observational forecasting results in Table~\ref{tab:comparative_forecasting} show performance across varying data sizes (i.e., number of instances) and forecast horizons for all datasets. 
\ac{VAR} emerges as the top-performing model, consistently matching or closely approaching Oracle's scores for all datasets. 
DLinear usually achieves predictive accuracy close to \ac{VAR} for $1$-step forecasts, presumably due to the common linear nature of both models. 
TiDE and TSMixer consistently underperform compared to \ac{VAR} and DLinear for German and Pendulum. 
For all models (including Oracle), performance on the Pendulum dataset is uniformly worse than on the German, highlighting the greater challenge in forecasting given the system's stronger stochasticity and variables changing more rapidly over time.
On Census, \ac{VAR} and TiDE provide the best results,  TiDE slightly outperforming \ac{VAR} in $5$-step horizon.

\begin{figure*}[t!]
  \centering
  \includegraphics[width=\textwidth]{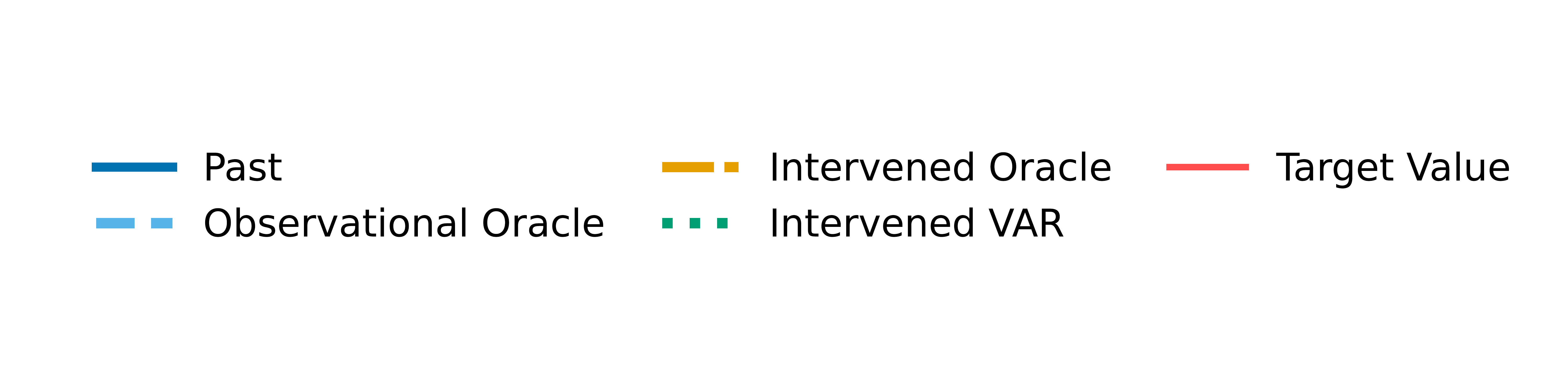}
  % Prima riga di figure
  \begin{subfigure}[t]{0.49\textwidth}
    \centering
    \vspace{-1.3cm}
    \includegraphics[width=0.85\textwidth]{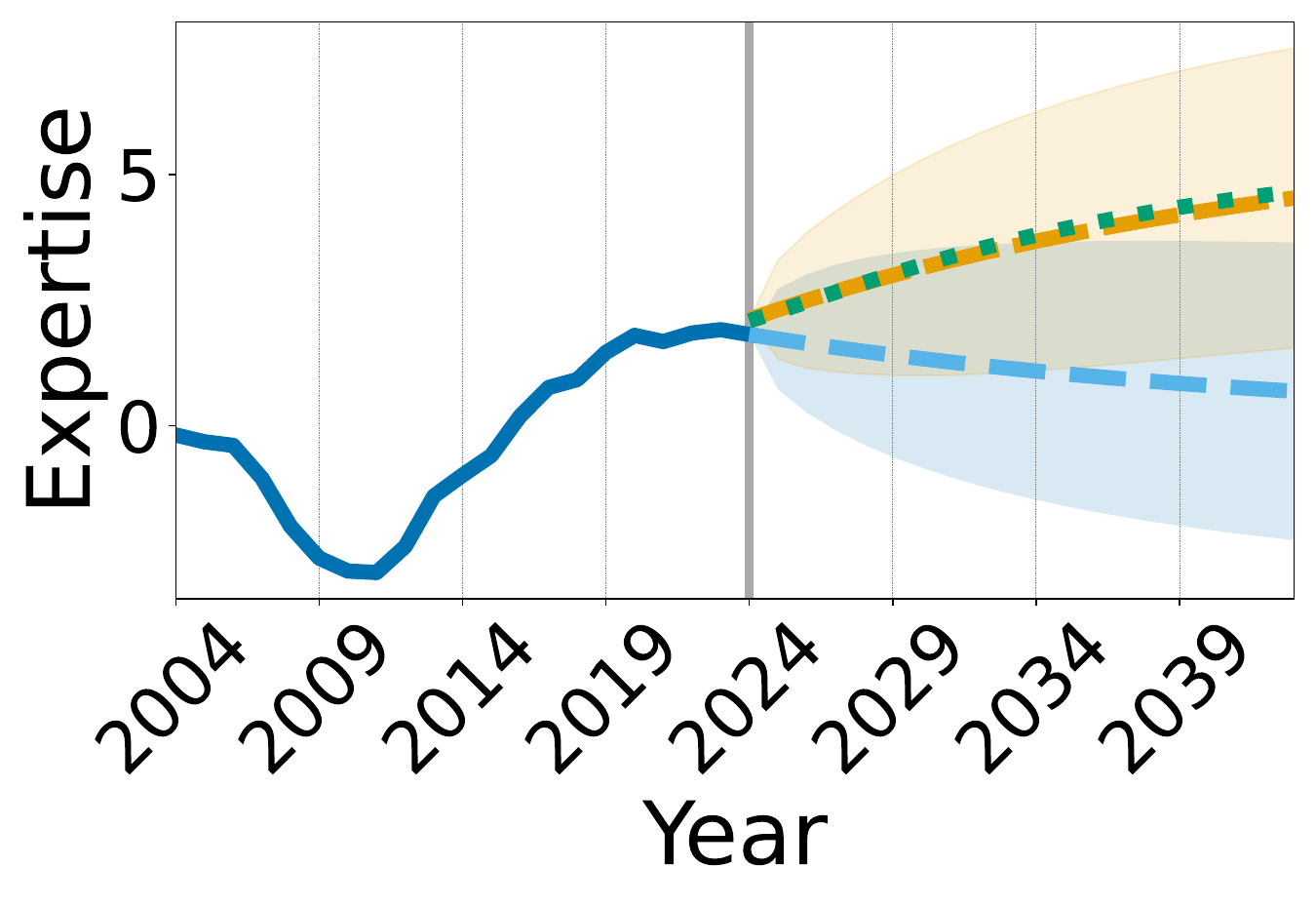}
  \end{subfigure}
  \hfill
  \begin{subfigure}[t]{0.49\textwidth}
    \vspace{-1.3cm}
    \centering
    \includegraphics[width=0.85\textwidth]{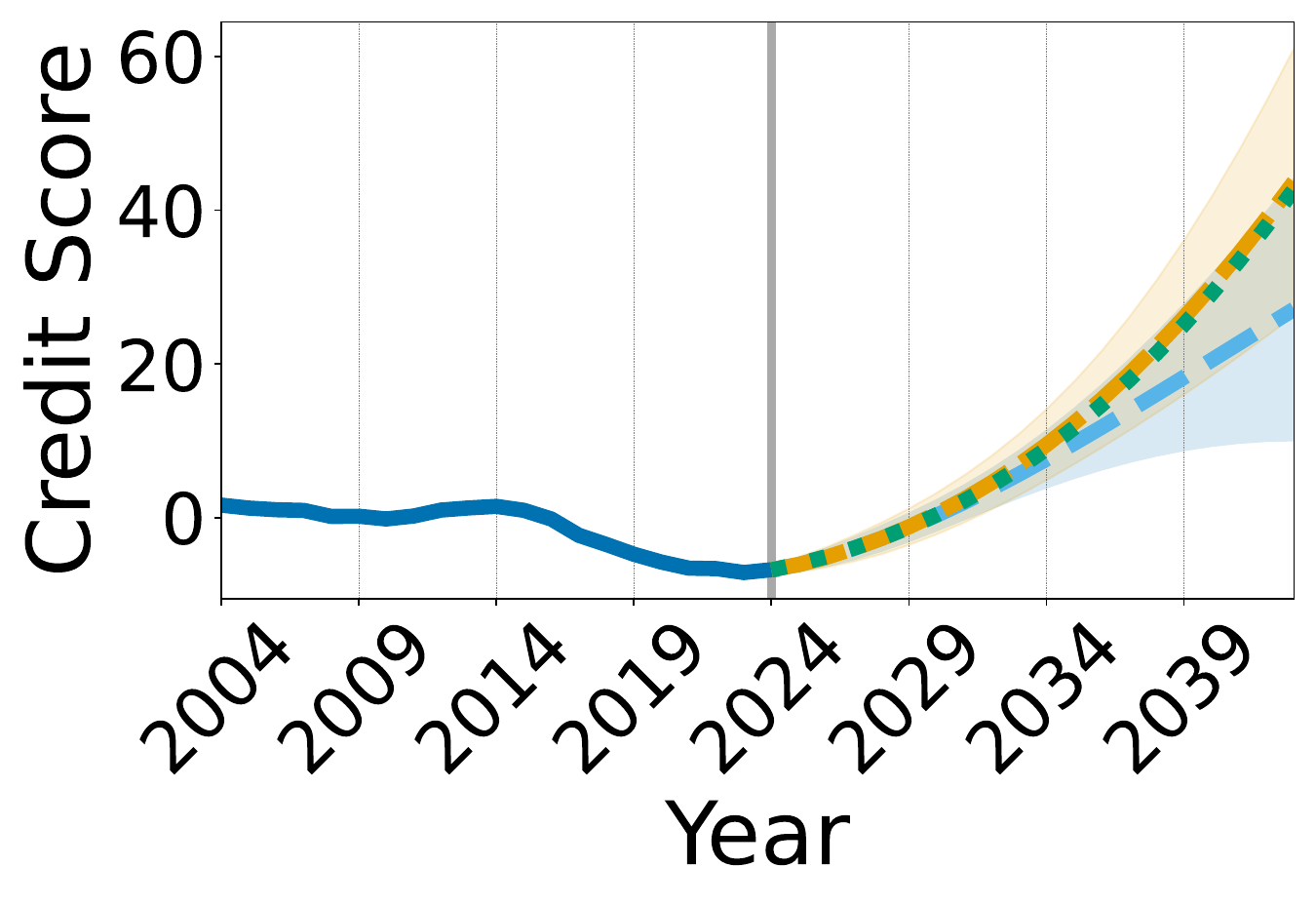}
  \end{subfigure}
  
  % Caption per la prima riga
  \caption{\textbf{Additive Intervention.} (Left) Intervention on \textit{Expertise} with $\bm{F} = 0.2$. (Right) Effect on \textit{Credit Score}. Shaded regions denote $95\%$ confidence bounds.}
  \label{fig:additive_int}
  
  \vspace{0.5cm} % Spaziatura tra le righe

  % Seconda riga di figure
  \begin{subfigure}[t]{0.49\textwidth}
    \centering
    \includegraphics[width=.85\textwidth]{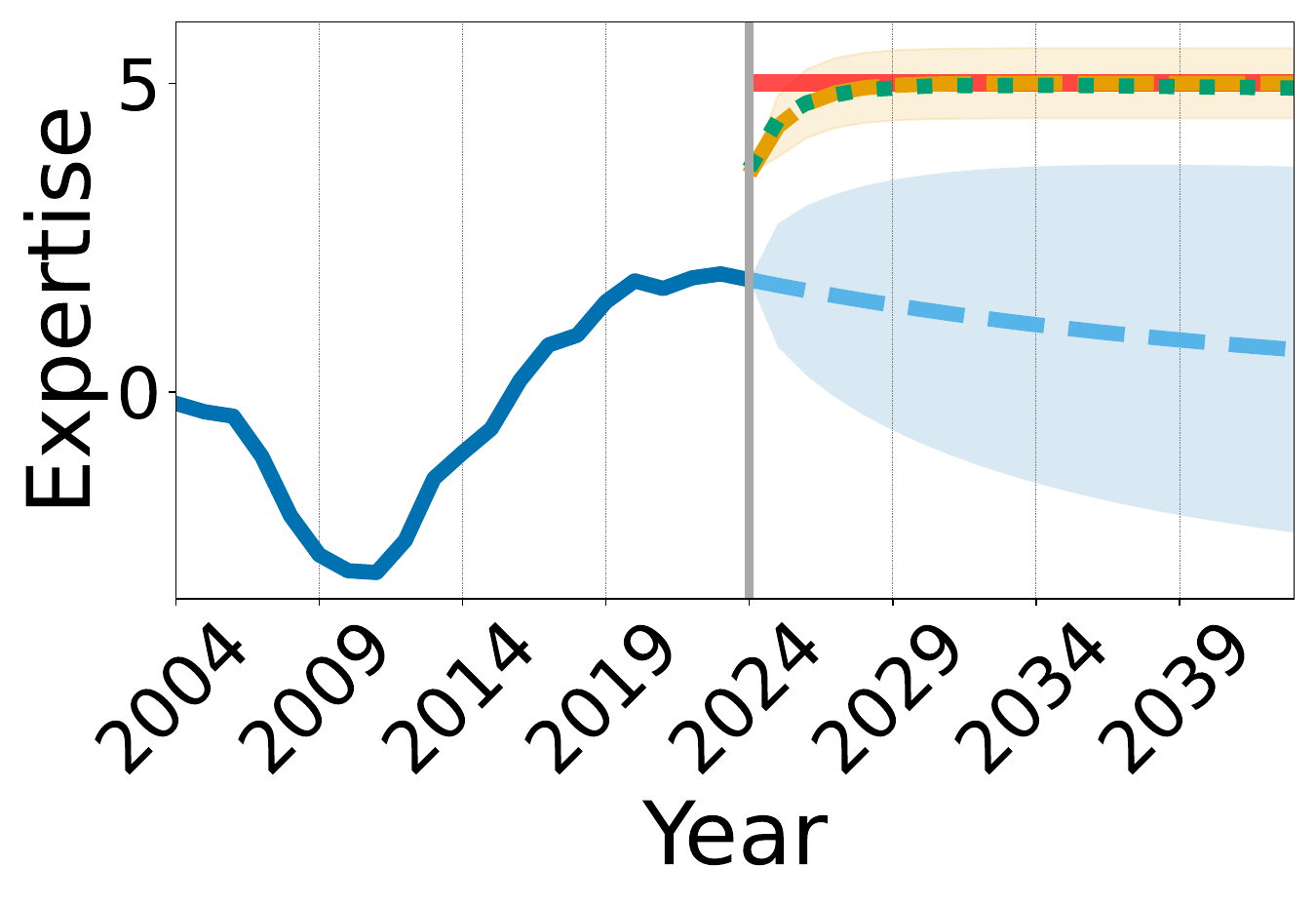}
  \end{subfigure}
  \hfill
  \begin{subfigure}[t]{0.49\textwidth}
    \centering
    \includegraphics[width=.85\textwidth]{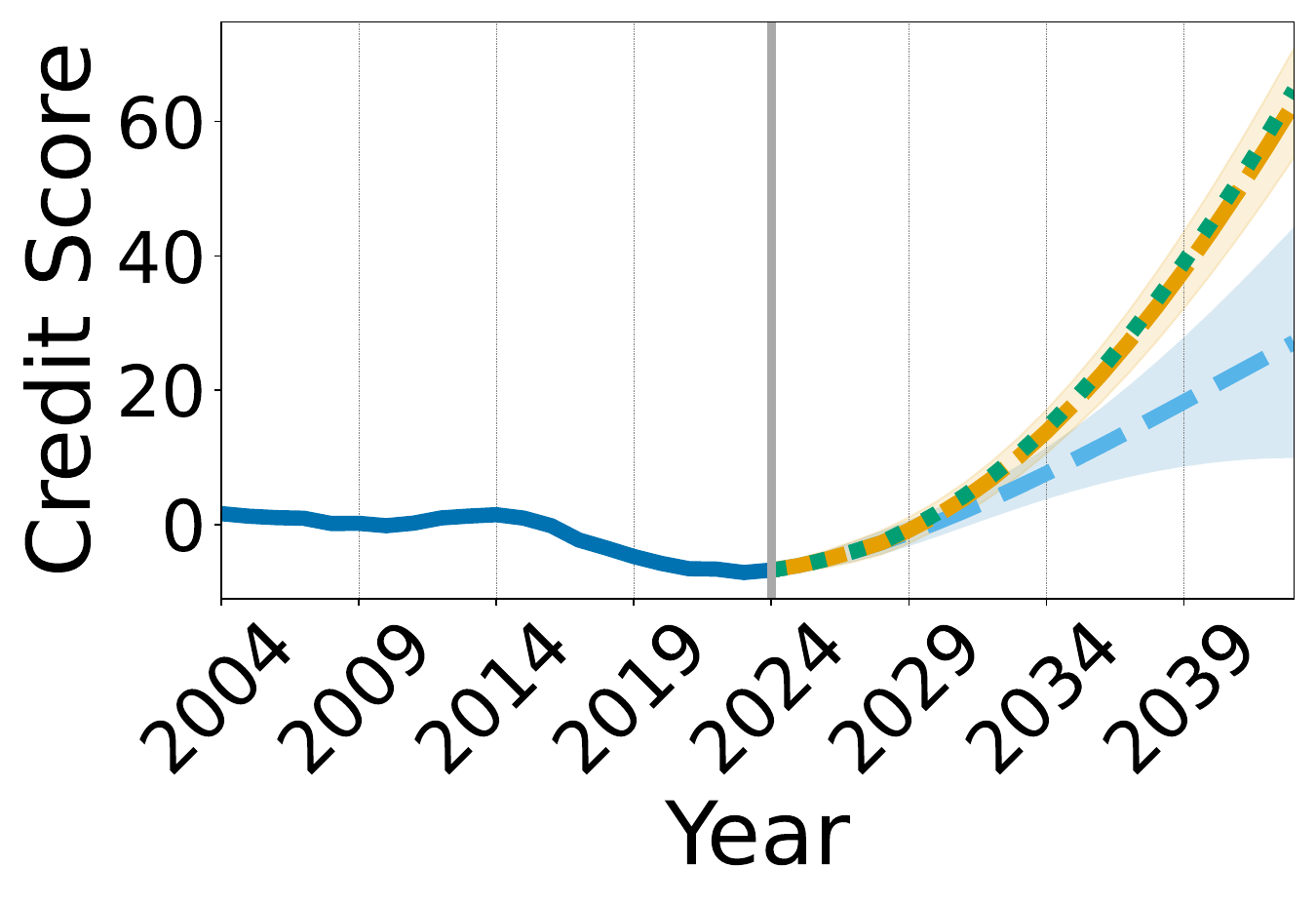}
  \end{subfigure}
  
  % Caption per la seconda riga
  \caption{\textbf{Forcing Intervention.} (Left) Intervention on \textit{Expertise} with $\bm{F} = 1$ and target $\hat{E}=5$. (Right) Effect on \textit{Credit Score}. Confidence bounds are narrower compared to Fig.~\ref{fig:additive_int}.}
  \label{fig:forcing_int}
\end{figure*}

\subsection{Interventional Forecasting}\label{subsec:interventional_forecasting}
We evaluate the causal \ac{VAR}'s forecast in estimating the causal effects on German.
%We measure the forecast accuracy of the causal \ac{VAR} framework in estimating the causal effect of an intervention over time. We focus on the German dataset. Refer to~\ref{app:int_forecasting} for the analysis of Census and Pendulum datasets. 
%

\paragraph{Baselines} 
Since state-of-the-art methods do not allow computing the causal effect of interventions on dynamical systems, we use an oracle forecaster 
as a benchmark for theoretically optimal performance.
Specifically, the ground truth values are estimated as $\ce_{t+h} = \expected{\X^\intervention_{t+h} - \X_{t+h}|\X_{<t}}$, while the predicted values from the proposed \ac{VAR} framework as $\hat{\ce}_{t+h} = (\hat{\X}^\intervention_t - \hat{\X}_{t+h})|\X_{<t}$.

\paragraph{Interventions} 
We perform causal interventions on the root node \textit{Expertise} and observe the effect on the target variable \textit{Credit Score}.
For the additive case, we apply $\bm{F}=0.2$, while for forcing, we use $\bm{F}=1$ with a target value of $\hat{E} = 5$. These values are selected for illustrative purposes such that the long-term expected value of \textit{Expertise} is the same for both interventions (i.e., $5$).

\paragraph{RQ2: Additive and forcing interventions}
Figure~\ref{fig:additive_int} illustrates the additive intervention.
%on \textit{Expertise} (one of the root nodes of the causal graph) with $\bm{F}=1$, while \ref{fig:add_credit} shows the resulting effect on the target variable \textit{Credit Score}.
\textit{Expertise} is a variable that typically necessitates several years for acquisition in practical scenarios. 
The causal \ac{VAR} accurately captures such delayed impact as its effect on \textit{Credit Score} appears after several years.
As the system maintains its dynamic characteristics unchanged, the forecasted covariance remains the same even after the intervention.
Figure~\ref{fig:forcing_int} shows the forcing intervention, 
%on $\textit{Expertise}$ with $\bm{F}=1$ where we compel the variable to converge to a target value $\hat{E}=2$, while \ref{fig:forcing_credit} shows the resulting effect on \textit{Credit Score}.
 where the interventional forecasting exactly aligns with the target value. %This result stems from the intervention's purpose of simulating scenarios that required fixing a specific value (\ref{subsec:do-like}). 
Moreover, we stress that even for a low value of $\bm{F}$, the forcing intervention resembles a do-intervention (shrinking the variances significantly) even though theoretically convergence is guaranteed only for $F \to \infty$.
%This can also be observable from the near-zero variance on the intervened variable.

\begin{table}[t!]
    \centering
    \caption{\textbf{Interventional Forecasting.} 
    MAE scores (\emph{lower is better}) for the proposed causal \ac{VAR} framework on the German dataset. Results averaged over $10$ runs, with standard deviation in subscript. Scores are scaled by a factor of $10^2$ to ease readability.}
    \label{tab:interventional_forecasting_german}
    \setlength{\tabcolsep}{4pt}
    \footnotesize
    \renewcommand{\arraystretch}{1.05}
    \begin{tabular}{ccccc}
    \toprule
    & & & \multicolumn{2}{c}{Interventional Forecasting} \\
    \cmidrule(lr){4-5}
    Dataset & Size & Horizon & Additive & Forcing \\
    \midrule
    \multirow{4}{*}{German} & \multirow{2}{*}{100} & 1 & $.000_{.000}$ & $.000_{.000}$ \\
& & 10 & $.043_{.028}$ & $.364_{.297}$ \\
\cmidrule(lr){2-5}
 & \multirow{2}{*}{500} & 1 & $.000_{.000}$ & $.000_{.000}$ \\
& & 10 & $.018_{.014}$ & $.115_{.081}$ \\
\bottomrule
\end{tabular}
%\vspace{-0.1cm}
\end{table}

\paragraph{RQ3: Estimating the causal effect of interventions over time} 
Table~\ref{tab:interventional_forecasting_german} summarizes results on interventional forecasting, showing  errors with varying data size %(i.e., number of instances in the training set) 
and forecast horizons.  
At $1$-step, both interventions lead to perfect performance since \textit{Credit Score} is a slow-changing variable and requires at least $3$ time steps for an intervention to take effect.
At $10$-step, our causal effect estimates remain highly accurate.

%for forcing interventions result in lower forecasting error than the additive interventions due to their mechanism of driving the intervened variable toward a specific target. 
%the forcing mechanism drives the intervened variable toward a specific target value.
%\IV{ be more specific}
%, thereby reducing its variance. This constraint leads to more predictable system behavior, resulting in estimates that more closely align with the oracle's predictions. 
%In contrast, additive interventions shift the variable by a constant amount, allowing for greater variability and potential divergence from the oracle's forecast

\subsection{Experimental details and extra results}
\label{app:extra_results}

\definecolor{myorange}{HTML}{E69F00}
\definecolor{mypurple}{HTML}{8B008B}

In this section, we expand the description of the experimental section from Section~\ref{sec:experiments}, and provide the reader with additional results in the following subsections.
First, we outline the details that are consistent across all experiments and then explore the specifics of each experiment in their respective subsections.

\paragraph{Hardware}
To run all experiments, we used a machine with 18 cores, 36 threads, equipped with Intel(R) Core (TM) i9-10980XE CPU @ 3.00GHz, OS Ubuntu 20.04.3, programming language Python 3.12.0. 

\paragraph{Datasets}\label{app:dataset}
This section provides all the information on datasets used in the empirical evaluation of Section~\ref{sec:experiments} of the main manuscript and the following subsections. 

\vspace{0.2cm}
\textbf{German.} We aimed to capture relationships between the variables that appeared intuitive to us and, to a certain extent, reflected a real-world loan approval scenario. Therefore, we selected the following variables from the German Credit dataset\footnote{\url{https://www.kaggle.com/datasets/uciml/german-credit}}: 
\textit{Expertise} ($E$), 
\textit{Responsibility} ($R$), 
\textit{Loan Amount} ($L$), 
\textit{Loan Duration} ($D$), 
\textit{Income} ($I$), 
\textit{Savings} ($S$) 
and \textit{Credit Score} ($C$). 
We have indexed the variables from $0$ to $6$, following the sequence in which they are mentioned.
The corresponding causal graph is designed as in Figure~\ref{fig:german_graph}. 
To generate a $7$-dimensional vector-valued stationary time series, we simulate a stable \svar{}$(4)$ process, considering four lags to be a reasonable timeframe for conducting our analysis.
In modeling the autoregressive behavior of socio-economic variables, we use different matrices tailored to capture temporal dependencies in the data. Below, we provide interpretations for such matrices and specify all the non-zero coefficients:
\begin{itemize}
    \item $A_0$ denotes the instantaneous impact between variables and affects only those nominal features (theoretical value of a good or financial instrument defined by a consensus or standard), e.g., the \textit{Credit Score}. 
    This is highlighted because the tangible attributes of an applicant require a certain period to change. In our model, we set $A_0[3, 6] = 0.5$; $A_0 [4, 6] = -0.3$; $A_0 [5, 6] = -0.5$. 
    Furthermore, all diagonal entries of $A_0$ are set to $1$;
    \item $A_1$ has diagonal values that are all $0.95$ except for the last one (ie., \textit{Credit Score)}, which is $0$. This matrix gives the dynamic system a memory function.
    \item $A_2$ includes interactions among variables we classify as rapid. These effects are considered short-term and are typically the most direct and intense, as the system's memory extends only one step back in time. 
    In our model, we set $A_2 [1, 4] = 0.3$; $A_2 [2, 6] = 0.5$; $A_2 [4, 5] = 0.2$;
    \item $A_3$ contains medium-term effects, illustrating how conditions from three steps back affect the current state of the system. In our model, we set $A_3 [2, 3] = 0.5$;
    \item $A_4$ reflects long-term relationships within the system. In our model, we set $A_4 [0, 1] = 0.3$; $A_4 [0, 4] = 0.8$.
\end{itemize}

\textbf{Pendulum.} For the description of the model see \cref{example:stability_not_granted}. To generate a $2$-dimensional vector-valued stationary time series, we simulate a stable SVAR$(1)$ process. In this case, we set $A_0$ as an identity matrix, while $A_1$ is set as in~\Eqref{eq:pendulum_matrix}. 

\vspace{0.15cm}

\textbf{Census.} We consider open data extracted from the United States Census Bureau\footnote{\url{https://www.census.gov}}.
This source provides a panel dataset spanning yearly demographic characteristics for $227$ countries, starting from $1950$ up to $2023$.
We conducted a preprocessing procedure to align the data with our experimental requisites. Specifically, we selected the $50$ countries with high income\footnote{\url{https://blogs.worldbank.org/en/opendata/world-bank-country-classifications-by-income-level-for-2024-2025}}. 
Since most countries had missing values in the initial decades of data collection, we decided to start the dataset from $1991$, considering the $32$-year period sufficient for our work. The total number of instances is therefore $1600$.
Then, we chose only the essential variables for our analysis, which led to the construction of the ground truth causal graph. The chosen variables include: \textit{Total Population} for three age groups namely $0-14$, $15-64$, and $65-99$, \textit{Net Migration}, \textit{Total Births}, \textit{Total Deaths}. The assumed causal graph is shown in Figure~\ref{fig:census_graph}. 
To adapt the data to the autoregressive model, we modified the assignment of values for \textit{Net Migration}, \textit{Total Births}, and \textit{Total Deaths} by shifting them to the values of the previous year. %As number of lags, we set $1$.

\paragraph{Metrics} 
We measure the discrepancy between the \textit{h}-step forecast $\hat{\X}_{t+h}|\X_{<t}$ and the true value $\X_{t+h}$ in the test set using the root mean square error (RMSE) and the symmetric mean absolute percentage error (SMAPE) metrics~\citep{montgomery2015introduction}.
We report scores focusing on the target variables (i.e., \textit{Credit Score} for German, $\Xcomp[1]$ for Pendulum, and age groups for Census).

\paragraph{Training and evaluation methodology}
We generated synthetic time-series datasets with varying training samples. We used $300$ validation and $2400$ test samples for German while $100$ validation and $2200$ test samples for Pendulum.

We implemented DLinear, TiDe, and TSMixer models using the Darts library~\citep{herzen2022darts}. Each model was trained for $200$ epochs, and final evaluations were performed on the test set. 
We employed an early stopping mechanism to prevent overfitting, monitoring the validation loss with patience set to $10$ and a minimum delta of $1e-5$. 
We used the Adam optimizer~\citep{KingmaB14} with an initial learning rate of $1e-3$ and an exponential learning rate scheduler\footnote{\url{https://pytorch.org/docs/stable/generated/torch.optim.lr_scheduler.ExponentialLR.html}} with a decay rate of $0.99$.
For hyperparameter tuning, we always select the best hyperparameter combination according to validation loss, reporting results from the test dataset in the manuscript. 
%The input chunk length, representing the number of historical time steps used as model input, was set to the number of lags for each dataset. 
%The output chunk length, denoting the number of time steps predicted simultaneously by the model, was fixed at $1$.
We applied Mean Squared Error as the loss function for DLinear and quantile regression as the likelihood function for TiDE and TSMixer.

For the Census dataset, we constrained the fitting of the \ac{VAR} framework by setting the matrix coefficients without dependencies in the causal graph to zero. All experiments were repeated $10$ times, with results reported as averages and standard deviations.

\begin{table}[t]
    \centering
    \caption{\textbf{Observational forecasting.} MAE scores in Table~\ref{tab:comparative_forecasting}. Here, we also report the standard deviation in subscript.}\label{tab:observational_forecasting_MAE}
    \setlength{\tabcolsep}{1.9pt}
    \footnotesize
    \renewcommand{\arraystretch}{1.2}
    \begin{tabular}{lcccccccc}
    \toprule
    & & & \multicolumn{5}{c}{MAE} \\
    \cmidrule(lr){4-8}
    Dataset & Size & Horizon & Oracle & VAR & DLinear & TiDE & TSMixer \\
    \midrule
    German & \multirow{4}{*}{\centering100} & \multirow{2}{*}{\centering1} & $\mathtt{.004}$ & $\mathbf{.008}_{.004}$ & $.009_{.004}$ & $.011_{.000}$ & $.014_{.001}$ \\
Pendulum & & & $\mathtt{.042}$ & $\mathbf{.043}_{.001}$ & $\mathbf{.043}_{.001}$ & $.218_{.000}$ & $.217_{.000}$ \\
\cmidrule(lr){3-8}
German & & \multirow{2}{*}{\centering10} & $\mathtt{.014}$ & $\mathbf{.055}_{.030}$ & $\mathbf{.055}_{.027}$ & $.094_{.003}$ & $.139_{.005}$ \\
Pendulum & & & $\mathtt{.399}$ & $\mathbf{.420}_{.018}$ & $.440_{.009}$ & $1.43_{.002}$ & $1.43_{.001}$ \\
\midrule
German & \multirow{4}{*}{\centering500} & \multirow{2}{*}{\centering1} & $\mathtt{.004}$ & $\mathbf{.004}_{.000}$ & $\mathbf{.004}_{.000}$ & $.011_{.000}$ & $.014_{.000}$ \\
Pendulum & & & $\mathtt{.042}$ & $\mathbf{.042}_{.000}$ & $\mathbf{
.042}_{.000}$ & $.218_{.001}$ & $.217_{.000}$ \\
\cmidrule(lr){3-8}
German & & \multirow{2}{*}{\centering10} & $\mathtt{.014}$ & $\mathbf{.015}_{.001}$ & $\mathbf{.015}_{.000}$ & $.093_{.002}$ & $.135_{.004}$ \\
Pendulum & & & $\mathtt{.399}$ & $\mathbf{.401}_{.002}$ & $.405_{.010}$ & $1.43_{.003}$ & $1.43_{.001}$ \\
%\midrule
%German & \multirow{4}{*}{\centering1000} & \multirow{2}{*}{\centering1} & $\mathtt{.004}$ & $\mathbf{.004}_{.000}$ & $.004_{.000}$ & $.011_{.000}$ & $.012_{.000}$ \\
%Pendulum & & & $\mathtt{.042}$ & $\mathbf{.042}_{.000}$ & $.042_{.000}$ & $.218_{.000}$ & $.217_{.000}$ \\
%\cmidrule(lr){3-8}
%German & & \multirow{2}{*}{\centering10} & $\mathtt{.014}$ & $\mathbf{.014}_{.000}$ & $.014_{.000}$ & $.093_{.003}$ & $.125_{.004}$ \\
%Pendulum & & & $\mathtt{.399}$ & $\mathbf{.400}_{.001}$ & $.402_{.006}$ & $1.43_{.002}$ & $1.43_{.001}$ \\
\bottomrule
    \end{tabular}
\end{table}

\subsection*{Observational Forecasting}\label{app:obs_forecasting}
\paragraph{How does the VAR performance compare with SOTA
models for forecasting multivariate time series?}
Table~\ref{tab:observational_forecasting_MAE} presents results that
mirror those in Table~\ref{tab:comparative_forecasting} with standard deviation in subscript. 
In addition to the results from Section~\ref{subsec:observational_forecasting}, Table~\ref{tab:observational_forecasting_combined} provides the forecasting accuracy, measured using RMSE and SMAPE scores, across different data sizes and forecast horizons for German and Pendulum datasets. 
The patterns observed here are consistent with those in the main text.
\ac{VAR} continues to lead in performance, often aligning closely with the Oracle. DLinear remains competitive for shorter forecast horizons thanks to its linear structure, occasionally surpassing \ac{VAR} by a small margin.%, though generally trailing slightly overall.

\begin{table*}[ht!]
    \centering
    \caption{\textbf{Observational forecasting.} RMSE and SMAPE scores (\emph{lower is better}) for \ac{VAR}, DLinear~\citep{zeng2023dlinear}, TiDE~\citep{das2023tide} and TSMixer~\citep{chen2023tsmixer}, benchmarked against the oracle forecaster. Results averaged over ten runs, with standard deviation in subscript.}
    \label{tab:observational_forecasting_combined}
    \setlength{\tabcolsep}{1.5pt}
    \footnotesize
    \renewcommand{\arraystretch}{1.18}
    \begin{tabular}{lcccccccccccc}
    \toprule
    & & & \multicolumn{5}{c}{RMSE} & \multicolumn{5}{c}{SMAPE} \\
    \cmidrule(lr){4-8} \cmidrule(lr){9-13}
    Dataset & Size & Horizon & Oracle & VAR & DLinear & TiDE & TSMixer & Oracle & VAR & DLinear & TiDE & TSMixer \\
    \midrule
    German & \multirow{4}{*}{\centering100} & \multirow{2}{*}{\centering1} & $\mathtt{.005}$ & $\mathbf{.010}_{.004}$ & $.011_{.005}$ & $.063_{.009}$ & $.071_{.002}$ & $\mathtt{2.01}$ & $\mathbf{3.81}_{1.44}$ & $4.12_{1.56}$ & $14.6_{2.16}$ & $20.4_{0.33}$ \\
    Pendulum & & & $\mathtt{.053}$ & $\mathbf{.053}_{.001}$ & $.054_{.001}$ & $.333_{.033}$ & $.264_{.000}$ & $\mathtt{15.7}$ & $\mathbf{15.8}_{0.22}$ & $15.8_{0.23}$ & $62.0_{3.73}$ & $54.2_{0.00}$ \\
    \cmidrule(lr){3-13}
    German & & \multirow{2}{*}{\centering10} & $\mathtt{.018}$ & $.069_{.038}$ & $\mathbf{.068}_{.033}$ & $.306_{.053}$ & $.284_{.008}$ & $\mathtt{7.33}$ & $17.2_{6.14}$ & $\mathbf{17.0}_{5.36}$ & $53.7_{3.75}$ & $58.8_{0.46}$ \\
    Pendulum & & & $\mathtt{.494}$ & $\mathbf{.520}_{.022}$ & $.545_{.013}$ & $1.89_{.069}$ & $1.74_{.000}$ & $\mathtt{82.7}$ & $\mathbf{87.2}_{4.61}$ & $88.9_{3.21}$ & $175_{2.32}$ & $180_{0.03}$ \\
    \midrule
    German & \multirow{4}{*}{\centering500} & \multirow{2}{*}{\centering1} & $\mathtt{.005}$ & $\mathbf{.005}_{.000}$ & $\mathbf{.005}_{.000}$ & $.037_{.003}$ & $.074_{.002}$ & $\mathtt{2.01}$ & $2.15_{0.08}$ & $\mathbf{2.13}_{0.08}$ & $8.45_{0.46}$ & $17.7_{0.44}$ \\
    Pendulum & & & $\mathtt{.053}$ & $\mathbf{.053}_{.000}$ & $\mathbf{.053}_{.000}$ & $.352_{.014}$ & $.264_{.000}$ & $\mathtt{15.7}$ & $\mathbf{15.7}_{0.02}$ & $\mathbf{15.7}_{0.03}$ & $63.8_{1.72}$ & $54.2_{0.00}$ \\
    \cmidrule(lr){3-13}
    German & & \multirow{2}{*}{\centering10} & $\mathtt{.018}$ & $\mathbf{.019}_{.001}$ & $\mathbf{.019}_{.001}$ & $.200_{.014}$ & $.320_{.008}$ & $\mathtt{7.33}$ & $\mathbf{7.55}_{0.19}$ & $7.68_{0.20}$ & $40.9_{1.63}$ & $56.5_{0.35}$ \\
    Pendulum & & & $\mathtt{.494}$ & $\mathbf{.496}_{.002}$ & $.502_{.011}$ & $1.93_{.030}$ & $1.74_{.000}$ & $\mathtt{82.7}$ & $\mathbf{83.3}_{0.75}$ & $83.8_{1.08}$ & $174_{0.70}$ & $180_{0.04}$ \\
    \midrule
    German & \multirow{4}{*}{\centering1000} & \multirow{2}{*}{\centering1} & $\mathtt{.005}$ & $\mathbf{.005}_{.000}$ & $\mathbf{.005}_{.000}$ & $.025_{.001}$ & $.051_{.002}$ & $\mathtt{2.01}$ & $2.09_{0.03}$ & $\mathbf{2.07}_{0.04}$ & $6.57_{0.21}$ & $12.4_{0.41}$ \\
    Pendulum & & & $\mathtt{.053}$ & $\mathbf{.053}_{.000}$ & $\mathbf{.053}_{.000}$ & $.373_{.018}$ & $.264_{.000}$ & $\mathtt{15.7}$ & $\mathbf{15.7}_{0.02}$ & $\mathbf{15.7}_{0.03}$ & $65.9_{1.89}$ & $54.1_{0.01}$ \\
    \cmidrule(lr){3-13}
    German & & \multirow{2}{*}{\centering10} & $\mathtt{.018}$ & $\mathbf{.018}_{.000}$ & $\mathbf{.018}_{.000}$ & $.161_{.006}$ & $.264_{.007}$ & $\mathtt{7.33}$ & $\mathbf{7.38}_{0.09}$ & $7.43_{0.13}$ & $34.2_{0.66}$ & $51.8_{0.84}$ \\
    Pendulum & & & $\mathtt{.494}$ & $\mathbf{.495}_{.001}$ & $.498_{.006}$ & $1.97_{.037}$ & $1.74_{.000}$ & $\mathtt{82.7}$ & $\mathbf{82.8}_{0.54}$ & $83.2_{0.66}$ & $173_{0.90}$ & $180_{0.01}$ \\
    \bottomrule
    \end{tabular}
\end{table*}
\begin{table*}[ht!]
    \centering
    \caption{\textbf{Observational forecasting.} RMSE and SMAPE scores (\emph{lower is better}) for \ac{VAR}, DLinear, TiDE and TSMixer on the Census dataset.}
    \label{tab:census_forecasting}
    \setlength{\tabcolsep}{3pt}
    \footnotesize
    \renewcommand{\arraystretch}{1.18}
    \begin{tabular}{lcccccccccc}
    \toprule
    & & & \multicolumn{4}{c}{RMSE} & \multicolumn{4}{c}{SMAPE} \\
    \cmidrule(lr){4-7} \cmidrule(lr){8-11}
    Dataset & Size & Horizon & VAR & DLinear & TiDE & TSMixer & VAR & DLinear & TiDE & TSMixer \\
    \midrule
    \multirow{2}{*}{\centering Census} & \multirow{2}{*}{\centering50$\times$32} & \multirow{1}{*}{\centering1} & $\mathbf{.002}$ & $.007$ & $\mathbf{.002}$ & $.009$ & $.409$ & $1.76$ & $\mathbf{.401}$ & $1.79$ \\
    \cmidrule(lr){3-11}
     & & \multirow{1}{*}{\centering5} & $.018$ & $.026$ & $\mathbf{.013}$ & $.024$ & $3.25$ & $6.36$ & $\mathbf{2.92}$ & $6.28$ \\
    \bottomrule
    \end{tabular}
\end{table*}

In contrast, TSMixer and TiDE generally underperform in comparison, with their scores consistently following behind \ac{VAR} and DLinear. The challenging nature of the Pendulum dataset is apparent, with all models (including Oracle) showing greater difficulty in maintaining the accuracy as for German.
Table~\ref{tab:census_forecasting} presents the results for the Census dataset.  
The performance gap between the models narrows. For $1$-step forecasts, both \ac{VAR} and TiDE achieve similar, strong results, while for $5$-step forecasts, TiDE slightly outperforms \ac{VAR} in both metrics.

\subsection*{Interventional Forecasting}\label{app:int_forecasting}

\begin{figure*}[t!]
  \centering
  \includegraphics[width=1\textwidth]{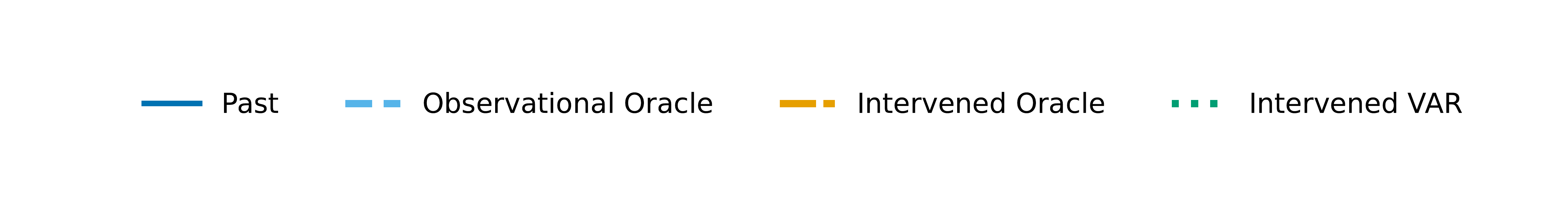}
  \begin{minipage}[t]{0.49\textwidth}
    \centering
    \begin{minipage}{\textwidth}
    \vspace{-0.6cm} 
    \begin{subfigure}[b]{0.49\textwidth}
      \centering
      \includegraphics[width=\textwidth]{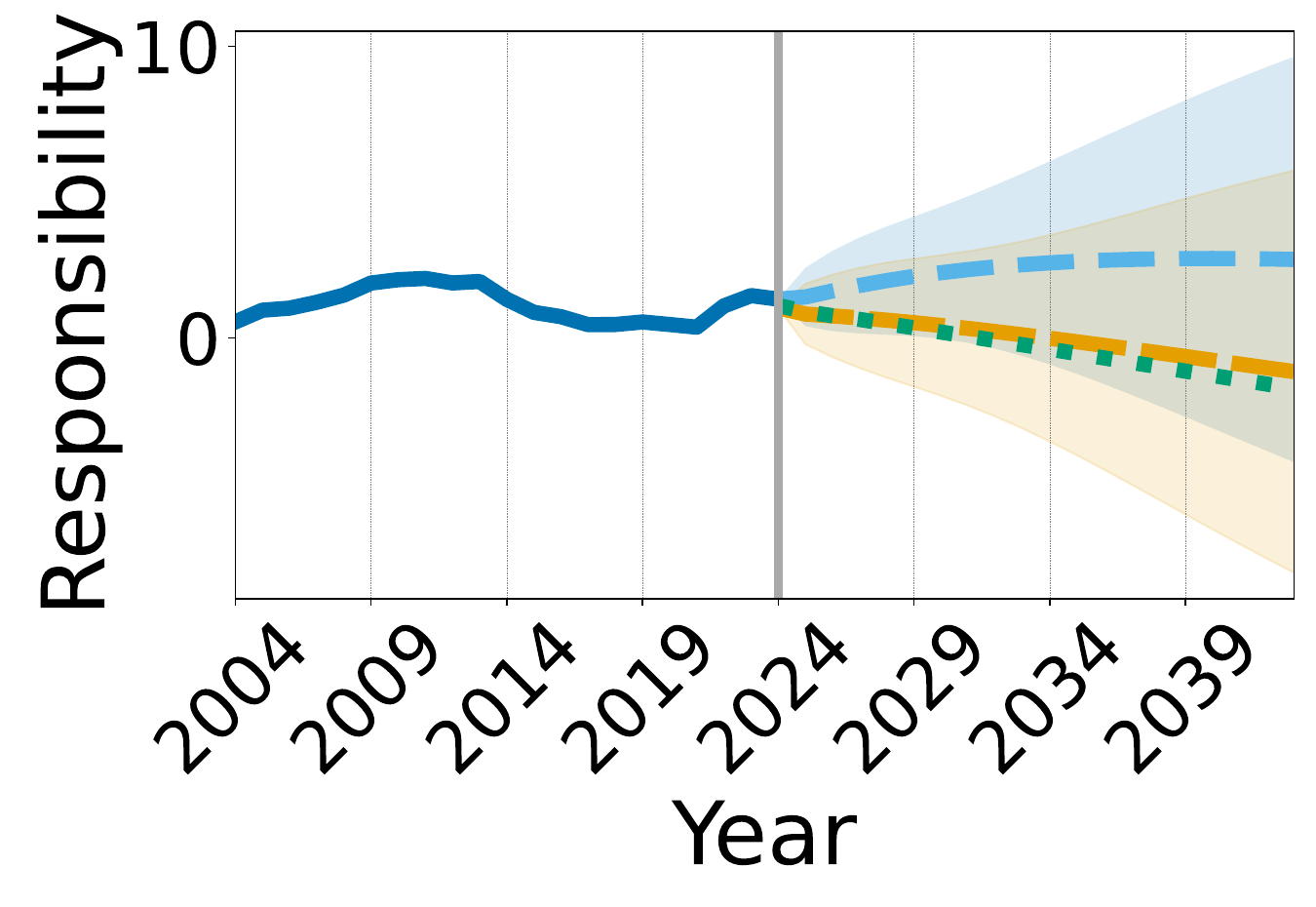}
      \caption{Intervention on $R$}
    \end{subfigure}
    %\hfill
    \begin{subfigure}[b]{0.49\textwidth}
      \centering
      \includegraphics[width=\textwidth]{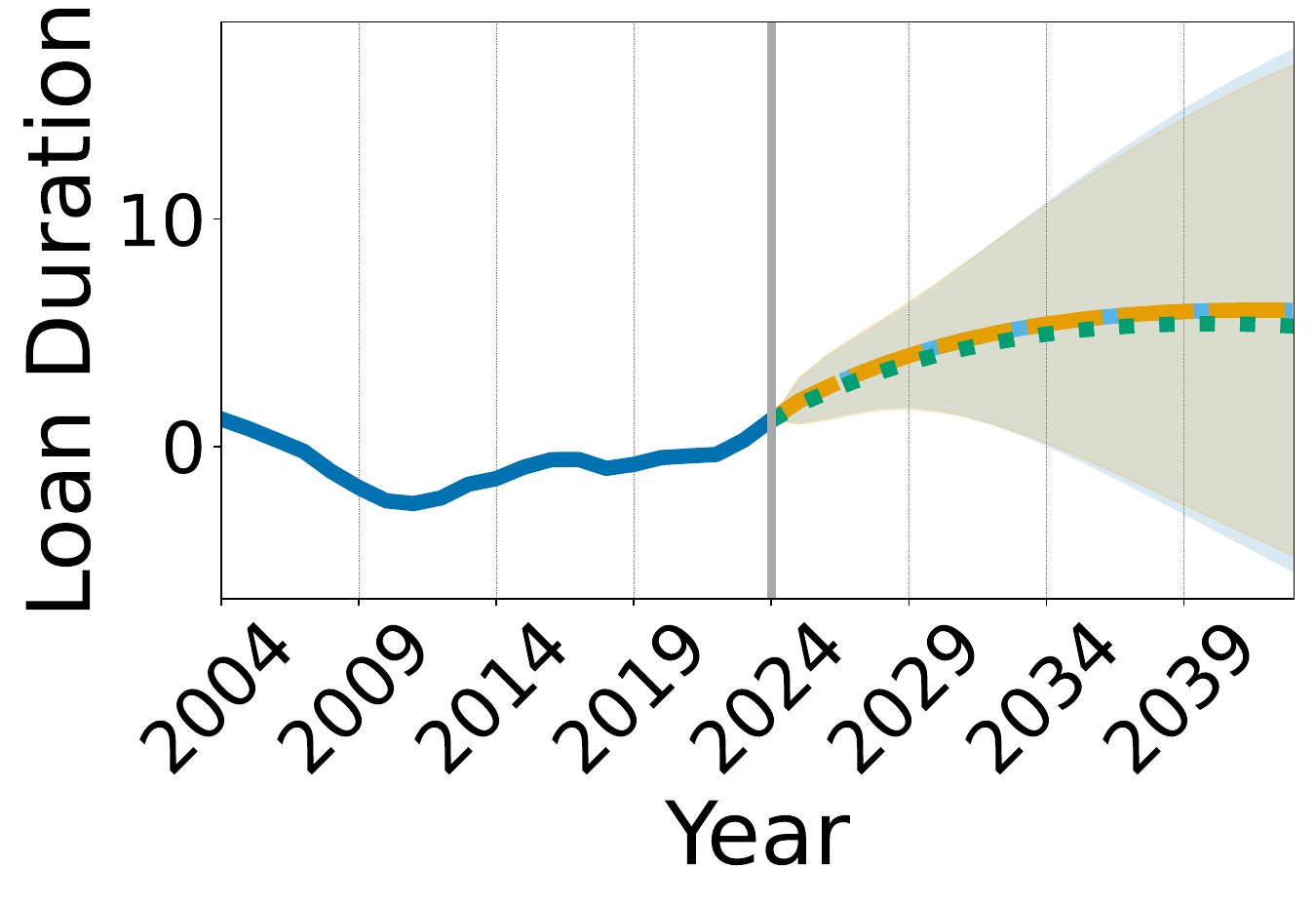}
      \caption{\centering Effect on $D$}
    \end{subfigure}
    \end{minipage}
  \end{minipage}
  \begin{minipage}[t]{0.49\textwidth}
    \centering
    \begin{minipage}{\textwidth}
    \vspace{-1cm} 
    \begin{subfigure}[b]{0.49\textwidth}
      \centering
      \includegraphics[width=\textwidth]{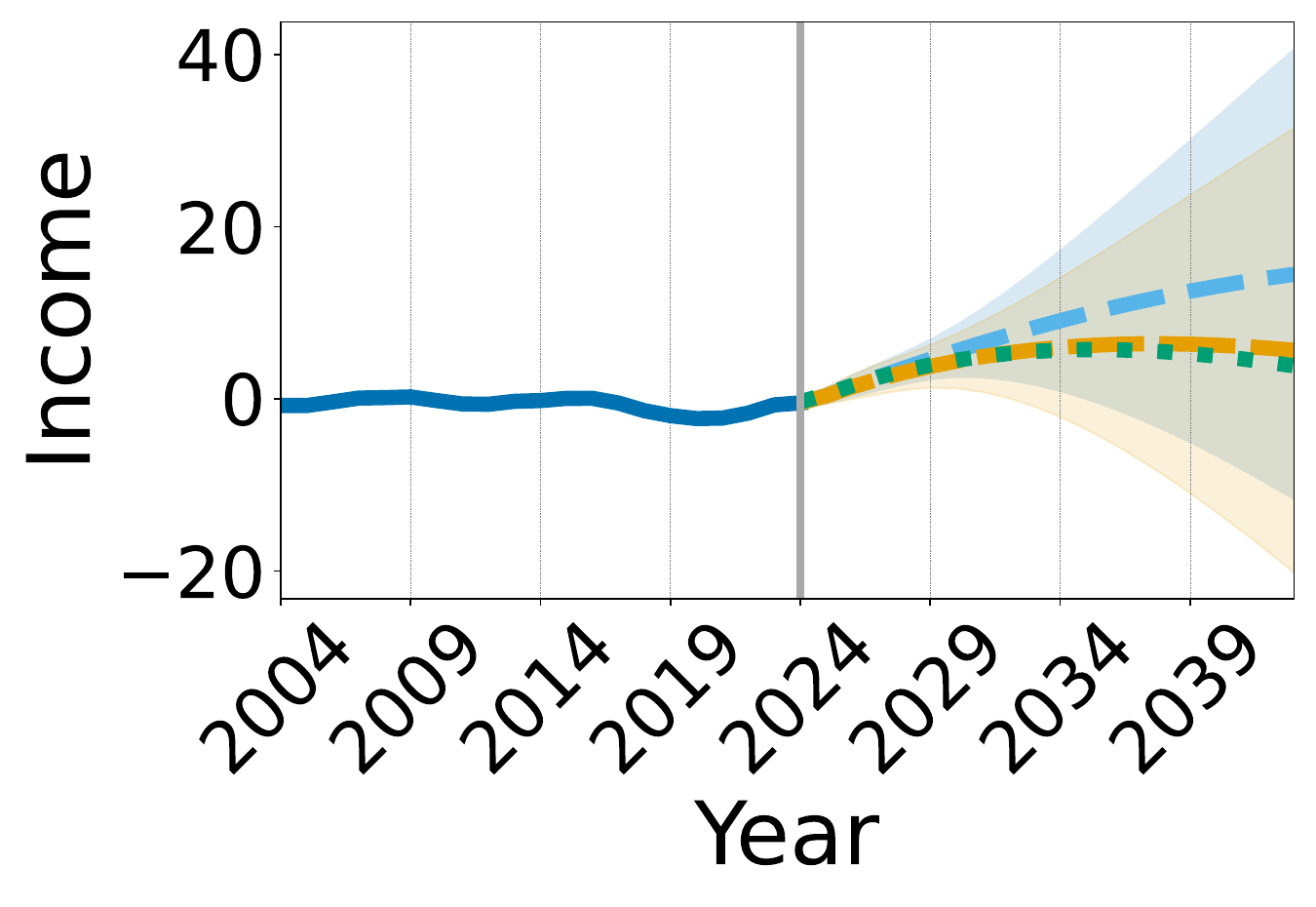}
      \caption{\centering Effect on $I$}
    \end{subfigure}
    \hfill
    \begin{subfigure}[b]{0.49\textwidth}
      \centering
      %\vspace{-0.5cm}
      \begin{tikzpicture}[
            scale=0.6,
            transform shape,
            vertex/.style={circle, draw, minimum size=1cm, inner sep=0pt, font=\footnotesize}
        ]
        \node[vertex, fill=myorange!85!white] (r) at (0,1.5) {$R$};
        \node[vertex] (e) at (0,0) {$E$};
        \node[vertex] (l) at (1.7,0) {$L$};
        \node[vertex] (d) at (3.4,0) {$D$};
        \node[vertex, fill=myorange!30!white] (i) at (1.7,1.5) {$I$};
        \node[vertex, fill=myorange!30!white] (s) at (1.7,3) {$S$};
        \node[vertex, fill=myorange!30!white] (y) at (3.4,1.5) {$C$};
        
        \draw[edge] (e) -- (r);
        \draw[edge] (l) -- (d);
        \draw[edge] (r) -- (i);
        \draw[edge] (e) -- (i);
        \draw[edge] (i) -- (s);
        \draw[edge] (l) -- (y);
        \draw[edge] (d) -- (y);
        \draw[edge] (i) -- (y);
        \draw[edge] (s) -- (y);
        \end{tikzpicture}
        \caption{German causal graph}
    \end{subfigure}
     \end{minipage}
  \end{minipage} 
  \centering
  \begin{minipage}[t]{0.49\textwidth}
    \centering
    \begin{minipage}{\textwidth}
    \vspace{0.5cm} 
    \begin{subfigure}[b]{0.49\textwidth}
      \centering
      \includegraphics[width=\textwidth]{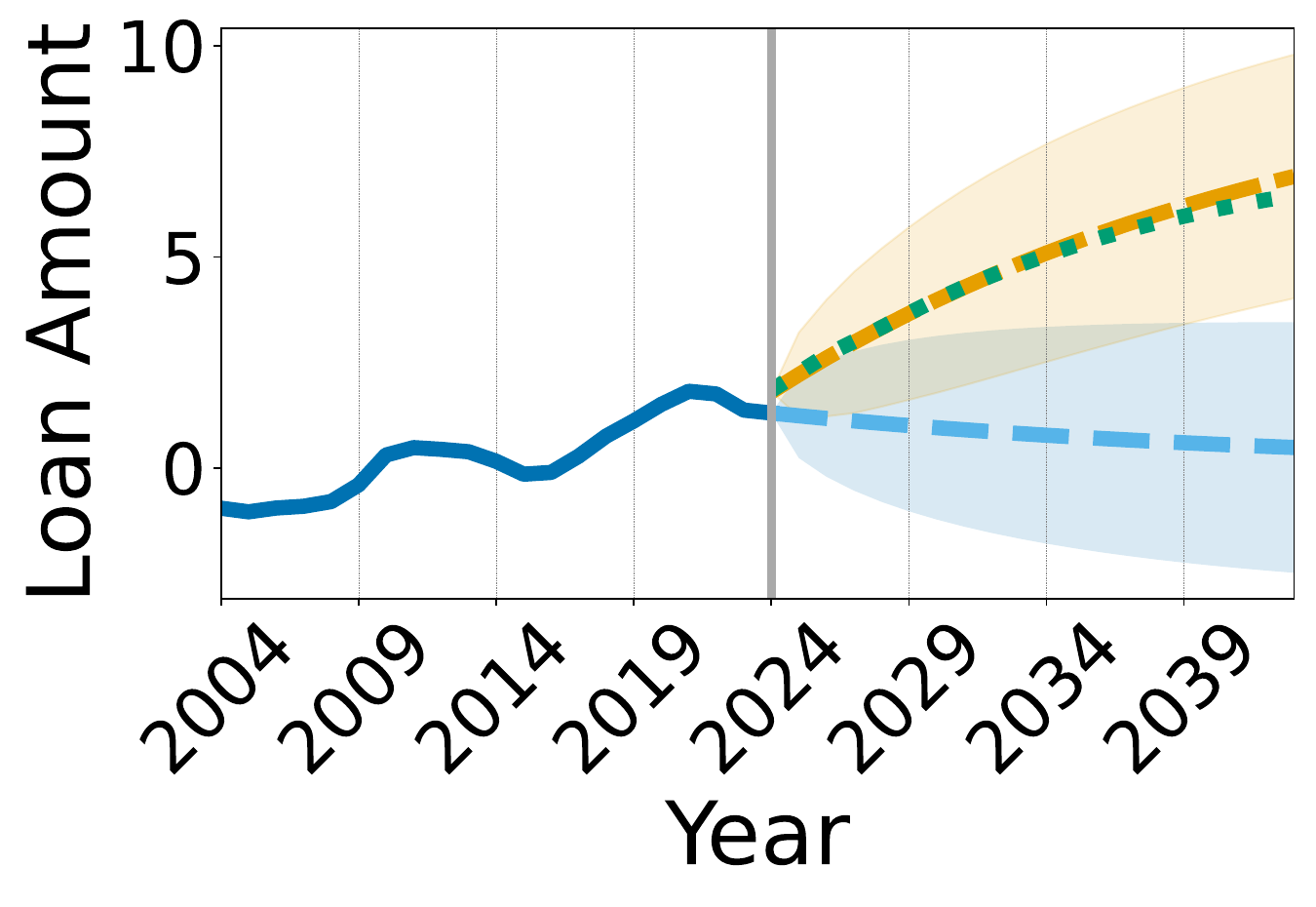}
      \caption{Intervention on $L$}
    \end{subfigure}
    %\hfill
    \begin{subfigure}[b]{0.49\textwidth}
      \centering
      \includegraphics[width=\textwidth]{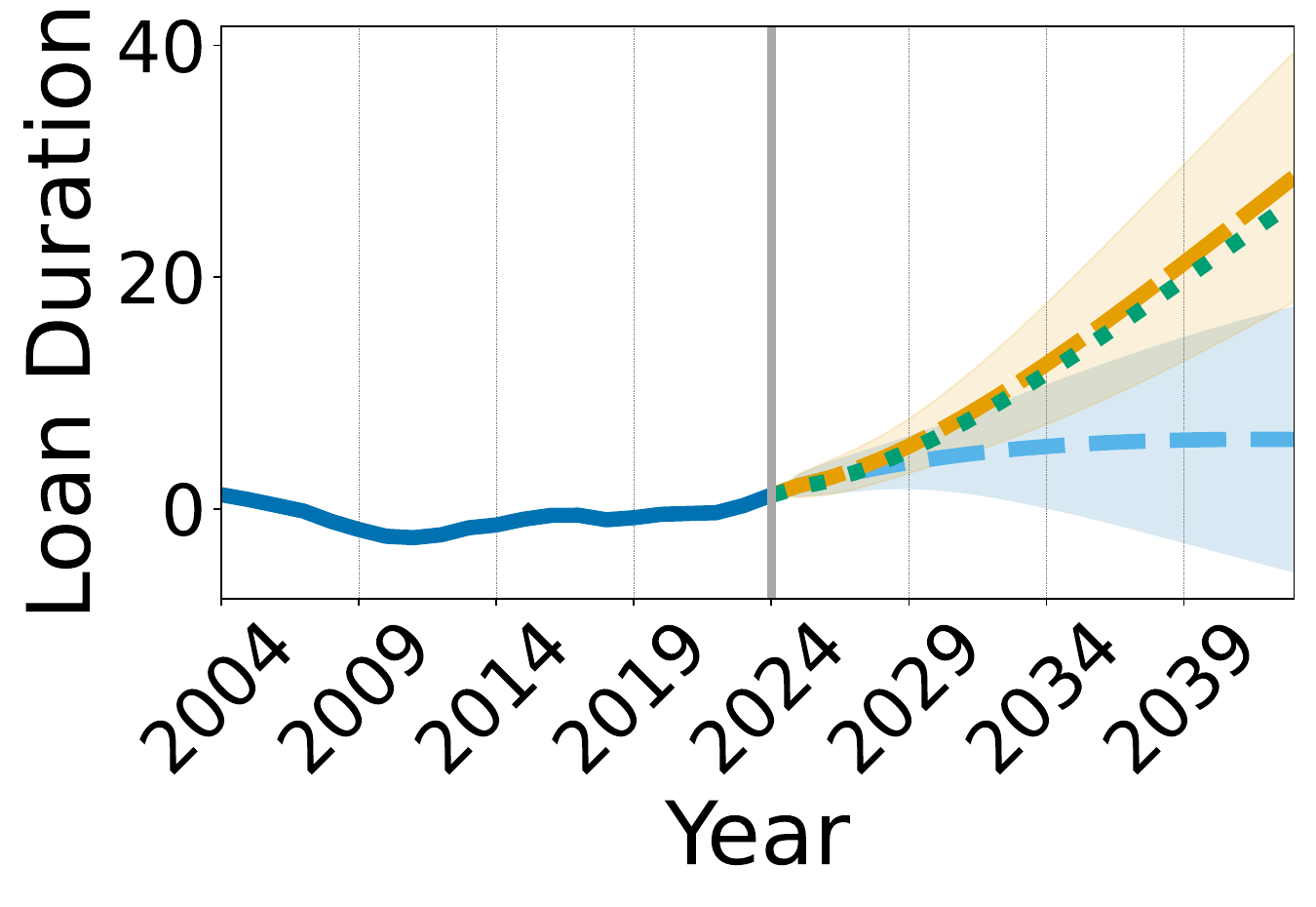}
      \caption{\centering Effect on $D$}
    \end{subfigure}
    \end{minipage}
  \end{minipage}
  \begin{minipage}[t]{0.49\textwidth}
    \centering
    \begin{minipage}{\textwidth}
    \vspace{0.5cm} 
    \begin{subfigure}[b]{0.49\textwidth}
      \centering
      \includegraphics[width=\textwidth]{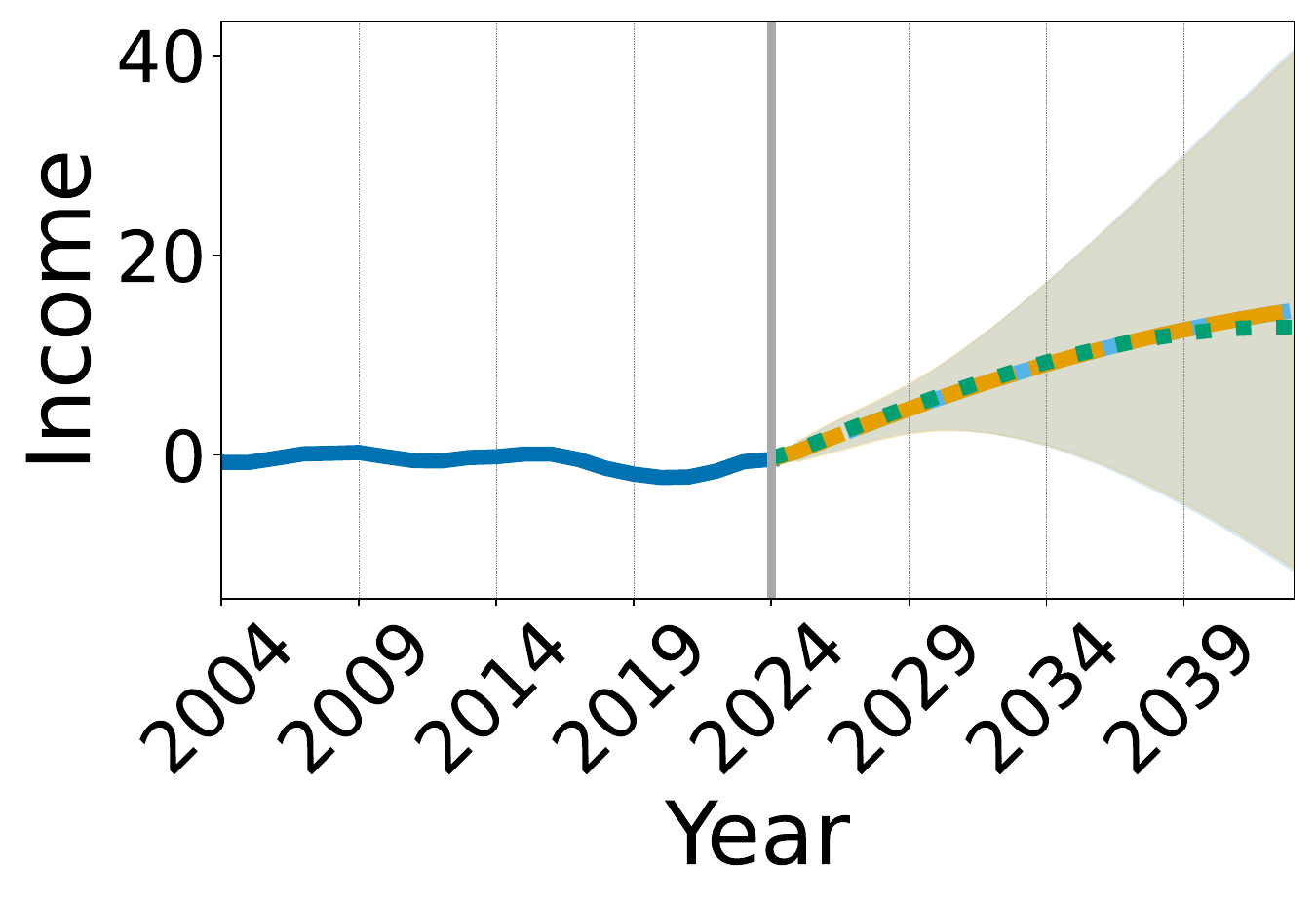}
      \caption{\centering Effect on $I$}
    \end{subfigure}
    \hfill
    \begin{subfigure}[b]{0.49\textwidth}
      \centering
      \begin{tikzpicture}[
            scale=0.6,
            transform shape,
            vertex/.style={circle, draw, minimum size=1cm, inner sep=0pt, font=\footnotesize}
        ]
        \node[vertex] (r) at (0,1.5) {$R$};
        \node[vertex] (e) at (0,0) {$E$};
        \node[vertex, fill=myorange!85!white] (l) at (1.7,0) {$L$};
        \node[vertex, fill=myorange!30!white] (d) at (3.4,0) {$D$};
        \node[vertex] (i) at (1.7,1.5) {$I$};
        \node[vertex] (s) at (1.7,3) {$S$};
        \node[vertex, fill=myorange!30!white] (y) at (3.4,1.5) {$C$};
        
        \draw[edge] (e) -- (r);
        \draw[edge] (l) -- (d);
        \draw[edge] (r) -- (i);
        \draw[edge] (e) -- (i);
        \draw[edge] (i) -- (s);
        \draw[edge] (l) -- (y);
        \draw[edge] (d) -- (y);
        \draw[edge] (i) -- (y);
        \draw[edge] (s) -- (y);
        \end{tikzpicture}
        \caption{German causal graph}
    \end{subfigure}
     \end{minipage}
  \end{minipage} 
  %\caption{\textbf{Additive Intervention.}  (a) Intervention on \textit{Loan Amount} with $\bm{F} = 0.5$. The resulting effect on: (b) \textit{Loan Duration} and (c) \textit{Income}.}
  %  \label{fig:additive_int_loan_amount}
    %\vspace{-0.1cm}
\caption{\textbf{Additive Interventions.} 
(a) Intervention on \textit{Responsibility} with $\bm{F} = -0.3$ and the resulting effect on: (b) \textit{Loan Duration} and (c) \textit{Income}. 
(e) Intervention on \textit{Loan Amount} with $\bm{F} = 0.5$ and the resulting effect on: (f) \textit{Loan Duration} and (g) \textit{Income}.
In (d) and (h), the deep orange node represents the intervened variable, while the lighter orange nodes are those influenced by the intervention. In all plots, shaded regions denote $95\%$ confidence bounds. 
} \label{fig:additive_int_appendix}
\end{figure*}
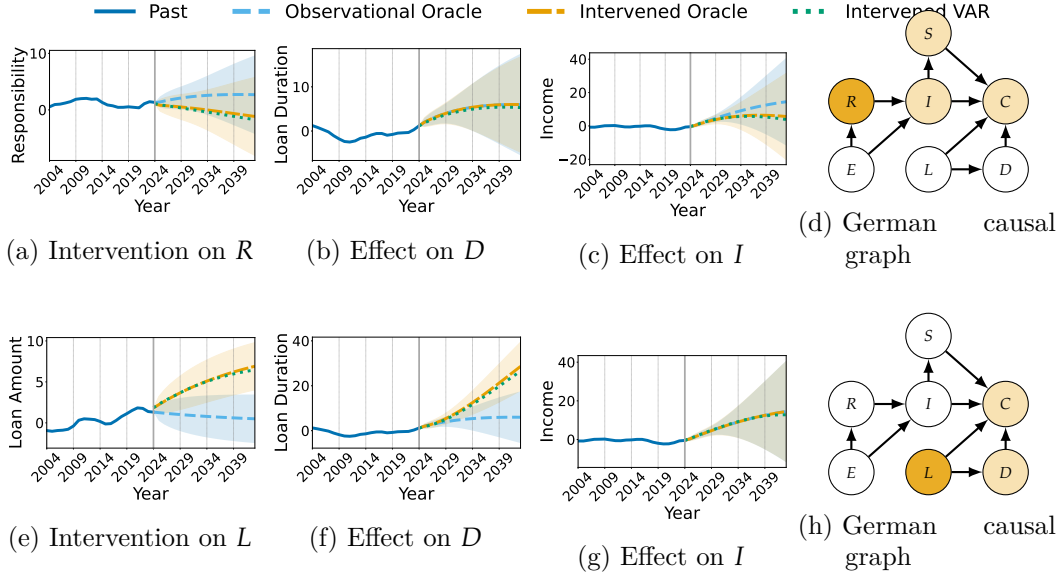

\begin{figure*}[th!]
  \centering
  \includegraphics[width=0.98\textwidth]{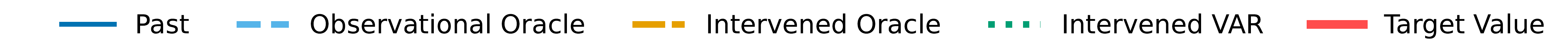}
  \begin{minipage}[t]{0.49\textwidth}
    \centering
    \begin{minipage}{\textwidth}
    \vspace{0.25cm} 
    \begin{subfigure}[b]{0.49\textwidth}
      \centering
      \includegraphics[width=\textwidth]{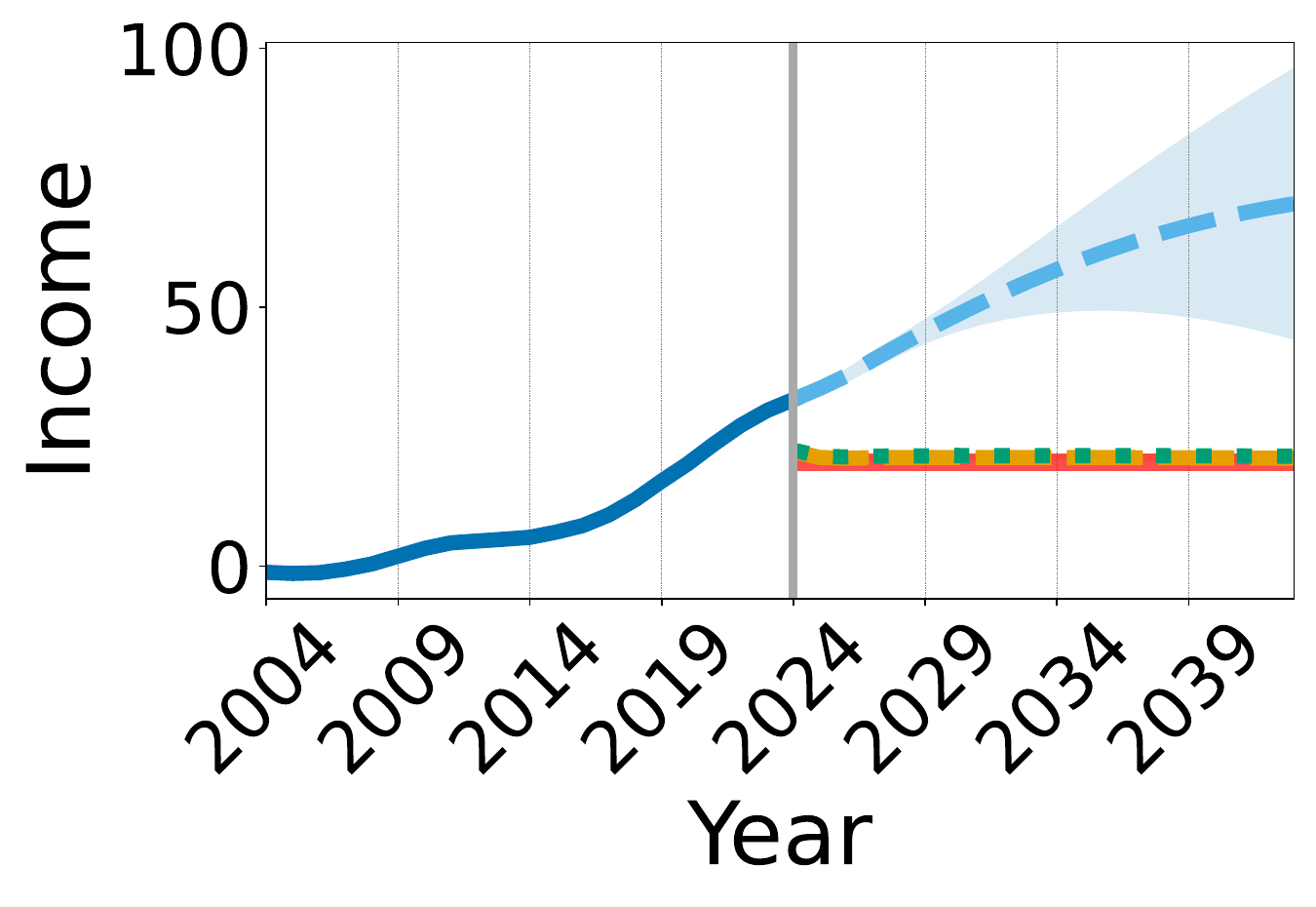}
      \caption{Intervention on $I$}
    \end{subfigure}
    %\hfill
    \begin{subfigure}[b]{0.49\textwidth}
      \centering
      \includegraphics[width=\textwidth]{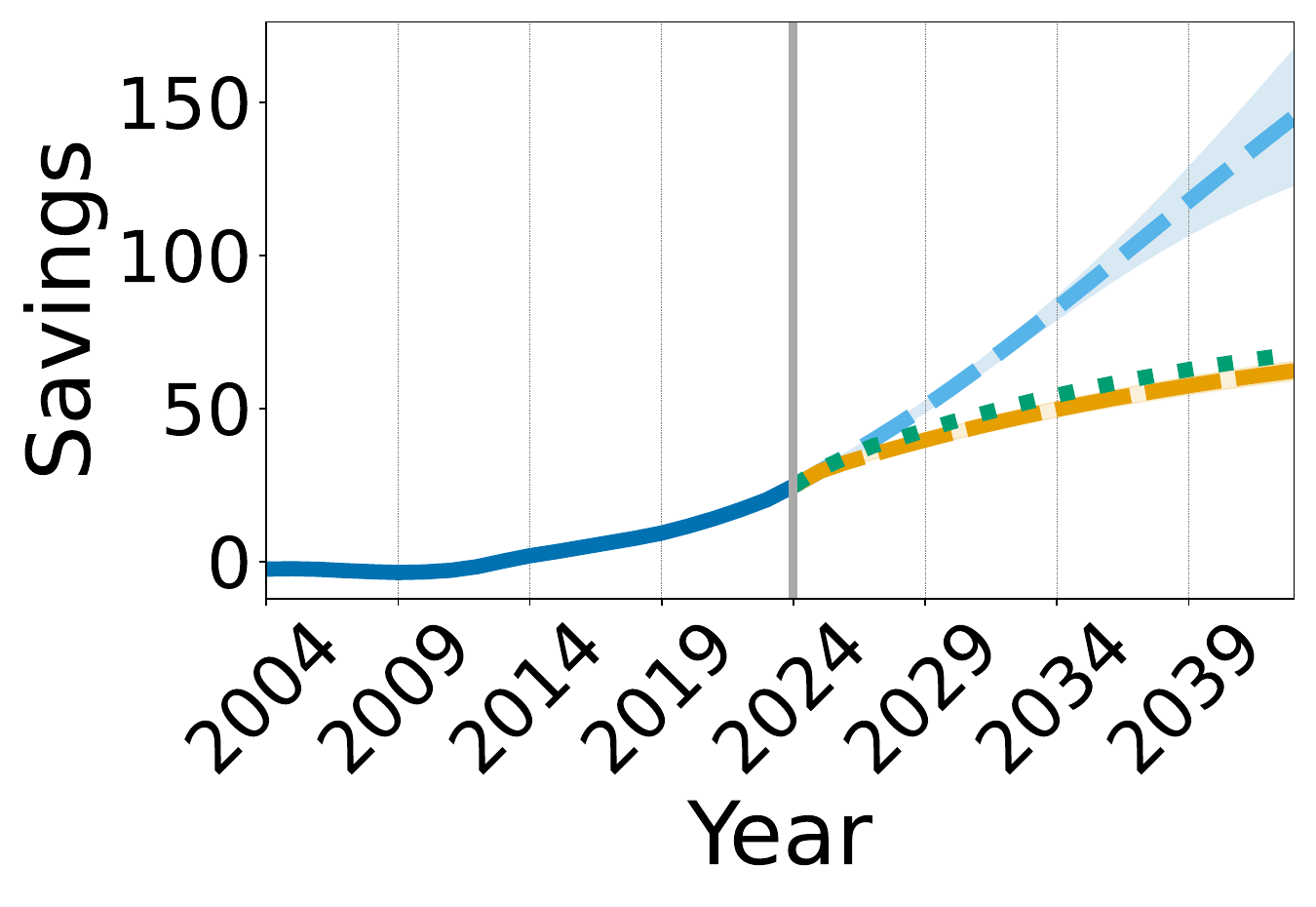}
      \caption{\centering Effect on $S$}
    \end{subfigure}
    \end{minipage}
  \end{minipage}
  \begin{minipage}[t]{0.49\textwidth}
    \centering
    \begin{minipage}{\textwidth}
    \vspace{0.25cm} 
    \begin{subfigure}[b]{0.49\textwidth}
      \centering
      \includegraphics[width=\textwidth]{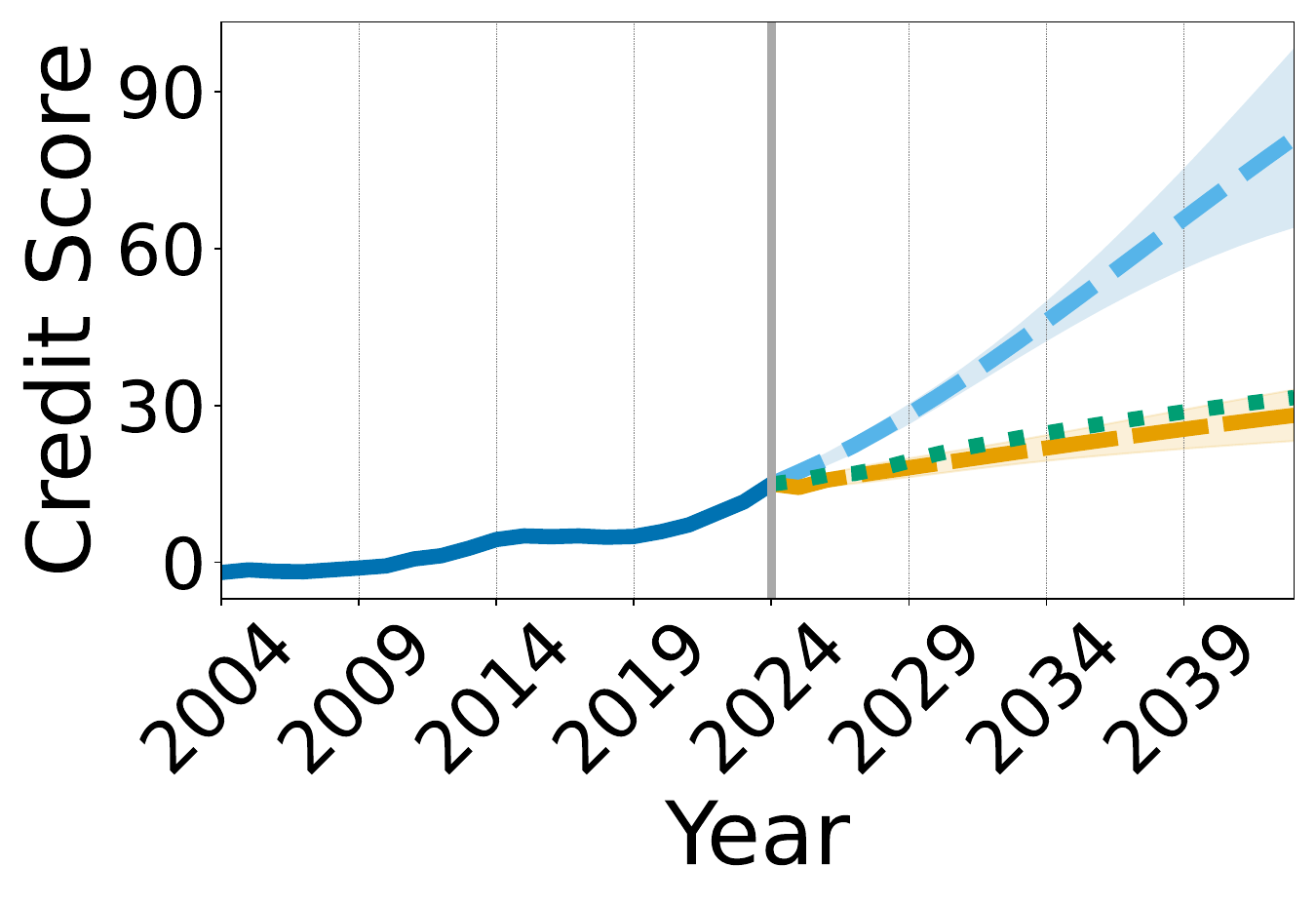}
      \caption{\centering Effect on $C$}
    \end{subfigure}
    \hfill
    \begin{subfigure}[b]{0.49\textwidth}
      \centering
      \begin{tikzpicture}[
            scale=0.6,
            transform shape,
            vertex/.style={circle, draw, minimum size=1cm, inner sep=0pt, font=\footnotesize}
        ]
        \node[vertex] (r) at (0,1.5) {$R$};
        \node[vertex] (e) at (0,0) {$E$};
        \node[vertex] (l) at (1.7,0) {$L$};
        \node[vertex] (d) at (3.4,0) {$D$};
        \node[vertex, fill=myorange!85!white] (i) at (1.7,1.5) {$I$};
        \node[vertex, fill=myorange!30!white] (s) at (1.7,3) {$S$};
        \node[vertex, fill=myorange!30!white] (y) at (3.4,1.5) {$C$};
        
        \draw[edge] (e) -- (r);
        \draw[edge] (l) -- (d);
        \draw[edge] (r) -- (i);
        \draw[edge] (e) -- (i);
        \draw[edge] (i) -- (s);
        \draw[edge] (l) -- (y);
        \draw[edge] (d) -- (y);
        \draw[edge] (i) -- (y);
        \draw[edge] (s) -- (y);
        \end{tikzpicture}
        \caption{German causal graph}
    \end{subfigure}
     \end{minipage}
  \end{minipage} 
  \centering
  \begin{minipage}[t]{0.49\textwidth}
    \centering
    \vspace{0.5cm}
    \begin{minipage}{\textwidth}
    \begin{subfigure}[b]{0.49\textwidth}
      \centering
      \includegraphics[width=\textwidth]{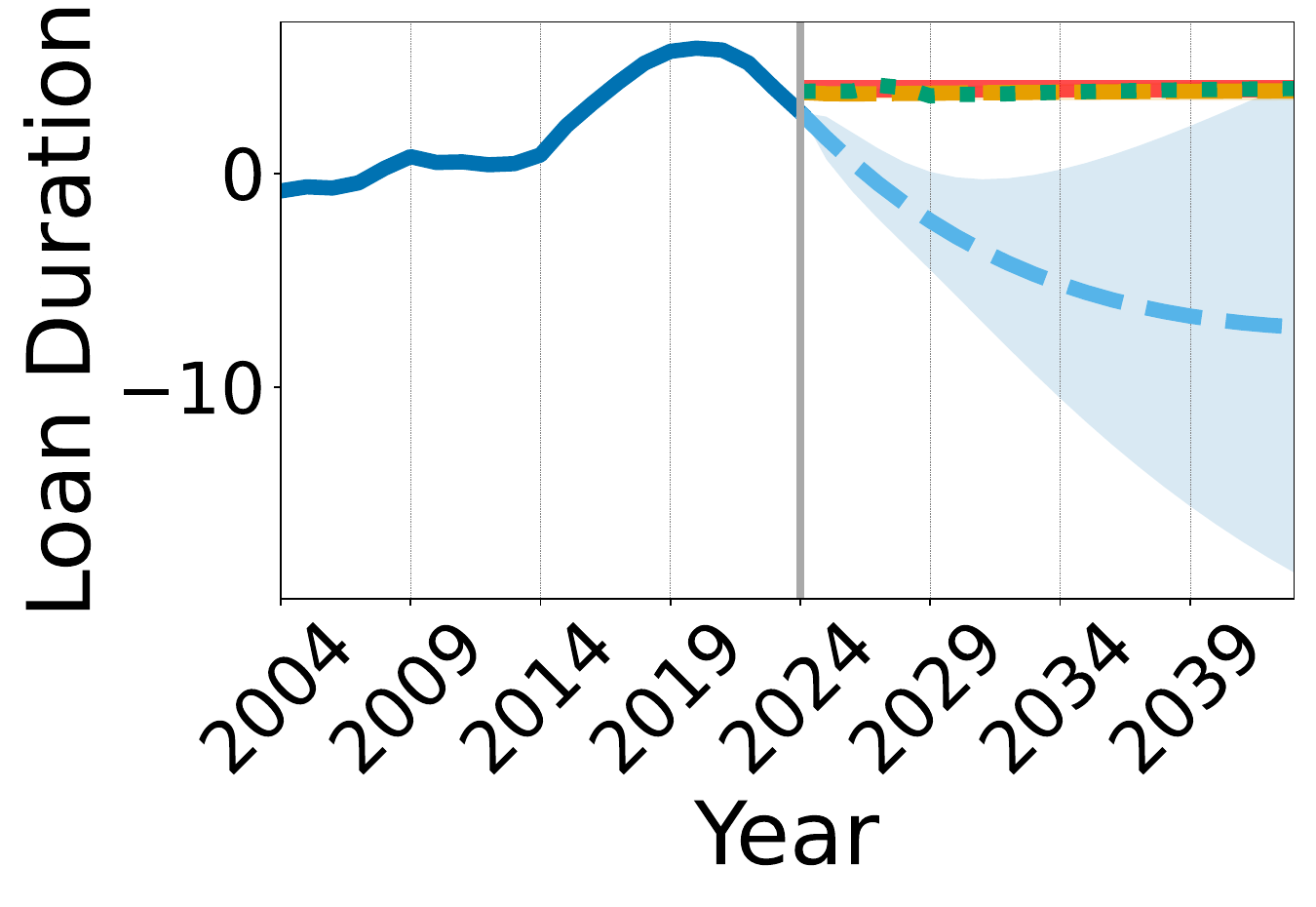}
      \caption{Intervention on $D$}
    \end{subfigure}
    %\hfill
    \begin{subfigure}[b]{0.49\textwidth}
      \centering
      \includegraphics[width=\textwidth]{figures/ciovertime/Forcing_Intervention_on_Income_with_Force_20.pdf}
      \caption{\centering Intervention on $I$}
    \end{subfigure}
    \end{minipage}
  \end{minipage}
  \begin{minipage}[t]{0.49\textwidth}
    \centering
    \vspace{0.5cm}
    \begin{minipage}{\textwidth}
    \begin{subfigure}[b]{0.49\textwidth}
      \centering
      \includegraphics[width=\textwidth]{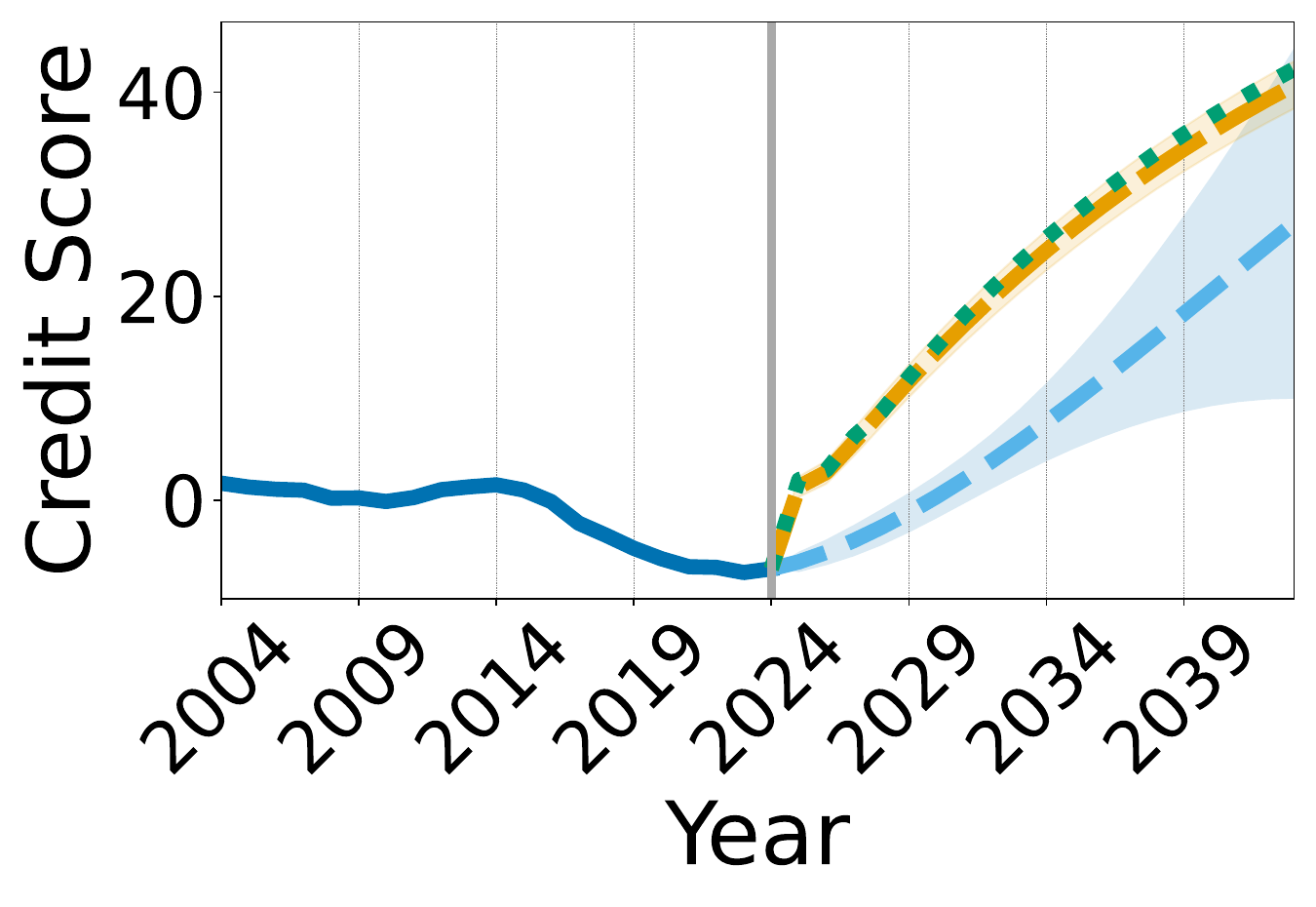}
      \caption{\centering Effect on $C$}
    \end{subfigure}
    \hfill
    \begin{subfigure}[b]{0.49\textwidth}
      \centering
      \begin{tikzpicture}[
            scale=0.6,
            transform shape,
            vertex/.style={circle, draw, minimum size=1cm, inner sep=0pt, font=\footnotesize}
        ]
        \node[vertex] (r) at (0,1.5) {$R$};
        \node[vertex] (e) at (0,0) {$E$};
        \node[vertex] (l) at (1.7,0) {$L$};
        \node[vertex, fill=myorange!85!white] (d) at (3.4,0) {$D$};
        \node[vertex, fill=myorange!85!white] (i) at (1.7,1.5) {$I$};
        \node[vertex, fill=myorange!30!white] (s) at (1.7,3) {$S$};
        \node[vertex, fill=myorange!30!white] (y) at (3.4,1.5) {$C$};
        
        \draw[edge] (e) -- (r);
        \draw[edge] (l) -- (d);
        \draw[edge] (r) -- (i);
        \draw[edge] (e) -- (i);
        \draw[edge] (i) -- (s);
        \draw[edge] (l) -- (y);
        \draw[edge] (d) -- (y);
        \draw[edge] (i) -- (y);
        \draw[edge] (s) -- (y);
        \end{tikzpicture}
        \caption{German causal graph}
    \end{subfigure}
     \end{minipage}
  \end{minipage} 
  \caption{\textbf{Forcing Interventions.}  
  (a) Intervention on \textit{Income} with $\bm{F} = 5$ and target value $\hat{I} = 20$.  The resulting effect on (b) \textit{Savings} and (c) \textit{Credit Score}. (e) Intervention on \textit{Loan Duration} with $\bm{F} = 3$ and target value $\hat{D} = 4$. (b) Intervention on \textit{Income} with $\bm{F} = 3$ and target value $\hat{I} = 20$. (g) Effect of interventions in (e-f) on \textit{Credit Score}.
  In (d) and (h), the deep orange node represents the intervened variable, while the lighter orange nodes are those influenced by the intervention. In all plots, shaded regions denote $95\%$ confidence bounds.
  }
\label{fig:forcing_int_appendix}
\end{figure*}

\paragraph{How do additive and forcing interventions affect the system dynamics?}
Figure~\ref{fig:additive_int_appendix} illustrates two additive interventions applied to German. Panel (a) depicts an intervention on \textit{Responsibility} with $\bm{F} = -0.3$, while panel (e) shows an intervention on \textit{Loan Amount} with $\bm{F} = 0.5$.
The causal graphs in panels (d) and (h) show how these interventions propagate through the network. Panels (b) and (c) display the effects of the \textit{Responsibility} intervention on \textit{Loan Duration} and \textit{Income}, respectively. Panels (f) and (g) illustrate the corresponding effects for the \textit{Loan Amount} intervention.
Specifically, in panels (a-c), we observe that the intervention on \textit{Responsibility} affects \textit{Income} while having no impact on \textit{Loan Duration}. 
This behavior aligns with the causal relationships depicted in panel (d) as \textit{Responsibility} is not an ancestor of \textit{Loan Duration}. Similarly, panels (e-g) illustrate that the intervention on \textit{Loan Amount} affects only \textit{Loan Duration}, as \textit{Loan Amount} is not an ancestor of \textit{Income}.

Figure~\ref{fig:forcing_int_appendix} shows two forcing interventions applied to German.
Panel (a) displays an intervention on \textit{Income} with $\bm{F} = 5$ and target value $\hat{I} = 20$. Panels (b) and (c) show how this intervention affects \textit{Savings} and \textit{Credit Score}, respectively.
Panel (e) presents an intervention on \textit{Loan Duration} with $\bm{F} = 3$ and target value $\hat{D} = 4$, while panel (f) shows a different intervention on \textit{Income} with $\bm{F} = 3$ and target value $\hat{I} = 20$. %Panel (g) shows the combined effect of these two interventions on \textit{Credit Score}.
Panels (a-c) present results that are consistent with those demonstrated in the manuscript, namely that interventional forecasting precisely aligns with the specified target value.
Panels (e-f) demonstrate the capability of the causal \ac{VAR} framework to implement multiple interventions within the system simultaneously. Panel (g) displays the combined impact of these interventions on \textit{Credit Score}. 
By visualizing this cumulative effect, we could gain insights into the complex interplay between multiple variables and their joint influence on the target variable.

\begin{table*}[th!]
    \centering
    \caption{\textbf{Interventional Forecasting.} RMSE and SMAPE scores (\emph{lower is better}) for the proposed causal \ac{VAR} framework on German and Inverted Pendulum datasets. RMSE scores are scaled by a factor of $10^2$ to ease readability.
    Results averaged over ten runs, with standard deviation in subscript.}
    \label{tab:interventional_forecasting_combined}
    \setlength{\tabcolsep}{6pt}
    \footnotesize
    \renewcommand{\arraystretch}{1.18}
    \begin{tabular}{lccccccc}
    \toprule
    & & & \multicolumn{2}{c}{RMSE (scaled)} & \multicolumn{2}{c}{SMAPE} \\
    \cmidrule(lr){4-5} \cmidrule(lr){6-7}
    Dataset & Train Steps & Horizon & Additive & Forcing & Additive & Forcing \\
    \midrule
    German & \multirow{4}{*}{\centering100} & \multirow{2}{*}{\centering1} & $.000_{.000}$ & $.000_{.000}$ & $.000_{.000}$ & $.000_{.000}$ \\
    Pendulum & & & $.042_{.031}$ & $.062_{.046}$ & $.146_{.107}$ & $.319_{.236}$ \\
    \cmidrule(lr){3-7}
    German & & \multirow{2}{*}{\centering10} & $.446_{.295}$ & $1.53_{1.27}$ & $.829_{.547}$ & $5.80_{4.57}$ \\
    Pendulum & & & $6.66_{4.47}$ & $.102_{.051}$ & $2.19_{1.57}$ & $.262_{.140}$ \\
    \midrule
    German & \multirow{4}{*}{\centering500} & \multirow{2}{*}{\centering1} & $.000_{.000}$ & $.000_{.000}$ & $.000_{.000}$ & $.000_{.000}$ \\
    Pendulum & & & $.024_{.023}$ & $.035_{.034}$ & $.083_{.080}$ & $.180_{.173}$ \\
    \cmidrule(lr){3-7}
    German & & \multirow{2}{*}{\centering10} & $.173_{.139}$ & $.464_{.338}$ & $.353_{.276}$ & $1.89_{1.33}$ \\
    Pendulum & & & $2.02_{1.52}$ & $.049_{.035}$ & $.690_{.451}$ & $.111_{.077}$ \\
    \midrule
    German & \multirow{4}{*}{\centering1000} & \multirow{2}{*}{\centering1} & $.000_{.000}$ & $.000_{.000}$ & $.000_{.000}$ & $.000_{.000}$ \\
    Pendulum & & & $.010_{.007}$ & $.014_{.011}$ & $.034_{.026}$ & $.074_{.056}$ \\
    \cmidrule(lr){3-7}
    German & & \multirow{2}{*}{\centering10} & $.101_{.079}$ & $.299_{.220}$ & $.210_{.163}$ & $1.24_{.892}$ \\
    Pendulum & & & $1.11_{.853}$ & $.023_{.012}$ & $.376_{.257}$ & $.054_{.031}$ \\
    \bottomrule
    \end{tabular}
\end{table*}

\begin{table}[th!]
    \centering
    \caption{\textbf{Interventional Forecasting.} MAE scores (\emph{lower is better}) for the proposed causal \ac{VAR} on the Pendulum dataset. Results averaged over $10$ runs, with standard deviation in subscript. Scores are scaled by a factor of $10^3$ to ease readability.}\label{tab:interventional_forecasting_pendulum}
    \setlength{\tabcolsep}{5pt}
    \footnotesize
    \renewcommand{\arraystretch}{1.2}
    \begin{tabular}{ccccc}
    \toprule
    & & & \multicolumn{2}{c}{MAE} \\
    \cmidrule(lr){4-5}
    Dataset & Train Steps & Horizon & Additive & Forcing \\
    \midrule
    \multirow{4}{*}{Pendulum} &  \multirow{2}{*}{100} & 1 & $.094_{.069}$ & $.310_{.230}$ \\
& & 10 & $6.32_{4.44}$ & $.264_{.142}$ \\
\cmidrule(lr){2-5}
 & \multirow{2}{*}{500} & 1 & $.053_{.052}$ & $.174_{.169}$ \\
& & 10 & $1.95_{1.39}$ & $.112_{.078}$ \\
%\cmidrule(lr){2-5}
% & \multirow{2}{*}{1000} & 1 & $.000_{.000}$ & $.000_{.000}$ \\
%& & 10 & $.011_{.008}$ & $.075_{.056}$ \\
\bottomrule
    \end{tabular}
    \end{table}

\paragraph{Interventions} 
We provide the intervention type and the force values used to evaluate the interventional forecasting described in the following paragraph. 
For German, we perform causal interventions on root node \textit{Expertise}, %~one 
and observe the resulting effect on the target variable \textit{Credit Score}.
For the additive case, we apply $\bm{F}=0.2$, while for the forcing case, we use $\bm{F}=1$ with a target value of $\hat{E} = 5$. 
For Pendulum, we perform causal interventions on $\Xcomp[2]$, and observe the resulting effect on $\Xcomp[1]$.
For the additive case, we apply $\bm{F}=0.4$, while for forcing, we use $\bm{F}=1$ with a target value of $\hat{\Xcomp[2]}= 1$.

\paragraph{How accurate is the causal VAR framework in estimating the causal effect of interventions over time?}
Table~\ref{tab:interventional_forecasting_combined} provides RMSE and SMAPE scores for interventional forecasting across different data sizes and forecast horizons on the German and Pendulum datasets. These scores further validate the effectiveness of the proposed causal \ac{VAR} framework.
For $1$-step forecasts, the Pendulum dataset shows the intervention effect immediately, in contrast to the German dataset, where variables change slowly and require several time steps to exhibit measurable effects. 
For $10$-step forecasts, the scores demonstrate that the model achieves high accuracy. Notably, the RMSE values remain low across both datasets and intervention types, indicating the robustness of the model in estimating causal effects over longer horizons. Additionally, while the SMAPE scores show some variation, the results suggest that the model is well-suited for handling more complex forecasting scenarios. 

Table~\ref{tab:interventional_forecasting_pendulum} outlines the results of interventional forecasting, detailing errors across different data sizes and forecast horizons for the Pendulum dataset. 
Our causal effect estimates remain highly accurate for both forecast horizons. In contrast to the German dataset in Table~\ref{tab:interventional_forecasting_german}, the intervention effect is observable even at the $1$-step forecast.
These results reinforce the findings presented in the main paper, highlighting the effectiveness of the proposed causal \ac{VAR} framework in interventional forecasting tasks.

\begin{figure*}[th!]
  \centering
  \includegraphics[width=0.98\textwidth]{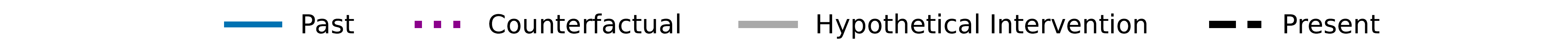}
  \begin{minipage}[t]{0.49\textwidth}
    \centering
    \begin{minipage}{\textwidth}
    \vspace{0.5cm} 
    \begin{subfigure}[b]{0.49\textwidth}
      \centering
      \includegraphics[width=\textwidth]{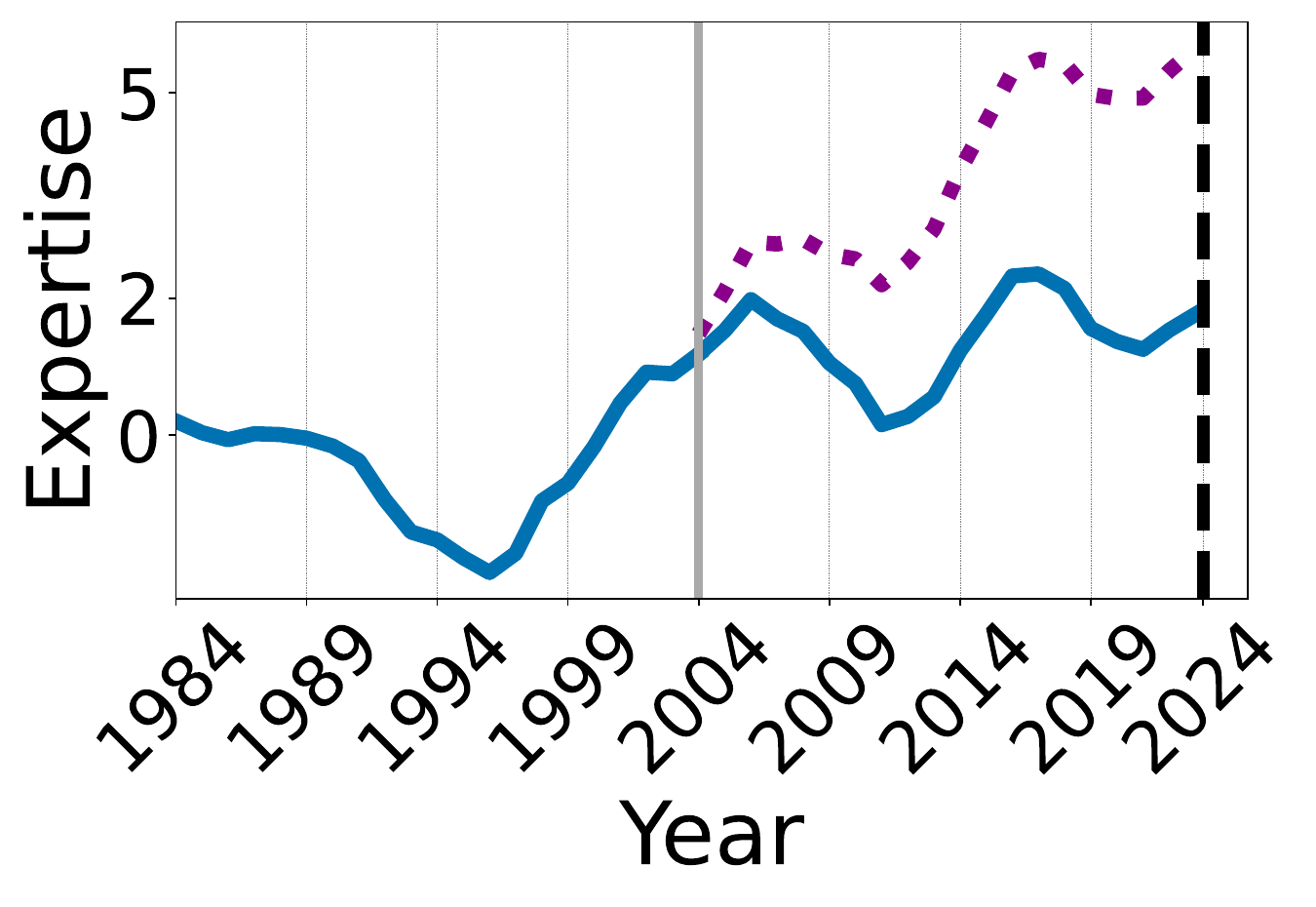}
      \caption{Intervention on $E$}
    \end{subfigure}
    %\hfill
    \begin{subfigure}[b]{0.49\textwidth}
      \centering
      \includegraphics[width=\textwidth]{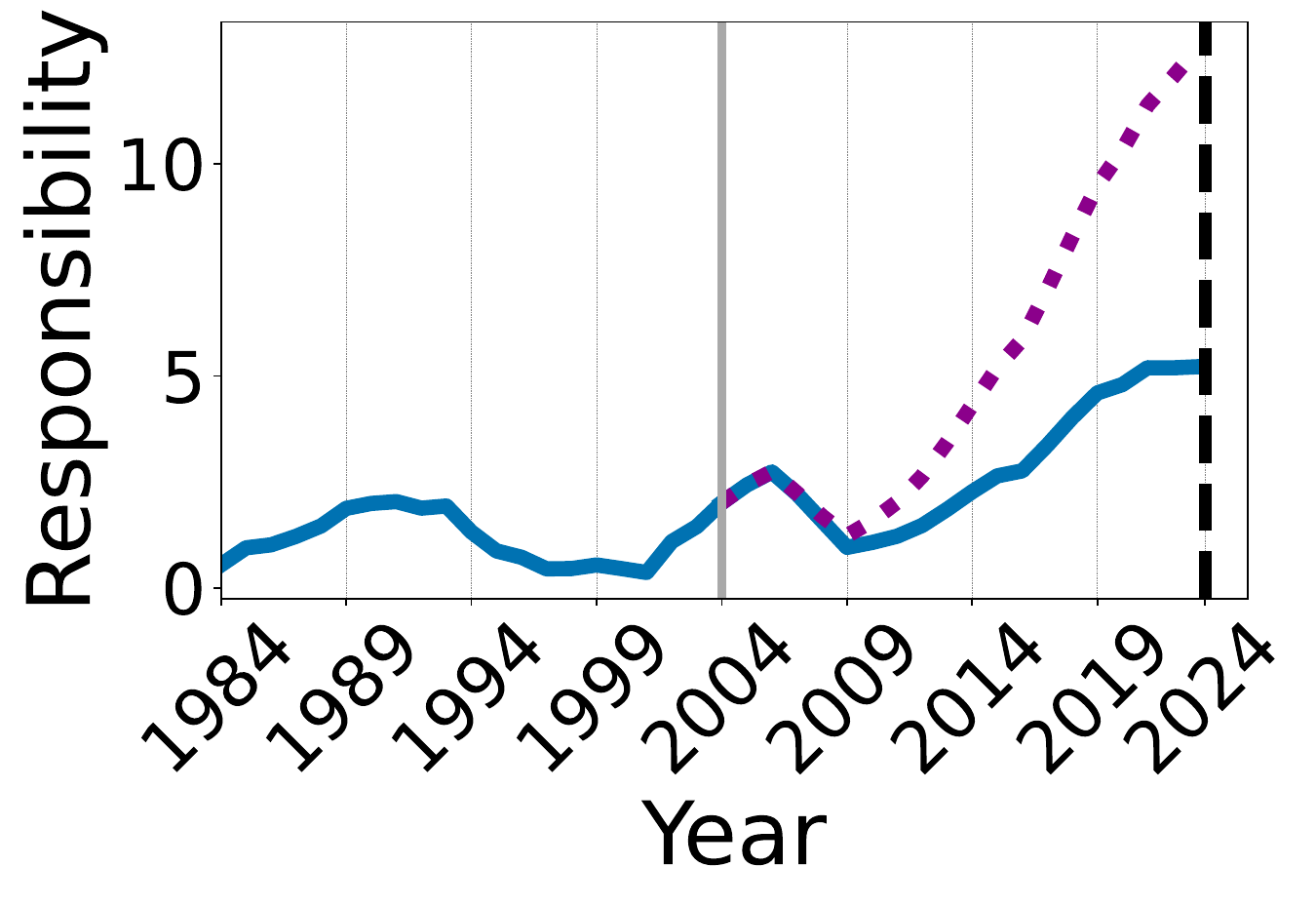}
      \caption{\centering Effect on $R$}
    \end{subfigure}
    \end{minipage}
  \end{minipage}
  \begin{minipage}[t]{0.49\textwidth}
    \centering
    \begin{minipage}{\textwidth}
    \vspace{0.5cm} 
    \begin{subfigure}[b]{0.49\textwidth}
      \centering
      \includegraphics[width=\textwidth]{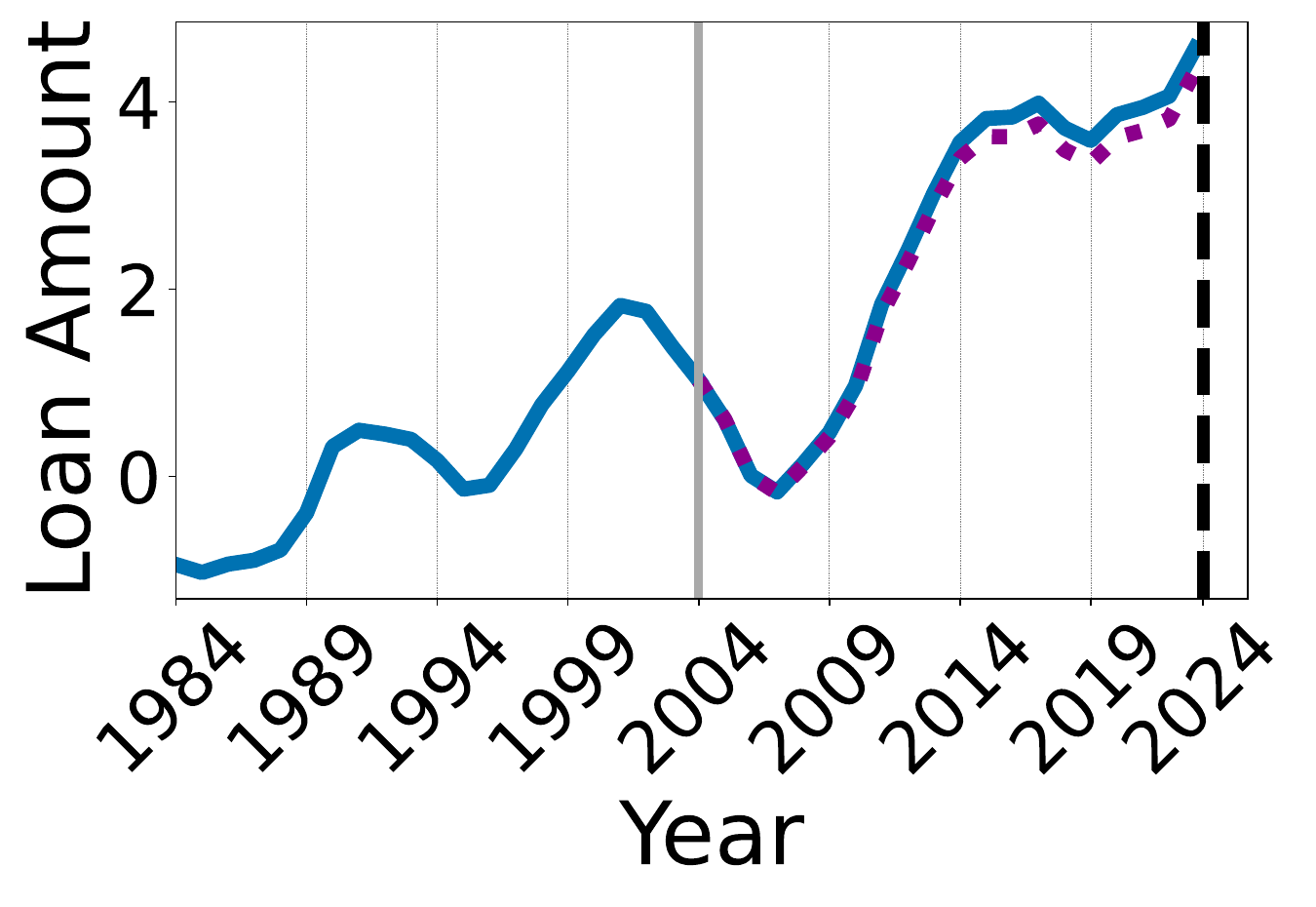}
      \caption{\centering Effect on $L$}
    \end{subfigure}
    \hfill
    \begin{subfigure}[b]{0.49\textwidth}
      \centering
      %\vspace{-0.5cm}
      \begin{tikzpicture}[
            scale=0.6,
            transform shape,
            vertex/.style={circle, draw, minimum size=1cm, inner sep=0pt, font=\footnotesize}
        ]
        \node[vertex, fill=mypurple!15!white] (r) at (0,1.5) {$R$};
        \node[vertex, fill=mypurple!45!white] (e) at (0,0) {$E$};
        \node[vertex] (l) at (1.7,0) {$L$};
        \node[vertex] (d) at (3.4,0) {$D$};
        \node[vertex, fill=mypurple!15!white] (i) at (1.7,1.5) {$I$};
        \node[vertex, fill=mypurple!15!white] (s) at (1.7,3) {$S$};
        \node[vertex, fill=mypurple!15!white] (y) at (3.4,1.5) {$C$};
        
        \draw[edge] (e) -- (r);
        \draw[edge] (l) -- (d);
        \draw[edge] (r) -- (i);
        \draw[edge] (e) -- (i);
        \draw[edge] (i) -- (s);
        \draw[edge] (l) -- (y);
        \draw[edge] (d) -- (y);
        \draw[edge] (i) -- (y);
        \draw[edge] (s) -- (y);
        \end{tikzpicture}
        \caption{German causal graph}
    \end{subfigure}
     \end{minipage}
  \end{minipage} 
  \centering

  \begin{minipage}[t]{0.49\textwidth}
    \centering
    \begin{minipage}{\textwidth}
    \vspace{0.5cm} 
    \begin{subfigure}[b]{0.49\textwidth}
      \centering
      \includegraphics[width=\textwidth]{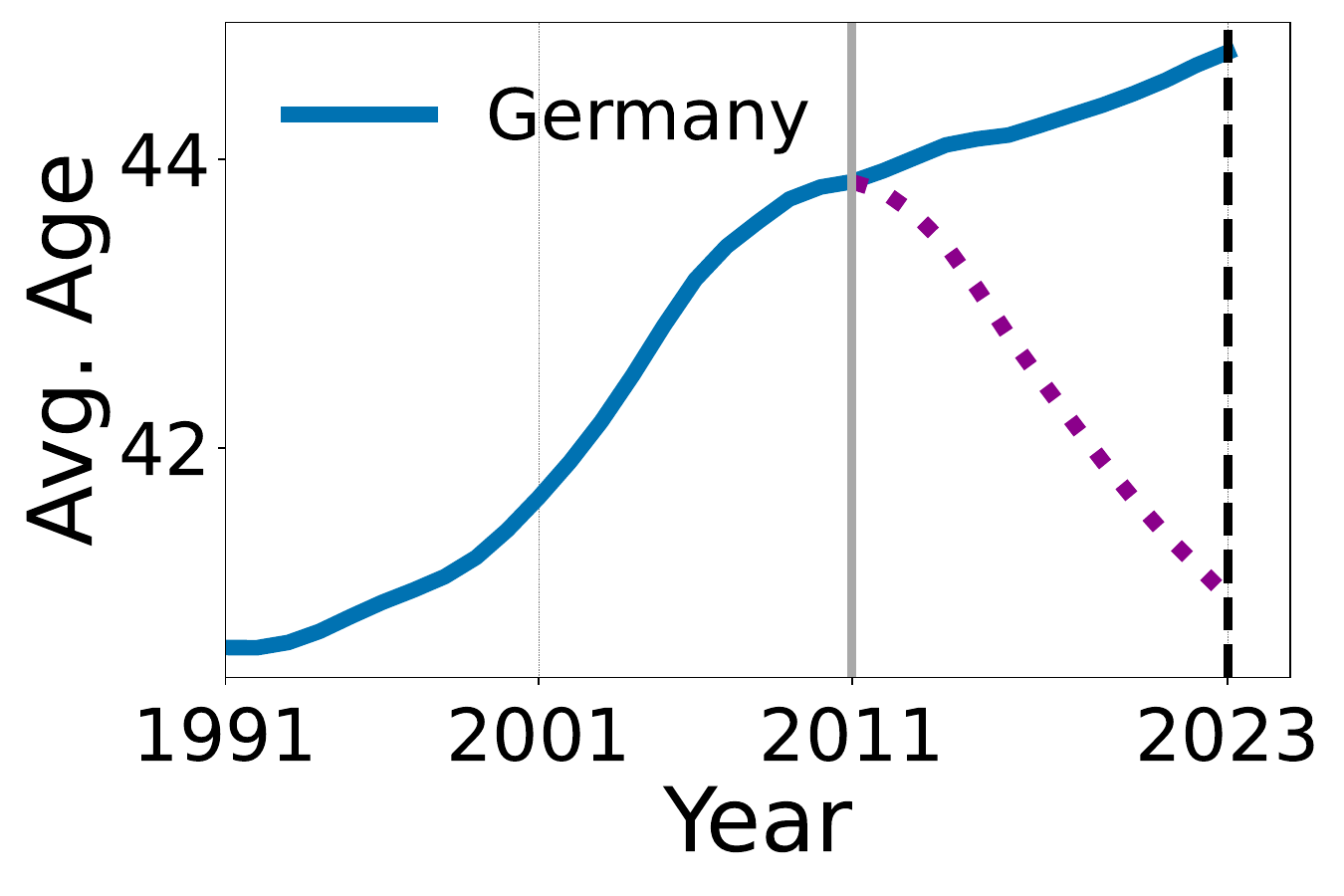}
      \caption{\centering Germany}
    \end{subfigure}
    %\hfill
    \begin{subfigure}[b]{0.49\textwidth}
      \centering
      \includegraphics[width=\textwidth]{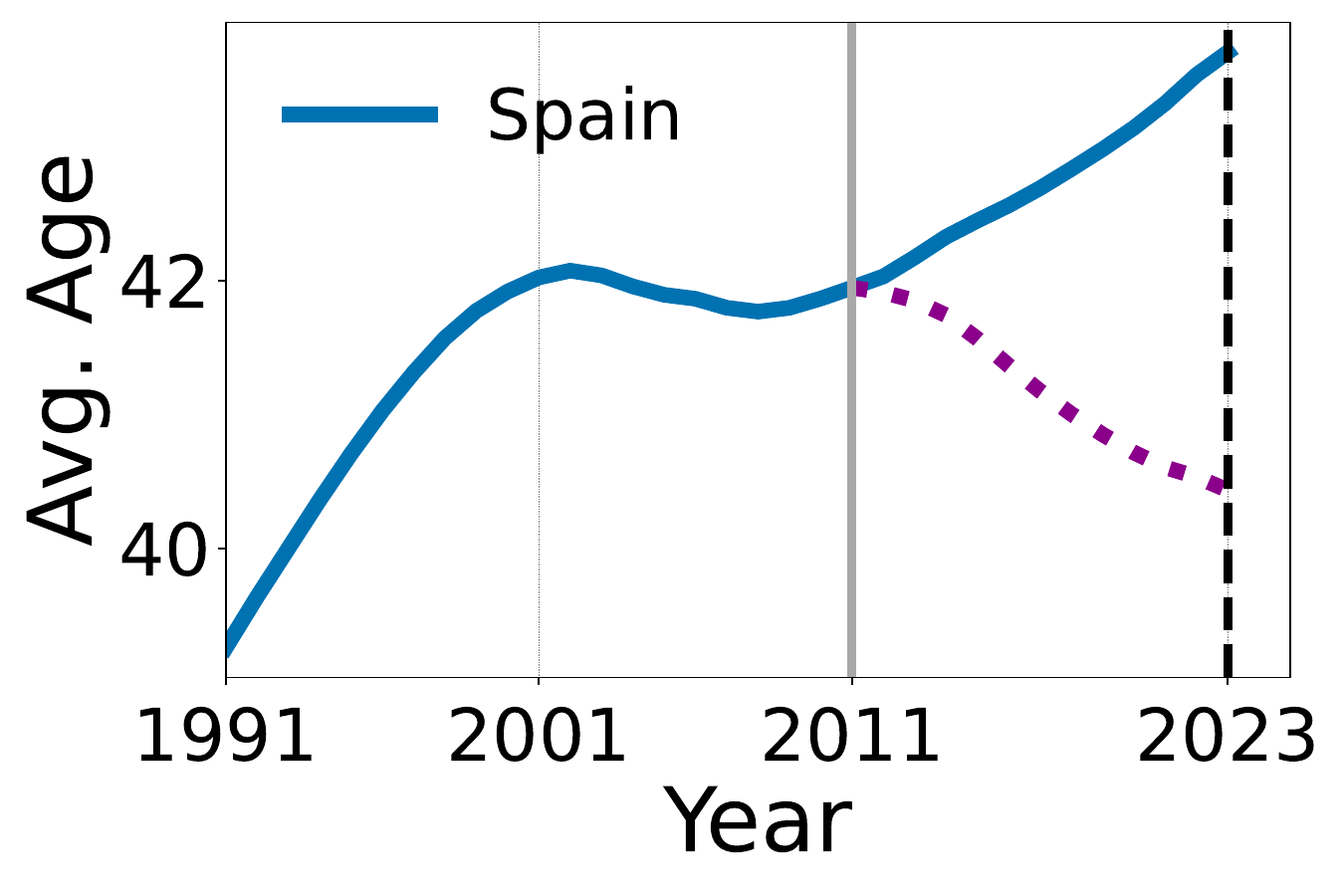}
      \caption{\centering Spain}
    \end{subfigure}
    \end{minipage}
  \end{minipage}
  \centering
  \begin{minipage}[t]{0.49\textwidth}
    \centering
    \begin{minipage}{\textwidth}
    \vspace{0.5cm} 
    \begin{subfigure}[b]{0.49\textwidth}
      \centering
      \includegraphics[width=\textwidth]{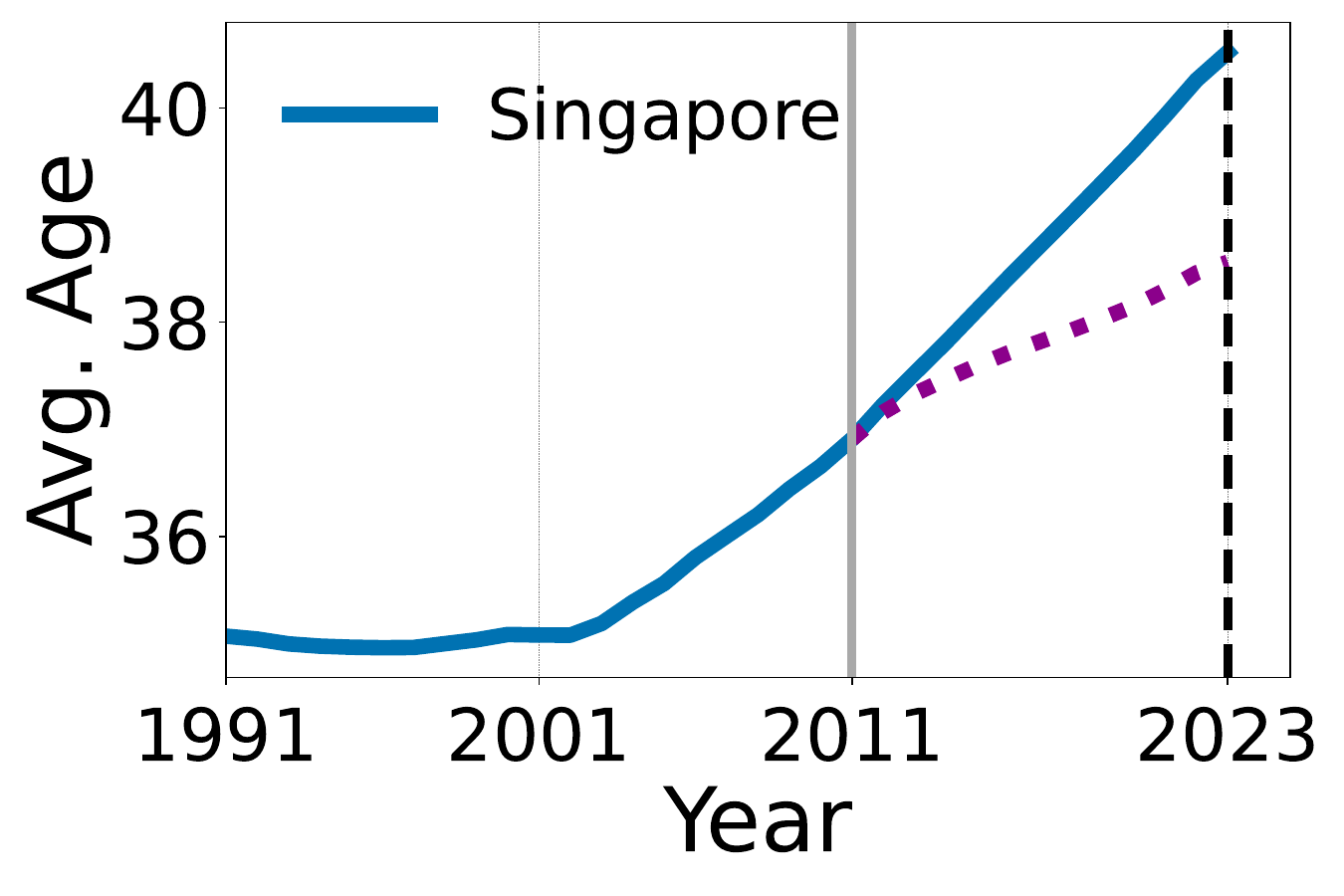}
      \caption{\centering Singapore}
    \end{subfigure}
    \hfill
    \begin{subfigure}[b]{0.49\textwidth}
      \centering
      %\vspace{-0.5cm}
      \begin{tikzpicture}[
            scale=0.6,
            transform shape,
            vertex/.style={circle, draw, minimum size=1.1cm, inner sep=0pt, font=\small}
        ]
        \node[vertex, fill=mypurple!45!white] (b) at (0, 1.5) {Birth};
        \node[vertex, fill=mypurple!15!white] (d) at (3.4, 3) {Death};
        \node[vertex] (m) at (3.4,0) {Migra.};
        \node[vertex, fill=mypurple!15!white] (a-0) at (1.7,1.5) {0-14};
        \node[vertex, fill=mypurple!15!white] (a-1) at (3.4,1.5) {15-64};
        \node[vertex, fill=mypurple!15!white] (a-2) at (5.1,1.5) {65-99};
        
        \draw[edge] (b) -- (a-0);
        \draw[edge] (a-1) to [out=120,in=60](b); 
        \draw[edge] (d) -- (a-2);
        \draw[edge] (d) -- (a-1);
        \draw[edge] (d) -- (a-0);
        \draw[edge] (m) -- (a-0);
        \draw[edge] (m) -- (a-1);
        \draw[edge] (m) -- (a-2);
        \draw[edge] (a-0) -- (a-1);
        \draw[edge] (a-1) -- (a-2);
        \draw[edge] (m) to [out=190,in=-60](b);
        \draw[edge] (a-2) to [out=90, in=0](d);
        \end{tikzpicture}
        \caption{Census causal graph}
    \end{subfigure}
     \end{minipage}
  \end{minipage} 
  \caption{\textbf{Retrospective Counterfactuals.}
  Panels (a-d) focus on the German dataset.  
  (a) Intervention on \textit{Expertise} with $\bm{F} = 0.3$. 
  Effect of the intervention on (b) \textit{Responsibility} and (c) \textit{Loan Amount}. 
  Panels (e-h) examine the Census dataset with an intervention on \textit{Births} with $\bm{F} = 0.004$. Effect of the intervention on the population's average age of (e) Germany, (f) Spain, and (g) Singapore. 
  In (d) and (h), the deep purple node represents the intervened variable, while the lighter purple nodes are those influenced by the intervention.}
  \label{fig:counterfactuals}
\end{figure*}
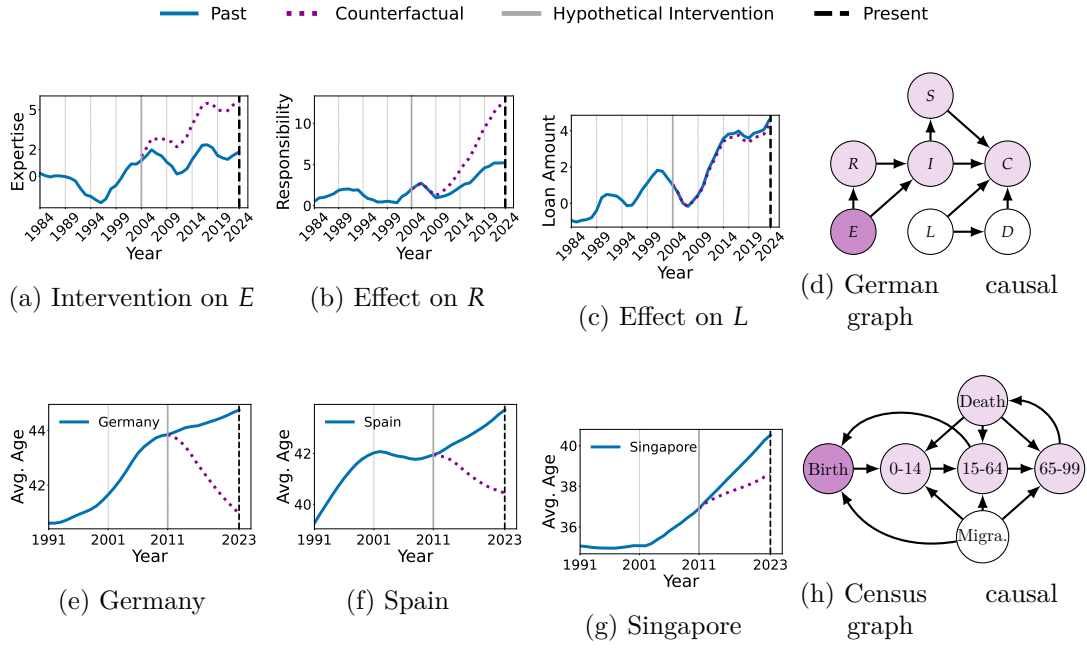

\subsection*{Retrospective Counterfactuals}\label{app:counterfactuals}
In this section, we demonstrate that our causal \ac{VAR} framework can also address the causal queries discussed Section~\ref{sec:practical_implications} in terms of retrospective counterfactuals.

Figure~\ref{fig:counterfactuals} illustrates two counterfactual interventions, with panels (a-d) related to German and panels (e-h) focused on Census.
For German, panel (a) shows an intervention in 2004 on \textit{Expertise} with $\bm{F} = 0.3$ and the effect of this intervention on (b) \textit{Responsibility} and (c) \textit{Loan Amount}. 
This scenario could potentially answer the question: ``What would be the outcome in 2024 if an applicant had improved their expertise in 2004?''. 
Panel (b) shows that an increase in \textit{Expertise} would have led to a corresponding rise in \textit{Responsibility}. Notably, panel (c) reveals the \textit{Loan Amount} remains unchanged, reflecting the structure of our causal graph where \textit{Expertise} is not an ancestor of this node.
For the Census dataset, we examine an intervention in 2011 on \textit{Births} with $\bm{F} = 0.004$ and its effect on the population's average age (computed as a
weighted mean of age groups) of three countries: (e) Germany, (f) Spain, and (g) Singapore. 
The force value indicates that \textit{Births} increase by $0.4\%$ of each country's total population annually.
This analysis could potentially explore the question: ``How would a policy promoting increased births in 2011 have affected a country’s current average population age?''. 
Panels (e-f) display that such an intervention would have significantly lowered the average age in Germany and Spain, given their relatively older population structures. 
In contrast, panel (g) shows that the effect on Singapore would have been more moderate, consistent with its younger demographic profile at the intervention time.

\section{Use cases} \label{sec:use-cases}

\begin{figure*}[t!]
  \centering
  % Prima riga di figure: German
  \includegraphics[width=\textwidth]{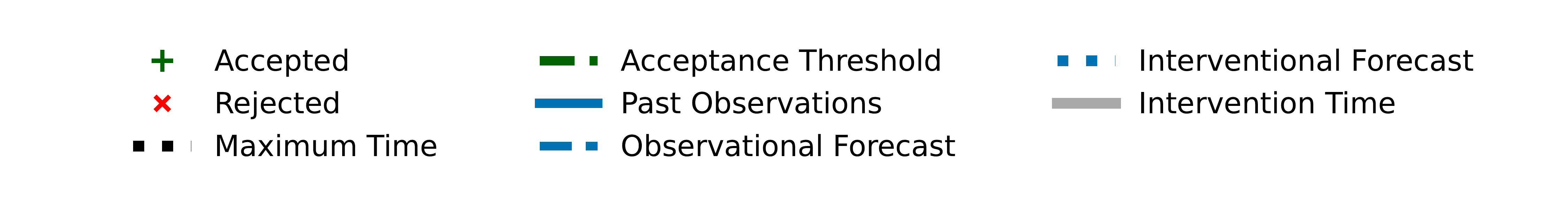}
  \begin{subfigure}[t]{0.49\textwidth}
    \centering
    \includegraphics[width=.9\textwidth]{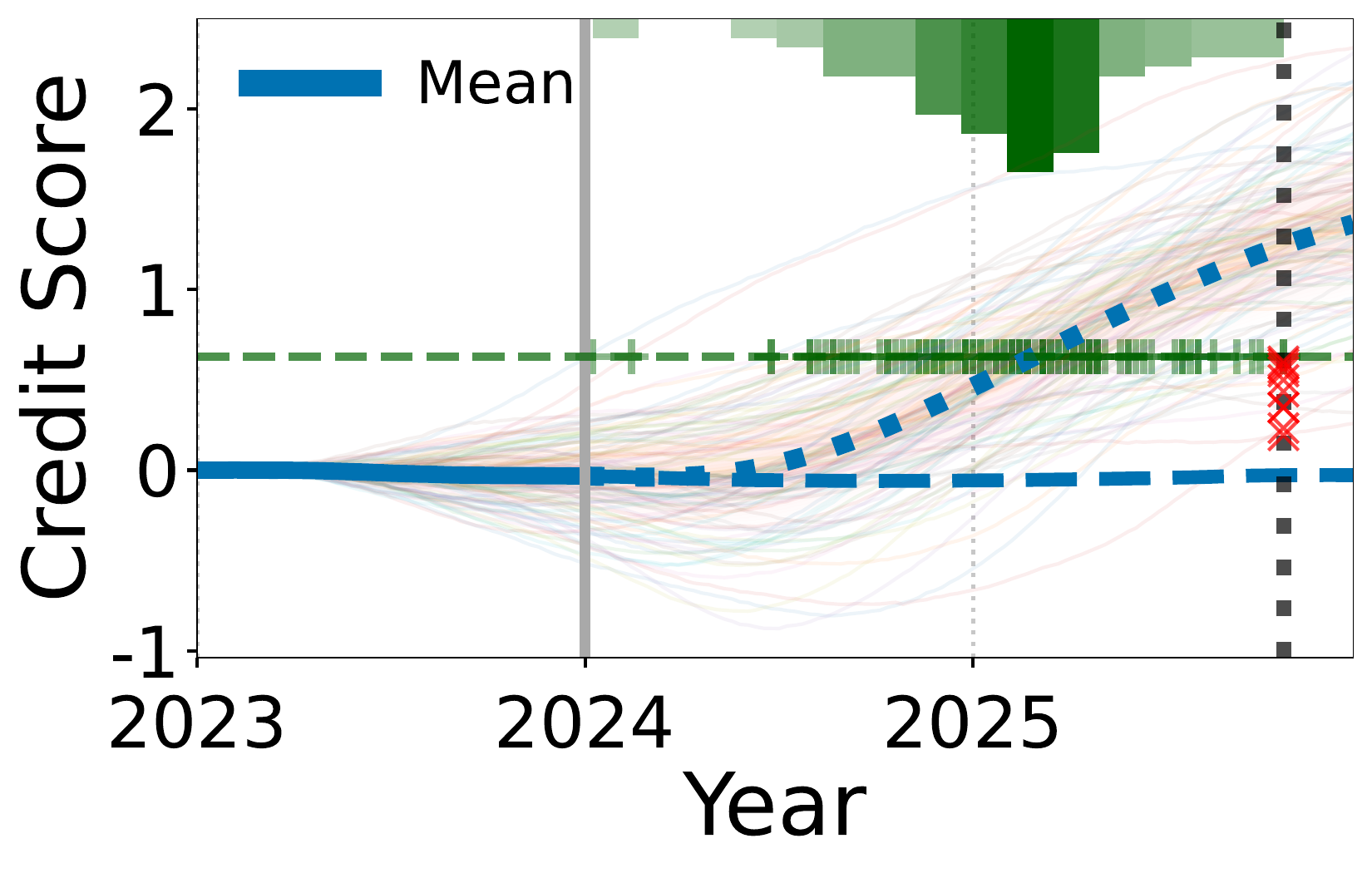}
  \end{subfigure}
  \hfill
  \begin{subfigure}[t]{0.49\textwidth}
    \centering
    \includegraphics[width=.9\textwidth]{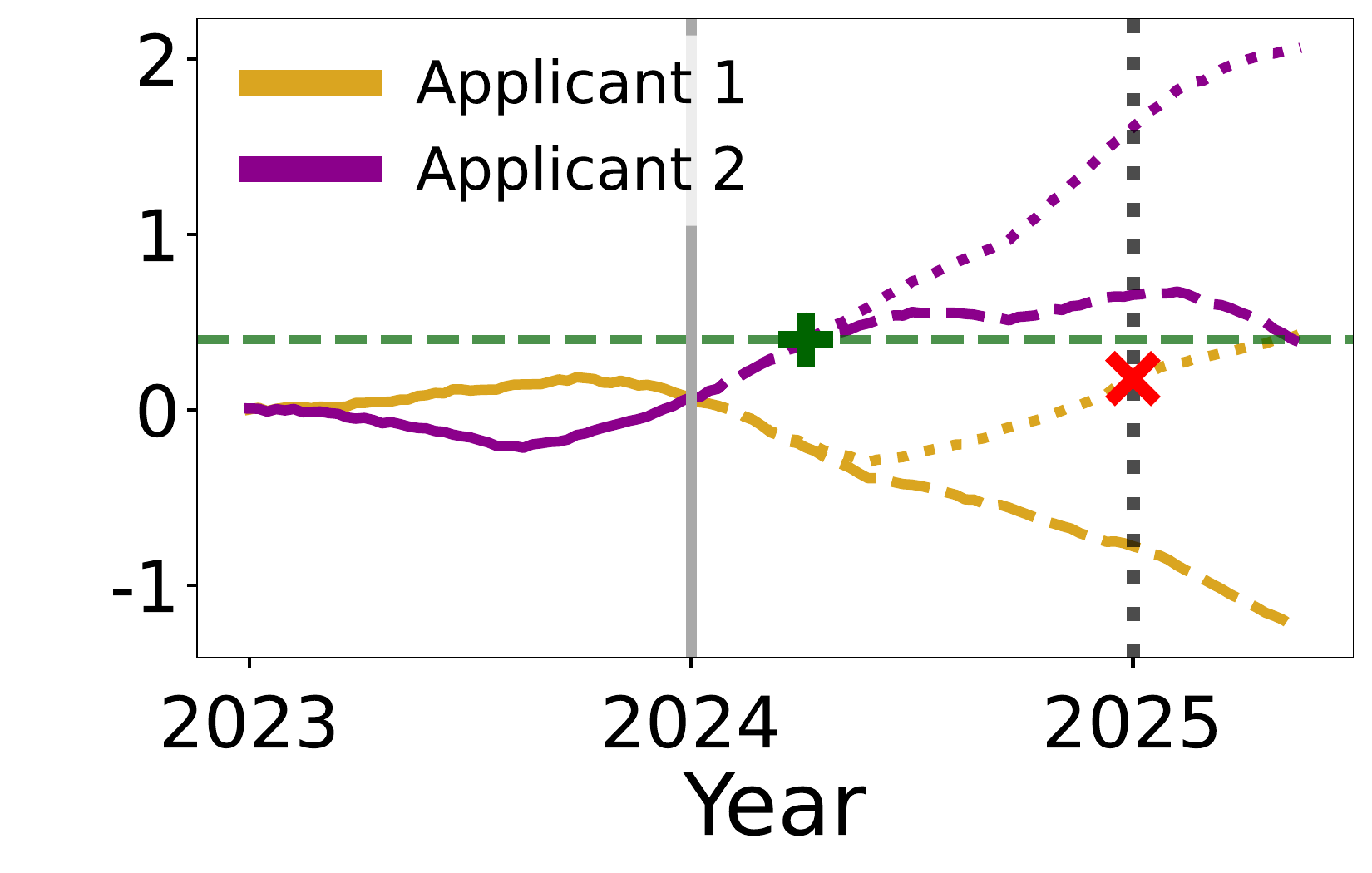}
  \end{subfigure}
  \caption{\textbf{German}. Effect of increasing \textit{Expertise} on \textit{Credit Score}. (Left) The time each loan applicant takes to cross or not the acceptance threshold. The histogram shows the distribution of crossing times. (Right) Comparison of two loan applicants, i.e., trajectories, with similar scores at intervention time. Forecasts are dashed for observational and dotted for interventional.}
  \label{fig:use_case_german}

  %\vspace{0.3cm} % Spazio tra le righe di figure

  % Seconda riga di figure: Census
  \includegraphics[width=\textwidth]{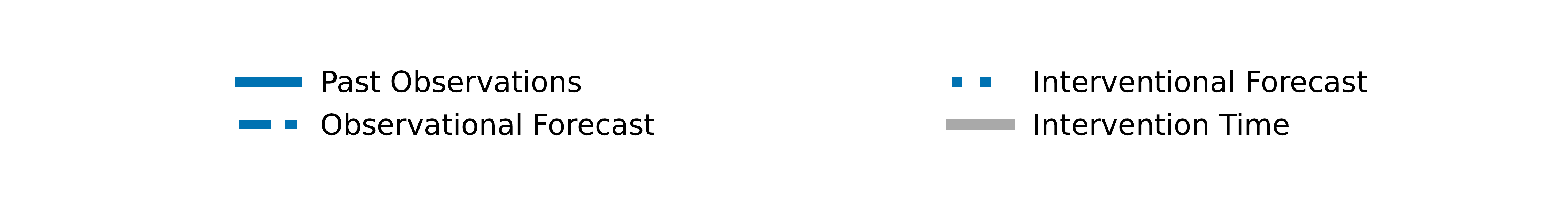}
  \begin{subfigure}[t]{0.49\textwidth}
    \centering
    \vspace{-0.5cm}
    \includegraphics[width=.9\textwidth]{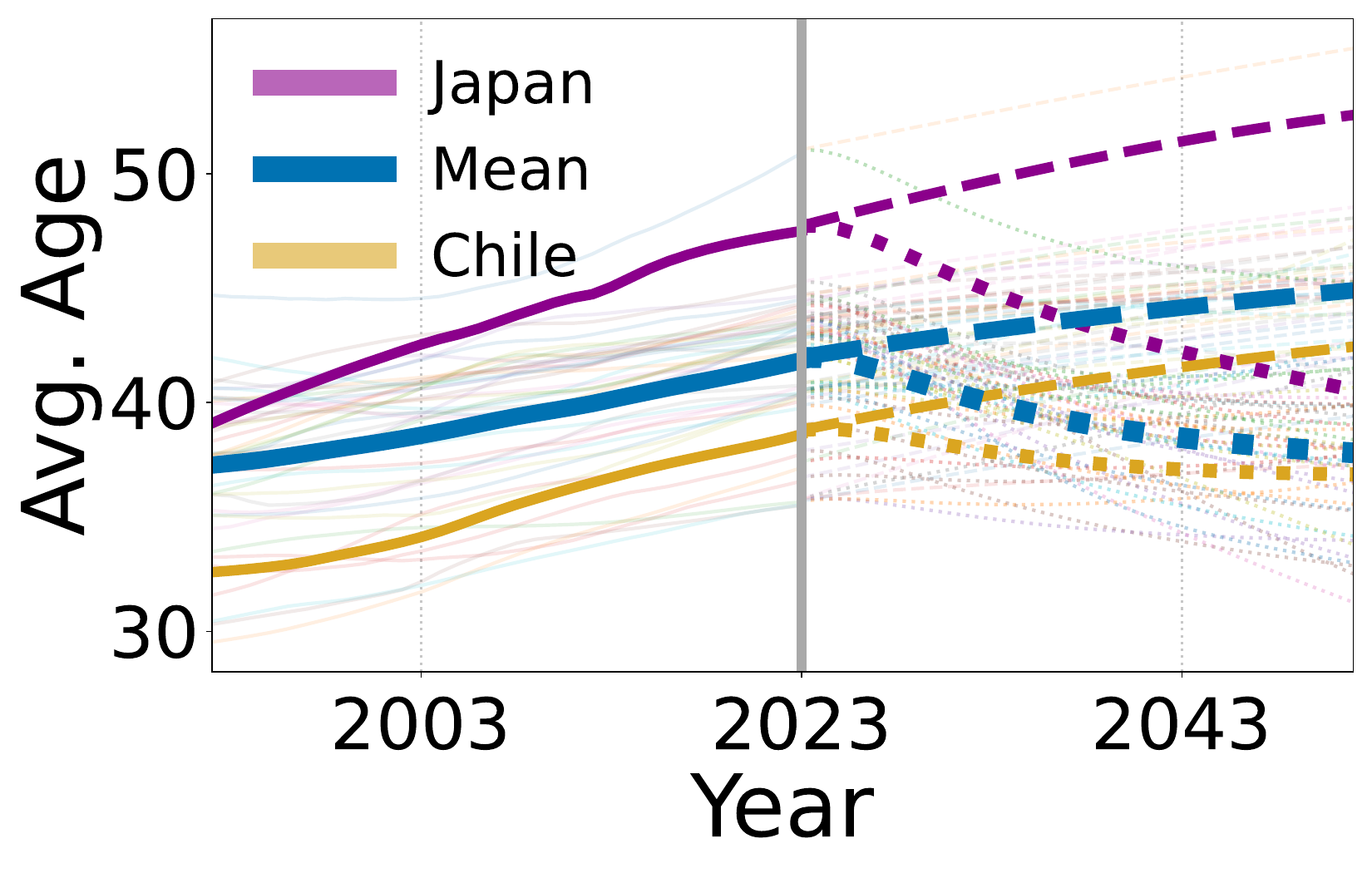}
  \end{subfigure}
  \hfill
  \begin{subfigure}[t]{0.49\textwidth}
    \centering
     \vspace{-0.5cm}
    \includegraphics[width=.9\textwidth]{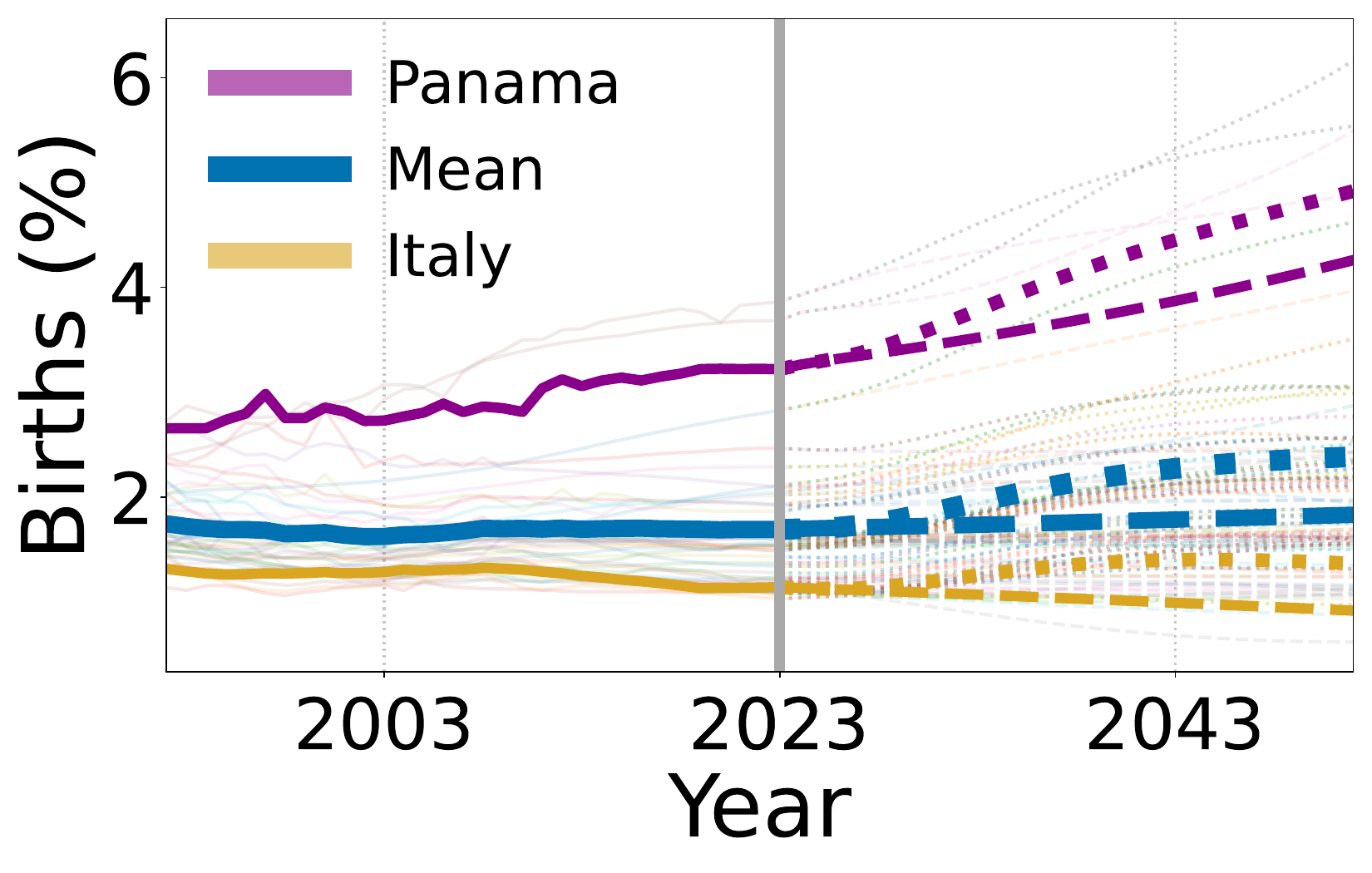}
  \end{subfigure}
  \caption{\textbf{Census}. Additive intervention across all countries. (Left) Intervention on \textit{Births} with $\bm{F}=0.004$ and its effect on the population's average age. (Right) Intervention on \textit{Migration} with $\bm{F}=0.04$ and its effect on \textit{Births}.}
  \label{fig:use_case_census}
\end{figure*}
In this section, we show real-world scenarios where estimating causal effects over time represents a useful step toward a 
realistic modeling of the phenomena.
We focus on the German and Census datasets, presenting two analyses for each. For German, Figure~\ref{fig:use_case_german} illustrates the effect of increasing \textit{Expertise} by $\bm{F} = 0.38$ on \textit{Credit Score}.

\paragraph{German 1 -- Same intervention, different outcomes}
The left panel of Figure~\ref{fig:use_case_german} shows the distribution of the time required for loan applicants to cross or not the acceptance threshold after the intervention. 
Access to this information allows quantification of intervention efficacy, identification of credit-building patterns, and infer the key factors influencing loan eligibility outcomes.
It can also inform the recommendation of actions (\eg{} in algorithmic recourse~\citep{DBLP:journals/csur/KarimiBSV23}) 
within a reasonable timeframe, fostering trust in the system and promoting user acceptance.

\paragraph{German 2 -- Similar cross-sectional values, different causal effects over time}%Do you really need an intervention?} 
%We point out how the reliance on instantaneous state information is insufficient for generating effective recommendations. 
The right panel of Figure~\ref{fig:use_case_german} shows trajectories that, while seeming similar at a given time, may have significantly divergent historical and future behaviors. For instance, the purple trajectory may have autonomously crossed the threshold without intervention, whereas in the case of the yellow one, the applied intervention may be inadequate to ensure the desired outcome. 
Such divergence highlights the importance of moving beyond models that rely only on cross-sectional data, motivating the need for techniques, such as the proposed causal \ac{VAR} framework, that %more accurately
capture individual applicant behavior over time.

For Census, Figure~\ref{fig:use_case_census} presents two additive interventions across all countries. % with Case 1 {(Left)} and Case 2 {(Right)}.

\paragraph{Census 1 -- Impact of Births on Average Age} 
The left panel of Figure~\ref{fig:use_case_census} shows the increase on \textit{Births} with $\bm{F} = 0.004$ and its effect on the population’s average age (computed as a weighted mean of age groups). The force value means that \textit{Births} increase by $0.4\%$ of each country's total population every year.
Examining how they affect population age over time could allow policymakers to identify which countries might benefit most from specific types of demographic interventions and evaluate the long-term viability of systems. 
We can also observe that the intervention in Japan caused a more evident decrease in the average age than in Chile.

\paragraph{Census 2 -- Impact of Migration on Births} 
The right panel of Figure~\ref{fig:use_case_census} reports how a $4\%$ growth in \textit{Migration} w.r.t. the total population influences \textit{Birth} rates. 
We observe that increased migration leads to a rise in births. However, its impact is less evident (observational and interventional forecasting trajectories are closer) compared to the result on the average age shown in Figure~\ref{fig:use_case_census} (left panel).

\section{Conclusions} \label{sec:conclusions_ciovertime}
In this chapter, we have established a link between discrete-time dynamical systems at equilibrium and \acp{SCM}. 
Moreover, we have provided an explicit procedure for mapping \acp{VAR} to linear \acp{SCM} and demonstrated that, under specific model stability conditions, interventions on the dynamical system and the \ac{SCM} yield equivalent results.
To conduct causal inference over time, we have introduced two classes of interventions (additive and forcing) for \acp{VAR}.

\paragraph{Limitations} 
When systems exhibit strongly nonlinear dynamics, linear \acp{VAR} may prove less effective than alternative nonlinear approaches. 
Moreover, our framework requires prior knowledge of the causal graph.  In scenarios where this information is lacking, the process of causal discovery can present significant challenges.

\paragraph{Future work} In the next chapter, we will investigate the use of non-linear \acp{DSP}, as they may lead to better convergence rates and allow for %aiming to propose a 
causal interpretation of a broader family of dynamical models.
Moreover, %we  believe that 
our work opens several interesting research directions (~\cref{sec:practical_implications} for concrete examples) and applications (\eg{} causal inference over time in high-dimensional contexts such as climate science).

\chapter{Continuous Time and Non-linear Systems}\label{ch:CIcontinuous_time}

\paragraph{Author's Note}
This chapter represents a major advancement in the PhD research, extending the discrete-time framework introduced earlier to the more general setting of continuous-time and non-linear systems. It constitutes a natural evolution of~\cref{ch:Ci_over_time} and reformulates its key ideas within the theoretical framework of Itô calculus and stochastic differential equations. This is an original contribution, developed specifically during the writing of the thesis and not yet published. \cref{sec:differential_equations} provides a general overview of the mathematical theory used and does not contain original results.
\vspace{30pt}

In~\cref{ch:Ci_over_time}, we discussed the implementation of \emph{causal inference over time} using discrete-time models, represented by \acp{dE}. From a practical perspective, we built a linear framework to estimate the evolution equations of a system under study and to perform causal queries on it. In this chapter, we aim to extend the discussion to more general models by removing the two limitations of the previous approach. In particular, we focus on continuous-time models, using \acp{SDE} as the fundamental theoretical framework, and we consider equations that are not necessarily linear. 

The chapter is structured as follows.
%In \ref{sec:continuos_time_motivation}, we provide a broad perspective to motivate the extension of the framework, discussing both theoretical and methodological advantages. We will also identify connections with \cref{sec:causal-perspective} and highlight the limitations of the general framework in practical implementations. 

In \cref{sec:differential_equations}, we introduce the necessary mathematical foundation to describe the problem at hand using \acp{SDE}. 

In \cref{sec:gradient_flow}, we propose a reformulation of the evolution equations based on the theory of \ac{GF}, allowing us to generalize the concept of \emph{causal inference over time} (\cref{ch:Ci_over_time}) to the broader setting. 

In \cref{sec:causal_sdes}, we develop a causal interpretation of \acp{SDE}, introducing the formalization of interventions and their effects in continuous-time dynamical systems.

In \cref{sec:long_run_continuous}, we focus on estimating the asymptotic causal effects of interventions, analyzing the stationary behavior of the system under different intervention strategies.

In \cref{sec:connection_discrete_continuous}, we establish a formal connection between discrete-time and continuous-time frameworks, showing how results can be transferred between \acp{SdE} and \acp{SDE}.
A schematic summary of the overall chapter's contributions is provided in Table~\ref{tab:results_recap}.

%Finally, in \cref{sec:S_gradient_descent}, we will discuss how discrete-time and continuous-time models are related, showing that the former can be seen as an approximation of the latter. This approximation is particularly useful for simulating underlying continuous dynamics through algorithmic solutions.

%\section{Why Non-Linearity is a Relevant Step Forward}\label{sec:continuos_time_motivation}

\section{Differential Equations}\label{sec:differential_equations}
Consider a dynamical system described by the (deterministic) \ac{ODE}\footnote{For the sake of simplicity, throughout this chapter, we will consider one-dimensional systems. We will extend the discussion to the \( n \)-dimensional case only when necessary for the exposition.}

\begin{equation}\label{eq:differential_equation}
\diff{x}{t} = f(x),
\end{equation}

which models the evolution over time of a quantity of interest, $x$. We call $f(x)$ the vector field of the system, as it represents the direction and speed of $x$, given its position or state. Typically, one wants to study how $x$ evolves starting from a specific initial state $x_0$. If $f(x)$ satisfies certain regularity conditions, knowing both $x_0$ and $f(x)$ is enough to determine a unique solution $x(t)$, which describes the complete trajectory of $x$ as a function of time. In this sense, $f(x)$ characterizes the entire system and allows us to analyze its behavior\footnote{for a causal interpretation of \acp{ODE} see \citet{DBLP:journals/corr/abs-2402-18477}}.

Similarly, we can consider the case where $x_0$ is not precisely known. Here, the initial state may belong to a set of possible values, requiring us to apply the evolution described by \cref{eq:differential_equation} to each case. More generally, if we treat $x_0$ as a random variable $X_0$, the solution of \cref{eq:differential_equation} becomes a \ac{SP} $X: T \times \Omega \mapsto \R$ (\cref{def:dsp}), with the key difference that $t$ is now a continuous parameter in $\R$ rather than a discrete index. Recall that $\Omega$ denotes the sample space on which we define our probability space.

Moreover, we assume that $f$ can also behave probabilistically, depending on a random variable $\noise : \Omega \mapsto \R^d$. In other words, the equation might have additional random parameters that are separate from the initial condition $X_0$. We can then rewrite \cref{eq:differential_equation} in the following form:

\begin{equation}\label{eq:random_differential_equation}  
    \diff{X}{t} = f(X,\noise),  
\end{equation}  

and refer to it as a \ac{RDE}\footnote{See \citep{bongers2018causal}, $f$  for further clarification on \acp{RDE}}.

\begin{example}[Random Logistic Differential Equation]
Consider the equation 
\begin{equation}\label{eq:logistic_equation}
    \diff{X}{t} = X (1-X),
\end{equation}
which is the standard form of the logistic differential equation, commonly used to describe models such as population growth. It is straightforward to show that the solution to this equation is
$$X(t) = \frac{1}{1+\frac{1-X_0}{X_0}e^{-t}}.$$

Notice that, regardless of the initial state, all trajectories converge to the value $1$. Indeed, \cref{eq:logistic_equation} admits two equilibrium points for $x$ ($0$ and $1$), i.e., values at which its velocity $f(x)$ is zero. 
\begin{figure}
\centering
\includegraphics[width=0.95\linewidth]{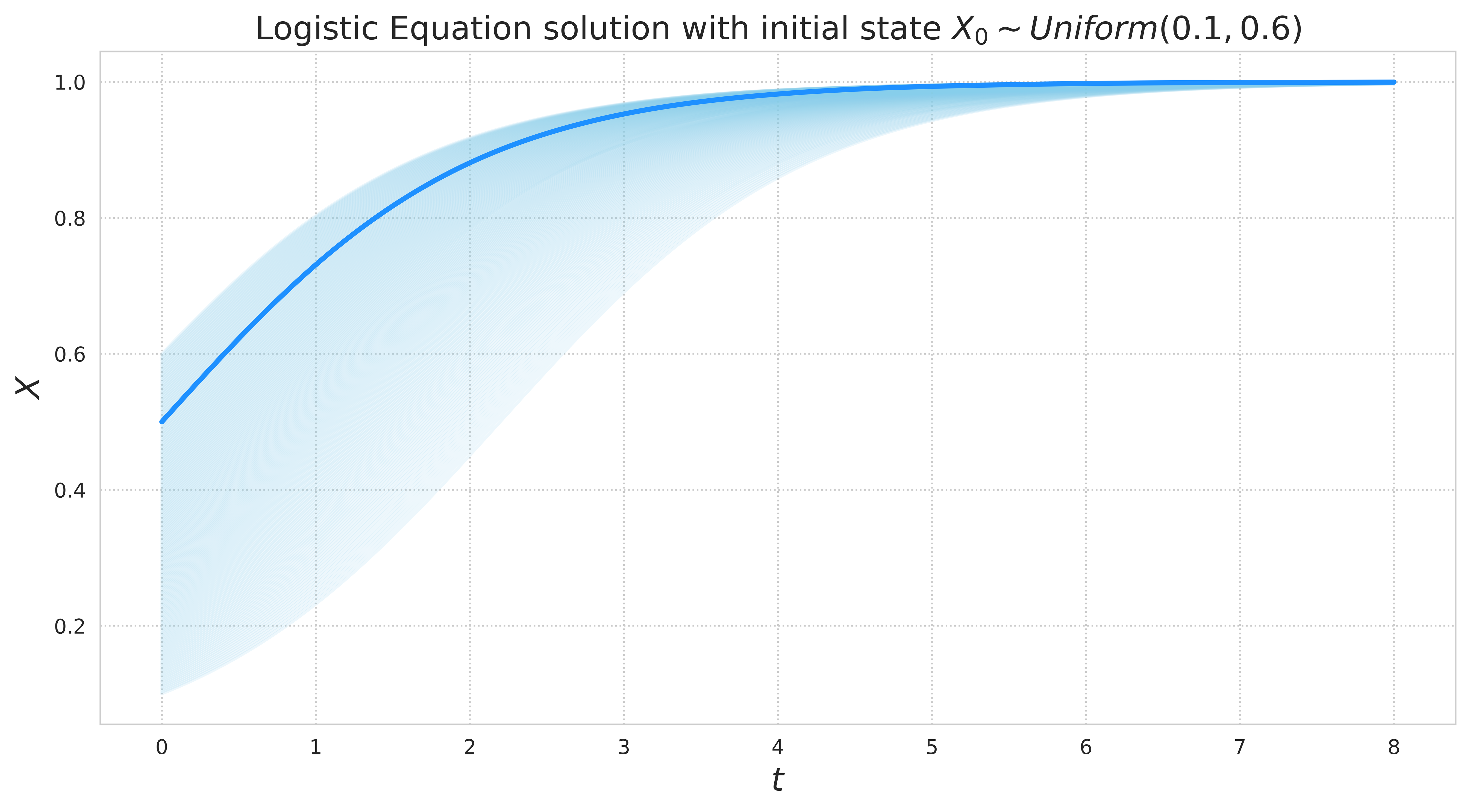}
\caption{\textbf{Logistic Equation.} Representation of the evolution of a \ac{SP} driven by the logistic equation. The initial distribution $X_0$ has a uniform density between $0.1$ and $0.6$. A particular solution of this equation is the sigmoid function $\sigma(t) = 1/(1+e^{-t})$ (darker blue line), obtained by setting $x_0=1/2$.}
\label{fig:logistic_equation}
\end{figure}

As a slight variation, consider the equation
\begin{equation}\label{eq:logistic_equation_random_coeff}
\diff{X}{t} = X (k - X).
\end{equation}
where $k$ is a random variable representing different limit values. In this case, rather than converging to a single limit value, we observe a limiting distribution: as $t \to \infty$, $X(t)$ increasingly resembles a specific random variable, which we denote by $X_\infty$.

\begin{figure}
    \centering
    \includegraphics[width=0.9\linewidth]{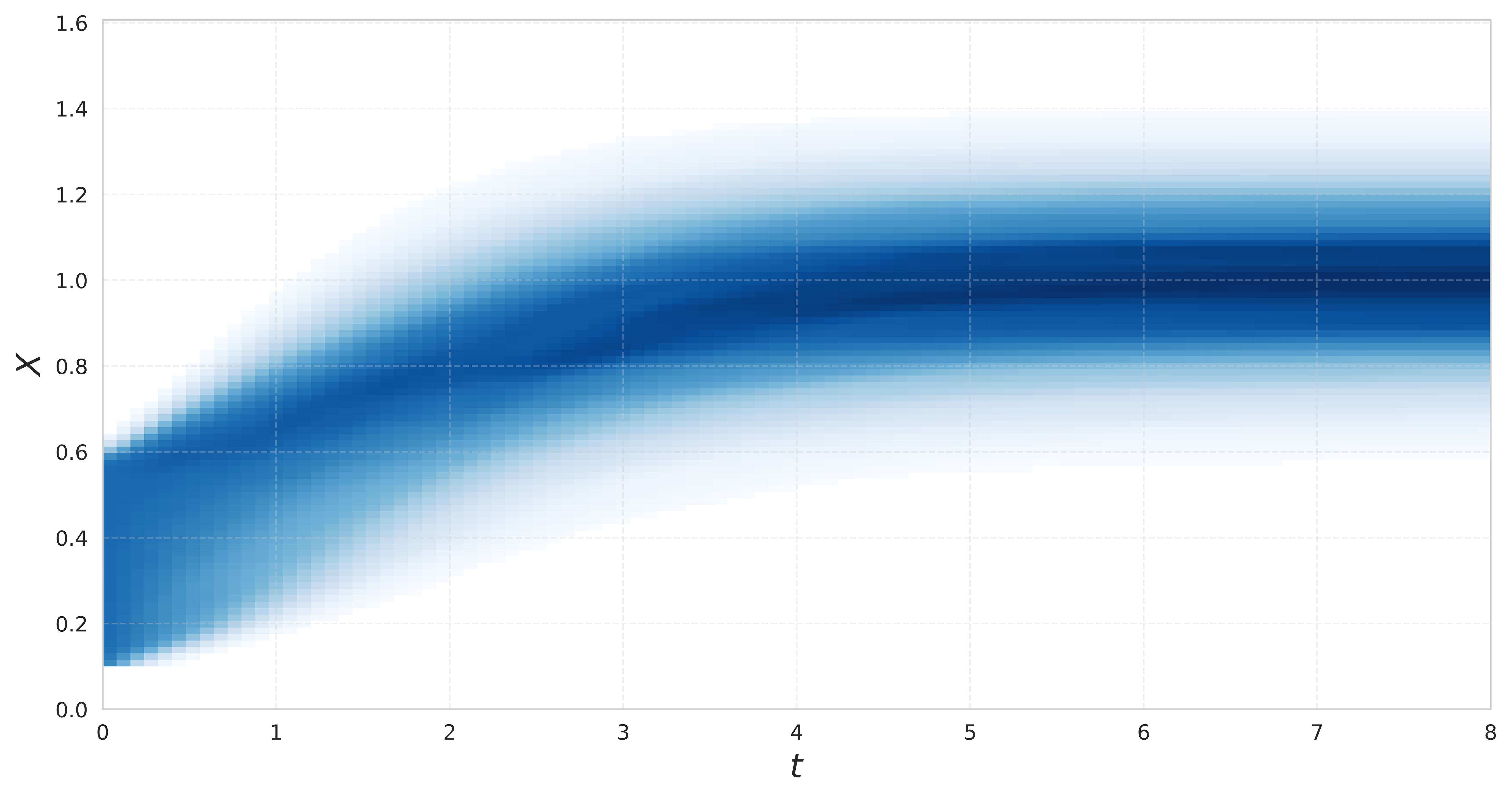}
    \caption{\textbf{Random Logistic Equation.} Evolution of the \ac{SP} driven by \cref{eq:logistic_equation_random_coeff}, with $k\sim \normal(1,0.1)$ and $X_0 \sim \uniform(0.2,0.6)$. Note that, assuming $k$ and $X_0$ are independent of each other, $X(t)$ converges to $k$.}
    \label{fig:random_coeff_logistic_process}
\end{figure}
\end{example}

\ac{RDE}s are generally easy to handle, as the randomness primarily affects the \quotes{initial conditions} of the process. Once a specific realization $\omega \in \Omega$ is fixed and both $X_0(\omega)$ and $\noise(\omega)$ are determined, $X(\omega,t)$ becomes the solution of an \ac{ODE}. In general, the solution of an \ac{RDE} can be expressed in integral form (we are explicitly indicating the dependence of $X$ and $\noise$ on $\Omega$ to highlight where stochastic elements appear in the equation):  

\begin{equation}\label{eq:rde_solution}
X(t,\Omega) = X_0(\Omega) + \int_0^t f(X(s,\Omega),\noise(\Omega)) \dd s;
\end{equation}  

By fixing a specific $\omega \in \Omega$ and introducing the notations $\hat{x_0} := X_0(\omega)$ and $\hat{f}(x) := f(x,\noise(\omega))$, we can rewrite the equation as  

\begin{equation}\label{eq:rde-to-ode}
    x(t) = \hat{x_0} + \int_0^t \hat{f}(x(s)) \dd s,
\end{equation}  

which corresponds to the classical solution of an initial value problem with initial condition $\hat{x}_0$ and an \ac{ODE} as in \cref{eq:differential_equation}.

In \citep{bongers2018causal}, it is shown how a direct link can be established between \acp{RDE} and \acp{SCM}. Similar to what was discussed in \cref{sec:causal-perspective}, the latter can be interpreted as the compressed version of the former, removing all information about the temporal component of the process and considering only the invariant distribution, obtained in the limit $t \to \infty$. In other words, \ac{SCM} model the long-run behavior of \acp{RDE}.

In this work, we aim to extend the discussion to \acp{SDE}, which provide a more sophisticated framework for modeling inherently stochastic systems. Unlike \acp{RDE}, where randomness is primarily confined to the initial conditions, \acp{SDE} incorporate continuous random perturbations. This implies that the system is influenced by both deterministic dynamics and stochastic fluctuations at every point in time, making it impossible to reduce the problem to an initial condition formulation of an \ac{ODE}. Instead, the stochastic component introduces a time-dependent randomness that fundamentally shapes the system's evolution.

To rigorously analyze and solve \acp{SDE}, advanced mathematical tools are required. These include Brownian motion, which models continuous random perturbations, and Itô calculus, a specialized framework that extends standard calculus to stochastic processes. In the following sections, we introduce these concepts in detail.

\subsection{Brownian Motion}
\ac{BM} refers to the random, irregular movement of pollen particles suspended in water, first observed by the botanist Robert Brown in 1828. 
This movement is caused by the continuous collisions of water molecules with the pollen, leading to the gradual spreading or diffusion of the particles in the water. 
The concept of \ac{BM} has broad applications beyond the study of microscopic particles in liquids. It is used to model various phenomena, such as stock price fluctuations, thermal noise in electrical circuits, and random disturbances in physical, biological, economic, and management systems. 
Its universality stems from the minimality of the assumptions required for its theoretical framework. 
Essentially, any process involving random disturbances at each moment in time can be analyzed using \ac{BM}. 
In the following, we introduce its definition and briefly discuss its fundamental properties.

\begin{definition}[Brownian Motion]\label{def:brownian_motion}
A continuous-time stochqstic process $W$ is called a Brownian motion if it satisfies the following properties\footnote{The definition provided here presents the properties in a simplified form, as we do not intend to introduce the concept of filtration on a probability space in this thesis for the sake of simplicity. For a rigorous definition, we refer to \citet{pascucci2024probability}.}:

\begin{enumerate}
    \item $W(0) = 0$ almost surely,
    \item $W$ is continuous almost surely,
    \item $W(t) - W(s)$ is independent of $W(s)$ for every $t > s > 0$,
    \item $W(t) - W(s) \sim \normal(0, t-s)$ for every $t > s > 0$.
\end{enumerate}

\begin{figure}
    \centering
    \includegraphics[width=0.9\linewidth]{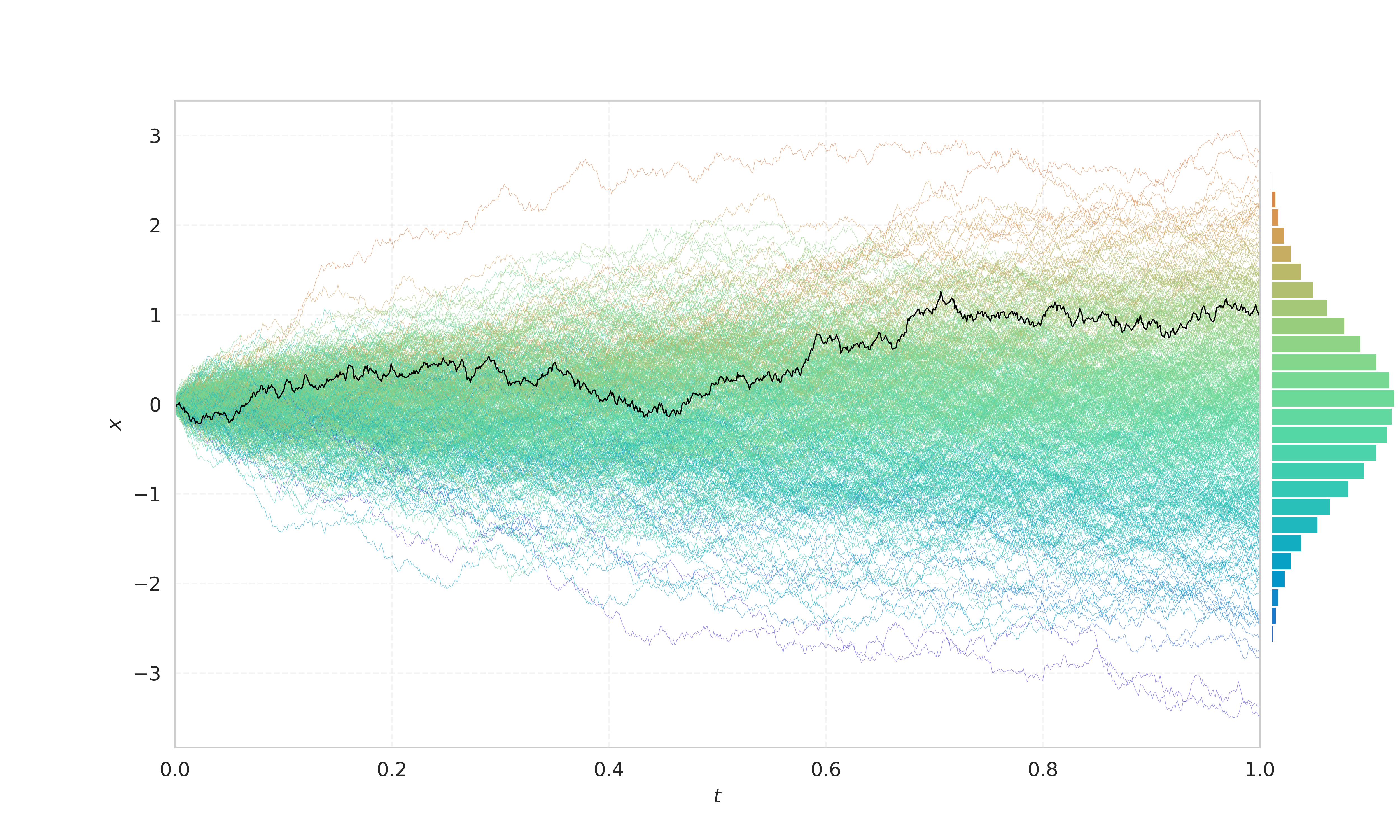}
    \caption{\textbf{Brownian Motion.} Depiction of a standard \ac{BM} over the time interval \([0, 1]\). The black line represents a single trajectory. The distribution of the process at the final time \(t = 1\) follows a \(\normal(0, 1)\) (standard normal) distribution, as illustrated by the histogram on the right.}
    \label{fig:brownian_motion}
\end{figure}
\end{definition}

The above properties can be interpreted as follows:
$(1)$ states that the trajectories of a \ac{BM} always start from the origin; 
$(2)$ ensures that each individual trajectory is continuous (although, as can be intuitevely expected from \cref{fig:brownian_motion}, it is almost surely not differentiable); 
$(3)$ represents the property of \textit{independent increments}: the process, considered over separate intervals, is completely memoryless;
$(4)$ states two important facts: the first one is that the process is \textit{time-invariant}, indeed it holds that $W(t) - W(s) = W(t-s)$; the second one is that the distribution of any increment is normal. This property provides an additional interpretative insight \ie{}, the observed variation over any time interval can be understood as the macroscopic effect of infinitely many microscopic perturbations occurring on much shorter time scales. Regardless of their specific nature, assuming these perturbations are i.i.d., their cumulative sum follows a Gaussian distribution, as stated by the Central Limit Theorem. This further clarifies why \ac{BM} is so prevalent across diverse scientific disciplines. 

\begin{remark*}
\ac{BM} is a Markov process, meaning that its future values depend only on its present state and not on past observations. Formally, for any sequence $t_1<t_2<\dots<t_{n-1}<t_n$,
\begin{equation}
P(W(t_n) | W(t_{n-1}), \dots, W(t_1)) = P(W(t_n) | W(t_{n-1})).
\end{equation}
Furthermore, \ac{BM} is also a martingale, e.g., for any $s<t$,
\begin{equation}
\mathbb{E}[W(t) | W(s)] = W(s).
\end{equation}
This last property can be directly derived from (3) and (4).
\end{remark*}

\paragraph{\ac{BM} is not the solution of any \ac{RDE}}  

Suppose we want to model \ac{BM}. Instead of  considering~\cref{eq:random_differential_equation} and trying to find a process that satisfies it, we aim to identify the vector field $f$ such that  

$$
\diff{W}{t} = f(W,\noise).
$$

We immediately observe that the left-hand side term is ill-posed since \ac{BM} is \textit{almost everywhere non-differentiable} w.r.t. $t $. Even if we interpret the equation in a weak sense, we would still need to transfer the irregularity of $W$ onto $f$.  

Moreover, it is easy to recognize an expressivity issue. Firstly, since $W_0 = 0$ almost surely, there is no randomness in the initial state of the process. Additionally, we have assumed that the random parameters of $f$ are finite, which contradicts the properties of \ac{BM}: no matter how many observations we have at times $t_1, \dots, t_m$, the next increment will always be independent of the entire history.  

Thus, we must consider a new class of processes specifically designed to handle such irregularities. In the next section, we introduce Itô processes.

\subsection{Itô Processes and Stochastic Differential Equations}

A (time-homogeneous) Itô process is a \ac{SP} defined by\footnote{For clarity and consistency with the literature on Itô calculus \citep{pascucci2024probability}, we switch the notation from $X(t)$ to $X_t$, to indicate the process at time $t$.}  

\begin{equation}\label{eq:ito-process}  
X_t = X_0 + \int_0^t \mu(X_s) \dd s + \int_0^t \sigma(X_s) \dd W_s,  
\end{equation}  

where the random variable $X_0$ is the initial state of $X$, $\mu$ and $\sigma$ are called respectively \textit{the drift and the diffusion coefficients} of the process. As a first observation, notice that as $\sigma\to 0$, the equation reduces to \cref{eq:random_differential_equation}. In fact, the integral on the right represents the irregular perturbation of the process due to the contribution of \ac{BM}.

We can express \cref{eq:ito-process} in its differential form as  

\begin{equation}\label{eq:stochastic_differential_equation}  
dX_t = \mu(X_t) \dd t + \sigma(X_t)\dd W_t.  
\end{equation}  

This equation is what is referred to as a \acf{SDE}. It is important to note that it is not a precise mathematical statement but rather a compact way of writing \cref{eq:ito-process}.

The terms \textit{drift} and \textit{diffusion} have a precise meaning: Considering an infinitesimal time interval in which the process evolves, say from $\tau$ to $\tau+\delta$, we obtain  

$$\expected{X_{\tau+\delta}-X_{\tau}}=\mu(X_{\tau})\delta,$$  

which means that $\mu$ controls the \textit{average infinitesimal displacement} of the process. Similarly, we get  

$$\Var[X_{\tau+\delta}-X_{\tau}]=\sigma^2(X_t)\delta,$$  

which means that $\sigma$ controls the \textit{infinitesimal rate of diffusion} of the process.  

In fact, for small time intervals, we have  

$$X_{\tau+\delta}-X_{\tau} \sim \normal(\delta\mu(X_{\tau}),\delta\sigma^2(X_{\tau})). $$
\begin{example*}[\ac{BM} with constant drift and diffusion]
Suppose $\mu$ and $\sigma$ are constants and $X_0 = 0$. then Equation \eqref{eq:stochastic_differential_equation} simplifies to  
$$X_t = \mu t + \sigma W_t. \quad (\text{differential form:} \quad \dd X_t = \mu \dd t + \sigma \dd W_t)$$  
It is easy to verify that, starting from \cref{def:brownian_motion}, property $(4)$, we immediately obtain
$$\expected{X_t}=\expected{\mu t} + 
\underbrace{\expected{\sigma W_t}}_{=0}=\mu t$$
and 
$$\Var[X_t] = \underbrace{\Var[\mu t]}_{=0} + \Var[\sigma W_t] = \sigma^2 \Var[W_t] = \sigma^2 t.$$
\end{example*}
\subsection{Stochastic Integrals}  

One of the main challenges in stochastic calculus is to provide a precise meaning to the integral  

$$\int_0^t \sigma(X_s) \dd W_s.$$  

The main difficulty arises from the fact that \ac{BM} is not a rectifiable function. As a result, it is not possible to naturally rewrite the integral in the form  

$$\int_0^t \sigma(X_s) \diff{W_s}{t} \dd t.$$  

The classical approach to defining this integral begins with simple (\eg, piecewise constant) processes, and then progressively extends to more general classes of integrands through successive refinements. Eventually, the integral is well-defined for the class of locally square-integrable processes. The final result is encapsulated in the well-known Itô formula. Without delving into the technical details of the construction, we discuss the implications of this result for estimating key quantities of the process. Let us begin by introducing Itô's formula in the simplest case, where the reference \ac{SP} is precisely \ac{BM}.

\begin{theorem}[Itô formula - Brownian Motion]\label{thm:ito_formula_BM}
Let $f:\R \to \R$ be a $C^2$ function (e.g., with a continuous second derivative), then the following relation holds:
\begin{equation}\label{eq:ito_formula_BM}
    f(W_t) = f(0) + \int_0^t f'(W_s) \dl{W_s} + \frac{1}{2} \int_0^t f''(W_s) \dl{s}  
\end{equation}
    
\end{theorem}

A first observation is that, in the special case where $f$ is the identity function $f(x)=x$, we obtain  

$$\int_0^t \dl{W_s} = W_t - W_0,$$

since the last integral vanishes. Moreover, the expected value of this integral is $0$, which follows from the properties of \ac{BM}. Less intuitively, under fairly general conditions on $f$, it also holds that  

$$\expected{\int_0^t f(W_s) \dl{W_s}}=0.$$  

Informally, this result follows from interpreting the integral as a weighted sum of many small increments $\dl{W_s}$, each of which has mean zero due to the properties of \ac{BM}.  

The rightmost integral is the most interesting feature of the relation. Without it, the formula would be entirely equivalent to the fundamental theorem of calculus. In fact, it represents the fundamental difference between stochastic calculus and classical calculus.  

Rewriting \cref{eq:ito_formula_BM} in differential form, we obtain  
$$\dl{f(W_t)} = f'(W_t)\dl{W_t} + \frac{1}{2}f''(W_t)\dl{t}.$$  
We can observe a strong similarity with the second-order Taylor expansion of $f$ in its differential form, which is written as  
$$\dl{f(W_t)} = f'(W_t)\dl{W_t} + \frac{1}{2}f''(W_t)\dl{W_t^2}.$$  
The formulas would be entirely identical if the identity $\dl{W_t^2}=\dl{t}$ held. In fact, it turns out that, symbolically, this identity incorporates a rule consistent with Itô's formula, allowing us to reach the same conclusions. Specifically, using the relations  
\begin{equation}\label{eq:ito_formula_differential}
\dl{t^2}=0, \qquad \dl{t}\dl{W_t}=0, \qquad \dl{W_t^2}=\dl{t}    
\end{equation}  
we can formulate the extended version of Itô's formula for any process defined as in \cref{eq:ito-process}.  

\begin{theorem}[Itô formula for Itô Processes]\label{thm:ito_formula_for_ito-processes}  
Let $f:\R^2 \to \R$ be a $C^2$ function and $X_t$ an Itô process as in \cref{eq:ito-process}, defined by the \ac{SDE}  
$$dX_t = \mu(X_t) \dd t + \sigma(X_t)\dd W_t,$$  
then the following relation holds:  
\begin{equation}\label{eq:ito_formula_ito_process}  
    \dd f(X_t,t) = \left[ \difcp{f}{t} + \difcp{f}{x} \mu + \frac{1}{2}\difcp[2]{f}{x}\sigma^2 \right] \dl{t} + \difcp{f}{x} \sigma \dl{W_t}.  
\end{equation}  
\end{theorem}  
\begin{proof}  
    Consider the (differential) second-order Taylor expansion of $f$  
    $$\dd f(X_t,t) = \difcp{f}{x}\dl{X_t} + \difcp{f}{t} \dl{t} + \frac{1}{2}\left[ \difcp[2]{f}{x} \dl[2]{X_t} + 2 \difcp{f}{x,t} \dl{X_t}\dl{t} + \difcp[2]{f}{t} \dl[2]{t}\right].$$  
    Substituting $\dl{X_t}$ with $\mu\dl{t}+\sigma \dl{W_t}$ we get  
    \begin{equation*}  
    \begin{split}  
        \dd f(X_t,t) = &\difcp{f}{x}\mu \dl{t} + \difcp{f}{x}\sigma \dl{W_t} + \difcp{f}{t} \dl{t} + \\  
        & \frac{1}{2}\left[ \difcp[2]{f}{x} (\mu\dl{t}+\sigma\dl{W_t})^2 + 2 \difcp{f}{x,t} (\mu \dl{t} + \sigma \dl{W_t})\dl{t} + \difcp[2]{f}{t} \dl[2]{t}\right].  
    \end{split}  
    \end{equation*}  
    Applying \cref{eq:ito_formula_differential} to the expressions in square brackets, the only nonzero contribution is $$\frac{1}{2} \difcp[2]{f}{x}\sigma^2 \dl[2]{W_t} = \frac{1}{2} \difcp[2]{f}{x}\sigma^2 \dl{t}.$$  
    Grouping all the remaining terms, we obtain \cref{eq:ito_formula_ito_process}.  
\end{proof}  

In practice, Itô’s formula allows us to access relevant observables of the system under study, represented as an \ac{SP}. For example, by considering $f(x,t)=x$ and $f(x,t)=x^2$, we can estimate the mean and variance of $X_t$, respectively, and analyze their evolution over time. More generally, it enables us to solve fairly general \acp{SDE}, often finding explicit closed-form solutions. By extending classical calculus to the stochastic setting, Itô’s formula serves as a cornerstone of modern applied probability, providing powerful tools for both theoretical analysis and practical applications.  

\subsection{Fokker Plack Equation}

The study of a \ac{SP} through its \ac{SDE} provides very detailed information about its evolution, potentially describing the behavior of each individual trajectory. However, we are often interested in a more macroscopic description of the system—a coarser but simpler resolution. To this end, we can focus only on the probability density of $X_t$, which we denote by $\prob(x,t)$. 

If our goal, for example, is to estimate the evolution of the mean and variance of $X_t$ over time, this object is as good as the original \ac{SP}. Remarkably, the Fokker–Planck equation (also called the Forward Kolmogorov equation) allows us to translate an \ac{SDE} into an associated Partial Differential Equation, which models the evolution of $\prob(x,t)$.

\begin{theorem}[Fokker Planck Equation]
Let $X_t$ be a \ac{SP} modeled by the \ac{SDE} as in \cref{eq:stochastic_differential_equation}, and let $\prob(x,t)$ be its probability density. The following Partial Differential Equation holds:
\begin{equation}\label{eq:fokker_planck}
\difcp{\prob(x,t)}{t} = - \difcp{[\mu(x)\prob(x,t)]}{x} + \frac{1}{2} \difcp[2]{[\sigma^2(x)\prob(x,t)]}{x}.
\end{equation}
\end{theorem}

\begin{proof}
Let us consider an arbitrary function $f(X_t)$. By applying \cref{eq:ito_formula_ito_process}, we obtain\footnote{Note that $f$ depends only on $X_t$ in this case, so the partial derivative with respect to $t$ is $0$.}
\begin{equation*}
\dd f = \left[ \difcp{f}{x} \mu + \frac{1}{2}\difcp[2]{f}{x}\sigma^2 \right] \dl{t} + \difcp{f}{x} \sigma \dl{W_t}.  
\end{equation*}
Taking the expectiation and differentiating w.r.t to $t$  we get
\begin{equation*}
\diff**t{\expected{f}} = \expected{\difcp{f}{x} \mu  + \frac{1}{2}\difcp[2]{f}{x}\sigma^2}.
\end{equation*}
Note that the term with $\dl{W_t}$ disappears, since the expectiation of $\dl{W_t}$ is zero. Rewriting expectation as an integral over $\prob(x,t)$, we get
\begin{equation*}
    \diff**t{\int_{-\infty}^{\infty}f(x)\prob(x,t)\dl{x}} = \int_{-\infty}^{\infty} \left[ \mu(x)\difcp{f(x)}{x}+ \frac{1}{2}\sigma^2(x)\difcp[2]{f(x)}{x} \right] \prob(x,t) \dl{x}.
\end{equation*}
Consider the drift part of the integral on the right part. Integrating by parts, we obtain
\begin{equation*}
    \int_{-\infty}^{\infty}\mu(x)\difcp{f(x)}{x}\prob(x,t) \dl{x} = \underbrace{\mu(x)f(x)p(x,t)\big|^\infty_{-\infty}}_{=0} - \int_{-\infty}^{\infty} f(x) \difcp{[\mu(x)\prob(x,t)]}{x} \dl{x}.
\end{equation*}
Applying the same procedure for the diffusion term, intergating by parts two times, we obtain in the end
\begin{equation*}
    \int_{-\infty}^{\infty} f(x) \left[ \difcp{\prob(x,t)}{t} +\difcp{[\mu(x)\prob(x,t)]}{x}-\frac{1}{2}\difcp[2]{[\sigma^2(x)\prob(x,t)]}{x}\right]\dl{x} = 0
\end{equation*}
Since $f(x)$ is arbitrary, the term in square brackets must be identically zero.  
\end{proof}

\subsection{Multidimensional Stochastic Calculus}

In this section, we extend the definitions and results from the previous sections to the multidimensional setting.

\begin{definition}[$d$-dimensional \ac{BM}]\label{def:d-dim_BM}
    A $d$-dimensional \ac{BM}, denoted by $\W$, is a $d$-dimensional vector of independent \acp{BM}. Specifically, it satisfies
    $$\W_t - \W_s \sim \normal(0, (t-s)\I_d), \quad \text{for all} \ t>s>0,$$
    where $\I_d$ is the $d \times d$ identity matrix.
\end{definition}

\begin{definition}[$n$-dimensional \ac{SDE}]\label{def:d-dim_SDE}
An $n$-dimensional \ac{SDE} is a system of \acp{SDE} of the form
\begin{equation}\label{eq:d-dim-SDE}
    \dl{\X_t} = \bm{\mu}(\X_t) \dl{t} + \bm{\sigma}(\X_t) \dl{\W_t},
\end{equation}
where:
\begin{enumerate}
    \item $\X_t$ is an $n$-dimensional \ac{SP}.
    \item $\bm{\mu}: \R^n \to \R^n$ is a vector-valued function, referred to as the \textit{drift vector}.
    \item $\bm{\sigma}: \R^n \to \R^{n \times d}$ is a matrix-valued function (\textit{diffusion matrix}).
    \item $\W_t$ is a $d$-dimensional \ac{BM}.
\end{enumerate}
\end{definition}

\begin{definition}\label{ito_formal_rules_d_dimensional}
The formal rules for computing products of differentials (\cref{eq:ito_formula_differential}) extend to the multidimensional case as follows:
$$\dl{t^2}=0, \qquad \dl{t}\dl{W_t^i}=0, \qquad \dl{W_t^i}\dl{W_t^j}=\delta_{ij}\dl{t},$$
where $\delta_{ij}$ denotes the Kronecker delta. These rules arise from the fact that the components of $\W_t$ are mutually independent.
\end{definition}

\begin{definition}[$n$-dimensional Itô formula for Itô processes]\label{def:ito_formula-n-dimension}
    Let $\X_t$ be a \ac{SP} governed by an $n$-dimensional \ac{SDE}. The Itô formula (\cref{thm:ito_formula_for_ito-processes}) extends to
    \begin{equation}\label{eq:ito-formula-n-dimensional}
        \dd {f(\X_t,t)} = \left[ \difcp{f}{t} + \sum_{i=1}^n\mu^i\difcp{f}{x_i} + \frac{1}{2}\sum_{i,j=1}^n \Sigma^{ij}\difcp{f}{x_i x_j} \right] \dl{t} + \sum_{i=1}^n \sum_{j=1}^d \sigma^{ij}\difcp{f}{x_i} \dl{W_t^j}, 
    \end{equation}
    where $\bm{\Sigma} := \bm{\sigma}\bm{\sigma}^T$. 
\end{definition}

\begin{definition}[$n$-dimensional Fokker–Planck equation]
Consider an $n$-dimensional \ac{SDE} with drift vector $\bm{\mu}(x)$ and diffusion matrix $\bm{\sigma}(x)$. Assume that $\bm{\Sigma}$ is positive definite \ie, $\bm{\Sigma} \succ \rho \bm{I}$ for some $\rho>0$. The Fokker–Planck equation (\cref{eq:fokker_planck}) generalizes to
\begin{equation}\label{eq:fokker_planck_n_dimensional}
    \difcp{\prob(\x,t)}{t} = - \sum_{i=1}^n \difcp{[\mu^i(\x)\prob(\x,t)]}{x_i} + \frac{1}{2} \sum_{i,j=1}^n\difcp{[\Sigma^{ij}(\x)\prob(\x,t)]}{x_i x_j}.
\end{equation}
Using the divergence operator notation, $\nabla \cdot a := \sum_{i=1}^n \difcp{a}{x_i}$, the equation can be written in a more compact form as
$$\difcp{\prob}{t} = - \nabla \cdot (\bm{\mu}\prob) + \frac{1}{2} \nabla \cdot (\nabla \cdot(\bm{\Sigma}\prob)).$$
\end{definition}

\begin{example}[Constant, isotropic diffusion]
    Consider the $d$-dimensional \ac{SDE}
    $$ \dl{\X_t} = \bm{\mu}(\X_t) \dl{t} + \sigma \I_d \dl{\W_t}.$$
    Here, each component of $\X_t$ is influenced by an independent one-dimensional \ac{BM}, meaning that diffusion occurs uniformly in all directions. 
    In this case, we have $\bm{\Sigma} = \sigma^2 \I_d$, and the divergence term simplifies to $\nabla \cdot(\bm{\Sigma}p) = \sigma^2 \nabla p$. Consequently, \cref{eq:fokker_planck_n_dimensional} reduces to
    \begin{equation}\label{eq:fokker_planck_isotropic}
        \difcp{\prob}{t} = - \nabla \cdot (\bm{\mu}\prob) + \frac{\sigma^2}{2} \Delta \prob,
    \end{equation}
    where $\Delta a := \nabla \cdot \nabla a = \sum_{i=1}^d \difcp[2]{a}{x_i}$ is the Laplacian operator. 

    This equation contains two important special cases:
    \begin{itemize}
        \item Setting $\bm{\mu} = 0$ yields $\difcp{\prob}{t} = \frac{\sigma^2}{2} \Delta \prob$, which corresponds to the well-known \textit{heat equation} with diffusive constant $\sigma^2/2$.
        \item Setting $\sigma = 0$ results in $\difcp{\prob}{t} = - \nabla \cdot (\bm{\mu}\prob)$, which is the \textit{continuity equation} with vector field $\bm{\mu}$.
    \end{itemize}  
\end{example}

\section{Gradient Flow}\label{sec:gradient_flow}
In \cref{sec:causal-perspective}, we discussed how to interpret a discrete-time dynamical system through a causal lens, providing a meaning for causal interventions on \acp{SdE} and demonstrating the existence of a link between these and \acp{SCM}. However, the existence of an explicit mapping between the two frameworks has only been established for linear systems (\cref{thm:var_to_scm}). The goal of this section is to lay the foundation for extending this result to nonlinear and/or continuous-time systems. 

A key assumption in our approach was the existence of an invariant distribution for the \ac{DSP} under consideration. Indeed, invariant distributions in both the observational and interventional dynamical systems served as the bridge between the temporal framework (\ac{SdE}) and the static framework (\ac{SCM}). 

This leads to the following question: What conditions must the drift vector $\bm{\mu}(\x)$ and the diffusion matrix $\bm{\sigma}(\x)$ satisfy for the associated \ac{SDE} to admit an invariant distribution? In general, this problem is challenging to address \citep{berglund2021long, locherbach2015ergodicity, hairer2010convergence}. However, the solution is particularly elegant and well-structured for an important special case \textit{Stochastic Gradient Flow}, also known as the \textit{Overdamped Langevin Equation},  while remaining sufficiently general for the purposes of our study.

\subsection{Stochastic Gradient Flow}
Consider the $d$-dimensional \ac{SDE}\footnote{Compared to \cref{eq:fokker_planck_isotropic}, we apply a slight transformation to the diffusion coefficient to simplify the equations that will follow. Specifically, we leave the identity matrix implicit and replace the term $\sigma$ with $\sqrt{2\sigma}$.}  
\begin{equation}\label{eq:langevin_equation}
    \dl{\X_t} = - \nabla V(\X_t) \dl{t} + \sqrt{2}\sigma \dl{\W_t}.
\end{equation}  
This equation has two important features. First, the diffusion matrix is constant and isotropic. Second, the drift vector corresponds to the (negative) gradient of a function $V$, which we refer to as the \textit{potential function}.

\begin{remark*}  
    Under the condition  
    $$\mathcal{Z} := \int_{\R^d} e^{\frac{-V(\x)}{\sigma^2}} \dl{\x} < \infty,$$  
    meaning that \( V(\x) \) grows faster than logarithmically as \( |\x| \to \infty \), the stochastic differential equation in \cref{eq:langevin_equation} has a well-defined invariant distribution, known as the \textit{Gibbs distribution}:  
\begin{equation}\label{eq:langevin_invariant_distribution}  
        \prob_\infty(\x) = \frac{1}{\mathcal{Z}}e^{\frac{-V(\x)}{\sigma^2}}.  
    \end{equation}  
\end{remark*} 

\begin{proof}
    The Fokker-Planck equation for \cref{eq:langevin_equation} is given by  
    $$ \difcp{\prob}{t} =  \nabla \cdot (\nabla V \prob) + \nabla \cdot ( \sigma^2\nabla\prob).$$  
    To prove that \cref{eq:langevin_invariant_distribution} is indeed an invariant distribution, we need to verify that \( \difcp{\prob_\infty}{t} = 0 \), which leads to the condition  
    $$ \nabla \cdot (\nabla V \prob_\infty + \sigma^2 \nabla \prob_\infty) = 0.$$  
    Computing explicitly, we find  
    $$ \nabla \prob_\infty = - \frac{1}{\mathcal{Z} \sigma^2} \nabla V e^{\frac{-V(\x)}{\sigma^2}} = - \frac{\nabla V \prob_\infty}{\sigma^2}.$$  
    Substituting this into the equation, we obtain  
    $$\nabla \cdot (\nabla V \prob_\infty -\frac{\sigma^2}{\sigma^2} \nabla V \prob_\infty) = \nabla \cdot 0 = 0.$$  
\end{proof}

We are now ready to introduce the concept of causal intervention on a \ac{SDE}. The key idea behind our approach is that $ V(\x)$ fully determines the equilibrium distribution of the dynamical system. \textit{Any intervention modifies its properties, leading to a distributional shift, which manifests through a new Gibbs distribution once the new equilibrium is reached}.
\section{A Causal Perspective on SDEs}\label{sec:causal_sdes}
The concept of a causal system varies significantly across scientific theories. For example, in the context of dynamical systems, particularly in control theory, a system is considered causal if its governing function depends only on its past or present states, never on future states. This is why such systems are also referred to as non-anticipative or physical systems. This interpretation is so broad that it is difficult to conceive of a system that could be meaningfully modeled in a non-causal manner. As a result, rather than serving as an epistemological principle, causality in this sense becomes a methodological prerequisite for describing the natural constraints of evolutionary models.

In contrast, Pearl’s theory of causality defines causal relationships as inherently asymmetric, making a clear distinction between interaction-based phenomena and those governed by causation. The assumption of acyclicity in causal graphs is not merely a mathematical convenience but a fundamental principle of how causal relationships are conceptualized. In this framework, the system's variables are arranged within a rigid hierarchical structure in which root causes influence their descendants, propagating information along chains of cause and effect. 

In this work, we take an intermediate stance, adopting a classical definition that is compatible with both perspectives. We define a variable as causal with respect to another if intervening on it results in a corresponding change in the second variable. The key aspect of this definition is that the intervention is conceived as external to the system under study, originating from an external agent's deliberate choice. In other words, causality is established through experimentation. This methodological approach has already been considered in the past, for example in \cite{sokol2014causal},\cite{mogensen2022graphical},\cite{commenges2009general}.

However, even this perspective has limitations, as the boundary between it and the weaker notion of causality in control theory is not always well defined. Consider, for instance, a gas composed of a vast number of interacting particles. Changing the state of a single particle would, in principle, influence the others. Yet, in high-dimensional systems, this view of causality becomes less meaningful. Throughout the following sections, we will explore this issue from various angles, gradually developing a causal perspective. 

One key aspect of our discussion will be the following: The \emph{causalization} of a system can be understood as a form of coarse-grained representation. This representation preserves essential system behavior while omitting unnecessary details, offering a form of representational compression that aligns with an intelligence \emph{embedded} in the environment. This perspective accounts for an agent’s ability to intervene and generalize when facing distribution shifts. In fact, the causal graph can be regarded as the most compressed representation that remains consistent with the system’s probability distribution \citep{PetersJanzingSchoelkopf17}.

\subsection{Causal Interventions on \acp{SDE}}

Consider a $d$-dimensional \ac{SDE} as in \cref{eq:langevin_equation}, written as a system of equations
\begin{equation}
    \left\{ \dl{X^{i}_t} = \mu^i(\X) \dl{t} + \sqrt{2}\sigma\dl{W^i_t} \right\}_{i=1}^d  
\end{equation}
We define an intervention $\intervention$ on the \ac{SDE} at time $t_\intervention$ as a modification of one or more drift functions $\mu^i(\x)$ according to the rule:
\begin{equation}
    \mu^i_\intervention(\x) =\begin{cases}
    \mu^i(\x), & \text{if $t<t_\intervention$}.\\
    \tilde{\mu}^i(\x), & \text{if $t\geq t_\intervention$}.
  \end{cases}
\end{equation}

\begin{comment}
[Why Force?]
In the overdamped regime, viscous forces dominate over inertial effects. When this occurs, any motion rapidly stabilizes at a terminal velocity, effectively rendering acceleration negligible.  

As a simple model, consider an object moving through a highly viscous medium under the influence of a constant force $F$. Newton's second law yields the differential equation:  
$$m \diff[2]{x}{t} = F - b \diff{x}{t}$$
where $m$ is the mass of the object, and $b$ is the damping coefficient.  

The terminal velocity of the system \ie, the one that one eventually reaches applying the same force indefintely, is given by $F / b$, and the characteristic time scale for reaching this velocity is $\tau = m / b$. In the overdamped limit, where $m/b$ becomes very small, $\tau$ becomes negligible. This implies that the system rapidly equilibrates, reaching terminal velocity almost instantaneously, effectively eliminating acceleration from the dynamics.  

In this regime, the equation of motion simplifies to a quasi-static form where velocity is directly proportional to the applied force:  
$$F = b \diff{x}{t}$$
\section{Discretizing time: Gradient Descent}\label{sec:S_gradient_descent}
\end{comment}

In practice, we modify the mechanism that drives the evolution of $\Xcomp[i]$, replacing it with a new one starting from the intervention time $t_\intervention$. As in \cref{ch:Ci_over_time}, we focus on two types of intervention:
\begin{itemize}
    \item \textbf{Additive Intervention}: we modify the drift function by adding a constant, $\tilde{\mu}^i(\x) := \mu^i(\x) + F^i$
    \item \textbf{Forcing Intervention}: we include a term that pulls the variable toward a target value $\hat{x}^i$, $\tilde{\mu}^i(\x) := \mu^i(\x) + F^i(\hat{x}^i - \xcomp[i])$
\end{itemize}
Similarly, we denote the intervention with $\intervention$ and the new solution of the intervened \ac{SDE} with $\X^\intervention$.
As in the discrete-time case (\cref{def:ce_t_discrete}), we define the Causal Effect over time $\ce{}_t$ of $\intervention$ as

\begin{equation}\label{eq:causal_effet_time_continuous}
        \ce{}^\intervention_t := \expected{\X^\intervention_t - \X_t \mid \X_{t_\intervention}}.%\:|\: C},
\end{equation}

\section{Estimating the Long-Run Causal Effect}\label{sec:long_run_continuous}

If the dynamical system admits a stationary distribution both before and after the intervention, we may be interested in estimating the long-term effect of the intervention, that is,
$$
\ce{} = \lim_{t\to\infty} \ce{}_t = \mathbb{E}[\X^\intervention_\infty - \X_\infty].
$$

A brute-force approach consists in simulating the system for a sufficiently long time, until it reaches equilibrium. However, this method has two major limitations. First, we typically lack information about the convergence rate to the stationary distribution, which could, in principle, be very slow. Second, due in part to the first issue, this procedure may be computationally expensive.

In practice, if we have a way to estimate $\prob_\infty$ directly without relying on numerical simulation, that approach is almost certainly preferable.

Using the formulation in \cref{eq:langevin_equation}, we ask whether the drift vector can be written as the gradient of a potential function. In fact, writing the expectation explicitly as
$$
\mathbb{E}[\X_\infty] = \int_\Xspace \X \, \prob_\infty(\X) \, \dl{\X},
$$
makes the estimation of $\ce{}$ straightforward.

\paragraph{One-dimensional case}\label{sec:one-dimensional_langevin_equilibration}

Given the \ac{SDE}
\[
\dl{X_t} = \mu(X_t)\,\dl{t} + \sqrt{2}\sigma\,\dl{W_t},
\]
we can always express $\mu(x)$ as the (negative) derivative of another function by defining
\begin{equation}\label{eq:mu_to_potential}
    V(x) := - \int_0^x \mu(s)\,\dl{s}.
\end{equation}

As a result, if $V(x)$ grows sufficiently fast as $|x|\to \infty$, a stationary distribution always exists and is given by
$$
\prob_\infty(x) = \frac{1}{\mathcal{Z}} e^{\frac{-V(x)}{\sigma^2}}.
$$

In particular, under an additive intervention, the potential function $V(x)$ becomes
\begin{equation}\label{eq:potential_additive_intervention}
    V^\intervention(x) := - \int_0^x \left[\mu(s) + F\right]\,\dl{s} = V(x) - Fx,
\end{equation}
and the new stationary distribution $\prob^\intervention_\infty$ adjusts accordingly.

\begin{remark*}[Additive intervention is not always stable]
In the discrete-time setting (see \cref{sec:stability_additive_discrete}), additive interventions preserved system stability regardless of the magnitude of $F$. This is no longer guaranteed in the continuous-time case. Specifically, the existence of a stationary distribution requires that $V(x)$ grow faster than logarithmically as $|x| \to \infty$. Adding a linear term may violate this condition. For instance, if $V(x) = |x|$, an additive intervention with $F > 1$ not only fails to satisfy the condition, but also pushes the probability mass indefinitely toward larger values of $x$. A sufficient condition for an additive intervention to preserve stability is that $V(x)$ grows faster than linearly at the tails.
\end{remark*}

Similarly, the forcing intervention transforms $V(x)$ into
\begin{equation}\label{eq:potential_forcing_intervention}
    V^\intervention(x) := - \int_0^x \left[\mu(s) + F(\hat{x}-s)\right]\,ds = V(x) + \frac{F}{2}\left[(x-\hat{x})^2 - \hat{x}^2\right].
\end{equation}

Without loss of generality, and by ignoring the constant term $-F\hat{x}^2/2$, we can write
\begin{equation}\label{eq:pinf_forcing_continuos}
    \prob^\intervention_\infty(x) = \frac{1}{\mathcal{Z}} \exp\left({\frac{-2V(x)-F(x-\hat{x})^2}{2\sigma^2}}\right).
\end{equation}

Note that for large values of $F$, the rightmost term in the numerator becomes dominant, making the distribution increasingly similar to a Gaussian distribution with variance $\sigma^2 / F$. In particular, as $F \to \infty$, the distribution approaches the Dirac delta function $\delta(\hat{x})$, making forcing intervention equivalent to the operation $do(x = \hat{x})$.

\paragraph{n-dimensional case} Consider the \ac{SDE}
\begin{equation}\label{eq:n-dimensional_SDE}
    \dl{\X_t} = \bm{\mu}(\X_t) \dl{t} + \sqrt{2}\sigma \dl{\W_t}.
\end{equation}
Here, $\bm{\mu}$ is a \textit{vector field}, and we need to determine whether it admits a potential function $V$. For this to be true, assuming $V$ is sufficiently regular, the mixed partial derivatives of $V$ must be equal, that is,
\begin{equation*}
    \difcp{V}{\xcomp[i]\xcomp[j]} = \difcp{V}{\xcomp[j]\xcomp[i]}.
\end{equation*}
This condition is equivalent to requiring that $\bm{\mu}$ is an \textit{irrotational vector field}. Formally, this means that the following must hold:
\begin{equation}\label{eq:irrotational_field}
    \difcp{\mu^i}{\xcomp[j]} = \difcp{\mu^j}{\xcomp[i]}, \quad \forall i,j.
\end{equation}

This requirement is quite restrictive. To see this, consider the linear case
\begin{equation*}
    \dl{\X_t} = \A \X_t \dl{t} + \sqrt{2}\sigma \dl{\W_t}.
\end{equation*}
In this case, \cref{eq:irrotational_field} becomes the condition $A_{ij} = A_{ji}$; in other words, \textit{the evolution matrix must be symmetric}. While this condition is reasonable in physical systems, as a consequence of Newton’s third law\footnote{For every action (force) in nature, there is an equal and opposite reaction. If object A exerts a force on object B, then object B exerts an equal and opposite force on object A.}, it is incompatible with the framework of causality, which seeks to investigate asymmetries in interactions between objects. For this reason, a direct application of the Langevin equation is not satisfactory. We need to follow a different approach.

\paragraph{Acyclic n-dimensional case}\label{sec:acyclic_non_linear_continuos_equilibration}
Starting from \cref{eq:n-dimensional_SDE}, let us restrict ourselves to the case where the causal graph is acyclic. In other words, the Jacobian matrix of $\bm{\mu}$ is always lower triangular (not strict). In this setting, we can apply the results from \cref{sec:one-dimensional_langevin_equilibration} one variable at a time, following the causal hierarchy, \ie, moving from causes to effects. The system of equations then takes the form
\begin{equation}
    \begin{cases}
        \dl{\Xcomp[1]_t} = \mu^1(\Xcomp[1]_t) \dl{t} + \sqrt{2}\sigma\,dW^1_t \\
        \dl{\Xcomp[2]_t} = \mu^2(\Xcomp[1]_t,\Xcomp[2]_t) \dl{t} + \sqrt{2}\sigma\,dW^2_t\\
        \vdots \\
        \dl{\Xcomp[n]_t} = \mu^n(\Xcomp[1]_t,\Xcomp[2]_t,\dots, \Xcomp[n]_t) \dl{t} + \sqrt{2}\sigma\,\dl{W^n_t}.
    \end{cases}
\end{equation}
We proceed as follows:
\begin{enumerate}
    \item We begin by finding the distribution $\prob^1_\infty(\xcomp[1])$, using the fact that the first equation is a one-dimensional \ac{SDE}, as described in \cref{sec:one-dimensional_langevin_equilibration}.
    
    \item Extending to $\Xcomp[2]$, we study its evolution by fixing $\Xcomp[1] = \xcomp[1]$. In this way, the second equation also becomes a one-dimensional \ac{SDE}, and we obtain the limit distribution $\prob^2_{\infty[\xcomp[1]]}(\xcomp[2])$. We note that this is exactly the conditional distribution of $\Xcomp[2]$ given $\Xcomp[1] = \xcomp[1]$: $\prob^2_{\infty[\xcomp[1]]}(\xcomp[2]) \equiv \prob_\infty(\xcomp[2]|\xcomp[1])$.
\item Proceeding in the same way, for each $\Xcomp[k]$ we compute the conditional probability $\prob^k_{\infty[\xcomp[1],\dots,\xcomp[k-1]]}(\xcomp[k]) \equiv \prob_\infty(\xcomp[k]|\xcomp[1],\dots,\xcomp[k-1])$.
\item Once the conditional distribution for $\Xcomp[n]$ given all the previous variables is computed, we can reconstruct the entire joint distribution. Moreover, we have the canonical factorization (\cref{eq:canonical_factorization})
$$
		\prob_\infty(\x) = \prod_{i=1}^{n} \prob_\infty(\xcomp[i]|\xcomp[<i]).
$$
\item Knowing the form of each individual term in the factorization, we can explicitly express $\prob_\infty$ as
\begin{equation}\label{eq:p_inf_explicit_factorization}
    \prob_\infty(\x) = \frac{1}{\mathcal{Z}} \prod_{i=1}^n \exp\left(\frac{-V^i(\xcomp[1],\dots,\xcomp[i])}{\sigma^2}\right) = \frac{1}{\mathcal{Z}}  \exp\left(\frac{- \sum_{i=1}^nV^i(\xcomp[1],\dots,\xcomp[i])}{\sigma^2}\right),
\end{equation}
where $V^i(\xcomp[1],\dots,\xcomp[i]) = - \int_0^x \mu^i(\xcomp[1],\dots,s) \dl{s}$ is the potential associated with $\mu^i$ and $\mathcal{Z} = \prod_{i=1}^n \mathcal{Z}^i$ is the joint product of the partition functions at each step of the pipeline.
\end{enumerate}
We observe in point (5) that $\prob_\infty$ keeps its original structure, that is, it remains a Gibbs distribution, with $V(\x) = \sum_{i=1}^n V^i(\x)$.

\begin{remark*}
    The procedure followed is based on the \textit{principle of modularity of causal mechanisms}: an intervention on a variable modifies only the mechanism that determines it, leaving all others unchanged. Concerning the canonical factorization, this principle corresponds to the independence between conditional probabilities. When intervening on $\Xcomp[i]$, the only term that changes is $\prob_\infty(\xcomp[i]|\xcomp[<i])$. In this context, the acyclicity of the causal graph plays a fundamental role. Without this assumption, we would not be able to write $\prob_\infty$ as a product of independent components. As a result, estimating causal effects in the cyclic case becomes less straightforward.
\end{remark*}

\paragraph{Every intervention is identifiable}
We have seen that the interventions considered  add an extra term to the original potential, which in both cases is a function only of the variable $\Xcomp[i]$ that has been intervened on. Specifically, for the additive intervention (\cref{eq:potential_additive_intervention}) we obtain a linear term $F\xcomp[i]$, and for the forcing intervention (\cref{eq:potential_forcing_intervention}) a quadratic term $\tilde{F}(\xcomp[i]-\hat{x})^2$. 

It is useful to notice that such terms can only arise from modifying the corresponding function $\mu^i$, and not the others, assuming sufficient regularity of $\bm{\mu}$. In fact, by the integral definition of $V^i$, this function cannot contain terms that do not depend in some way on $\xcomp[i]$. 

To see this, assume that 
$$V^j(\xcomp[1],\dots,\xcomp[i],\dots,\xcomp[j]) = g(\xcomp[1],\dots,\xcomp[i],\dots,\xcomp[j]) + h(\xcomp[i]), \quad \text{with} \quad i \neq j.$$

From the definition of $V^j$ (\cref{eq:mu_to_potential}) as a definite integral from $0$ to $\xcomp[j]$, we have

$$-\mu^j = \difcp{V^j}{\xcomp[j]} =  \difcp{g}{\xcomp[j]}+ \underbrace{\difcp{h}{\xcomp[j]}}_{=0} = \difcp{g}{\xcomp[j]},$$

since $h$ does not depend on $\xcomp[j]$. Therefore, $h$ appears only as an integration constant in $V^j$. In particular, since $V^j$ is computed as a definite integral, the contribution of $h$ is cancelled out, as it does not depend on the integration limits. Therefore, every intervention considered leads to the addition of a term in the final potential that depends only on the variable that was modified. As a result, the presence of such a term uniquely identifies which intervention was applied to the system.

\section{Connection with Discrete-Time Models}\label{sec:connection_discrete_continuous}

In this section, we discuss the theoretical bridges between the discrete-time framework presented in \cref{ch:Ci_over_time} and the continuous-time framework explored in this chapter. Specifically, by making the connection between \acp{SdE} and \acp{SDE} explicit, the goal is to enable the transfer of results obtained in one framework to the other, and vice versa. As shown in \cref{tab:results_recap}, we aim to focus on two specific cases for applying this transfer: \textit{extending the continuous-time framework to linear systems with loops}, and \textit{extending the discrete-time framework to the nonlinear acyclic case}. At present, we are not yet able to handle nonlinear systems with loops in either framework.

\begin{table}[h]
\centering
\renewcommand{\arraystretch}{1.2}
\begin{tabular}{@{} lccc @{}}
\toprule
\textbf{Framework} & \textit{Linear Acyclic} & \textit{Linear with Loops} & \textit{Nonlinear Acyclic} \\
\midrule
\textit{Discrete Time}   & $\checkmark_{\ref{subsec:mapping_var_scm}}$ & $\checkmark_{\ref{subsec:mapping_var_scm}}$ & \colorbox{gray!10}{$\Uparrow$} \\
\textit{Continuous Time} & $\checkmark_{\ref{sec:acyclic_non_linear_continuos_equilibration}}$ & \colorbox{gray!10} {$\Downarrow$} & $\checkmark_{\ref{sec:acyclic_non_linear_continuos_equilibration}}$ \\
\bottomrule
\end{tabular}
\caption{\textbf{Solved Cases.} Summary of solved cases and directions of result transferability between discrete- and continuous-time frameworks. The arrows indicate results that are expected to transfer but remain to be formally established.}\label{tab:results_recap}
\end{table}

\subsection{From \acp{SDE} to \acp{VAR}}  
We begin by showing how a linear \ac{SDE} can be associated with a \ac{VAR}(1). Consider the linear \ac{SDE}
\begin{equation}\label{eq:general_linear_SDE}
    \dl{\X_t} = \A \X_t \dl{t} + \bm{D} \dl{\W_t},
\end{equation}
where $\A$ is the drift matrix and $\bm{D}$ is the diffusion matrix. The explicit solution is:
\begin{equation*}
    \X_t = e^{\A t} \X_0 + \int_0^t e^{\A(t-s)} \bm{D} \, \dl{\W_s}.
\end{equation*}

While the second term is a stochastic integral, its covariance can be written as a deterministic integral:
\begin{equation*}
    \Cov[\X_t] = \int_0^t e^{\A (t-s)} \bm{D} \bm{D}^\top e^{\A^\top (t-s)} \, \dl{s}.
\end{equation*}

Since the integrand is deterministic, $\X_t$ is a multivariate Gaussian process. Hence, it is fully characterized by its mean, \ie, $e^{\A t} \X_0$, and the covariance above. By evaluating the solution at discrete time steps $t = k\Delta t$, we obtain a \ac{VAR}(1) model of the form:
\begin{equation*}
    \X_k = e^{\A \Delta t} \X_{k-1} + \vectornoise_k, \quad \vectornoise_k\sim\normal(0,\cov[\vectornoise]),
\end{equation*}
where
\begin{equation}\label{eq:covariance_sde_to_var}
    \cov[\vectornoise] = \int_0^{\Delta t} e^{\A s} \bm{D} \bm{D}^\top e^{\A^\top s} \, \dl{s}.
\end{equation}

With this equivalence in hand, we can generalize \cref{thm:var_to_scm} to define a mapping between \acp{SDE} and \acp{SCM}. Formally, we define this mapping as the composition of the map from \ac{SDE} to \ac{VAR} and the map from \ac{VAR} to \ac{SCM}, as illustrated in \cref{fig:sde_to_scm_graph}.

\begin{figure}[h!]
\centering
\begin{tikzpicture}
    [scale=1.0,
     every node/.style={
        rectangle split,
        rectangle split parts=2,
        rectangle split horizontal=true,
        rounded corners=0.1cm,
        minimum height=1cm,
        align=center,
        scale=0.8
    }]

    % Nodes
    \node[draw] (sde) at (0,1.5) {\nodepart[text width=1.2cm]{one}\ac{SDE}\nodepart[text width=1.2cm]{two}$\X_t$};
    \node[draw] (var) at (6,1.5) {\nodepart[text width=1.2cm]{one}\ac{VAR}\nodepart[text width=1.2cm]{two}$\X_k$};
    \node[draw] (scm) at (12,1.5) {\nodepart[text width=1.2cm]{one}\ac{SCM}\nodepart[text width=1.2cm]{two}$\X$};

    % Arrows
    \draw[|->, shorten >=4pt, shorten <=4pt, color=MidnightBlue] (sde) -- (var) node[midway, above, xshift=6pt] {\small Time discretization};
    \draw[|->, shorten >=4pt, shorten <=4pt, color=MidnightBlue] (var) -- (scm) node[midway, above] {\small Equilibration};

    % Composed arrow
    \draw[|->, dashed, thick, shorten >=4pt, shorten <=4pt, bend left=30, color=BrickRed]
        (sde.north east) to node[above, yshift=4pt] {\small Composition: SDE $\Rightarrow$ SCM}
        (scm.north west);
\end{tikzpicture}
\caption{\textbf{\acp{SDE} to \acp{SCM}.} Graphical summary of the mapping: the \ac{SDE} is first discretized into a \ac{VAR} model, and then equilibrated to yield an \ac{SCM}. The composition of these two steps defines a mapping from continuous-time stochastic dynamics to a causal model.}
\label{fig:sde_to_scm_graph}
\end{figure}
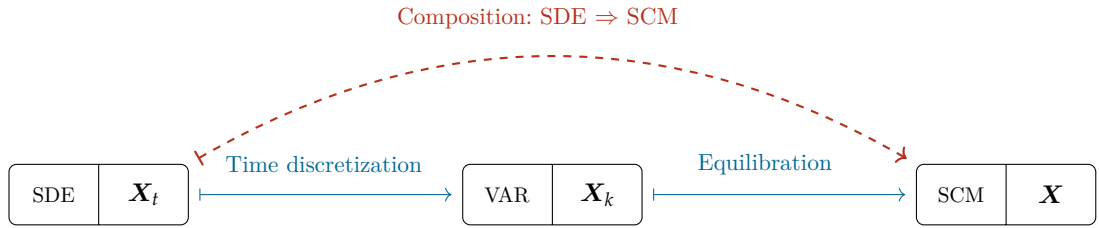

\begin{theorem}\label{thm:sde_to_scm_linear} 
Consider an \ac{SDE} as in \cref{eq:general_linear_SDE}. Then, there exists a linear \ac{SCM} of the form \begin{equation}\label{eq:sde_to_scm_linear} \X = e^{\A \Delta t} \X + \vectornoise, \quad \vectornoise_k \sim \normal(0, \cov[\vectornoise]), \end{equation} 
for which the properties stated in \cref{thm:var_to_scm} hold. In particular, note that $\Delta t$ can be chosen arbitrarily: the result can be interpreted as the existence of a family of equivalent \acp{SCM}, each derived from a discrete-time system with a different time step.\end{theorem}

\begin{remark*}
    The equivalence between the continuous-time and discrete-time models should not be interpreted in a strong (i.e., pathwise) sense. Brownian motion cannot, in any meaningful way, be reduced to the sampling of a random variable, regardless of how small $\Delta t$ is chosen. In other words, from a pathwise perspective, the \ac{SDE} is significantly more expressive than the corresponding \ac{VAR} model: for each trajectory of the discrete model, there exist infinitely many continuous-time trajectories that share only the initial state $\X_t$ and the terminal state $\X_{t + \Delta t}$. However, this discrepancy disappears when viewed from a distributional standpoint. The probability distributions observed at any given time are identical across the two models. One may interpret the continuous-time model as the underlying physical process and the discrete-time model as a sampled representation of it, much like a video recording of a real-world scene. The continuous model flows uninterrupted, while the discrete model offers a frame-by-frame reconstruction. The finite frame rate of the video corresponds to the time discretization through $\Delta t$: smooth motion becomes a sequence of snapshots. Although this approximation captures the overall behavior, it inevitably omits the infinitesimal fluctuations that occur between frames. Nonetheless, once the system reaches its stationary distribution, both perspectives agree not only that the system has stabilized, but also on its equilibrium state. This alignment justifies the possibility of mapping both descriptions into the same \ac{SCM}.
\end{remark*}

\subsection{From \acp{SdE} to \acp{SDE}}
\paragraph{Approximate Solutions}

It is not always possible to discretize a \ac{SDE} in such a way that the resulting difference equation exactly follows the original trajectory — that is, so that the two models agree perfectly on the position of $\X_t$ at all discrete times. In particular, for nonlinear \acp{SDE}, this exact correspondence generally does not exist. However, it is possible to implement numerical simulation schemes that guarantee bounds on the error (for example, in terms of the norm of the difference between expected values), with the error controlled by the choice of the timestep $\Delta t$. In other words, the finer the discretization, the closer the discrete model will reproduce trajectories similar to those of the continuous model. Here, we present the simplest scheme and summarize its main properties.

\begin{definition}[Euler-Maruyama Scheme]\label{def:euler_scheme}
Consider an \ac{SDE} of the form
\begin{equation*}
    \dl{\X_t} = \bm{\mu}(\X_t) \dl{t} + \bm{\sigma}(\X_t) \dl{\W_t},
\end{equation*}
where $\bm{\mu}$ and $\bm{\sigma}$ are sufficiently regular functions, and $\W_t$ is an $n$-dimensional standard Brownian motion.

The Euler-Maruyama discretization with timestep $\Delta t > 0$ is defined recursively by
\begin{equation}\label{eq:euler_maruyama_process_discrete}
    \tilde{\X}_{n+1} = \tilde{\X}_n + \bm{\mu}(\tilde{\X}_n) \Delta t + \bm{\sigma}(\tilde{\X}_n) \Delta \W_n,
\end{equation}
where $\tilde{\X}_n$ approximates $\X_{t_n}$ with $t_n = n \Delta t$, and the increments $\Delta \W_n = \W_{t_{n+1}} - \W_{t_n}$ are independent Gaussian random variables with distribution $\normal(0, \Delta t \, I)$.
\end{definition}

\begin{remark}
The equation obtained by applying the Euler--Maruyama scheme can be adapted to match the standard form of the difference equation presented in \cref{eq:sde}, namely:
\begin{equation*}
\tilde{\X}_{n+1} = \bm{f}(\tilde{\X}_{n}) + \bm{g}(\tilde{\X}_{n}) \vectornoise_n, \quad \vectornoise_n \sim \mathcal{N}(0, I).
\end{equation*}
To achieve this, it suffices to define $$\bm{f}(\tilde{\X}_n) :=  \tilde{\X}_n + \bm{\mu}(\tilde{\X}_n) \Delta t, \quad \bm{g}(\tilde{\X}_n) := \sqrt{\Delta t} \, \bm{\sigma}(\tilde{\X}_n).$$

It should be noted that the diffusion term is multiplied by $\sqrt{\Delta t}$ as a consequence of Itô's formula, and specifically due to the equivalence $\mathbb{E}\left[(\Delta \W_n)^2\right] = \Delta t$. Furthermore, the final difference equation absorbs the terms dependent on $\Delta t$ into the functions $\bm{f}$ and $\bm{g}$, eliminating any explicit reference to the time scale chosen for the system's evolution.
\end{remark}

In order to formally quantify how a discrete-time model approximates its continuous-time counterpart, it is necessary to introduce a notion of convergence that captures, in a precise sense, the similarity between the two processes. 
To this end, we introduce the definition of strong convergence between \acp{DSP}.

\begin{definition}[Strong Convergence of Stochastic Processes]
Let $\X$ be a \ac{DSP} and let $\{\tilde{\X}^{(k)}\}_{k=1}^\infty$ be a sequence of \ac{DSP}, all defined on the same space $\Xspace$ and over the same time interval $[0,T]$. 
We say that $\tilde{\X}^{(k)}$ converges strongly to $\X$ if, for every $t \in [0,T]$, it holds that
\begin{equation}
    \lim_{i\to\infty} \expected{|\tilde{\X}_t^{(k)} - \X_t|} = 0.
\end{equation}
Moreover, we say that $\tilde{\X}^{(k)}$ converges with strong order $\gamma$ if there exist a constant $C>0$ such that
\begin{equation}
    \expected{\sup_{0<t<T}|\tilde{\X}_t^{(k)} - \X_t|}< C k^{-\gamma}.
\end{equation}
\end{definition}

In the case of interest, the sequence of stochastic processes represents progressively finer discretizations, corresponding to decreasing values of the timestep parameter $\Delta t$. 
For simplicity, we will denote by $\tilde{\X}^{(k)}$ the stochastic process obtained through discretization with a fixed timestep $\Delta t = 1/k$.

It is important to note that the definition of strong convergence introduced above concerns processes defined in continuous time. 
Therefore, before being able to study the convergence of the discrete-time Euler–Maruyama scheme, we must extend the resulting process to the entire interval $[0, T]$ in a suitable manner. We define the \emph{continuous-time extension} of the discrete process in~\cref{eq:euler_maruyama_process_discrete} as an Itô process with piecewise constant drift and diffusion coefficients. 
Specifically, for $t_n \leq t < t_{n+1}$, where $t_n = n\Delta t$, we define
\begin{equation}\label{eq:euler_maruyama_process_continuous}
    \tilde{\X}_t^{(k)} = \tilde{\X}_n^{(k)} + \bm{\mu}(\tilde{\X}_n^{(k)})(t - t_n) + \bm{\sigma}(\tilde{\X}_n^{(k)})(\W_t - \W_{t_n}).
\end{equation}

\begin{theorem}[Strong Convergence of the Euler--Maruyama Scheme]
Consider a \ac{SDE} satisfying global Lipschitz conditions; that is, there exists a constant $C > 0$ such that
$$
\forall x, y \in \Xspace, \quad \|\bm{\mu}(x) - \bm{\mu}(y)\| \leq C|x-y|, \quad \|\bm{\sigma}(x) - \bm{\sigma}(y)\| \leq C|x-y|.
$$
Then, the process $\tilde{\X}_t^{(k)}$, the continuous Itô process obtained by discretizing the original \ac{SDE} with timestep $\Delta t = 1/k$, converges to the solution $\X$ of the \ac{SDE} with strong order $1/2$. 
Formally,
$$
\expected{\sup_{0 \leq t \leq T} |\tilde{\X}_t^{(k)} - \X_t|} \leq C(T) k^{-1/2},
$$
where $C(T)$ is a constant independent of $k$.
\end{theorem}

\begin{remark}
The theorem above can be interpreted both as a convergence result and as a bound on the maximal error given a fixed timestep $\Delta t$. 
Specifically, since the causal effect $\ce$ is defined in terms of the expected value of $\X$, this bound can be used to estimate, up to an error of order $2C\sqrt{\Delta t}$, the long-term causal effect of a nonlinear discrete-time stochastic process. Of course, this is only a preliminary result: it does not guarantee the existence of an explicit associated \ac{SCM} for the process, nor does it ensure the ability to identify the invariant distribution explicitly. 

It is also important to note that the bound characterizes the observational process specifically. 
In practice, the study of real-world dynamical systems almost always relies on data collected at discrete time intervals, \textit{i.e.}, time series. In other words, the theorem provides a principled estimate of how the sampling rate might affect the quality of the reconstructed description of the true underlying system. If, however, the available data allow us to empirically establish a tighter error bound, introducing a new constant $\tilde{C} < C$, for instance through the evaluation of predictive performance over future observations, then we can achieve much finer control over the expected error when making causal claims about interventions.

As an additional observation, additive interventions do not affect the Lipschitz constant of the drift term $\bm{\mu}$. 
As a consequence, we expect that the expected error bound remains unchanged when moving from the observational to the interventional level.
\end{remark}

\chapter{Final Discussion}\label{ch:conclusions}

This chapter summarizes how the contributions presented throughout the thesis address the research questions formulated in~\cref{sec:research_questions}, organizing the discussion around the two main lines of investigation explored.
Specifically,~\cref{conclusions:part1} focuses on \ac{AR} and its limitations, while~\cref{conclusions:part2} explores the extension of causal inference to systems that evolve over time.
For each part, we review how the corresponding research questions have been answered, highlight key limitations, and outline promising directions for future research.

\section{Algorithmic Recourse}\label{conclusions:part1}

This part of the thesis examined the limitations of existing approaches to \ac{AR}, identifying the formulation of the cost function as a central theoretical challenge. In the absence of contextual information, standard methods often rely on default cost metrics, typically some form of distance in the instance space, which are chosen arbitrarily and may not reflect user preferences. We propose an alternative approach that leverages contextual knowledge to generate more realistic and user-relevant recommendations.

\subsection{Response to the Research Question}
\textit{RQ1} concerns whether principled frameworks for \ac{AR} can be developed to generate actionable and realistic recommendations. This question has been addressed from multiple perspectives. First, we explored how user-specific contextual knowledge can be integrated into the recommendation process~(\cref{ch:cost-subjective}). Second, we examined how incorporating causal structure can guide the design of more meaningful cost functions. In this context, we distinguished between \ac{CEx} and \ac{CRe}: while the former, a more common and simpler approach, identifies the optimal counterfactual instance, the latter goes further by specifying a sequence of actionable changes needed to reach that instance~(\cref{ch:algo_recourse}). Additionally, we studied how introducing a temporal dimension transforms the recourse framework, advancing toward more realistic and constrained recommendations~(\cref{ch:timecar}).

\subsection{Limitations}
\cref{ch:cost-subjective} assumes that users can effectively manage the optimization process. However, this assumption may break down in high-dimensional or non-monotonic settings, where navigating the decision boundary becomes cognitively demanding. While involving the user avoids overly prescriptive cost modeling, it does not eliminate the issue of plausibility: users may encounter suggestions that are unrealistic, inconsistent, or internally incoherent. The lack of plausibility checks or statistical constraints can lead to counterfactuals that are technically valid but practically meaningless. Moreover, as the number of actionable features increases, so does the user’s cognitive burden. Prioritizing the most influential variables or suggesting promising directions could help mitigate this challenge. Lastly, counterfactual recommendations can be subject to manipulation or ethical concerns, underlining the importance of responsible design in recourse systems.

\cref{ch:timecar} assumes full knowledge of the underlying causal structure, which may not hold in practical applications where causal graphs must be inferred from data. Furthermore, the proposed temporal extension remains preliminary, currently limited to the inclusion of an additional term in the cost function. 
Using waiting times along the causal graph provides a highly specific and simplified way to model system evolution, which may fall short in capturing more complex temporal dynamics. In this regard, the framework presented in~\cref{pt:time_aware_causality} offers a promising foundation for future development. Determining appropriate waiting times, however, remains a significant challenge. In this work, these values were assigned manually based on general plausibility, without a rigorous theoretical or empirical basis. Future research should investigate methods for estimating these values automatically, potentially by leveraging datasets with explicit temporal structures, such as time series. Alternatively, these values could be selected directly by the user when interacting with the recommendation tool, allowing the temporal cost to be tailored to the individual and adapted to the specific context.

\subsection{Future Work}
Regarding \cref{ch:cost-subjective}, a key direction for future research is to develop a statistical protocol that actively learns the cost function from user interaction. While such a model would not fully capture individual-level variation, due to the inherently subjective nature of preferences, it could still provide an \quotes{average} approximation. This could serve as a preliminary tool for generating reasonable recommendations, which users may then tailor to their specific needs.

For \cref{ch:timecar}, future research could aim to relax the assumption of full causal knowledge by incorporating causal discovery methods. Furthermore, the potential applications in fairness and actionability could be empirically validated, also within the more advanced framework proposed in \cref{pt:time_aware_causality}.

\section{Causal Inference over Time}\label{conclusions:part2}

This part of the thesis extended causal reasoning to dynamical systems, investigating how interventions and causal inference can be modeled in settings where variables evolve over time. 
By linking the equilibrium behavior of dynamical systems to \acp{SCM}, the proposed framework allowed causal questions to be formulated and answered in the presence of temporal dependencies. 

\subsection{Response to the Research Question}
\textit{RQ2} asks whether causal inference methods can be extended to dynamical systems in order to model the effects of interventions over time and support recommendations that are consistent across time. 
This question was addressed in~\cref{ch:Ci_over_time} by developing a framework that maps discrete-time dynamical systems to equivalent \acp{SCM}, under stability assumptions.
It was shown that, under appropriate conditions, interventions defined at the dynamical level correspond to interventions at the \ac{SCM} level, making it possible to compute long-term causal effects even in evolving systems. 

In~\cref{ch:CIcontinuous_time}, the framework was extended to continuous-time models with non-linear dependencies, using Itô stochastic integration theory. Finally, the relationship between the discrete-time and continuous-time frameworks was examined, showing how, and under which conditions, results can be transferred between the two.

\subsection{Limitations}
Both frameworks assume knowledge of the system’s causal structure, with the exception of cases where this assumption can be weakened to require only causal sufficiency. They also rely on specific structural assumptions about the evolution equations, e.g., stochastic noise is always treated additively. The systems considered are assumed to be stable, meaning they do not exhibit long-term trends. While extending the analysis to systems with trends appears feasible based on existing work, the equivalence results with \acp{SCM} would no longer apply.

Moreover, the current frameworks do not account for time-varying interventions or those confined to a specific period, e.g., a public policy designed to last three years, or one that adapts annually in response to observed outcomes.

Finally, the case of non-linear systems with feedback loops remains unexplored and represents the most significant open challenge. Given the generality of such systems, extending the framework in this direction would mark a major milestone in the development of this line of research.

\subsection{Future Work}
So far, the proposed framework has been validated only in the discrete-time case, with experiments on both synthetic and real-world datasets. A similar empirical evaluation is needed for the continuous-time extension, ideally through a specific work with a stronger focus on real-world applicability.

From a theoretical perspective, addressing the non-linear case with feedback loops remains the main priority, along with formalizing precisely how results transfer between the discrete and continuous-time settings. The long-term goal is to develop a unified and general framework that can accommodate all these specific cases.

Regarding the assumption of known causal structure, integrating causal discovery algorithms into the modeling pipeline could be a valuable first step toward a more complete framework.  However, as discussed in \cref{ch:causality_math_framework}, this approach has inherent limitations, both in terms of required assumptions and the robustness of the resulting models.

In the author's view, constructing realistic causal models requires a hybrid approach to structure identification, one that combines automatic procedures with traditional mathematical modeling. Statistical methods can meaningfully support scientific inquiry, but they are not a substitute for human-led research. In this spirit, it is important to develop mechanisms for knowledge injection, allowing researchers to embed domain-specific expertise of each particular application directly into the models.

%
%
%
%
%
%
%\addtocontents{toc}{\protect\clearpage} % <--- just debug stuff, ignore
%\include{multiToC} % <--- just debug stuff, ignore for your documents
% ********************************************************************
% Backmatter
%*******************************************************
%\appendix
%\include{Chapters/Appendix/app-causal-inference-time}
%\renewcommand{\thechapter}{\alph{chapter}}
%\cleardoublepage
%\part{Appendix}
%********************************************************************
% remove if no figures
\listoffigures
% remove if no tables
\listoftables
% Other Stuff in the Back
%*******************************************************
\cleardoublepage%********************************************************************
% Bibliography
%*******************************************************
% work-around to have small caps also here in the headline
% https://tex.stackexchange.com/questions/188126/wrong-header-in-bibliography-classicthesis
% Thanks to Enrico Gregorio
\defbibheading{bibintoc}[\bibname]{%
  \phantomsection
  \manualmark
  \markboth{\spacedlowsmallcaps{#1}}{\spacedlowsmallcaps{#1}}%
  \addtocontents{toc}{\protect\vspace{\beforebibskip}}%
  \addcontentsline{toc}{chapter}{\tocEntry{#1}}%
  \chapter*{#1}%
}
\printbibliography[heading=bibintoc]

% Old version, will be removed later
% work-around to have small caps also here in the headline
%\manualmark
%\markboth{\spacedlowsmallcaps{\bibname}}{\spacedlowsmallcaps{\bibname}} % work-around to have small caps also
%\phantomsection
%\refstepcounter{dummy}
%\addtocontents{toc}{\protect\vspace{\beforebibskip}} % to have the bib a bit from the rest in the toc
%\addcontentsline{toc}{chapter}{\tocEntry{\bibname}}
%\label{app:bibliography}
%\printbibliography

%\cleardoublepage\include{FrontBackmatter/Declaration}
%\cleardoublepage\include{FrontBackmatter/Colophon}
% ********************************************************************
% Game Over: Restore, Restart, or Quit?
%*******************************************************
\end{document}